\documentclass[11pt,oneside]{article}

\pdfoutput=1 % For arXiv; to ensure it is compiled with PDFLaTeX.

\usepackage{style_mjh}
\usepackage[numbers]{natbib}
\usepackage{amsmath}                                          
\usepackage{amsthm}                                           % Theorem formatting:
\theoremstyle{definition} \newtheorem{defn}{Definition}       %  Use \begin{defn}
\theoremstyle{plain}            %  Use \begin{prop}[Optional prop label]
\theoremstyle{plain} \newtheorem{thm}[defn]{Theorem}                %  Use \begin{thm}[Optional thm label]
\theoremstyle{plain} \newtheorem{lem}[defn]{Lemma}                  %  Use \begin{lem}[Optional lem label]
\theoremstyle{plain}               %  Use \begin{lem}[Optional lem label]
\theoremstyle{remark} \newtheorem{rmk}[defn]{Remark}                %  Use \begin{rmk}[Optional rmk label]
\theoremstyle{remark} \newtheorem{ex}[defn]{Example}                %  Use \begin{ex}[Optional rmk label]

\usepackage{enumitem} % for cross-referencing description list items.
\makeatletter
\def\namedlabel#1#2{\begingroup
    #2%
    \def\@currentlabel{#2}%
    \phantomsection\label{#1}\endgroup
}
\makeatother

\usepackage[pdfauthor={Matthew James Holland},%
pdftitle={holland2018rgd},%
colorlinks=true,%
linkcolor=blue,%
citecolor=blue,%
pdftex]{hyperref} % this must come last.

\begin{document}

\title{\textbf{Efficient learning with robust gradient descent}}
\author{
  Matthew J.~Holland\thanks{Please direct correspondence to \texttt{matthew-h@ids.osaka-u.ac.jp}.}\\
  Osaka University\\
  Yamada-oka 2-8, Suita, Osaka, Japan
  \and
  Kazushi Ikeda\\
  Nara Institute of Science and Technology\\
  Takayama-cho 8916-5, Ikoma, Nara, Japan
}
\date{} % empty date.

\maketitle

\begin{abstract}
Minimizing the empirical risk is a popular training strategy, but for learning tasks where the data may be noisy or heavy-tailed, one may require many observations in order to generalize well. To achieve better performance under less stringent requirements, we introduce a procedure which constructs a robust approximation of the risk gradient for use in an iterative learning routine. Using high-probability bounds on the excess risk of this algorithm, we show that our update does not deviate far from the ideal gradient-based update. Empirical tests using both controlled simulations and real-world benchmark data show that in diverse settings, the proposed procedure can learn more efficiently, using less resources (iterations and observations) while generalizing better.
\end{abstract}

\section{Introduction}\label{sec:intro}

Any successful machine learning application depends both on procedures for reliable statistical inference, and a computationally efficient implementation of these procedures. This can be formulated using a risk $R(\ww) \defeq \exx l(\ww;\zz)$, induced by a loss $l$, where $\ww$ is the parameter (vector, function, set, etc.) to be specified, and expectation is with respect to $\zz$, namely the underlying data distribution. Given data $\zz_{1},\ldots,\zz_{n}$, if an algorithm outputs $\wwhat$ such that $R(\wwhat)$ is small with high probability over the random draw of the sample, this is formal evidence for good generalization, up to assumptions on the distribution. Performance-wise, the statistical side is important because $R$ is always unknown, and the method of implementation is important since the only $\wwhat$ we ever have in practice is one we can actually compute.

Empirical risk minimization (ERM), which admits any minimizer of $n^{-1}\sum_{i=1}^{n}l(\cdot;\zz_{i})$, is the canonical strategy for machine learning problems, and there exists a rich body of literature on its generalization ability \citep{kearns1994a,bartlett1996a,alon1997a,bartlett2003a}. In recent years, however, some severe limitations of this technique have come into light. ERM can be implemented by numerous methods, but its performance is sensitive to this implementation \citep{daniely2014a,feldman2016a}, showing sub-optimal guarantees on tasks as simple as multi-class pattern recognition, let alone tasks with unbounded losses. A related issue is highlighted in recent work by \citet{lin2016a}, where we see that ERM implemented using a gradient-based method only has appealing guarantees when the data is distributed sharply around the mean in a sub-Gaussian sense. These results are particularly important due to the ubiquity of gradient descent (GD) and its variants in machine learning. They also carry the implication that ERM under typical implementations is liable to become highly inefficient whenever the data has heavy tails, requiring a potentially infinitely large sample to achieve a small risk. Since tasks with such ``inconvenient'' data are common \citep{finkenstadt2003Extreme}, it is of interest to investigate and develop alternative procedures which can be implemented as readily as the GD-based ERM (henceforth, ERM-GD), but which have desirable performance for a wider class of learning problems. In this paper, we introduce and analyze an iterative routine which takes advantage of robust estimates of the risk gradient.

\paragraph{Review of related work}

Here we review some of the technical literature related to our work. As mentioned above, the analysis of \citet{lin2016a} includes the generalization of ERM-GD for sub-Gaussian observations. ERM-GD provides a key benchmark to be compared against; it is of particular interest to find a technique that is competitive with ERM-GD when it is optimal, but which behaves better under less congenial data distributions. Other researchers have investigated methods for distribution-robust learning. One notable line of work looks at generalizations of the ``median of means'' procedure, in which one constructs candidates on disjoint partitions of the data, and aggregates them such that anomalous candidates are effectively ignored. These methods can be implemented and have theoretical guarantees, ranging from the one-dimensional setting \citep{lerasle2011a,minsker2017a} to multi-dimensional and even functional models \citep{minsker2015a,hsu2016a,lecue2017a}. Their main limitation is practical: when sample size $n$ is small relative to the complexity of the model, very few subsets can be created, and robustness is poor; conversely, when $n$ is large enough to make many candidates, cheaper and less sophisticated methods often suffice.

An alternative approach is to use all the observations to construct robust estimates $\widehat{R}(\ww)$ of the risk $R(\ww)$ for each $\ww$ to be checked, and subsequently minimize $\widehat{R}$ as a surrogate. An elegant strategy using M-estimates of $R$ was introduced by \citet{brownlees2015a}, based on fundamental results due to \citet{catoni2009a,catoni2012a}. While the statistical guarantees are near-optimal under very weak assumptions on the data, the proxy objective $\widehat{R}$ is defined implicitly, introducing many computational roadblocks. In particular, even if $R$ is convex, the estimate $\widehat{R}$ need not be, and the non-linear optimization required by this method can be both unstable and costly in high dimensions.

Finally, conceptually the closest recent work to our research are those also analyzing novel ``robust gradient descent'' algorithms, namely steepest descent procedures which utilize a robust estimate of the gradient vector of the underlying (unknown) objective of interest. The first works in this line are due to \citet{holland2017a} (a preliminary version of our work) and \citet{chen2017a} (later updated as \citet{chen2017b}), which appeared as pre-prints almost simultaneously. While the problem setting of \citet{chen2017b} and the technical approach to robustification are completely different from ours, the underlying motivation of replacing the empirical mean gradient estimate with a more robust alternative is shared. We utilize an M-estimator of the gradient coordinates which can be approximated using fixed-point iterative updates. On the other hand, \citet{chen2017a} utilize the geometric median to robustly aggregate multiple candidates constructed on subsets after partitioning the data. They consider a federated learning setting with many low-cost machines susceptible to arbitrarily bad performance, running in parallel, and provide rigorous learning guarantees within that problem setting. We on the other hand consider a single learning machine, with potentially heavy-tailed data, within a general risk-minimization framework. While the theoretical guarantees are not directly comparable, the dependence on sample size $n$, confidence $\delta$, and dimension $d$ are essentially the same, up to minor differences in log factors. The key advantage to our approach is the ease of computation. While the geometric median used by \citet{chen2017b} can indeed be computed using well-known iterative routines \citep{vardi2000a}, these suffer from substantial overhead in computing pairwise distances over all partitions at each iteration, and as mentioned above in reference to the work of \citet{minsker2015a} and \citet{hsu2016a}, can run into significant bias when the number of partitions cannot be made large enough. A more recent entry into this line of research comes from \citet{prasad2018a}, who follow the exact same strategy as \citet{chen2017b}, but consider a more general learning setting, very close to the general setting of our paper. They also provide new results for several concrete models under heavy-tailed data, although the practical weaknesses of their procedure are exactly the same as those inherent in the procedure of \citet{chen2017b}.

\paragraph{Our contributions}

To deal with these limitations of ERM-GD and its existing robust alternatives, the key idea here is to use robust estimates of the risk gradient, rather than the risk itself, and to feed these estimates into a first-order steepest descent routine. In doing so, at the cost of minor computational overhead, we get formal performance guarantees for a wide class of data distributions, while enjoying the computational ease of a gradient descent update. Our main contributions:
\begin{itemize}
\item A learning algorithm which addresses the vulnerabilities of ERM-GD, is easily implemented, and can be adapted to stochastic sub-sampling for big problems.
\item High-probability bounds on excess risk of this procedure, which hold under mild moment assumptions on the data distribution, and suggest a promising general methodology.
\item Using both tightly controlled simulations and real-world benchmarks, we compare our routine with ERM-GD and other cited methods, obtaining results that reinforce the practical utility and flexibility suggested by the theory.
\end{itemize}

\paragraph{Content overview}
In section \ref{sec:intuitive}, we introduce the key components of the proposed algorithm, and provide an intuitive example meant to highlight the learning principles taken advantage of. Theoretical analysis of algorithm performance is given in section \ref{sec:algo}, including a sketch of the proof technique and discussion of the main results. Empirical analysis follows in section \ref{sec:tests}, in which we elucidate both the strengths and limits of the proposed procedure, through a series of tightly controlled numerical tests. Finally, concluding remarks and a look ahead are given in section \ref{sec:conclusion}. Proofs and extra information regarding computation is given in appendix \ref{sec:appendix}. Additional empirical test results are provided in appendix \ref{sec:more_test_results}.

\section{Robust gradient descent}\label{sec:intuitive}

Before introducing the proposed algorithm in more detail, we motivate the practical need for a procedure which deals with the weaknesses of the traditional sample mean-based gradient descent strategy.

\subsection{Why robustness?}\label{sec:intuitive_2D}

Recall that since ERM admits any minima of $n^{-1}\sum_{i=1}^{n}l(\cdot;\zz_{i})$, the simplest implementation of gradient descent (for $\wwhat_{(t)} \in \RR^{d}$) results in the update
\begin{align}\label{eqn:ERM-GD}
\wwhat_{(t+1)} = \wwhat_{(t)} - \alpha_{(t)} \frac{1}{n}\sum_{i=1}^{n}l^{\prime}(\wwhat_{(t)};\zz_{i})
\end{align}
where $\alpha_{(t)}$ are scaling parameters. Taking the derivative under the integral we have $R^{\prime}(\cdot) = \exx l^{\prime}(\cdot;\zz)$, meaning ERM-GD uses the sample mean as an estimator of each coordinate of $R^{\prime}$, in pursuit of a solution minimizing the unknown $R$. Without rather strong assumptions on the tails and moments of the distribution of $l(\ww;\zz)$ for each $\ww$, it has become well-known that the sample mean fails to provide sharp estimates \citep{catoni2012a,minsker2015a,devroye2015a,lugosi2016a}. Intuitively, the issue is that we expect bad estimates to imply bad approximate minima. Does this formal sub-optimality indeed manifest itself in natural settings? Can principled modifications improve performance at a tolerable cost?

A simple example suggests affirmative answers to both questions. The plot on the left of Figure \ref{fig:2D_motivation} shows contour lines of a strongly convex quadratic risk to be minimized, as well as the trajectory of 10 iterations of ERM-GD, given four independent samples from a common distribution, initiated at a common $\wwhat_{(0)}$. With data $\zz = (\xx,y) \in \RR^{d+1}$, losses are generated as $l(\ww;\zz_{i})=(\langle \ww, \xx_{i} \rangle - y_{i})^{2}/2$. We consider the case where the ``noise'' $\langle \ww, \xx_{i} \rangle - y_{i}$ is heavy-tailed (log-Normal). Half of the samples saw relatively good solutions after ten iterations, and half saw rather stark deviation from the optimal procedure. When the sample contains errant observations, the empirical mean estimate is easily influenced by such points.

\begin{figure}[t]
\centering
\includegraphics[width=0.33\textwidth]{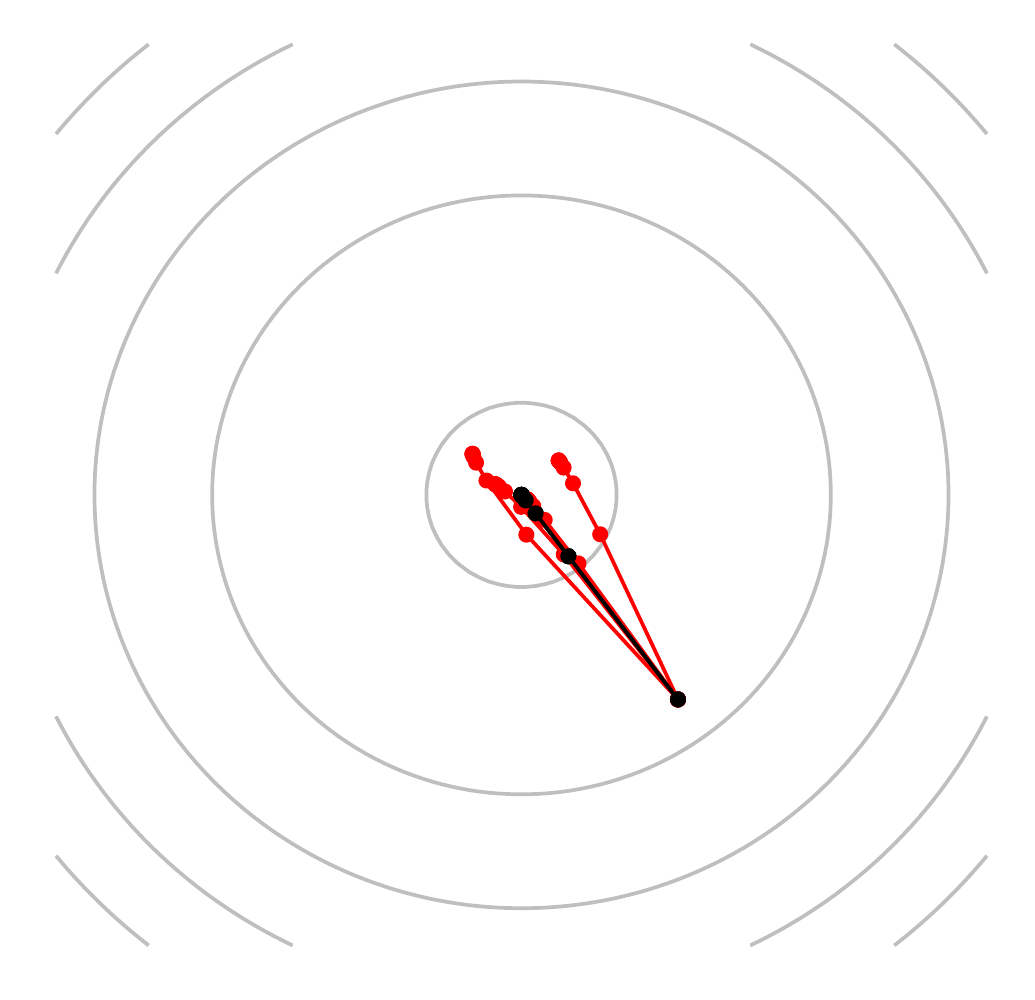}\includegraphics[width=0.33\textwidth]{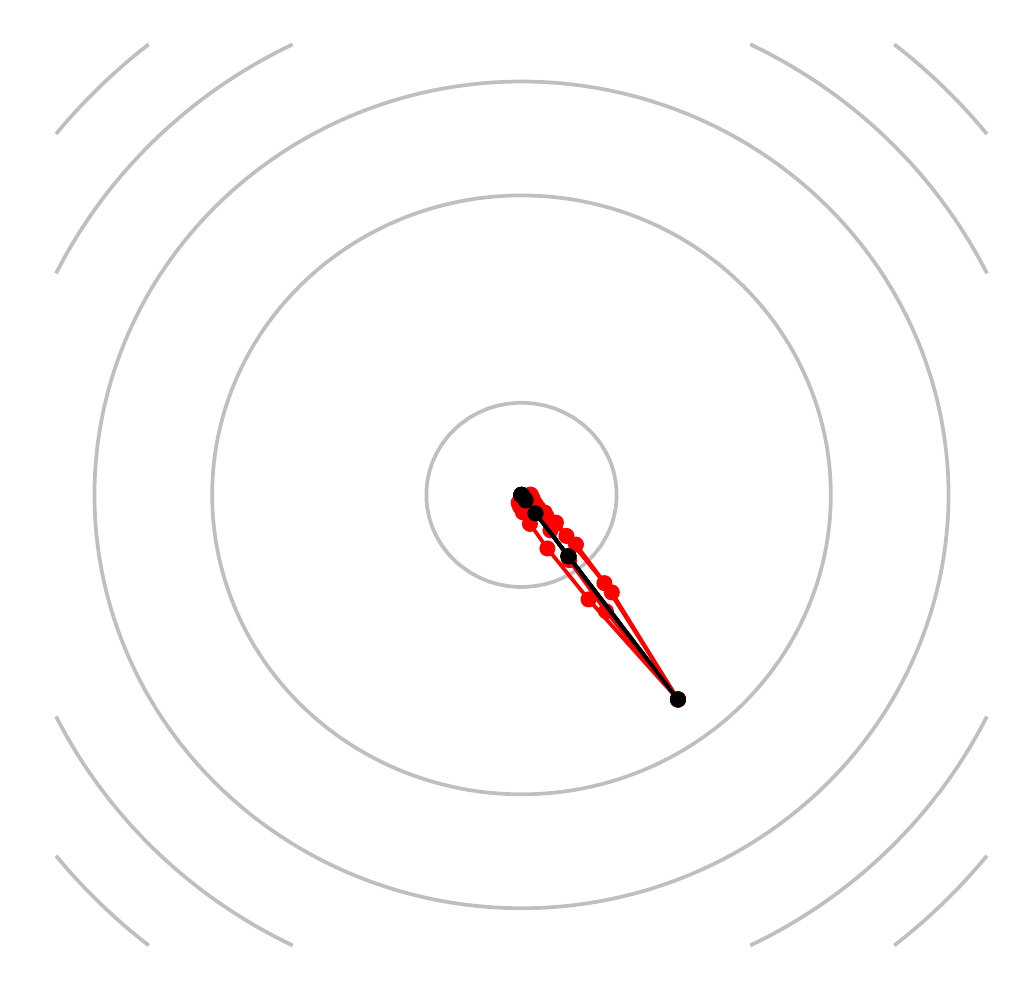}
\caption{A comparison of the minimizing sequence trajectories in a two-dimensional approximate risk minimization task, for the traditional ERM-based gradient descent (left) and a simple re-weighting procedure (right). Trajectories of the oracle update using $R^{\prime}$ (black) is pictured alongside the approximate methods (red). All procedures use $\alpha_{(t)}=0.35$, $t=0,\ldots,9$.}
\label{fig:2D_motivation}
\end{figure}

To deal with this, a classical idea is to re-weight the observations in a principled manner, and then carry out gradient descent as normal. That is, in the gradient estimate of (\ref{eqn:ERM-GD}), we replace the summands $n^{-1}l^{\prime}(\cdot;\zz_{i})$ with $\omega_{i}\,l^{\prime}(\cdot;\zz_{i})$, where $0 \leq \omega_{i} \leq 1$, $i = 1,\ldots,n$ and $\sum_{i=1}^{n}\omega_{i}=1$. For example, we could set
\begin{align*}
\omega_{i} \defeq \frac{\widetilde{\omega}_{i}}{\sum_{k=1}^{n}\widetilde{\omega}_{k}}, \quad \widetilde{\omega}_{i} \defeq \frac{\psi\left(\langle \ww, \xx \rangle - y_{i}\right)}{\left(\langle \ww, \xx \rangle - y_{i}\right)}
\end{align*}
where $\psi$ is an odd function of sigmoid form (see \ref{sec:prelims} and \ref{sec:appendix_computation}). The idea is that for observations $\zz_{i}$ that induce errors which are \textit{inordinately} large, the weight $\omega_{i}$ will be correspondingly small, reducing the impact. In the right-hand plot of Figure \ref{fig:2D_motivation}, we give analogous results for this procedure, run under the exact same settings as ERM-GD above. The modified procedure at least appears to be far more robust to random idiosyncrasies of the sample; indeed, if we run many trials, the average risk is far better than the ERM-GD procedure, and the variance smaller. The fragility observed here was in the elementary setting of $d=2$, $n=500$; it follows \textit{a fortiori} that we can only expect things to get worse for ERM-GD in higher dimensions and under smaller samples. In what follows, we develop a robust gradient-based minimization method based directly on the principles illustrated here.

\subsection{Outline of proposed procedure}\label{sec:intuitive_algo}

Were the risk to be known, we could update using
\begin{align}\label{eqn:GD_update_true}
\wwstar_{(t+1)} \defeq \wwstar_{(t)} - \alpha_{(t)} \gvec(\wwstar_{(t)})
\end{align}
where $\gvec(\ww) \defeq R^{\prime}(\ww)$, an idealized procedure. Any learning algorithm in practice will not have access to $R$ or $\gvec$, and thus must approximate this update with
\begin{align}\label{eqn:GD_update_approx}
\wwhat_{(t+1)} \defeq \wwhat_{(t)} - \alpha_{(t)} \gghat(\wwhat_{(t)}),
\end{align}
where $\gghat$ represents some sample-based estimate of $\gvec$. Setting $\gghat$ to the sample mean reduces to ERM-GD, and conditioned on $\wwhat_{(t)}$, $\exx \gghat(\wwhat_{(t+1)}) = \gvec(\wwhat_{(t+1)})$, a property used throughout the literature \citep{rakhlin2012a,leroux2012a,johnson2013a,shalev2013a,frostig2015a,murata2016a}. While convenient from a technical standpoint, there is no conceptual necessity for $\gghat$ to be unbiased. More realistically, as long as $\gghat$ is sharply distributed around $\gvec$, then an approximate first-order procedure should not deviate too far from the ideal, even if these estimators are biased. An outline of such a routine is given in Algorithm \ref{algo:rgd}.

\begin{algorithm}
\caption{Robust gradient descent outline}
\label{algo:rgd}
\begin{algorithmic}
\State \textbf{inputs:} $\wwhat_{0}$, $T>0$
\For{$t = 0,1,\ldots,T-1$} % the weight-updating loop
  \State $\displaystyle D_{(t)} \gets \{l^{\prime}(\wwhat_{(t)};\zz_{i})\}_{i=1}^{n}$ \State\Comment{\textit{Update loss gradients.}}
  \State $\displaystyle \widehat{\mv{\sigma}}_{(t)} \gets \text{\textsc{rescale}}(D_{(t)})$  \hfill\Comment{\textit{Eqn.~(\ref{eqn:dispersion_rough}).}}
  \State $\displaystyle \widehat{\mv{\theta}}_{(t)} \gets \text{\textsc{locate}}(D_{(t)},\widehat{\mv{\sigma}}_{(t)})$  \hfill\Comment{\textit{Eqns.~(\ref{eqn:location_rough}), (\ref{eqn:scale_rough}).}}
  \State $\displaystyle \wwhat_{(t+1)} \gets \wwhat_{(t)} - \alpha_{(t)}\widehat{\mv{\theta}}_{(t)}$ \hfill\Comment{\textit{Plug in to update.}}
\EndFor
\State \textbf{return:} $\wwhat_{(T)}$
\end{algorithmic}
\end{algorithm}

Let us flesh out the key sub-routines used in a single iteration, for the $\ww \in \RR^{d}$ case. When the data is prone to outliers, a ``soft'' truncation of errant values is a prudent alternative to discarding valuable data. This can be done systematically using a convenient class of M-estimators of location and scale \citep{vandervaart1998AS,huber2009a}. The \textsc{locate} sub-routine entails taking a convex, even function $\rho$, and for each coordinate, computing $\widehat{\mv{\theta}} = (\that_{1},\ldots,\that_{d})$ as
\begin{align}\label{eqn:location_rough}
\that_{j} \in \argmin_{\theta \in \RR} \sum_{i=1}^{n} \rho\left( \frac{l^{\prime}_{j}(\ww;\zz_{i})-\theta}{s_{j}} \right), \quad j = 1,\ldots,d.
\end{align}
Note that if $\rho(u) = u^{2}$, then $\that_{j}$ reduces to the sample mean of $\{l^{\prime}_{j}(\ww;\zz_{i})\}_{i=1}^{n}$, thus to reduce the impact of extreme observations, it is useful to take $\rho(u)=o(u^{2})$ as $u \to \pm\infty$. Here the $s_{j}>0$ factors are used to ensure that consistent estimates take place irrespective of the order of magnitude of the observations. We set the scaling factors in two steps. First is \textsc{rescale}, in which a rough dispersion estimate of the data is computed for each $j$ using
\begin{align}\label{eqn:dispersion_rough}
\widehat{\sigma}_{j} \in \left\{
\sigma > 0: \sum_{i=1}^{n} \chi\left(\frac{l^{\prime}_{j}(\ww;\zz_{i})-\gamma_{j}}{\sigma}\right) = 0 \right\}.
\end{align}
Here $\chi:\RR \to \RR$ is an even function, satisfying $\chi(0)<0$, and $\chi(u)>0$ as $u \to \pm\infty$ to ensure that the resulting $\widehat{\sigma}_{j}$ is an adequate measure of the dispersion of $l^{\prime}_{j}(\ww;\zz)$ about a pivot point, say $\gamma_{j}=\sum_{i=1}^{n}l^{\prime}_{j}(\ww;\zz_{i})/n$. Second, we adjust this estimate based on the available sample size and desired confidence level, as
\begin{align}\label{eqn:scale_rough}
s_{j} = \widehat{\sigma}_{j}\sqrt{n/\log(2\delta^{-1})}
\end{align}
where $\delta \in (0,1)$ specifies the desired confidence level $(1-\delta)$, and $n$ is the sample size. This last step appears rather artificial, but can be derived from a straightforward theoretical argument, given in section \ref{sec:algo_justification}. This concludes all the steps\footnote{For concreteness, in all empirical tests to follow we use the Gudermannian function \citep{abramowitz1964a}, $\rho(u) = \int_{0}^{u} \psi(x) \, dx$ where $\psi(u) = 2\atan(\exp(u))-\pi/2$, and $\chi(u) = u^{2}/(1+u^{2})-c$, for a constant $c>0$. General conditions on $\rho$, as well as standard methods for computing the M-estimates, namely the $\that_{j}$ and $\widehat{\sigma}_{j}$, are given in appendix \ref{sec:prelims}.} in one full iteration of Algorithm \ref{algo:rgd} on $\RR^{d}$.

In the remainder of this paper, we shall investigate the learning properties of this procedure, through analysis of both a theoretical (section \ref{sec:algo}) and empirical (section \ref{sec:tests}) nature. As an example, in the strongly convex risk case, our formal argument yields excess risk bounds of the form
\begin{align*}
R(\wwhat_{(T)}) - R^{\ast} \leq O\left(\frac{d(\log(d\delta^{-1})+d\log(n))}{n}\right) + O\left((1-\alpha\beta)^{T}\right)
\end{align*}
with probability no less than $1-\delta$, for small enough $\alpha_{(t)} = \alpha$ over $T$ iterations. Here $\beta > 0$ is a constant that depends only on $R$, and analogous results hold without strong convexity. Of the underlying distribution, all that is assumed is a bound on the variance of $l^{\prime}(\cdot;\zz)$, suggesting formally that the procedure should be competitive over a diverse range of data distributions.

\section{Theoretical analysis}\label{sec:algo}

Here we analyze the performance of Algorithm \ref{algo:rgd} on hypothesis class $\WW \subseteq \RR^{d}$, as measured by the risk achieved, which we estimate using upper bounds that depend on key parameters of the learning task. A general sketch is given, followed by some key conditions, representative results, and discussion. All proofs are relegated to appendix \ref{sec:appendix_proofs}.

\paragraph{Notation}
For integer $k$, write $[k] \defeq \{1,\ldots,k\}$ for all the positive integers from $1$ to $k$. Let $\mu$ denote the data distribution, with $\zz_{1},\ldots,\zz_{n}$ independent observations from $\mu$, and $\zz \sim \mu$ an independent copy. Risk is then $R(\ww) \defeq \exx_{\mu}l(\ww;\zz)$, its gradient $\gvec(\ww) \defeq R^{\prime}(\ww)$, and $R^{\ast} \defeq \inf_{\ww \in \WW} R(\ww)$. $\prr$ denotes a generic probability measure, typically the product measure induced by the sample. We write $\|\cdot\|$ for the usual ($\ell_{2}$) norm on $\RR^{d}$. For function $F$ on $\RR^{d}$ with partial derivatives defined, write the gradient as $F^{\prime}(\uu) \defeq (F^{\prime}_{1}(\uu),\ldots,F^{\prime}_{d}(\uu))$ where for short, we write $F^{\prime}_{j}(\uu) \defeq \partial F(\uu)/\partial u_{j}$.

\subsection{Sketch of the general argument}\label{sec:algo_justification}

The analysis here requires only two steps: (i) A good estimate $\gghat \approx \gvec$ implies that approximate update (\ref{eqn:GD_update_approx}) is near the optimal update. (ii) Under variance bounds, coordinate-wise M-estimation yields a good gradient estimate. We are then able to conclude that with enough samples and iterations, the output of Algorithm \ref{algo:rgd} can achieve an arbitrarily small excess risk. Here we spell out the key facts underlying this approach.

For the first step, let $\wwstar \in \RR^{d}$ be a minimizer of $R$. When the risk $R$ is strongly convex, then using well-established convex optimization theory \citep{nesterov2004ConvOpt}, we can easily control $\|\wwstar_{(t+1)}-\wwstar\|$ as a function of $\|\wwstar_{(t)}-\wwstar\|$ for any step $t \geq 0$. Thus to control $\|\wwhat_{(t+1)}-\wwstar\|$, in comparing the approximate case and optimal case, all that matters is the difference between $\gvec(\wwhat_{(t)})$ and $\gghat(\wwhat_{(t)})$ (Lemma \ref{lem:dueling_strong}). For the general case of convex $R$, since we cannot easily control the distance of the optimal update from any potential minimum, one can directly compare the trajectories of $\wwhat_{(t)}$ and $\wwstar_{(t)}$ over $t=0,1,\ldots,T$, which once again amounts to a comparison of $\gvec$ and $\gghat$. This inevitably leads to more error propagation and thus a stronger dependence on $T$, but the essence of the argument is identical to the strongly convex case.

For the second step, since both $\gghat$ and $\wwhat_{(t)}$ are based on a random sample $\{\zz_{1},\ldots,\zz_{n}\}$, we need an estimation technique which admits guarantees for any step, with high probability over the random draw of this sample. A basic requirement is that
\begin{align}\label{eqn:conf_uniform}
\prr\left\{ \max_{t \leq T} \|\gghat(\wwhat_{(t)})-\gvec(\wwhat_{(t)})\| \leq \varepsilon \right\} \geq 1-\delta.
\end{align}
Of course this must be proved (see Lemmas \ref{lem:grad_estimate} and \ref{lem:grad_estimate_varknown}), but if valid, then running Algorithm \ref{algo:rgd} for $T$ steps, we can invoke (\ref{eqn:conf_uniform}) to get a high-probability event on which $\wwhat_{(T)}$ closely approximates the optimal GD output, up to the accuracy specified by $\varepsilon$. Naturally this $\varepsilon$ will depend on confidence level $\delta$, which implies that to get $1-\delta$ confidence intervals, the upper bound in (\ref{eqn:conf_uniform}) will increase as $\delta$ gets smaller.

In the \textsc{locate} sub-routine of Algorithm \ref{algo:rgd}, we construct a more robust estimate of the risk gradient than can be provided by the empirical mean, using an ancillary estimate of the gradient variance. This is conducted using a smooth truncation scheme, as follows. One important property of $\rho$ in (\ref{eqn:location_rough}) is that for any $u \in \RR$, one has
\begin{align}\label{eqn:rho_Catoni_condition}
-\log(1-u+Cu^{2}) \leq \rho^{\prime}(u) \leq \log(1+u+Cu^{2})
\end{align}
for a fixed $C>0$, a simple generalization of the key property utilized by \citet{catoni2012a}. For the Gudermannian function (section \ref{sec:intuitive} footnote), we can take $C \leq 2$, with the added benefit that $\rho^{\prime}$ is bounded and increasing. As to the quality of these estimates, note that they are distributed sharply around the risk gradient, as follows.
\begin{lem}[Concentration of M-estimates]\label{lem:sharp_Mest}
For each coordinate $j \in [d]$, the estimates $\that_{j}$ of (\ref{eqn:location_rough}) satisfy
\begin{align}\label{eqn:sharp_Mest}
\frac{1}{2}|\that_{j}-g_{j}(\ww)| \leq\frac{C\vaa_{\mu}l_{j}^{\prime}(\ww;\zz)}{s_{j}} + \frac{s_{j}\log(2\delta^{-1})}{n}
\end{align}
with probability no less than $1-\delta$, given large enough $n$ and $s_{j}$.
\end{lem}
\noindent To get the tightest possible confidence interval as a function of $s_{j} > 0$, we must set
\begin{align*}
s_{j}^{2} = \frac{Cn\vaa_{\mu}l_{j}^{\prime}(\ww;\zz)}{\log(2\delta^{-1})},
\end{align*}
from which we derive (\ref{eqn:scale_rough}), with $\widehat{\sigma}_{j}^{2}$ corresponding to a computable estimate of $\vaa_{\mu}l_{j}^{\prime}(\ww;\zz)$. If the variance over all choices of $\ww$ is bounded by some $V < \infty$, then up to the variance estimates, we have $\|\gghat(\ww)-\gvec(\ww)\| \leq O(\sqrt{dV\log(2d\delta^{-1})/n})$, with $\gghat = \widehat{\mv{\theta}}$ from Algorithm \ref{algo:rgd}, yielding a bound for (\ref{eqn:conf_uniform}) free of $\ww$.

\begin{rmk}[Comparison with ERM-GD]
As a reference example, assume we were to run ERM-GD, namely using an empirical mean estimate of the gradient. Using Chebyshev's inequality, with probability $1-\delta$ all we can guarantee is $\varepsilon \leq O(\sqrt{d/(n\delta)})$. On the other hand, using the location estimate of Algorithm \ref{algo:rgd} provides guarantees with $\log(1/\delta)$ dependence on the confidence level, realizing an exponential improvement over the $1/\delta$ dependence of ERM-GD, and an appealing formal motivation for using M-estimates of location as a novel strategy.
\end{rmk}

\subsection{Conditions and results}\label{sec:algo_performance}

On the learning task, we make the following assumptions.
\begin{enumerate}
\item[\namedlabel{asmp:A1}{A1}.] Minimize risk $R(\cdot)$ over a closed, convex $\WW \subset \RR^{d}$ with diameter $\Delta < \infty$.
\item[\namedlabel{asmp:A2}{A2}.] $R(\cdot)$ and $l(\cdot;\zz)$ (for all $\zz$) are $\lambda$-smooth, convex, and continuously differentiable on $\WW$.
\item[\namedlabel{asmp:A3}{A3}.] There exists $\wwstar \in \WW$ at which $\gvec(\wwstar) = 0$.
\item[\namedlabel{asmp:A4}{A4}.] Distribution $\mu$ satisfies $\vaa_{\mu} l_{j}^{\prime}(\ww;\zz) \leq V < \infty$, for all $\ww \in \WW$, $j \in [d]$.
\end{enumerate}

Algorithm \ref{algo:rgd} is run following (\ref{eqn:location_rough}), (\ref{eqn:dispersion_rough}), and (\ref{eqn:scale_rough}) as specified in section \ref{sec:intuitive}. For \textsc{rescale}, the choice of $\chi$ is only important insofar as the scale estimates (the $\widehat{\sigma}_{j}$) should be moderately accurate. To make the dependence on this accuracy precise, take constants $c_{min},c_{max} > 0$ such that
\begin{align}\label{eqn:req_dispersion}
c_{min}^{2} \leq \frac{\widehat{\sigma}_{j}}{\vaa_{\mu}l_{j}^{\prime}(\ww;\zz)} \leq c_{max}^{2}, \quad j \in [d]
\end{align}
for all choices of $\ww \in \WW$, and write $c_{0} \defeq (c_{max}+C/c_{min})$. For $1-\delta$ confidence, we need a large enough sample; more precisely, for each $\ww$, it is sufficient if for each $j$,
\begin{align}\label{eqn:req_n_general}
\frac{1}{4} \geq \frac{C\log(2\delta^{-1})}{n}\left(1 + \frac{C\vaa_{\mu}l_{j}^{\prime}(\ww;\zz)}{\widehat{\sigma}_{j}^{2}}\right).
\end{align}
For simplicity, fix a small enough step size, 
\begin{align}\label{eqn:req_GD_update}
\alpha_{(t)} = \alpha, \forall \, t \in \{0,\ldots,T-1\}, \quad \alpha \in (0,2/\lambda).
\end{align}
Dependence on initialization is captured by two related factors $R_{0} \defeq R(\wwstar_{(0)})-R^{\ast}$, and $D_{0} \defeq \|\wwstar_{(0)}-\wwstar\|$. Under this setup, we can control the estimation error.
\begin{lem}[Uniform accuracy of gradient estimates]\label{lem:grad_estimate}
For all steps $t=0,\ldots,T-1$ of Algorithm \ref{algo:rgd}, we have
\begin{align*}
\|\widehat{\mv{\theta}}_{(t)}-\gvec(\wwhat_{(t)})\| \leq \frac{\widetilde{\varepsilon}}{\sqrt{n}} \defeq \frac{\lambda(\sqrt{d}+1)}{\sqrt{n}} + 2c_{0} \sqrt{\frac{dV(\log(2d\delta^{-1}) + d\log(3\Delta\sqrt{n}/2))}{n}}
\end{align*}
with probability no less than $1-\delta$.
\end{lem}

\paragraph{Under strongly convex risk}

In addition to assumptions \ref{asmp:A1}--\ref{asmp:A4}, assume that $R$ is $\kappa$-strongly convex. In this case, $\wwstar$ in \ref{asmp:A3} is the unique minimum. First, we control the estimation error by showing that the approximate update (\ref{eqn:GD_update_approx}) does not differ much from the optimal update (\ref{eqn:GD_update_true}).

\begin{lem}[Minimizer control]\label{lem:dueling_strong}
Consider the general approximate GD update (\ref{eqn:GD_update_approx}), with $\alpha_{(t)}=\alpha$ such that $0 < \alpha < 2/(\kappa+\lambda)$. Assume that (\ref{eqn:conf_uniform}) holds with bound $\varepsilon$. Write $\beta \defeq 2\kappa\lambda/(\kappa+\lambda)$. Then, with probability no less than $1-\delta$, we have
\begin{align*}
\|\wwhat_{(T)}-\wwstar\| \leq (1-\alpha\beta)^{T/2}D_{0} + \frac{2\varepsilon}{\beta}.
\end{align*}
\end{lem}
\noindent Since Algorithm \ref{algo:rgd} indeed satisfies (\ref{eqn:conf_uniform}), as proved in Lemma \ref{lem:grad_estimate}, we can use the control over the parameter deviation provided by Lemma \ref{lem:dueling_strong} and the smoothness of $R$ to prove a finite-sample excess risk bound.
\begin{thm}[Excess risk bounds]\label{thm:main_Rbound_strong}
Write $\wwhat_{(T)}$ for the output of Algorithm \ref{algo:rgd} after $T$ iterations, run such that (\ref{eqn:req_n_general})--(\ref{eqn:req_GD_update}) hold, with step size $\alpha_{(t)}=\alpha$ for all $0<t<T$, as in Lemma \ref{lem:dueling_strong}. It follows that
\begin{align*}
R(\wwhat_{(T)})-R^{\ast} \leq \lambda(1-\alpha\beta)^{T}D_{0}^{2} + \frac{4\lambda \widetilde{\varepsilon}}{\beta^{2} n}
\end{align*}
with probability no less than $1-\delta$, where $\widetilde{\varepsilon}$ is as given in Lemma \ref{lem:grad_estimate}.
\end{thm}

\begin{rmk}[Interpretation of bounds]
There are two terms in the upper bound of Theorem \ref{thm:main_Rbound_strong}, an optimization term decreasing in $T$, and an estimation term decreasing in $n$. The optimization error decreases at the usual gradient descent rate, and due to the uniformity of the bounds obtained, the statistical error is not hurt by taking $T$ arbitrarily large, thus with enough samples we can guarantee arbitrarily small excess risk. Finally, the most important assumption on the distribution is weak: finite second-order moments. If we assume finite kurtosis, the argument of \citet{catoni2012a} can be used to create analogous guarantees for an explicit scale estimation procedure, yielding guarantees whether the data is sub-Gaussian or heavy-tailed an appealing robustness to the data distribution.
\end{rmk}

\begin{rmk}[Doing projected descent]
The above analysis proceeds on the premise that $\wwhat_{(t)} \in \WW$ holds after all the updates, $t \in [T]$. To enforce this, a standard variant of Algorithm \ref{algo:rgd} is to update as
\begin{align*}
\wwhat_{(t+1)} \gets \pi_{\WW}\left( \wwhat_{(t)} - \alpha_{(t)}\widehat{\mv{\theta}}_{(t)} \right), \quad t \in \{0,\ldots,T-1\}
\end{align*}
where $\pi_{\WW}(\uu) \defeq \argmin_{\vv \in \WW}\|\uu-\vv\|$. By \ref{asmp:A1}, this projection is well-defined \citep[Sec.~3.12, Thm.~3.12]{luenberger1969Book}. Using this fact, it follows that $\|\pi_{\WW}(\uu)-\pi_{\WW}(\vv)\| \leq \|\uu-\vv\|$ for all $\uu,\vv \in \WW$, by which we can immediately show that Lemma \ref{lem:dueling_strong} holds for the \textit{projected robust gradient descent} version of Algorithm \ref{algo:rgd}.
\end{rmk}

\paragraph{With prior information}

An interesting concept in machine learning is that of the relationship between learning efficiency, and the task-related prior information available to the learner. In the previous results, the learner is assumed to have virtually no information beyond the data available, and the ability to set a small enough step-size. What if, for example, just the gradient variance was known? A classic example from decision theory is the dominance of the estimator of James and Stein over the maximum likelihood estimator, in multivariate Normal mean estimation using prior variance information. In our more modern and non-parametric setting, the impact of rough, data-driven scale estimates was made explicit by the factor $c_{0}$. Here we give complementary results that show how partial prior information on the distribution $\mu$ can improve learning.

\begin{lem}[Accuracy with variance information]\label{lem:grad_estimate_varknown}
Conditioning on $\wwhat_{(t)}$ and running one scale-location sequence of Algorithm \ref{algo:rgd}, with $\mv{\widehat{\sigma}}_{(t)}=(\widehat{\sigma}_{1},\ldots,\widehat{\sigma}_{d})$ modified to satisfy $\widehat{\sigma}_{j}^{2} = C\vaa_{\mu}l^{\prime}_{j}(\wwhat_{(t)};\zz)$, $j \in [d]$. It follows that
\begin{align*}
\|\mv{\that}_{(t)}-\gvec(\wwhat_{(t)})\| \leq 4\left(\frac{C\trace(\Sigma_{(t)})\log(2d\delta^{-1})}{n}\right)^{1/2}
\end{align*}
with probability no less than $1-\delta$, where $\Sigma_{(t)}$ is the covariance matrix of $l^{\prime}(\wwhat_{(t)};\zz)$.
\end{lem}

\noindent One would expect that with sharp gradient estimates, the variance of the updates should be small with a large enough sample. Here we show that the procedure stabilizes quickly as the estimates get closer to an optimum.
\begin{thm}[Control of update variance]\label{thm:variance_control}
Run Algorithm \ref{algo:rgd} as in Lemma \ref{lem:grad_estimate_varknown}, with arbitrary step-size $\alpha_{(t)}$. Then, for any $t<T$, taking expectation with respect to the sample $\{\zz_{i}\}_{i=1}^{n}$, conditioned on $\wwhat_{(t)}$, we have
\begin{align*}
\exx\|\wwhat_{(t+1)}-\wwhat_{(t)}\|^{2} \leq 2\alpha_{(t)}^{2}\left(\frac{32Cd\trace(\Sigma_{(t)})}{n} + \|\gvec(\wwhat_{(t)})\|^{2}\right).
\end{align*}
\end{thm}

\noindent In addition to these results, one can prove an improved version of Theorem \ref{thm:main_Rbound_strong} in a perfectly analogous fashion, using Lemma \ref{lem:grad_estimate_varknown}.

\section{Empirical analysis}\label{sec:tests}

The chief goal of our experiments is to elucidate the relationship between factors of the learning task (e.g., sample size, model dimension, initial value, underlying data distribution) and the behavior of the robust gradient procedure proposed in Algorithm \ref{algo:rgd}. We are interested in how these factors influence performance, both in an absolute sense and relative to the key competitors cited in section \ref{sec:intro}.

We have carried out three classes of experiments. The first considers a concrete risk minimization task given noisy function observations, and takes an in-depth look at how each experimental factor influences algorithm behavior, in particular the trajectory of performance over time (as we iterate). Second is an application of the proposed algorithm to the corresponding regression task under a large variety of data distributions, meant to rigorously evaluate the practical utility and robustness in an agnostic learning setting. Finally, we consider applications to classification tasks using real-world data sets.

\subsection{Controlled tests}\label{sec:tests_noisyopt}

\paragraph{Experimental setup}

Our first set of controlled numerical experiments uses a ``noisy convex minimization'' model, designed as follows. We construct a risk function taking a canonical quadratic form, setting $R(\ww) = \langle \Sigma\ww, \ww \rangle/2 + \langle \ww, \uu \rangle + c$, for pre-fixed constants $\Sigma \in \RR^{d \times d}$, $\uu \in \RR^{d}$, and $c \in \RR$. The task is to minimize $R(\cdot)$ without knowledge of $R$ itself, but rather only access to $n$ random function observations $r_{1},\ldots,r_{n}$. These $r:\RR^{d} \to \RR$ are generated independently from a common distribution, satisfying the property $\exx r(\ww) = R(\ww)$ for all $\ww \in \RR^{d}$. In particular, here we generate observations $r_{i}(\ww)=(\langle \wwstar-\ww, \xx_{i}\rangle + \epsilon_{i})^{2}/2$, $i \in [n]$, with $\xx$ and $\epsilon$ independent of each other. Here $\wwstar$ denotes the minimum, and we have that $\Sigma = \exx \xx\xx^{T}$. The inputs $\xx$ shall follow an isotropic $d$-dimensional Gaussian distribution throughout all the following experiments, meaning $\Sigma$ is positive definite, and $R$ is strongly convex.

We consider three main performance metrics in this section: the average excess empirical risk (based on the losses $r_{1},\ldots,r_{n}$), the average excess risk (based on true risk $R$), and the variance of the risk. Averages and variances are computed over trials, with each trial corresponding to a new independent random sample. For all tests, the number of trials is 250.

For these first tests, we run three procedures. First is ideal gradient descent, denoted \texttt{oracle}, which has access to the true objective function $R$. This corresponds to (\ref{eqn:GD_update_true}). Second, as a standard approximate procedure (\ref{eqn:GD_update_approx}) when $R$ is unknown, we use ERM-GD, denoted \texttt{erm} and discussed at the start of section \ref{sec:intuitive}, which approximates the optimal procedure using the empirical risk. Against these two benchmarks, we compare our Algorithm \ref{algo:rgd}, denoted \texttt{rgd}, as a robust alternative for (\ref{eqn:GD_update_approx}).

\paragraph{Impact of heavy-tailed noise}

Let us examine the results. We begin with a simple question: are there natural learning settings in which \texttt{rgd} outperforms ERM-GD? How does the same algorithm fare in situations where ERM is optimal? Under Gaussian noise, ERM-GD is effectively optimal \citep[Appendix C]{lin2016a}. We thus consider the case of Gaussian noise (mean $0$, standard deviation $20$) as a baseline, and use centered log-Normal noise (log-location $0$, log-scale $1.75$) as an archetype of asymmetric heavy-tailed data. Risk results for the two routines are given alongside training error in Figure \ref{fig:POC}.

\begin{figure}[t]
\centering
\includegraphics[width=1.0\textwidth]{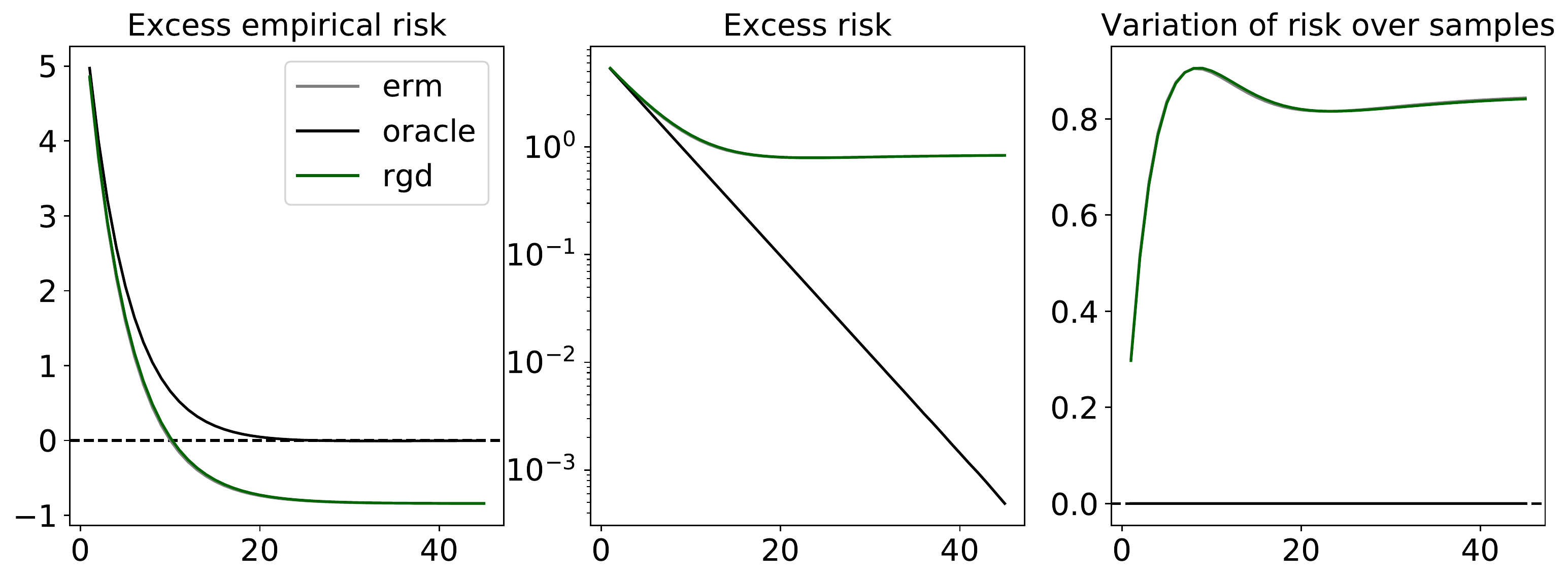}\\
\includegraphics[width=1.0\textwidth]{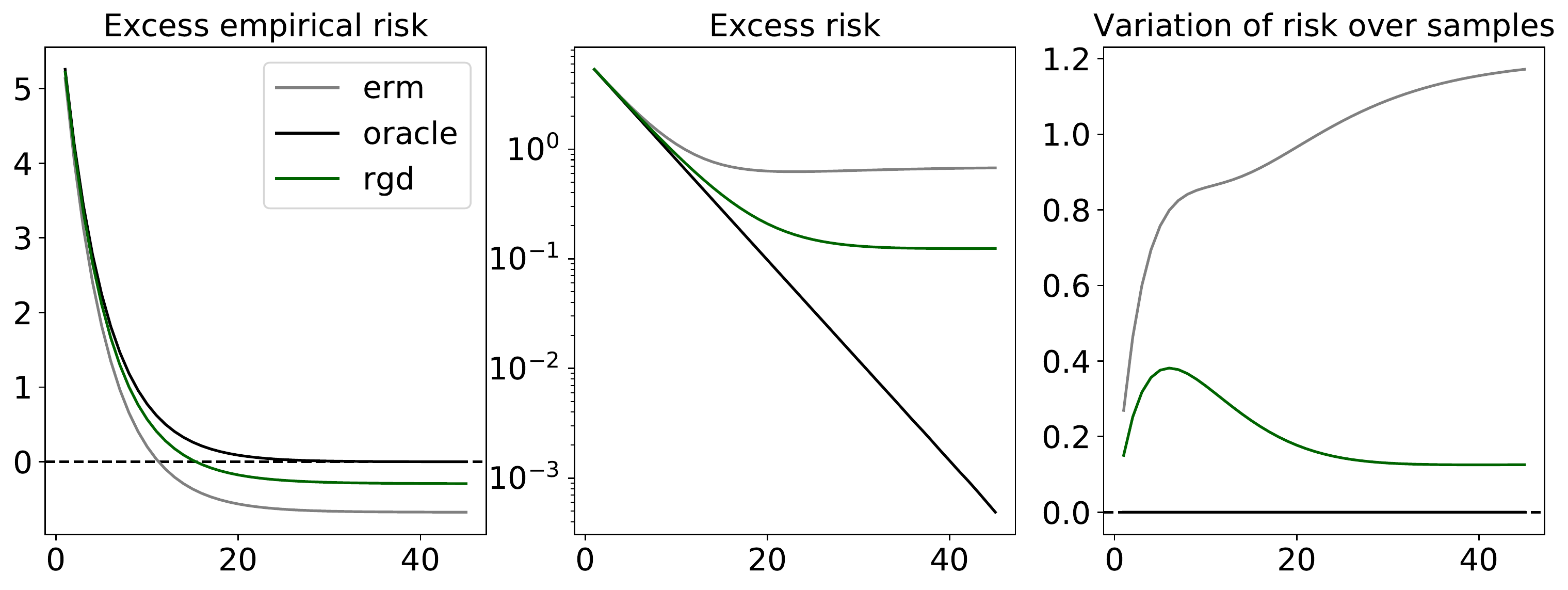}
\caption{Performance metrics as a function of iterative updates. Top row: Normal noise. Bottom row: log-Normal noise. Settings: $n = 500, d=2, \alpha_{(t)}=0.1$ for all $t$.}
\label{fig:POC}
\end{figure}

In the situation favorable to \texttt{erm}, differences in performance are basically negligible. On the other hand, in the heavy-tailed setting, the performance of \texttt{rgd} is superior in terms of quality of the solution found and the variance of the estimates. Furthermore, we see that at least in the situation of small $d$ and large $n$, taking $T$ beyond numerical convergence has minimal negative effect on \texttt{rgd} performance; on the other hand \texttt{erm} is more sensitive. Comparing true risk with sample error, we see that while there is some unavoidable overfitting, in the heavy-tailed setting \texttt{rgd} departs from the ideal routine at a slower rate, a desirable trait.

At this point, we still have little more than a proof of concept, with rather arbitrary choices of $n$, $d$, noise distribution, and initialization method. We proceed to investigate how each of these experimental parameters independently impacts performance.

\paragraph{Impact of initialization}

Given a fixed data distribution and sample size, how does the quality of the initial guess impact learning performance? We consider three initializations of the form $\wwstar + \text{Unif}[-\mv{\Delta},\mv{\Delta}]$, with $\mv{\Delta} = (\Delta_{1},\ldots,\Delta_{d})$, values ranging over $\Delta_{j} \in \{2.5, 5.0, 10.0\}$, $j \in [d]$, where larger $\Delta_{j}$ naturally correspond to potentially worse initialization. Relevant results are displayed in Figure \ref{fig:INIT}.

\begin{figure}[t]
\centering
\includegraphics[width=1.0\textwidth]{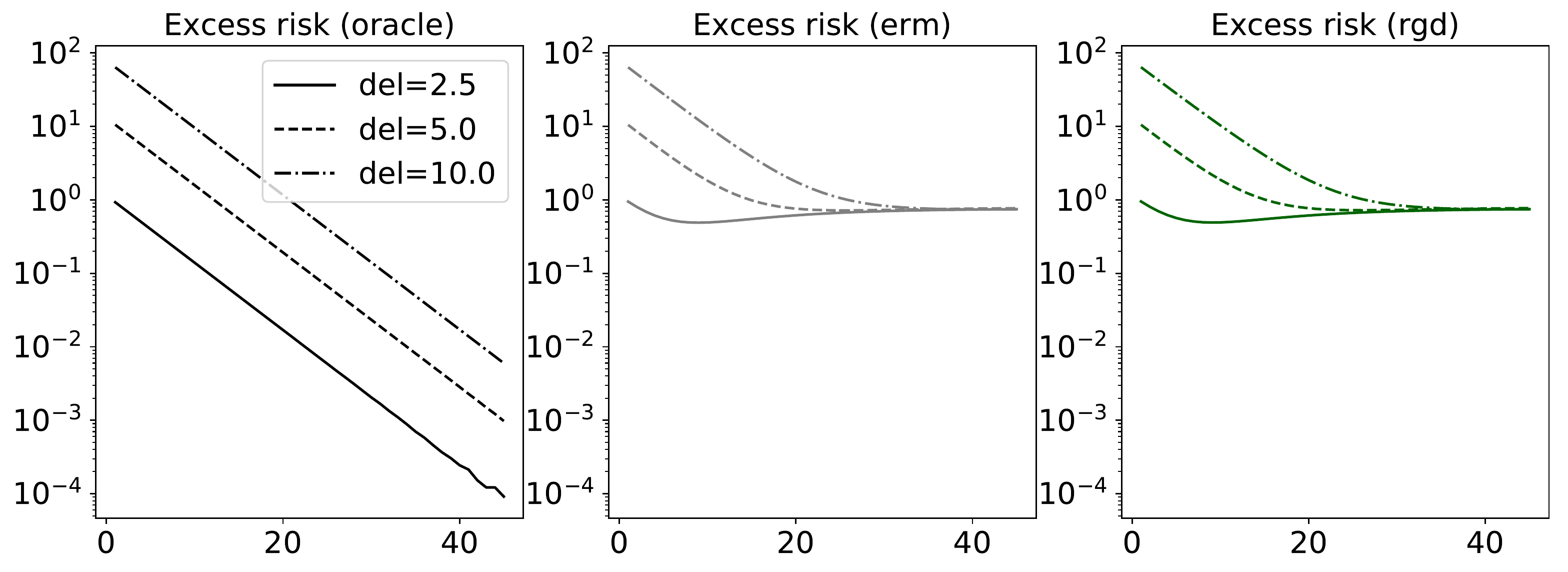}\\
\includegraphics[width=1.0\textwidth]{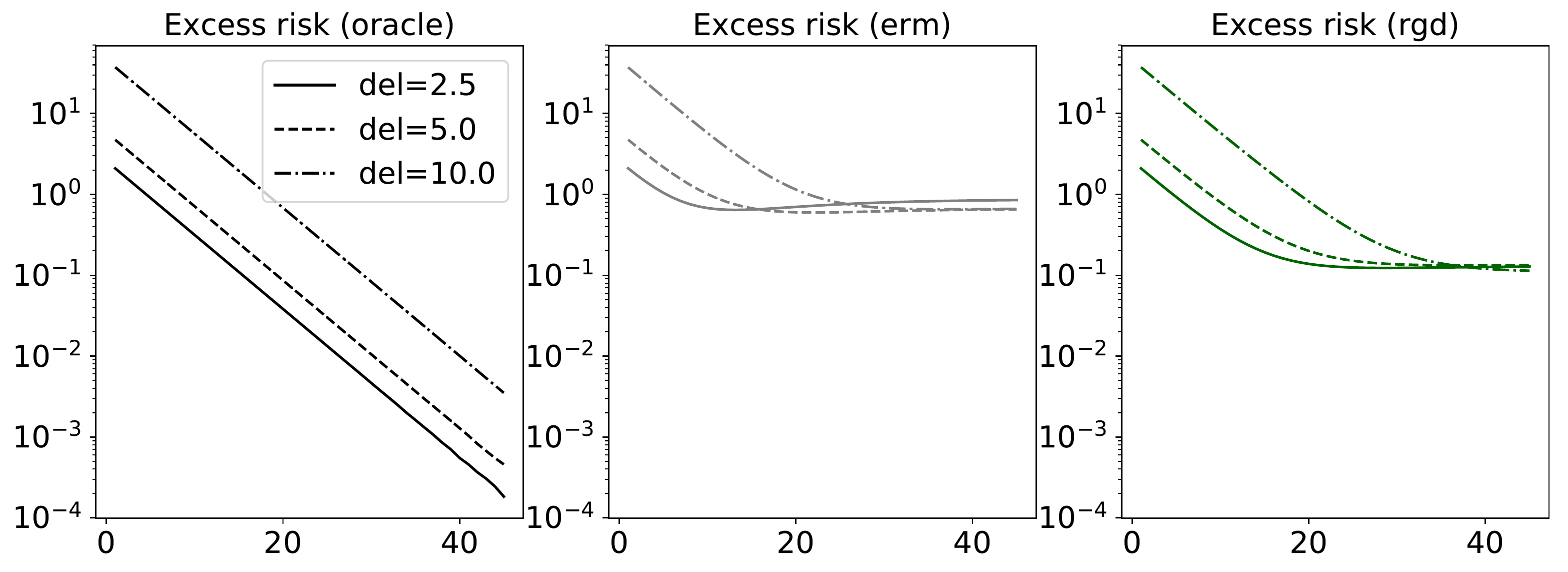}
\caption{Performance over iterations, under strong/poor initialization. Here \texttt{del} refers to $\Delta_{j}$. Top row: Normal noise. Bottom row: log-Normal noise. Settings: $n = 500, d=2, \alpha_{(t)}=0.1$ for all $t$.}
\label{fig:INIT}
\end{figure}

Some interesting observations can be made. That \texttt{rgd} matches \texttt{erm} when the latter is optimal is clear, but more importantly, we see that under heavy-tailed noise, \texttt{rgd} is far more robust to poor initial value settings. Indeed, while a bad initialization leads to a much worse solution in the limit for \texttt{erm}, we see that \texttt{rgd} is able to achieve the same performance as if it were initialized at a better value.

\paragraph{Impact of distribution}

It is possible for very distinct distributions to have exactly the same risk functions. Learning efficiency naturally depends heavily on the process generating the sample; the underlying optimization problem is the same, but the statistical inference task changes. Here we run the two algorithms of interest from common initial values as in the first experimental setting, and measure performance changes as the noise distribution is modified. We consider six situations, three for Normal noise, three for log-Normal noise. The location and scale parameters for the former are respectively $(0,0,0), (1,20,34)$; the log-location and log-scale parameters for the latter are respectively $(0,0,0), (1.25,1.75,1.9)$. Results are given in Figure \ref{fig:DIST}.

\begin{figure}[t]
\centering
\includegraphics[width=1.0\textwidth]{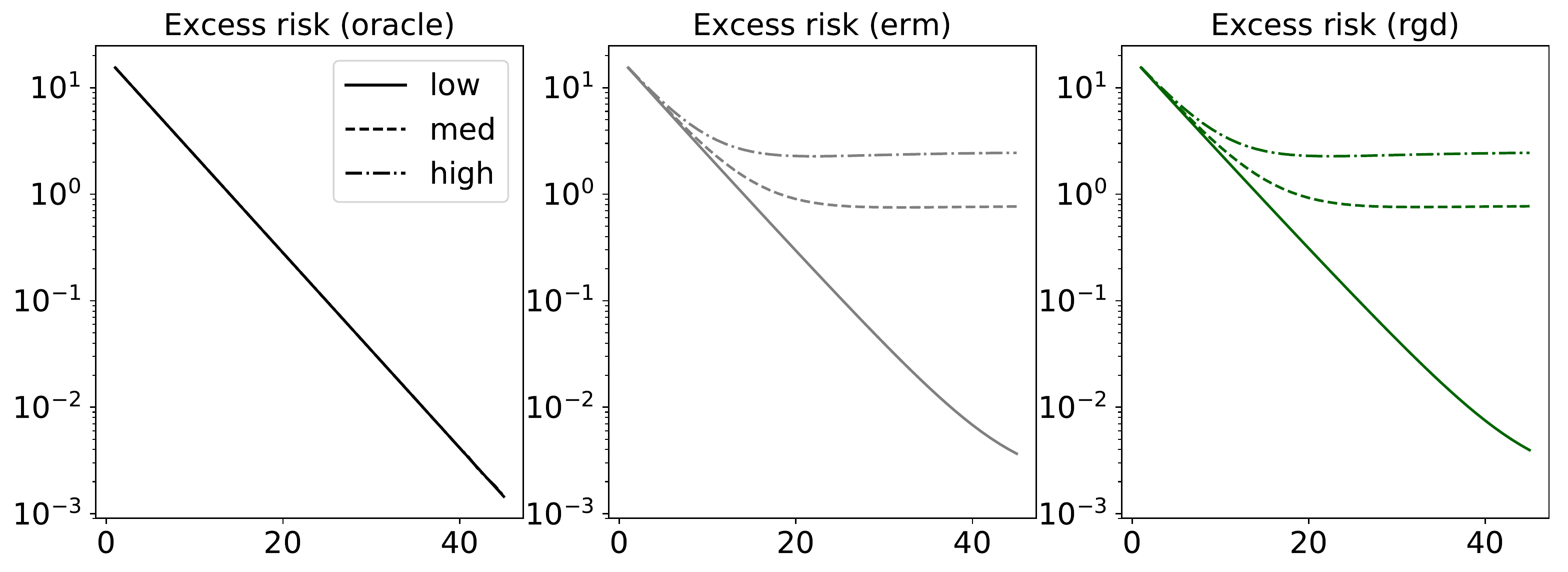}\\
\includegraphics[width=1.0\textwidth]{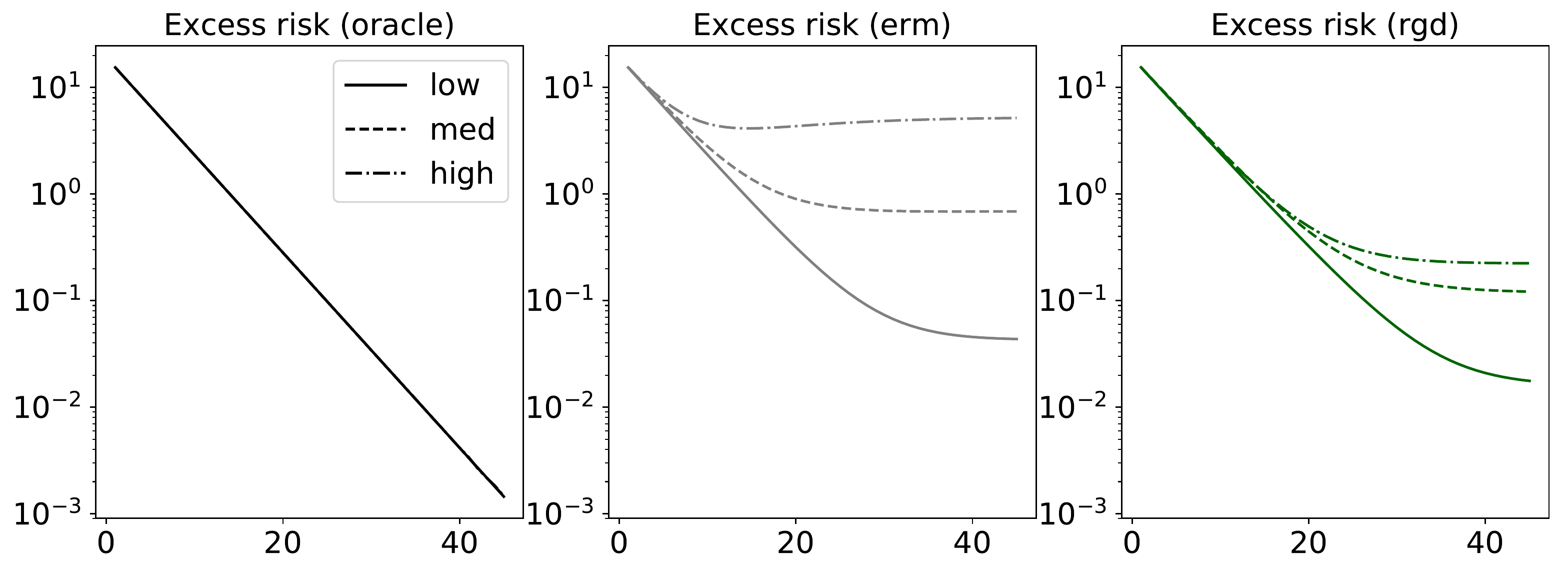}
\caption{Performance over iterations, under varying noise intensities. Here \texttt{low}, \texttt{med}, and \texttt{high} refer to the three noise distribution settings described in the main text. Settings: $n = 500, d = 2, \alpha_{(t)}=0.1$ for all $t$.}
\label{fig:DIST}
\end{figure}

Looking first at the Normal case, where we expect ERM-based methods to perform well, we see that \texttt{rgd} is able to match \texttt{erm} in all settings. In the log-Normal case, as our previous example suggested, the performance of \texttt{erm} degrades rather dramatically, and a clear gap in performance appears, which grows wider as the variance increases. This flexibility of \texttt{rgd} in dealing with both symmetric and asymmetric noise, both exponential and heavy tails, is indicative of the robustness suggested by the weak conditions of section \ref{sec:algo_performance}. In addition, it suggests that our simple dispersion-based technique ($\widehat{\sigma}_{j}$ settings in \ref{sec:intuitive_algo}) provides tolerable accuracy, implying a small enough $c_{0}$ factor, and reinforcing the insights from the proof of concept case seen in Figure \ref{fig:POC}.

\paragraph{Impact of sample size}

Since the true risk is unknown, the size and quality of the sample $\{\zz_{i}\}_{i=1}^{n}$ is critical to the output of all learners. To evaluate learning efficiency, we examine how performance depends on the available sample size, with dimension and all algorithm parameters fixed. Figure \ref{fig:NVAL} gives the accuracy of \texttt{erm} and \texttt{rgd} in tests analogous to those above, using common initial values across methods, and $n \in \{10, 40, 160, 640\}$.

\begin{figure}[t]
\centering
\includegraphics[width=1.0\textwidth]{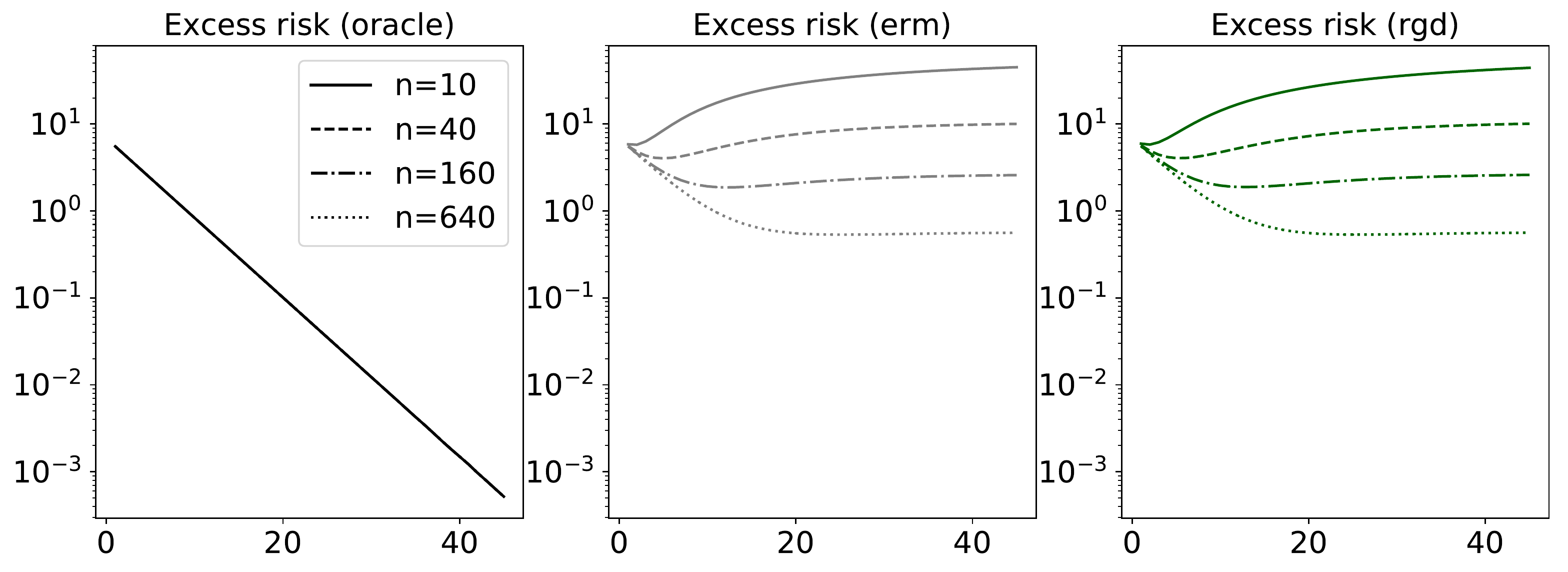}\\
\includegraphics[width=1.0\textwidth]{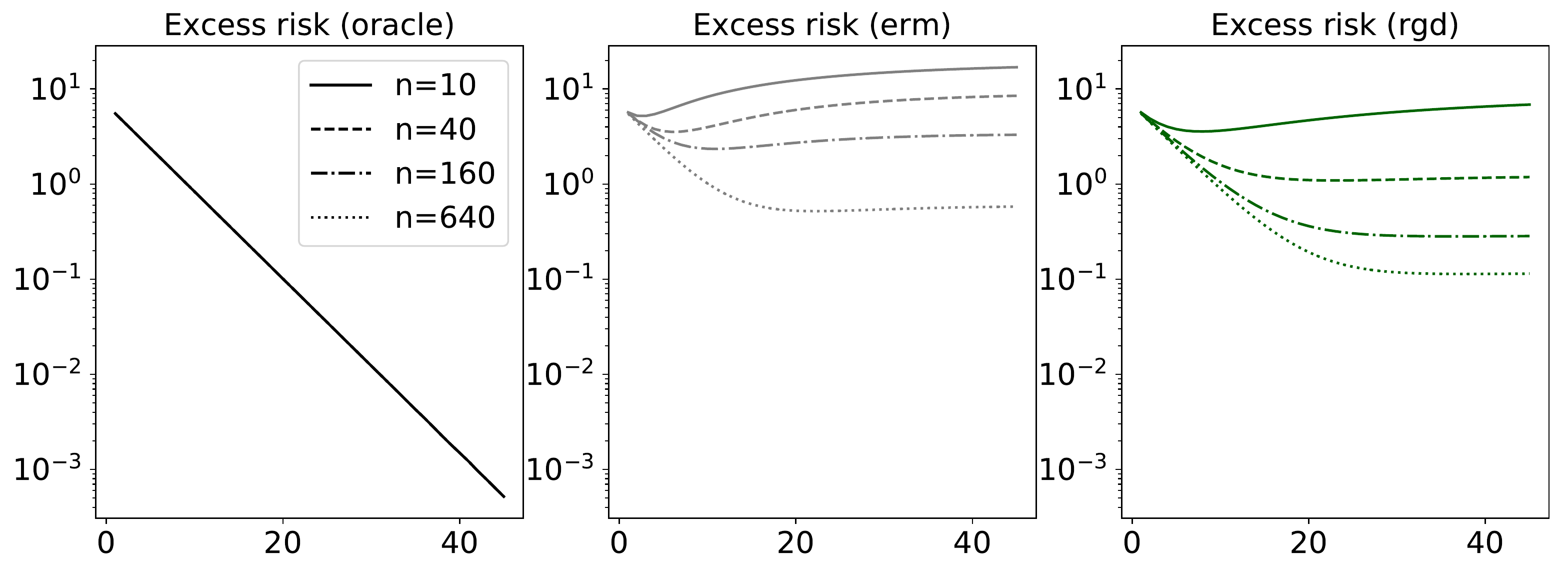}
\caption{Performance over iterations, under different sample sizes. Settings: $d=2, \alpha_{(t)}=0.1$ for all $t$.}
\label{fig:NVAL}
\end{figure}

Both algorithms naturally show monotonic performance improvements as the sample size grows, but the most salient feature of these figures is the performance of \texttt{rgd} under heavy-tailed noise, especially when sample sizes are small. When our data may be heavy-tailed, this provides clear evidence that the proposed RGD methods can achieve better generalization than ERM-GD with less data, in less iterations.

\paragraph{Impact of dimension}

Given a fixed number of observations, the role of dimension $d$, namely the number of parameters to be determined, plays an important from both practical and theoretical standpoints, as seen in the error bounds of section \ref{sec:algo_performance}. Fixing the sample size and all algorithm parameters as above, we investigate the relative difficulty each algorithm has as the dimension increases. Figure \ref{fig:DVAL} shows the risk of \texttt{erm} and \texttt{rgd} in tests just as above, with $d \in \{2, 8, 32, 128\}$.

\begin{figure}[t]
\centering
\includegraphics[width=1.0\textwidth]{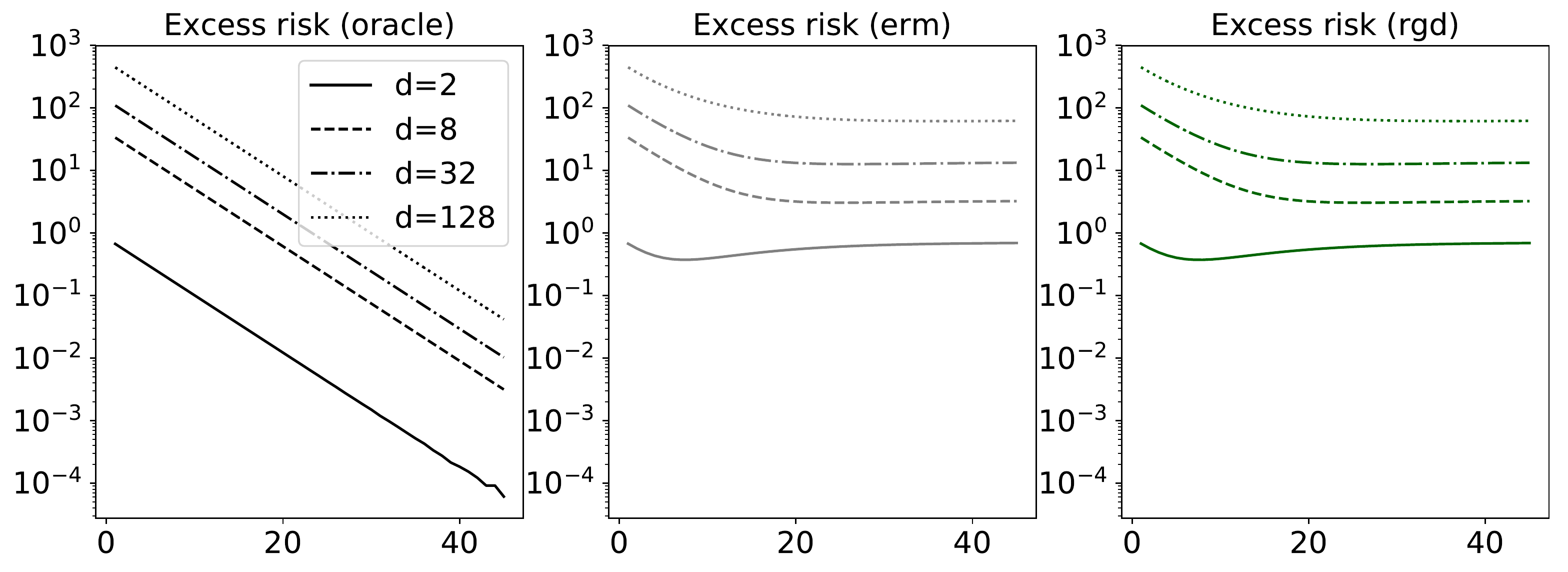}\\
\includegraphics[width=1.0\textwidth]{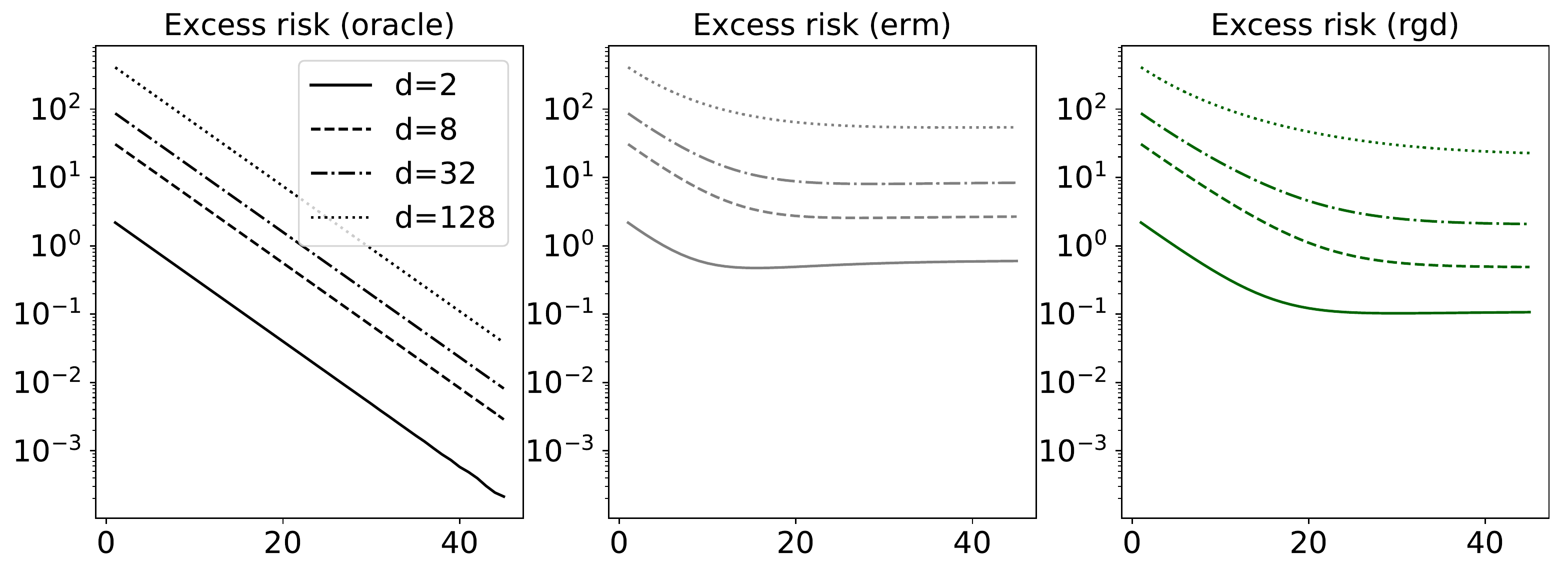}
\caption{Performance over iterations, under increasing dimension. Settings: $n = 500, \alpha_{(t)}=0.1$ for all $t$.}
\label{fig:DVAL}
\end{figure}

As the dimension increases, since the sample size is fixed, both non-oracle algorithms tend to require more iterations to converge. The key difference is that under heavy tails, the excess risk achieved by our proposed method is clearly superior to ERM-GD over all $d$ settings, while matching it in the case of Gaussian noise. In particular, \texttt{erm} hits bottom very quickly in higher dimensions, while \texttt{rgd} continues to improve for more iterations, presumably due to updates which are closer to that of the optimal (\ref{eqn:GD_update_true}).

\paragraph{Comparison with robust loss minimizer}

Another interesting question: instead of paying the overhead to robustify gradient estimates ($d$ dimensions to handle), why not just make robust estimates of the risk itself, and use those estimates to fuel an iterative optimizer? Just such a procedure is analyzed by \citet{brownlees2015a} (denoted \texttt{bjl} henceforth). To compare our gradient-centric approach with their loss-centric approach, we implement \texttt{bjl} using the non-linear conjugate gradient method of Polak and Ribi\`{e}re \citep{nocedal1999a}, which is provided by \texttt{fmin\_cg} in the \texttt{optimize} module of the SciPy scientific computation library (default maximum number of iterations is $200d$). This gives us a standard first-order general-purpose optimizer for minimizing the \texttt{bjl} objective. To see how well our procedure can compete with a pre-fixed max iteration number, we set $T=25$ for all settings. Computation time is computed using the Python \texttt{time} module. To give a simple comparison between \texttt{bjl} and \texttt{rgd}, we run multiple independent trials of the same task, starting both routines at the same (random) initial value each time, generating a new sample, and repeating the whole process for different settings of $d = 2, 4, 8, 16, 32, 64$. Median times taken over all trials (for each $d$ setting) are recorded, and presented in Figure \ref{fig:versusBJL} alongside performance results.

\begin{figure}[t]
\centering
\includegraphics[width=1.0\textwidth]{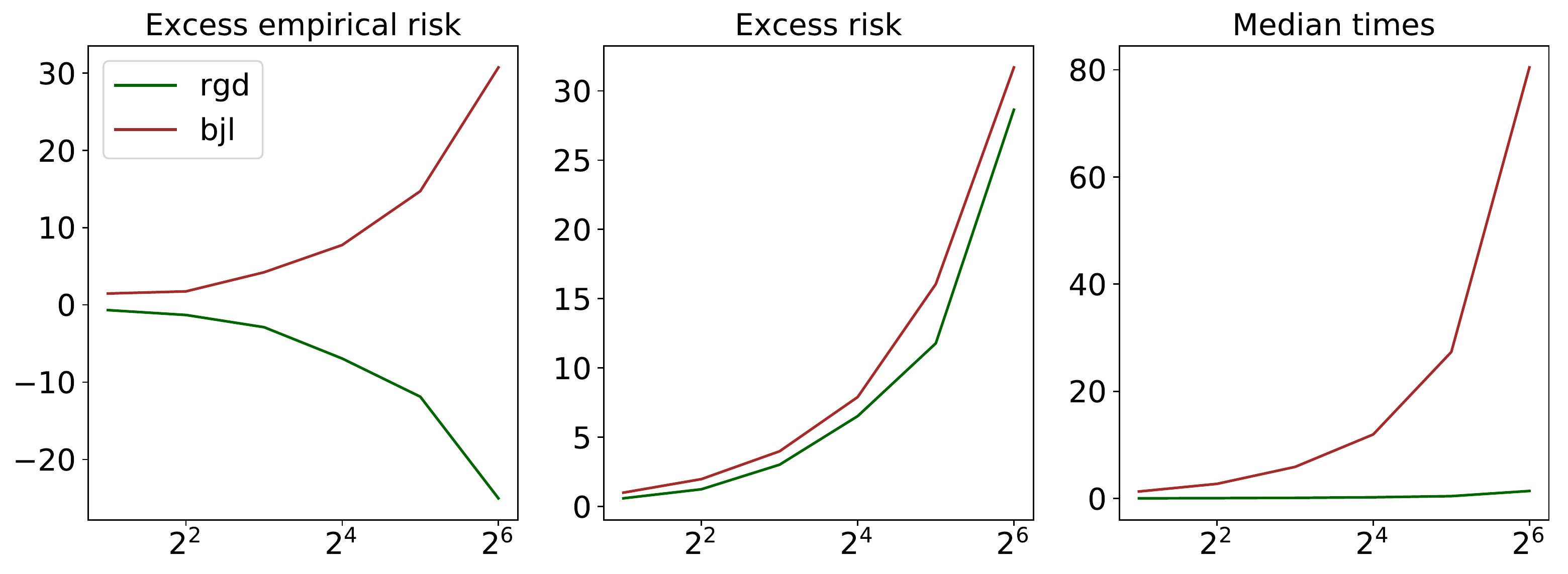}\\
\includegraphics[width=1.0\textwidth]{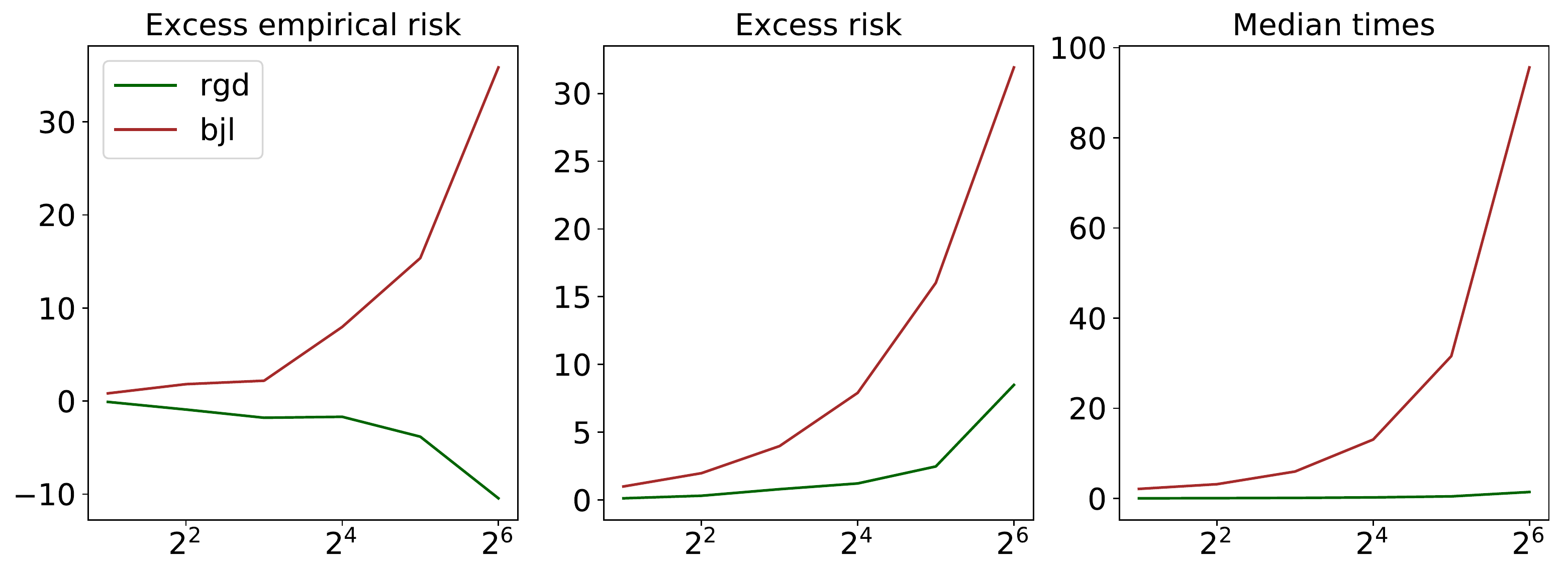}
\caption{Comparison of our robust gradient-based approach with the robust objective-based approach. Top: Normal noise. Bottom: log-Normal noise. Performance is given as a function of the number of $d$, the number of parameters to optimize, given in $\log_{2}$ scale. Settings: $n = 500, \alpha_{(t)}=0.1$ for all $t$.}
\label{fig:versusBJL}
\end{figure}

From the results, we can see that while the performance of both methods is similar in low dimensions and under Gaussian noise, in higher dimensions and under heavy-tailed noise, our proposed \texttt{rgd} realizes much better performance in much less time. Regarding excess empirical risk, random deviations in the sample cause the minimum of the empirical risk function to deviate away from $\wwstar$, causing the \texttt{rgd} solution to be closer to the ERM solution in higher dimensions. On the other hand, \texttt{bjl} is minimizing a different objective function. It should be noted that there are assuredly other ways of approaching the \texttt{bjl} optimization task, but all of which require minimizing an implicitly defined objective which need not be convex. We believe that \texttt{rgd} provides a simple and easily implemented alternative, while still utilizing the same statistical principles.

\paragraph{Regression application}

In this experiment, we apply our algorithm to a general regression task, under a wide variety of data distributions, and compare its performance against standard regression algorithms, both classical and modern. For each experimental condition, and for each trial, we generate $n$ training observations of the form $y_{i} = \xx_{i}^{T}\wwstar + \epsilon_{i}, i \in [n]$. Distinct experimental conditions are delimited by the setting of $(n,d)$ and $\mu$. Inputs $\xx$ are assumed to follow a $d$-dimensional isotropic Gaussian distribution, and thus our setting of $\mu$ will be determined by the distribution of noise $\epsilon$. In particular, we look at several families of distributions, and within each family look at 15 distinct noise levels, namely parameter settings designed such that $\sd_{\mu}(\epsilon)$ monotonically increases over the range 0.3--20.0, approximately linearly over the levels.

To capture a range of signal/noise ratios, for each trial, $\wwstar \in \RR^{d}$ is randomly generated as follows. Defining the sequence $w_{k} \defeq \pi/4 + (-1)^{k-1}(k-1)\pi/8, k=1,2,\ldots$ and uniformly sampling $i_{1},\ldots,i_{d} \in [d_{0}]$ with $d_{0}=500$, we set $\wwstar = (w_{i_{1}},\ldots,w_{i_{d}})$. Computing $\text{SN}_{\mu} = \|\wwstar\|_{2}^{2}/\vaa_{\mu}(\epsilon)$, we have $0.2 \leq SN_{\mu} \leq 1460.6$. Noise families: log-logistic (denoted \texttt{llog} in figures), log-Normal (\texttt{lnorm}), Normal (\texttt{norm}), and symmetric triangular (\texttt{tri\_s}). Even with just these four, we capture both bounded and unbounded sub-Gaussian noise, and heavy-tailed data both with and without finite higher-order moments. Results for many more noise distributions are given in appendix \ref{sec:more_test_results}.

Our key performance metric of interest is off-sample prediction error, here computed as excess RMSE on an independent large testing set, averaged over trials. For each condition and each trial, an independent test set of $m$ observations is generated identically to the corresponding $n$-sized training set. All competing methods use common sample sets for training and are evaluated on the same test data, for all conditions/trials. For each method, in the $k$th trial, some estimate $\wwhat(k)$ is determined. To approximate the $\ell_{2}$-risk, compute root mean squared test error $e_{k}(\wwhat) \defeq (m^{-1}\sum_{i=1}^{m}(\wwhat^{T}\xx_{k,i}-y_{k,i})^{2})^{1/2}$, and output prediction error as the average of normalized errors $e_{k}(\wwhat(k)) - e_{k}(\wwstar(k))$ taken over all $K$ trials. While $n$ values vary, in all experiments we fix $K=250$ and test size $m=1000$.

Regarding the competing methods, classical choices are ordinary least squares ($\ell_{2}$-ERM, denoted \texttt{OLD}) and least absolute deviations ($\ell_{1}$-ERM, \texttt{LAD}). We also look at two recent methods of practical and theoretical importance described in section \ref{sec:intro}, namely the robust regression routines of \citet{hsu2016a} (\texttt{HS}) and \citet{minsker2015a} (\texttt{Minsker}). For the former, we used the source published online by the authors. For the latter, on each subset the \texttt{ols} solution is computed, and solutions are aggregated using the geometric median (in $\ell_{2}$ norm), computed using the well-known algorithm of \citet[Eqn.~2.6]{vardi2000a}, and the number of partitions is set to $\max( 2, \lfloor n/(2d) \rfloor )$. For comparison to this, we also initialize \texttt{RGD} to the \texttt{OLS} solution, with confidence $\delta=0.005$, and $\alpha_{(t)} = 0.1$ for all iterations. Maximum number of iterations is $T \leq 100$; the routine finishes after hitting this maximum or when the absolute value of the gradient falls below $0.001$ for all conditions. Illustrative results are given in Figure \ref{fig:multinoise_linreg}.

\begin{figure}[t]
\centering
\includegraphics[width=0.25\textwidth]{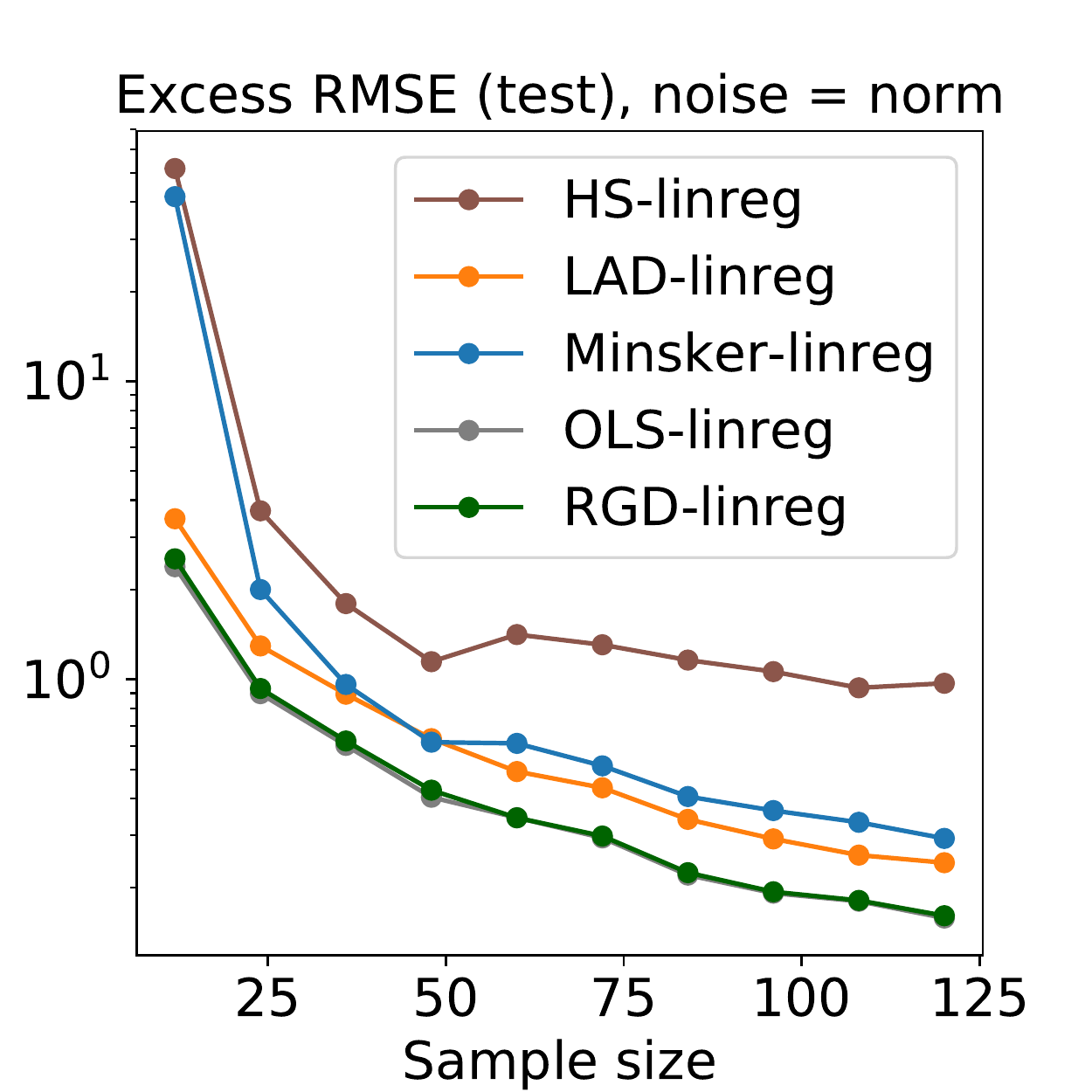}\,\includegraphics[width=0.25\textwidth]{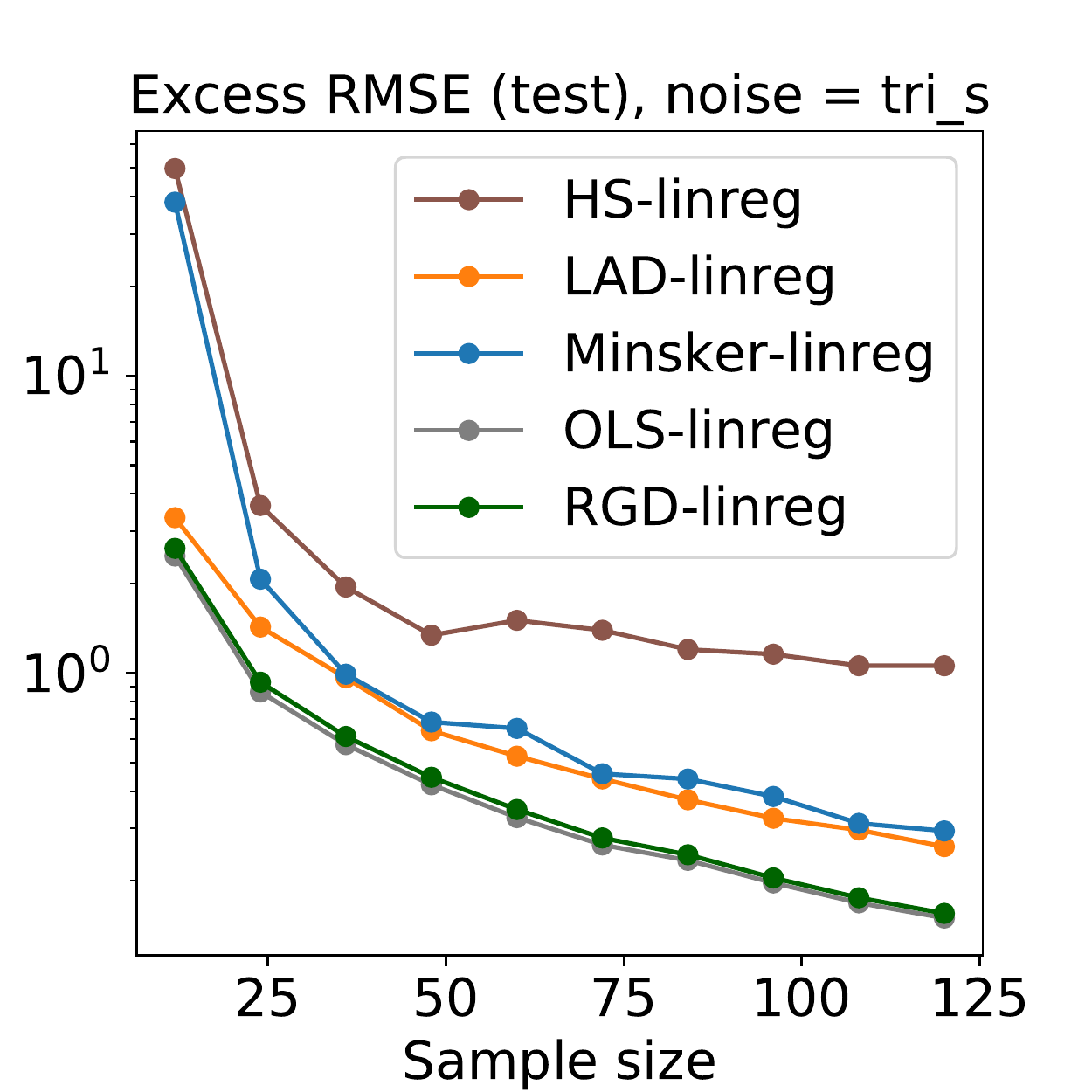}\,\includegraphics[width=0.25\textwidth]{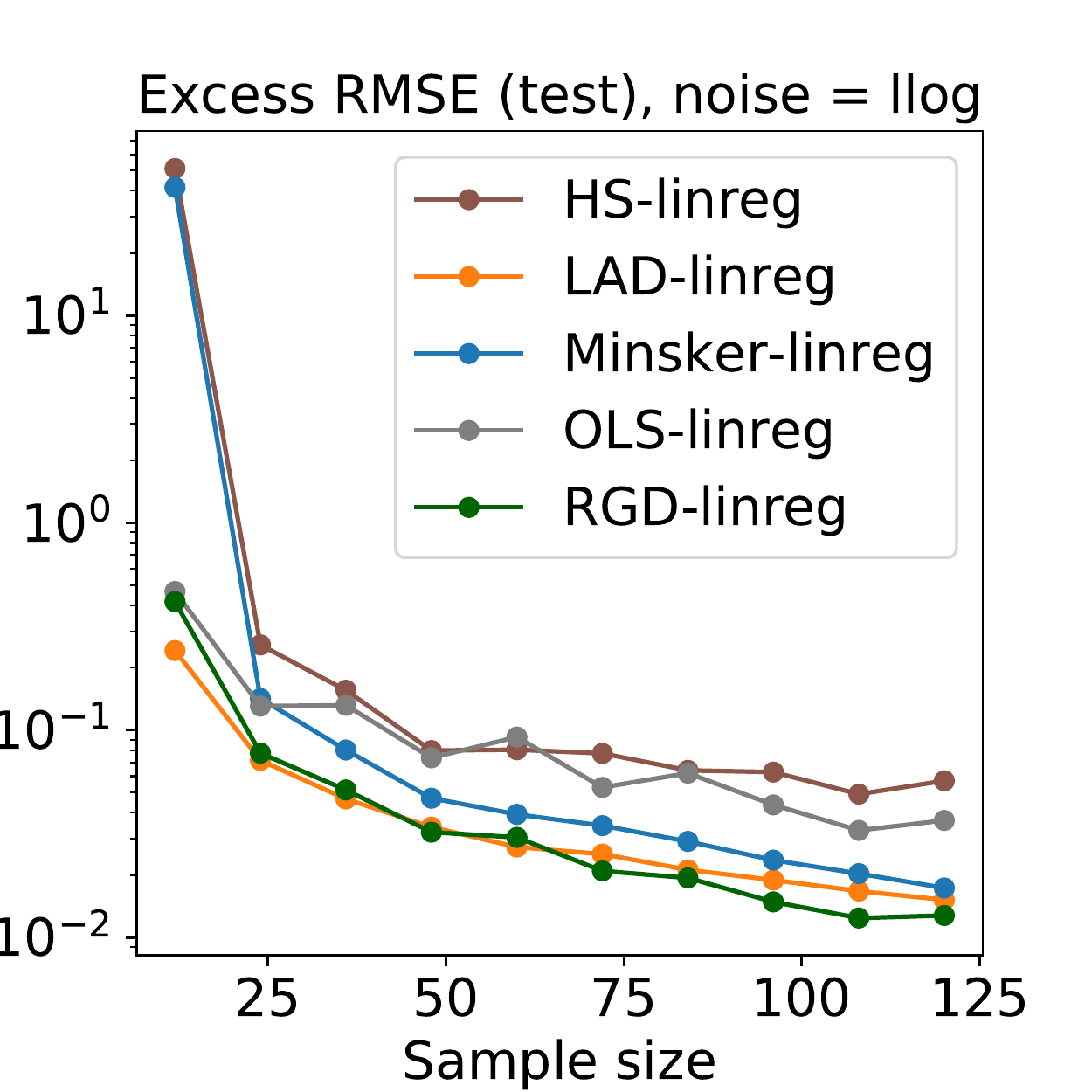}\,\includegraphics[width=0.25\textwidth]{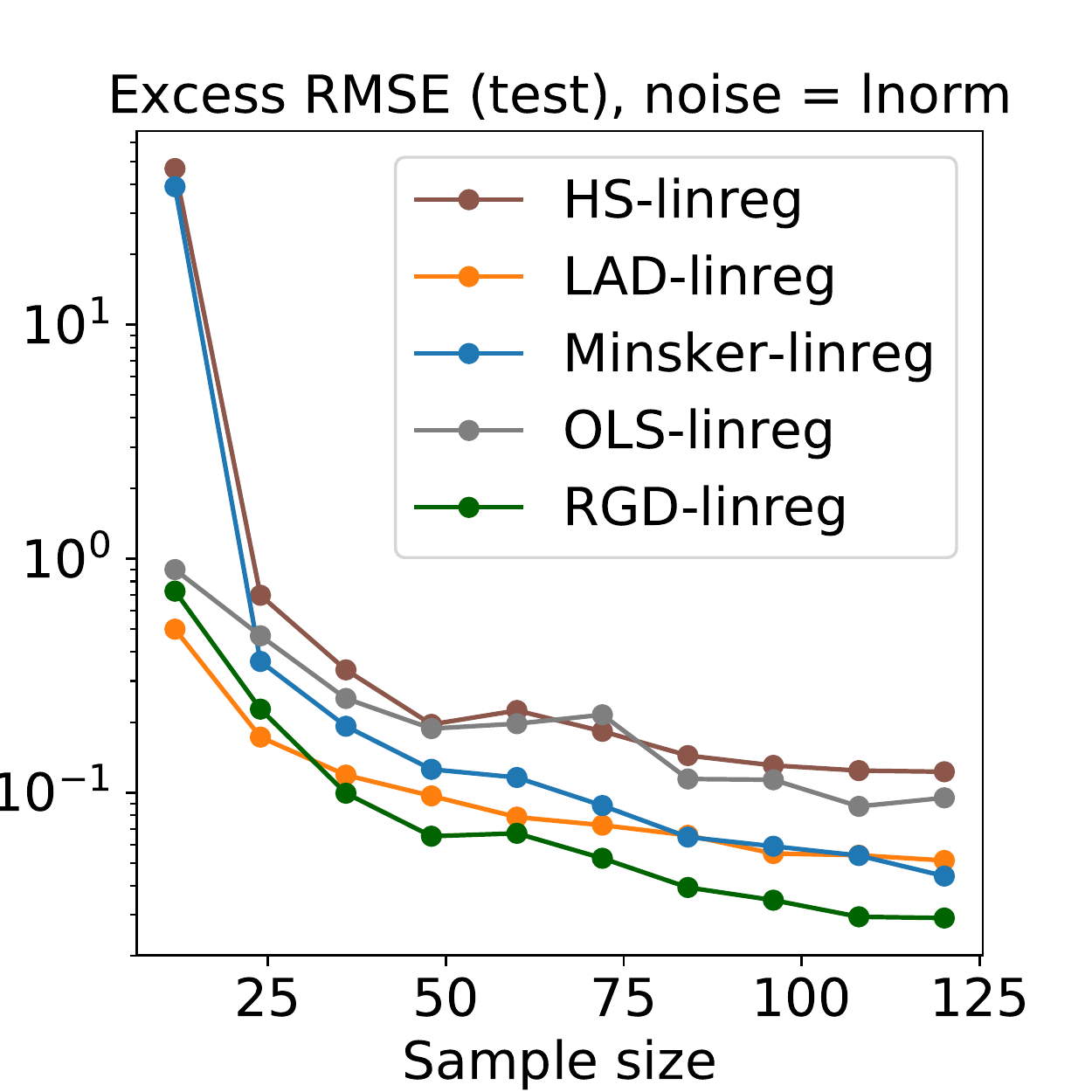}\\
\includegraphics[width=0.25\textwidth]{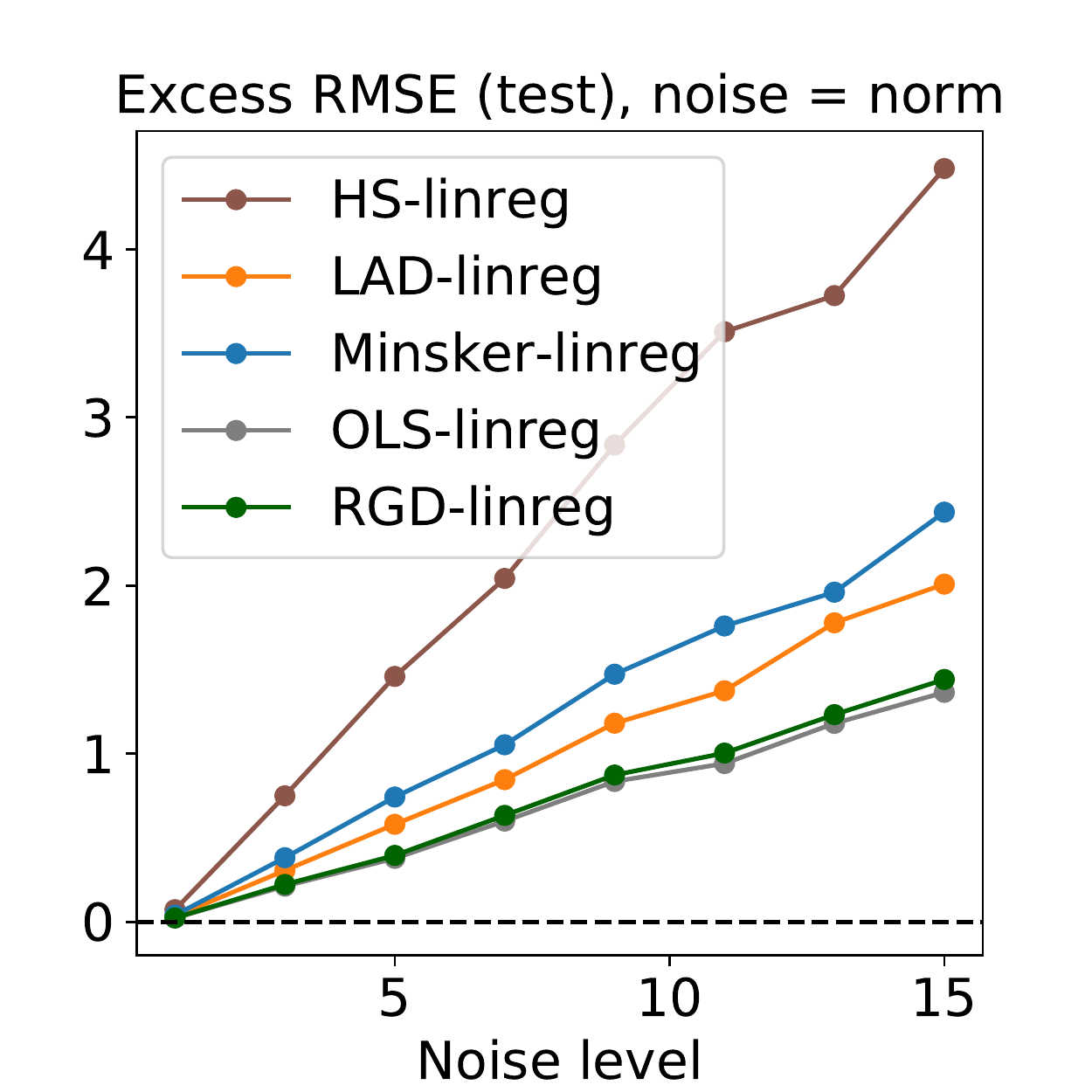}\,\includegraphics[width=0.25\textwidth]{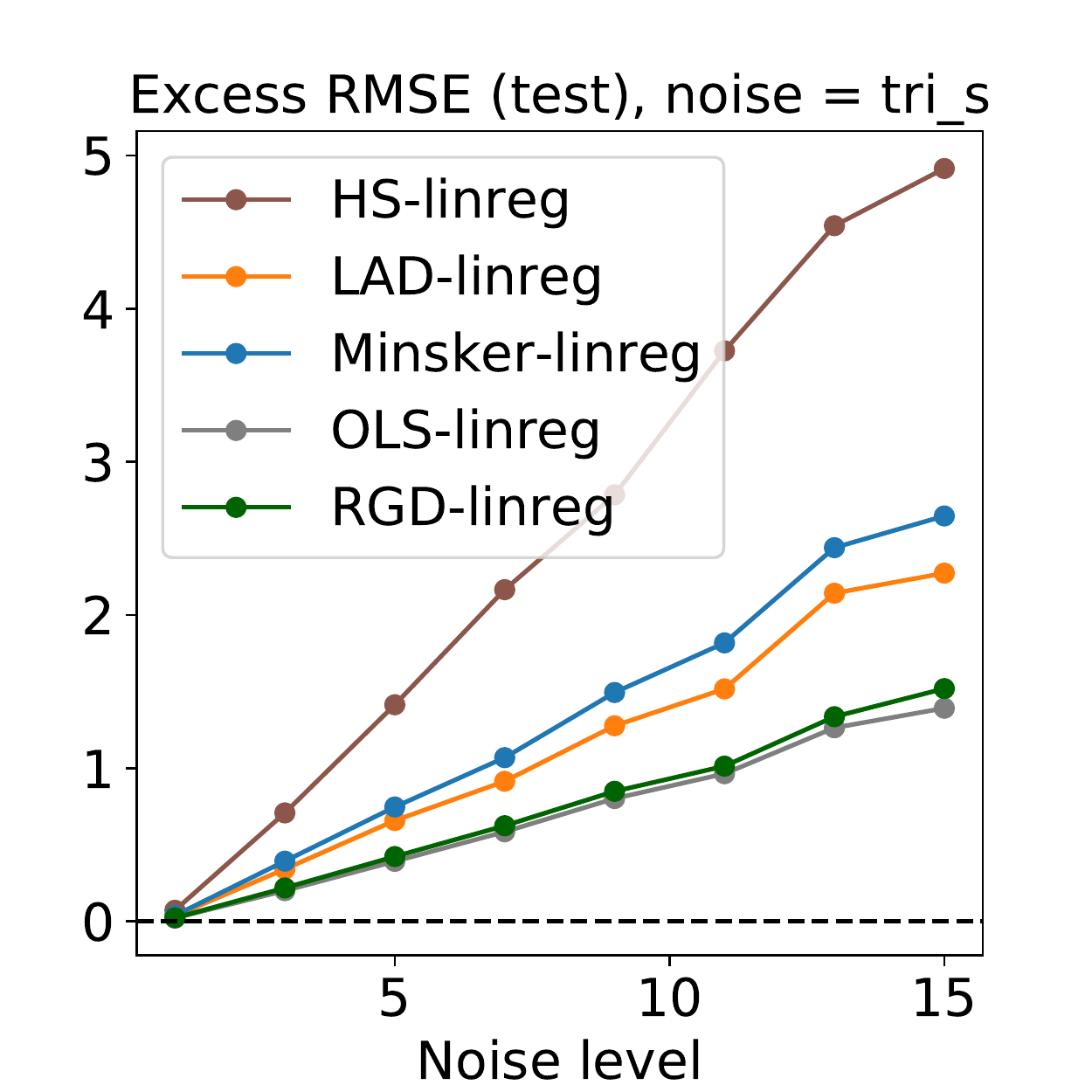}\,\includegraphics[width=0.25\textwidth]{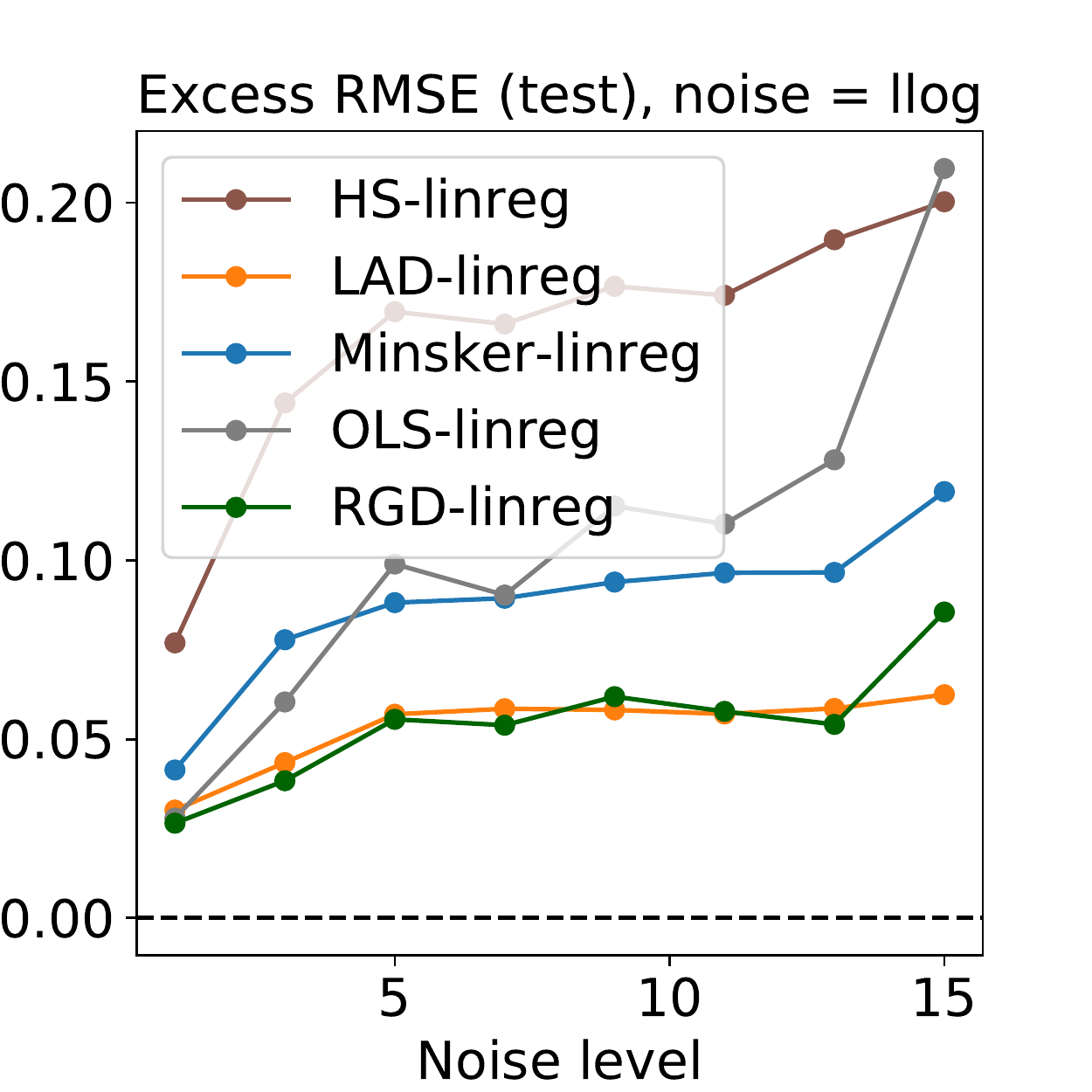}\,\includegraphics[width=0.25\textwidth]{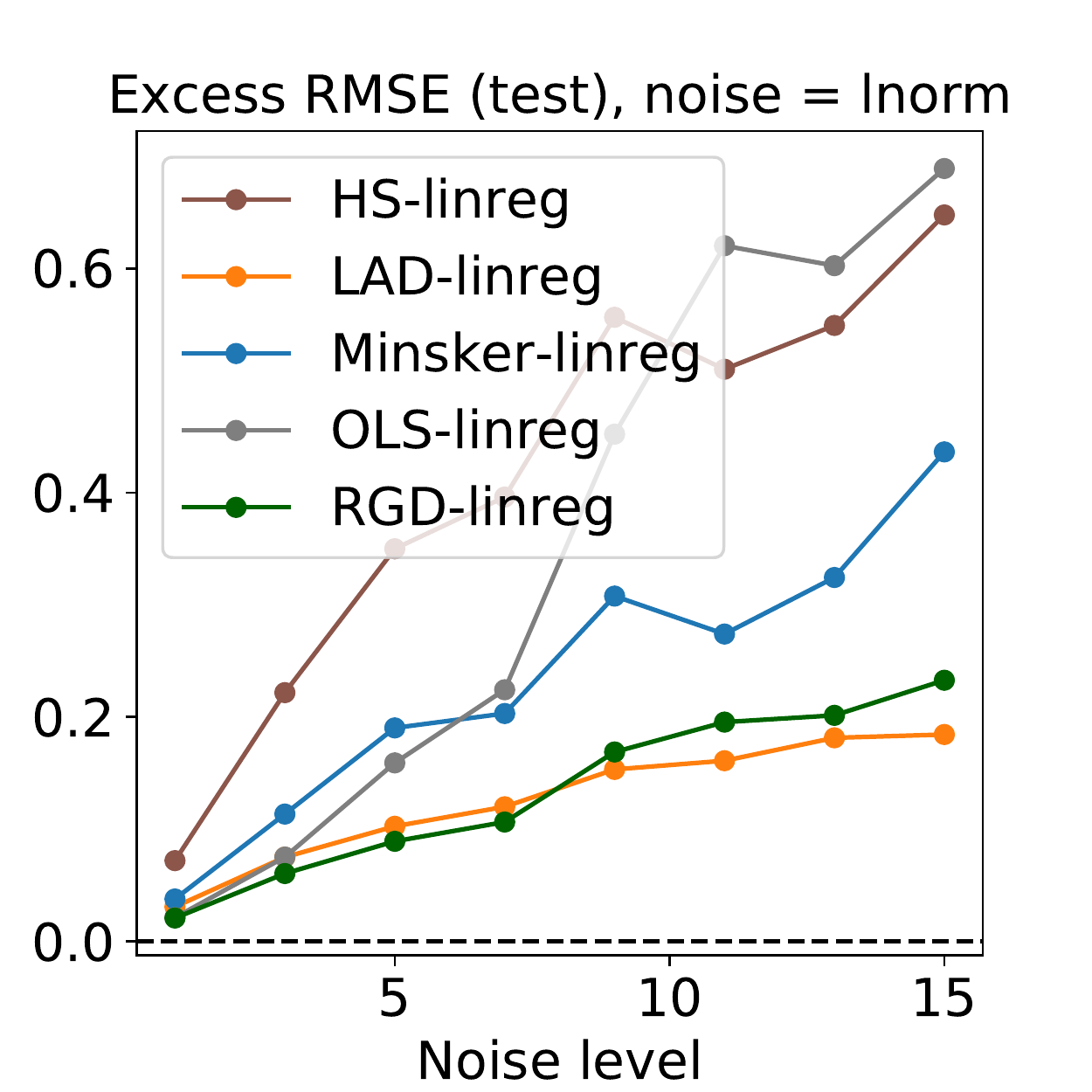}\\
\includegraphics[width=0.25\textwidth]{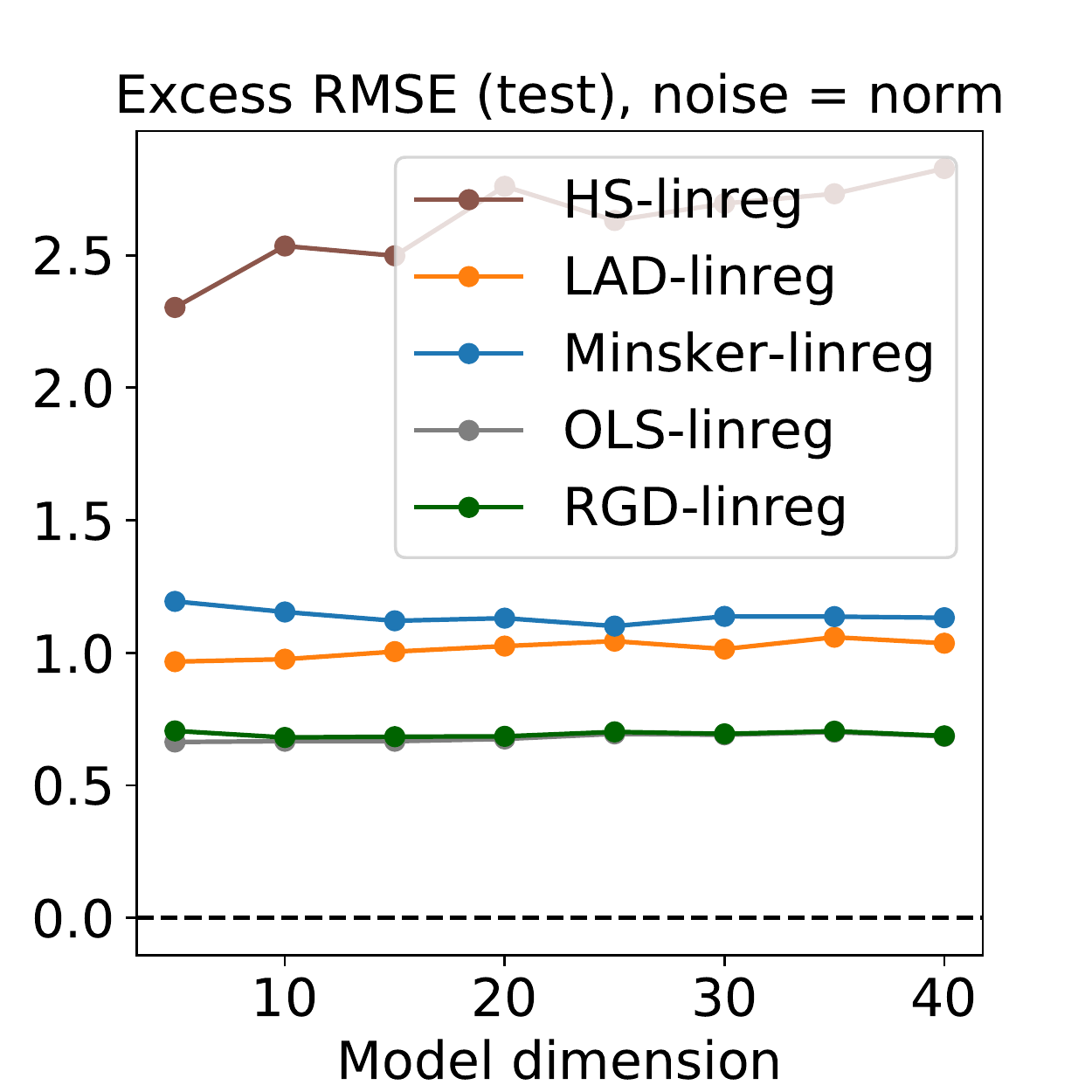}\,\includegraphics[width=0.25\textwidth]{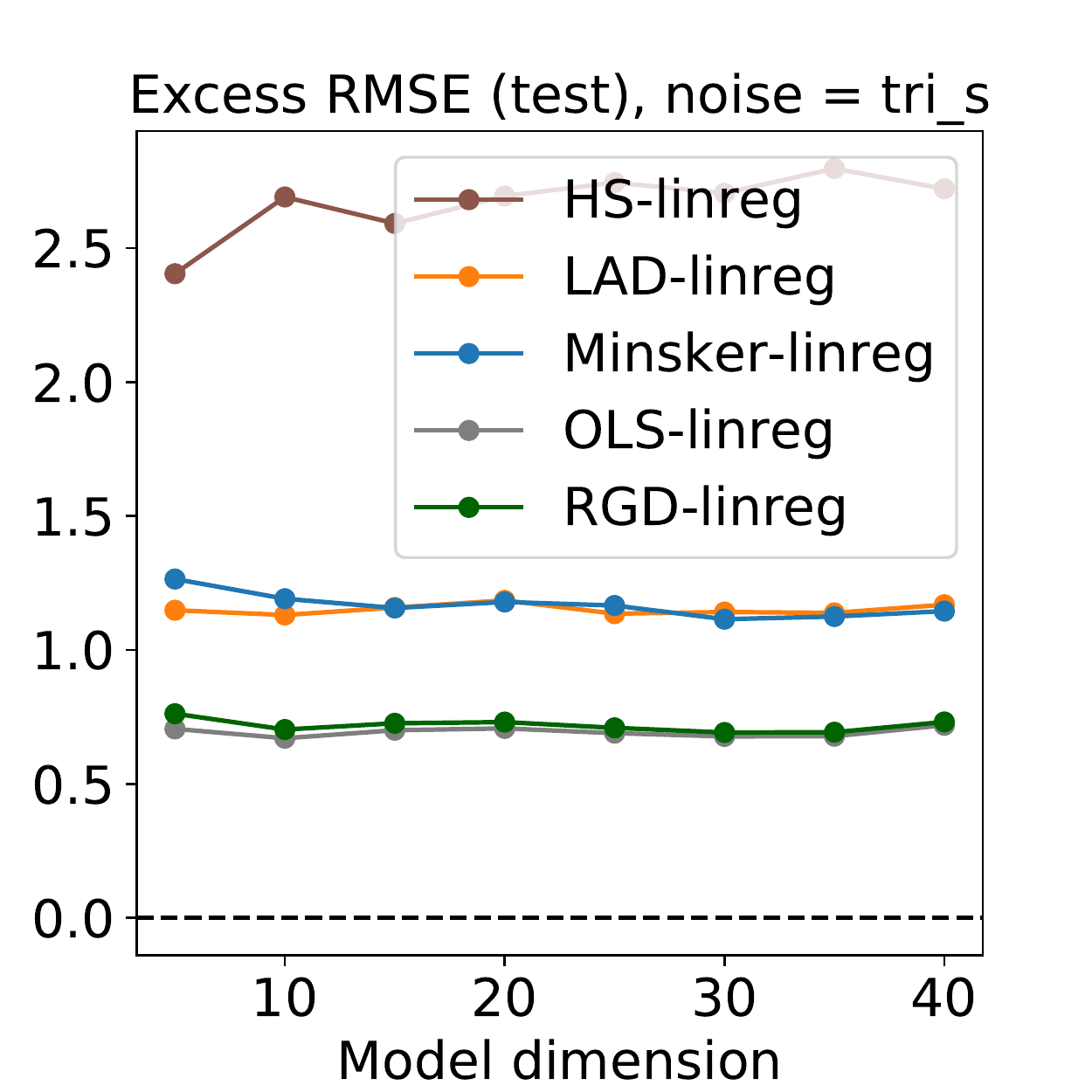}\,\includegraphics[width=0.25\textwidth]{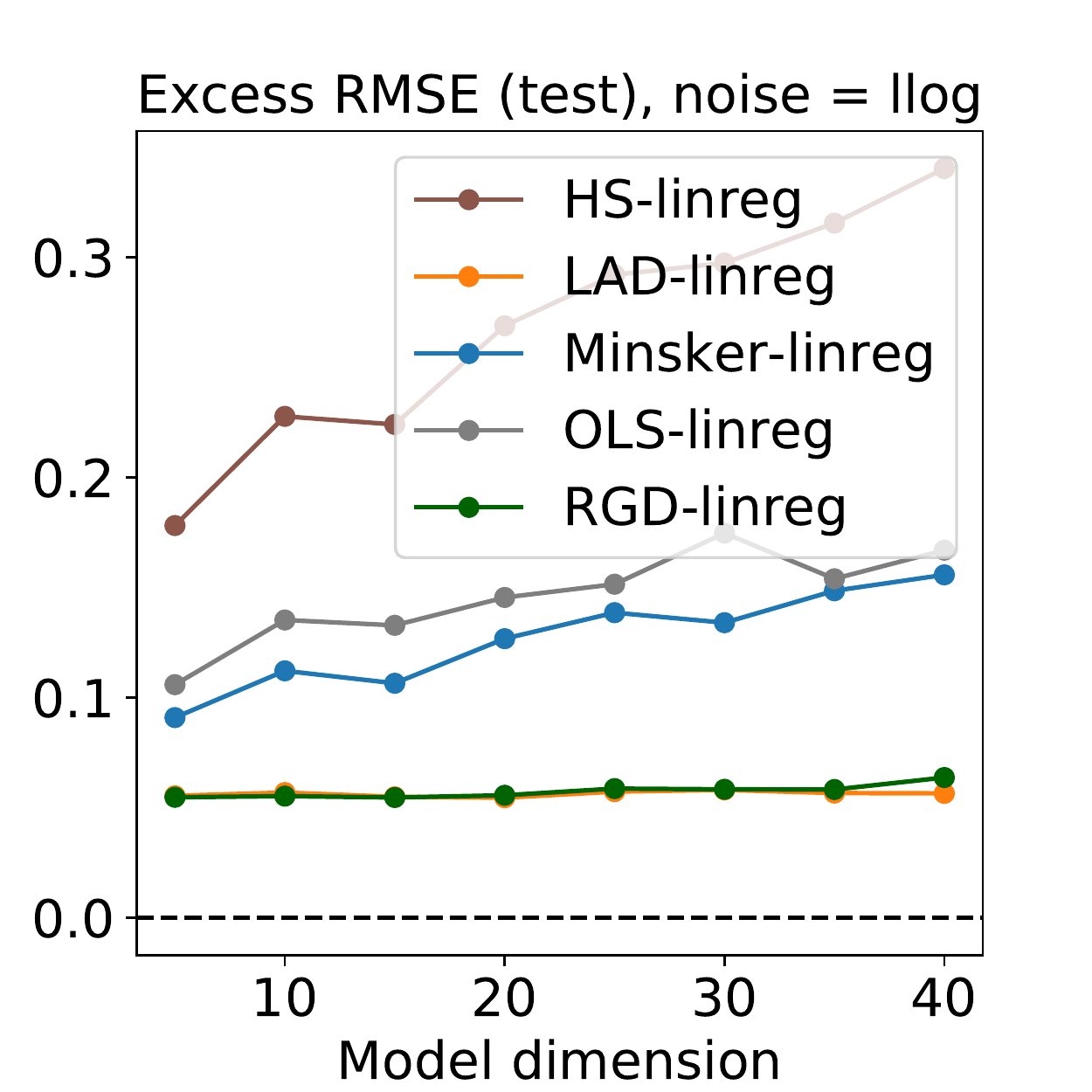}\,\includegraphics[width=0.25\textwidth]{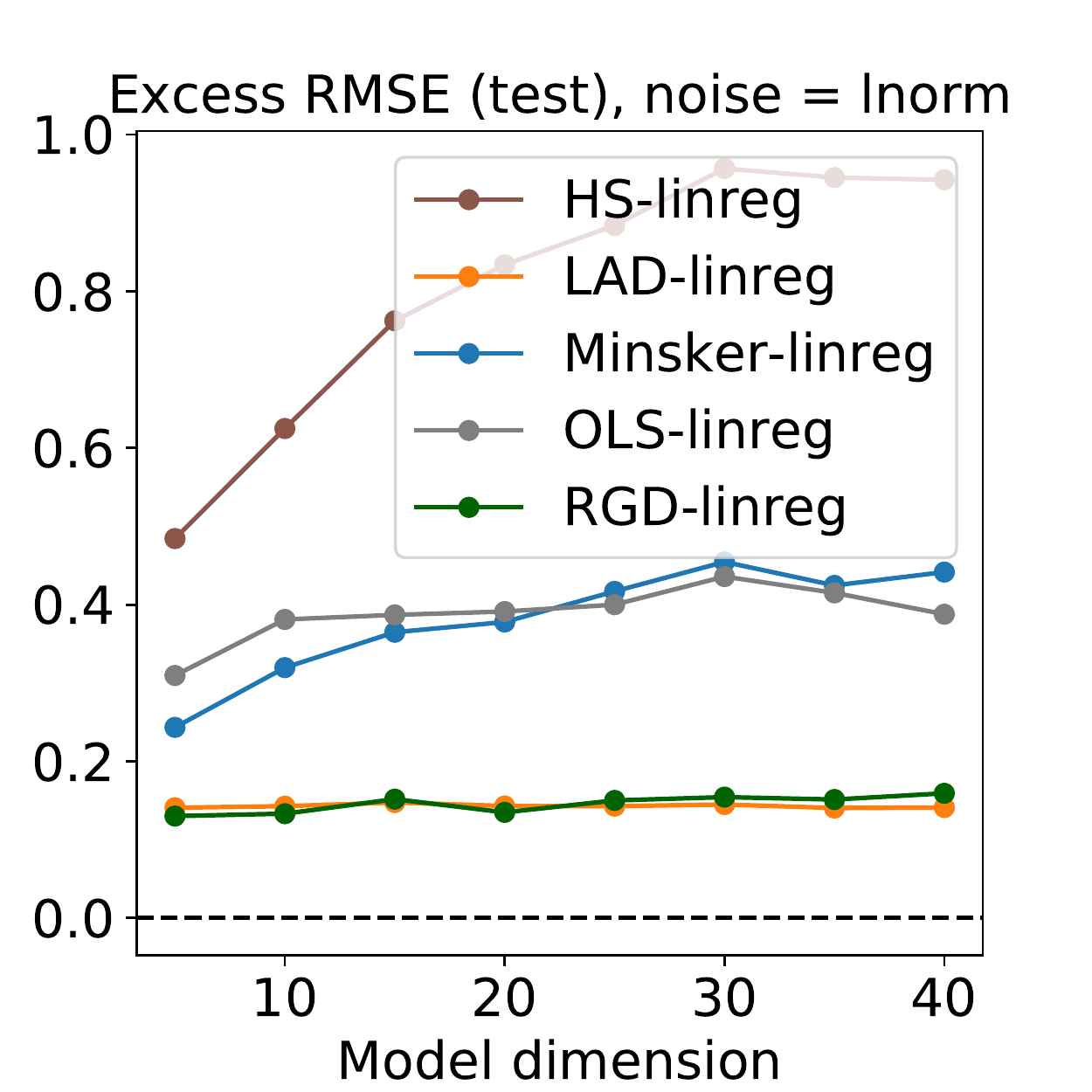}
\caption{Top row: Prediction error over sample size $12 \leq n \leq 122$, fixed $d=5$, noise level = $8$. Center row: Prediction error over noise levels, for $n=30, d=5$. Bottom row: Prediction error over dimensions $5 \leq d \leq 40$, with ratio $n/d = 6$ fixed, and noise level = $8$. Each column corresponds to a distinct noise family.}
\label{fig:multinoise_linreg}
\end{figure}

First we fix the model dimension $d$, and evaluate performance as sample size $n$ ranges from very small to quite large (top row of Figure \ref{fig:multinoise_linreg}). We see that regardless of distribution, \texttt{rgd} effectively matches the optimal convergence of OLS in the \texttt{norm} and \texttt{tri\_s} cases, and is resilient to the remaining two scenarios where \texttt{ols} breaks down. There are clear issues with the median of means based methods at very small sample sizes, though the geometric median based method does eventually at least surpass OLS in the \texttt{llog} and \texttt{lnorm} cases. Essentially the same trends can be observed at all noise levels.

Next, we look at performance over noise settings, from negligible noise to significant noise with potentially infinite higher-order moments (middle row of Figure \ref{fig:multinoise_linreg}). We see that \texttt{rgd} generalizes well, in a manner which is effectively uniform across the distinct noise families. We note that even in such diverse settings with pre-fixed step-size and iteration numbers, very robust performance is shown. It appears that under small sample size, \texttt{rgd} reduces the variance due to errant observations, while incurring a smaller bias than the other robust methods. When \texttt{ols} (effectively ERM-GD) is optimal, note that \texttt{rgd} follows it closely, with virtually negligible bias. When the former breaks down, \texttt{rgd} remains stable.

Finally, we fix the ratio $n/d$ and look at the role played by increasingly large dimension (bottom row of Figure \ref{fig:multinoise_linreg}). We see that for all distributions, the performance of \texttt{rgd} is essentially constant. This coincides with the theory of section \ref{sec:algo_performance}, and our intuition since Algorithm \ref{algo:rgd} is run in a by-coordinate fashion. On the other hand, competing methods show sensitivity to the number of free parameters, especially in the case of asymmetric data with heavy tails.

\subsection{Application to real-world benchmarks}\label{sec:tests_real}

To close out this section, and to gain some additional perspective on algorithm performance, we shift our focus to some nascent applications to real-world benchmark data sets.

Having already paid close attention to regression models in the previous section, here we consider applications of robust gradient descent to classification tasks, under both binary and multi-class settings. The model assumed is standard multi-class logistic regression: if the number of classes is $C$, and the number of input features is $F$, then the total number of parameters to be determined is $d = (C-1)F$. The loss function is convex in the parameters, and its partial derivatives all exist, so the model aligns well with our problem setting of interest. In addition, a squared $\ell_{2}$-norm regularization term $a \|\ww\|^{2}$ is added to the loss, with $a$ varying over datasets (see below). All learning algorithms are given a fixed budget of gradient computations, set here to $20n$, where $n$ is the size of the training set made available to the learner.

We use three well-known data sets for benchmarking: the CIFAR-10 data set of tiny images,\footnote{\url{http://www.cs.toronto.edu/~kriz/cifar.html}} the MNIST data set of handwritten digits,\footnote{\url{http://yann.lecun.com/exdb/mnist/}} and the protein homology dataset made popular by its inclusion in the KDD Cup.\footnote{\url{http://www.kdd.org/kdd-cup/view/kdd-cup-2004/Tasks}} For all data sets, we carry out 10 independent trials, with training and testing tests randomly sampled as will be described shortly. For all datasets, we normalize the input features to the unit interval $[0,1]$ in a per-dimension fashion. For CIFAR-10, we average the RGD color channels to obtain a single greyscale channel. As a result, $F=1024$. There are ten classes, so $C=10$, meaning $d=(C-1)F=9216$. We take a sample size of $n=4d=36864$ for training, with the rest for testing, and set $a=0.001$. For MNIST, we have $F=784$ and once again $C=10$, so $d=7056$. As with the previous dataset, we set $n=4d=28224$, and $a=0.0001$. Note that both of these datasets have all classes in equal proportions, so with uniform random sampling, class frequencies are approximately equal in each trial. On the other hand, the protein homology dataset (binary classification) has highly unbalanced labels, with only 1296 positive labels out of over 145,000 observations. We thus take random samples such that the training and test sets are balanced. For each trial, we randomly select 296 positively labeled examples, and the same amount of negatively labeled examples, yielding a test set of 592 examples. As for the training set size, we use the rest of the positive labels (1000 examples) plus a random selection of 1000 negatively labeled examples, so $n=2000$, and with $C=2$ and $F=74$, we have $d=74$. Regularization parameter $a$ is $0.001$. For all datasets, the parameter weights are initialized uniformly over the interval $[-0.05,0.05]$.

Regarding the competing methods used, we test out a random mini-batch version of robust gradient descent given in Algorithm \ref{algo:rgd}, with mini-batch sizes ranging over $\{5,10,15,20\}$, roughly on the order of $n^{-1/4}$ for the largest datasets. We also consider a mini-batch in the sense of randomly selecting coordinates to robustify: select $\min(100,d)$ indices randomly at each iteration, and run the RGD sub-routine on just these coordinates, using the sample mean for all the rest. Furthermore, we considered several minor alterations to the original routine, including using $\log\cosh(\cdot)$ instead of the Gudermannian function for $\rho$, updating the scale much less frequently (compared to every iteration), and different choices of $\chi$ for re-scaling. We compare our proposed algorithm with stochastic gradient descent (SGD), and stochastic variance-reduced gradient descent (SVRG) proposed by \citet{johnson2013a}. For each method, pre-fixed step sizes ranging over $\{0.0001, 0.001, 0.01, 0.05, 0.10, 0.15, 0.20\}$ are tested. SGD uses mini-batches of size 1, as does the inner loop of SVRG. The outer loop of SVRG continues until the budget is spent, with the inner loop repeating $n/2$ times.

\begin{figure}[t]
\centering
\includegraphics[width=0.33\textwidth]{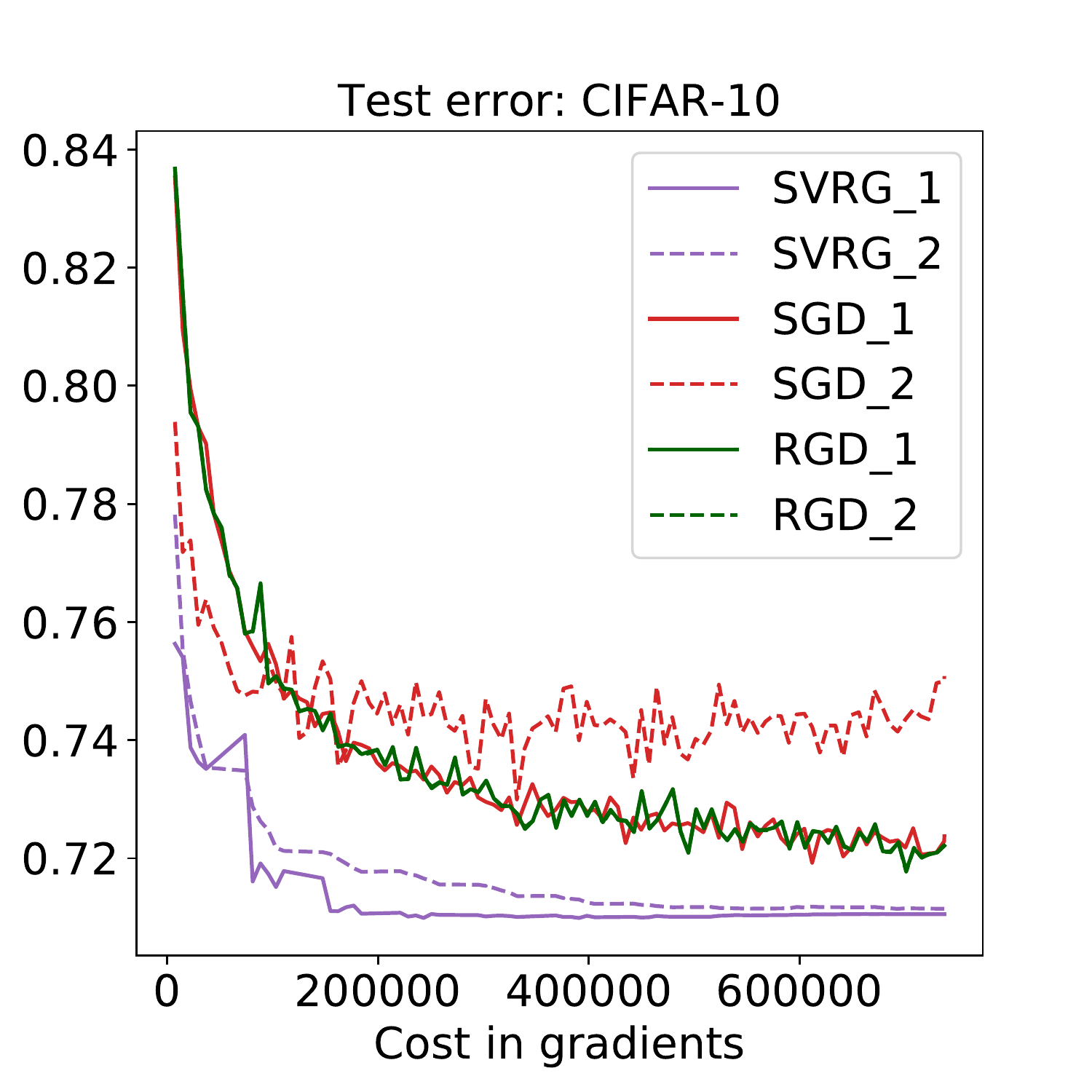}\,\includegraphics[width=0.33\textwidth]{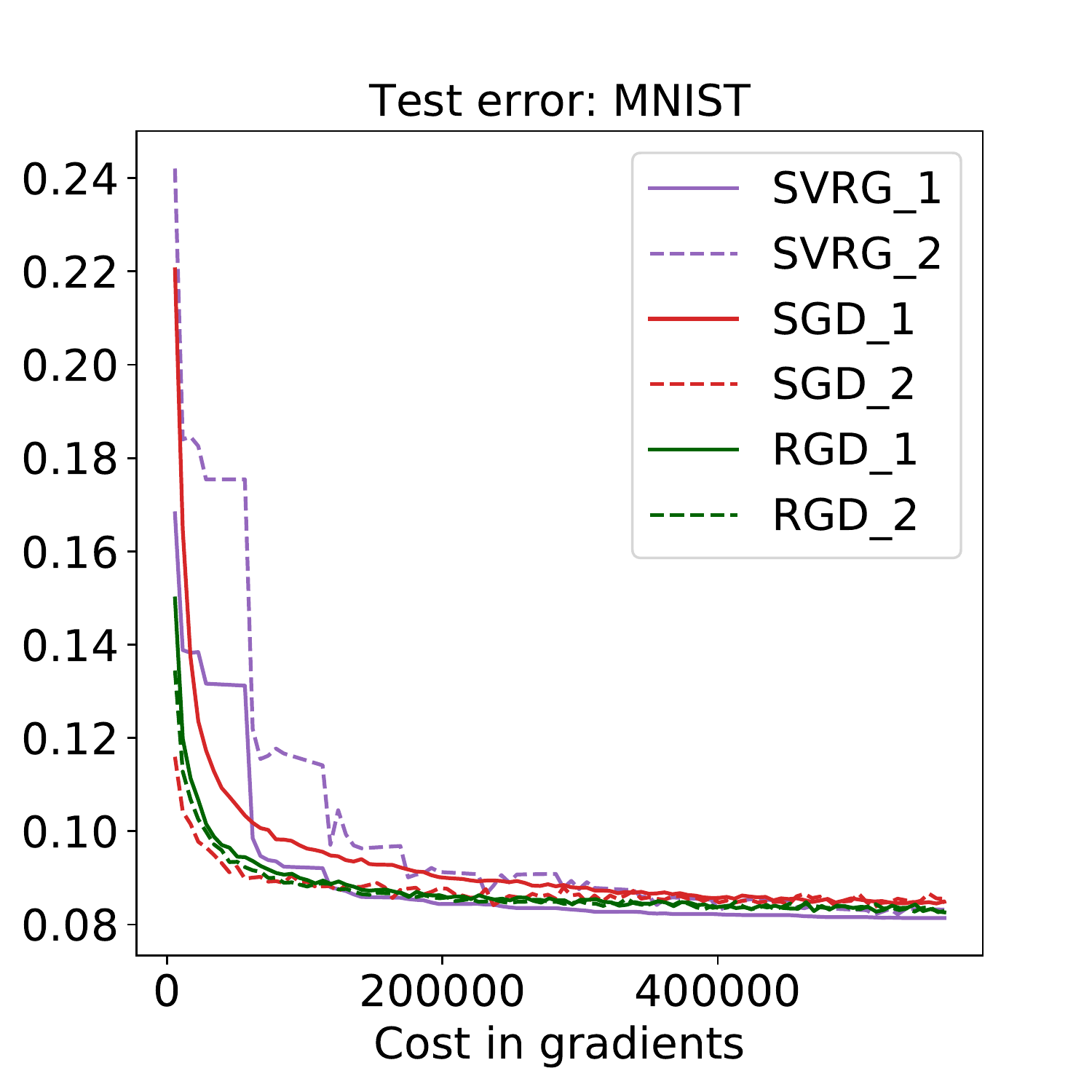}\,\includegraphics[width=0.33\textwidth]{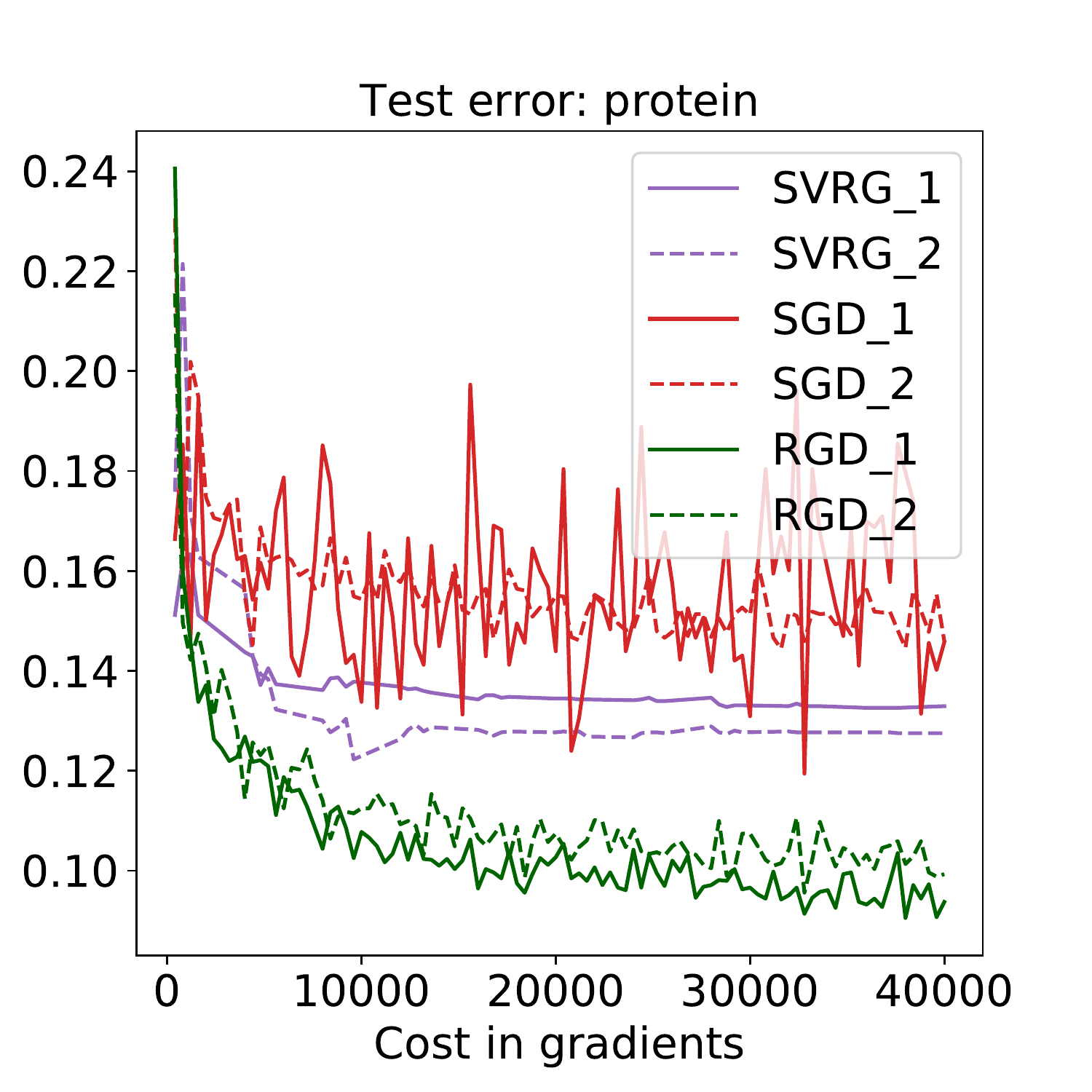}
\caption{Test error (misclassification rate) over budget spent, as measured by gradient computations, for the top two performers within each method class. Each plot corresponds to a distinct dataset.}
\label{fig:tests_real}
\end{figure}

Representative results are given in Figure \ref{fig:tests_real}. For each of the three methods of interest, and each dataset, we chose the top two performance settings, displayed as \texttt{*\_1} and \texttt{*\_2} respectively. Here ``top performance'' is measured by the median value of the last five iterations. We see that in general, robust gradient descent is competitive with the best settings of these well-known routines, has minimal divergence between the performance of its first- and second-best settings, and in the case of smaller data sets (protein homology), indeed significantly outperforms the competitors. While these are simply nascent applications of RGD, the strong initial performance suggests that further investigation of efficient strategies under high-dimensional data is a promising direction.

\section{Concluding remarks}\label{sec:conclusion}

In this work, we introduced and analyzed a learning algorithm which takes advantage of robust estimates of the unknown risk gradient, integrating statistical estimation and practical implementation into a single routine. Doing so allows us to deal with the statistical vulnerabilities of ERM-GD and partition-based methods, while circumventing computational issues posed by minimizers of robust surrogate objectives. The price to be paid is new computational overhead and potentially biased estimates. Is this price worth paying? Bounds on the excess risk are available under very weak assumptions on the data distribution, and we find empirically that the proposed algorithm has desirable learning efficiency, in that it can competitively generalize, with less samples, over more distributions than its competitors.

Moving forward, a more careful analysis of the role that prior knowledge can play on learning efficiency, starting with the first-order optimizer setting, is of significant interest. Characterizing the learning efficiency enabled by sharper estimates could lead to useful insights in the context of larger-scale problems, where a small overhead might save countless iterations and dramatically reduce budget requirements, while simultaneously leading to more consistent performance across samples. Another natural line of work is to look at alternative strategies which operate on the data vector as a whole (rather than coordinate-wise), integrating information across coordinates, in order to infer more efficiently.

{\small
\bibliographystyle{apalike}
\bibliography{refs_rgd.bib}
}

\appendix

\section{Technical appendix}\label{sec:appendix}

\subsection{Preliminaries}\label{sec:prelims}

Our generic data shall be denoted by $\zz \in \ZZ$. Let $\mu$ denote a probability measure on $\ZZ$, equipped with an appropriate $\sigma$-field. Data samples shall be assumed independent and identically distributed (iid), written $\zz_{1},\ldots,\zz_{n}$. We shall work with loss function $l:\RR^{d} \times \ZZ \to \RR_{+}$ throughout, with $l(\cdot;\zz)$ assumed differentiable for each $\zz \in \ZZ$. Write $\prr$ for a generic probability measure, most commonly the product measure induced by the sample. Let $f:\ZZ \to \RR$ be an measurable function. Expectation is written $\exx_{\mu}f(\zz) \defeq \int f \, d\mu$, with variance $\vaa_{\mu}f(\zz)$ defined analogously. For $d$-dimensional Euclidean space $\RR^{d}$, the usual ($\ell_{2}$) norm shall be denoted $\|\cdot\|$ unless otherwise specified. For function $F$ on $\RR^{d}$ with partial derivatives defined, write the gradient as $F^{\prime}(\uu) \defeq (F^{\prime}_{1}(\uu),\ldots,F^{\prime}_{d}(\uu))$ where for short, we write $F^{\prime}_{j}(\uu) \defeq \partial F(\uu)/\partial u_{j}$. For integer $k$, write $[k] \defeq \{1,\ldots,k\}$ for all the positive integers from $1$ to $k$. Risk shall be denoted $R(\ww) \defeq \exx_{\mu}l(\ww;\zz)$, and its gradient $\gvec(\ww) \defeq R^{\prime}(\ww)$. We make a running assumption that we can differentiate under the integral sign in each coordinate \citep{ash2000a,talvila2001a}, namely that
\begin{align}\label{eqn:diff_under_integral}
\gvec(\ww) = \left(\exx_{\mu}\frac{\partial l(\ww;\zz)}{\partial w_{1}}, \ldots, \exx_{\mu}\frac{\partial l(\ww;\zz)}{\partial w_{d}}\right).
\end{align}

Smoothness and convexity of functions shall also be utilized. For convex function $F$ on convex set $\WW$, say that $F$ is $\lambda$\textit{-Lipschitz} if, for all $\ww_{1},\ww_{2} \in \WW$ we have $|F(\ww_{1})-F(\ww_{2})| \leq \lambda \|\ww_{1}-\ww_{2}\|$. We say that $F$ is $\lambda$\textit{-smooth} if $F^{\prime}$ is $\lambda$-Lipschitz. Finally, $F$ is \textit{strongly convex} with parameter $\kappa > 0$ if for all $\ww_{1},\ww_{2} \in \WW$,
\begin{align*}
F(\ww_{1})-F(\ww_{2}) \geq \langle F^{\prime}(\ww_{2}), \ww_{1}-\ww_{2} \rangle + \frac{\kappa}{2}\|\ww_{1}-\ww_{2}\|^{2}
\end{align*}
for any norm $\|\cdot\|$ on $\WW$, though we shall be assuming $\WW \subseteq \RR^{d}$. If there exists $\wwstar \in \WW$ such that $F^{\prime}(\wwstar)=0$, then it follows that $\wwstar$ is the unique minimum of $F$ on $\WW$. Let $f:\RR^{d} \to \RR$ be a continuously differentiable, convex, $\lambda$-smooth function. The following basic facts will be useful: for any choice of $\uu,\vv \in \RR^{d}$, we have
\begin{align}
\label{eqn:facts_CO_1}
f(\uu)-f(\vv) & \leq \frac{\lambda}{2}\|\uu-\vv\|^{2} + \langle f^{\prime}(\vv), \uu-\vv \rangle\\
\label{eqn:facts_CO_2}
\frac{1}{2\lambda}\|f^{\prime}(\uu)-f^{\prime}(\vv)\|^{2} & \leq f(\uu)-f(\vv) - \langle f^{\prime}(\vv), \uu-\vv \rangle.
\end{align}
Proofs of these results can be found in any standard text on convex optimization, e.g.~\citep{nesterov2004ConvOpt}.

We shall leverage a special type of M-estimator here, built using the following convenient class of functions.
\begin{defn}[Function class for location estimates]\label{defn:rho}
Let $\rho:\RR \to [0,\infty)$ be an even function ($\rho(u)=\rho(-u)$) with $\rho(0)=0$ and the following properties. Denote $\psi(u) \defeq \rho^{\prime}(u)$.
\begin{enumerate}
\item $\rho(u) = O(u)$ as $u \to \pm\infty$.
\item $\rho(u)/(u^{2}/2) \to 1$ as $u \to 0$.
\item $\psi^{\prime} > 0$, and for some $C>0$, and all $u \in \RR$,
\begin{align*}
-\log(1-u+Cu^{2}) \leq \psi(u) \leq \log(1+u+Cu^{2}).
\end{align*}
\end{enumerate}
\end{defn}
\noindent Of particular importance in the proceeding analysis is the fact that $\psi=\rho^{\prime}$ is bounded, monotonically increasing and Lipschitz on $\RR$, plus the upper/lower bounds which let us generalize the technique of \citet{catoni2012a}.
\begin{ex}[Valid $\rho$ choices]
In addition to the Gudermannian function (section \ref{sec:intuitive} footnote), functions such as $2(\sqrt{1+u^{2}/2}-1)$ and $\log\cosh(u)$ are well-known examples that satisfy the desired criteria. Note that the wide/narrow functions of Catoni do not meet all these criteria, nor does the classic Huber function.
\end{ex}

\subsection{Proofs}\label{sec:appendix_proofs}

\begin{proof}[Proof of Lemma \ref{lem:sharp_Mest}]
For cleaner notation, write $x_{1},\ldots,x_{n} \in \RR$ for our iid observations. Here $\rho$ is assumed to satisfy the conditions of Definition~\ref{defn:rho}. A high-probability concentration inequality follows by direct application of the specified properties of $\rho$ and $\psi \defeq \rho^{\prime}$, following the general technique laid out by \citet{catoni2009a,catoni2012a}. For $u \in \RR$ and $s>0$, writing $\psi_{s}(u) \defeq \psi(u/s)$, and taking expectation over the random draw of the sample,
\begin{align*}
\exx\exp\left( \sum_{i=1}^{n}\psi_{s}(x_{i}-u)\right) & \leq \left(1 + \frac{1}{s}(\exx x-u) + \frac{C}{s^{2}}\exx (x^{2}+u^{2}-2xu) \right)^{n}\\
& \leq \exp\left( \frac{n}{s}(\exx x-u) + \frac{Cn}{s^{2}}(\vaa x + (\exx x - u)^{2}) \right).
\end{align*}
The inequalities above are due to an application of the upper bound on $\psi$, and and the inequality $(1+u) \leq \exp(u)$. Now, letting
\begin{align*}
A & \defeq \frac{1}{n}\sum_{i=1}^{n}\psi_{s}(x_{i}-u)\\
B & \defeq \frac{1}{s}(\exx x-u) + \frac{C}{s^{2}}(\vaa x + (\exx x - u)^{2})
\end{align*}
we have a bound on $\exx\exp(nA) \leq \exp(nB)$. By Chebyshev's inequality, we then have
\begin{align*}
\prr\{A > B + \varepsilon\} & = \prr\{\exp(nA) > \exp(nB+n\varepsilon)\}\\
& \leq \frac{\exx\exp(nA)}{\exp(nB+n\varepsilon)}\\
& \leq \exp(-n\varepsilon).
\end{align*}
Setting $\varepsilon=\log(\delta^{-1})/n$ for confidence level $\delta \in (0,1)$, and for convenience writing
\begin{align*}
b(u) \defeq \exx x - u + \frac{C}{s}(\vaa x + (\exx x-u)^{2}),
\end{align*}
we have with probability no less than $1-\delta$ that
\begin{align}\label{eqn:sharp_Mest_1}
\frac{s}{n}\sum_{i=1}^{n} \psi_{s}(x_{i}-u) \leq b(u) + \frac{s\log(\delta^{-1})}{n}.
\end{align}
The right hand side of this inequality, as a function of $u$, is a polynomial of order 2, and if
\begin{align*}
1 \geq D \defeq 4\left(\frac{C^2\vaa x}{s^2} + \frac{C\log(\delta^{-1})}{n}\right),
\end{align*}
then this polynomial has two real solutions. In the hypothesis, we stated that the result holds ``for large enough $n$ and $s_{j}$.'' By this we mean that we require $n$ and $s$ to satisfy the preceding inequality (for each $j \in [d]$ in the multi-dimensional case). The notation $D$ is for notational simplicity. The solutions take the form
\begin{align*}
u = \frac{1}{2}\left(2\exx x + \frac{s}{C} \pm \frac{s}{C}\left(1-D\right)^{1/2}\right).
\end{align*}
Looking at the smallest of the solutions, noting $D \in [0,1]$ this can be simplified as
\begin{align}\label{eqn:sharp_Mest_2}
u_{+} & \defeq \exx x + \frac{s}{2C}\frac{(1-\sqrt{1-D})(1+\sqrt{1-D})}{1+\sqrt{1-D}}\nonumber\\
& = \exx x + \frac{s}{2C}\frac{D}{1+\sqrt{1-D}}\nonumber\\
& \leq \exx x + sD/2C,
\end{align}
where the last inequality is via taking the $\sqrt{1-D}$ term in the previous denominator as small as possible. Now, writing $\xhat$ as the M-estimate using $s$ and $\rho$ as in (\ref{eqn:location_rough}), note that $\xhat$ equivalently satisfies $\sum_{i=1}^{n}\psi_{s}(\xhat-x_{i})=0$. Using (\ref{eqn:sharp_Mest_1}), we have
\begin{align*}
\frac{s}{n}\sum_{i=1}^{n}\psi_{s}(x_{i}-u_{+}) \leq b(u_{+}) + \frac{s\log(\delta^{-1})}{n} = 0,
\end{align*}
and since the left-hand side of (\ref{eqn:sharp_Mest_1}) is a monotonically decreasing function of $u$, we have immediately that $\xhat \leq u_{+}$ on the event that (\ref{eqn:sharp_Mest_1}) holds, which has probability at least $1-\delta$. Then leveraging (\ref{eqn:sharp_Mest_2}), it follows that on the same event,
\begin{align*}
\xhat-\exx x \leq sD/2C.
\end{align*}
An analogous argument provides a $1-\delta$ event on which $\xhat-\exx x \geq -sD/2C$, and thus using a union bound, one has that
\begin{align}\label{eqn:sharp_Mest_3}
|\xhat-\exx x| \leq 2\left( \frac{C \vaa x}{s} + \frac{s\log(\delta^{-1})}{n} \right)
\end{align}
holds with probability no less than $1-2\delta$. Setting the $x_{i}$ to $l_{j}^{\prime}(\ww;\zz_{i})$ for $j \in [d]$ and some $\ww \in \RR^{d}$, $i \in [n]$, and $\xhat$ to $\that_{j}$ corresponds to the special case considered in this Lemma. Dividing $\delta$ by two yields the $(1-\delta)$ result.
\end{proof}

\begin{proof}[Proof of Lemma \ref{lem:grad_estimate}]
For any fixed $\ww$ and $j \in [d]$, note that
\begin{align}\label{eqn:grad_estimate_1}
\nonumber
|\that_{j} - g_{j}(\ww)| & \leq \varepsilon_{j}\\
&  \defeq 2\left( \frac{C\vaa_{\mu}l_{j}^{\prime}(\ww;\zz)}{s_{j}} + s_{j}\log(2\delta^{-1}) \right)\\
\nonumber
& = 2 \sqrt{\frac{\log(2\delta^{-1})}{n}} \left( \frac{C\vaa_{\mu}l_{j}^{\prime}(\ww;\zz)}{\widehat{\sigma}_{j}} + \widehat{\sigma}_{j} \right)\\
& \leq \varepsilon^{\ast} \defeq 2 \sqrt{\frac{V\log(2\delta^{-1})}{n}} c_{0} %\left( \frac{C}{c_{min}} + c_{max} \right).
\end{align}
holds with probability no less than $1-\delta$. The first inequality holds via direct application of Lemma \ref{lem:sharp_Mest}, which holds under (\ref{eqn:req_n_general}) and using $\rho$ which satisfies (\ref{eqn:rho_Catoni_condition}). The equality follows immediately from (\ref{eqn:scale_rough}). The final inequality follows from \ref{asmp:A4} and (\ref{eqn:req_dispersion}), along with the definition of $c_{0}$.

Making the dependence on $\ww$ explicit with $\that_{j} = \that_{j}(\ww)$, an important question to ask is how sensitive this estimator is to a change in $\ww$. Say we perturb $\ww$ to $\wwtil$, so that $\|\ww - \wwtil\| = a > 0$. By \ref{asmp:A2}, for any sample we have
\begin{align*}
\|l^{\prime}(\ww;\zz_{i}) - l^{\prime}(\wwtil;\zz_{i})\| \leq \lambda \|\ww - \wwtil\| = \lambda a, \quad i \in [n]
\end{align*}
which immediately implies $|l^{\prime}_{j}(\ww;\zz_{i}) - l^{\prime}_{j}(\wwtil;\zz_{i})| \leq \lambda a$ for all $j \in [d]$ as well. Given a sample of $n \geq 1$ points, the most extreme shift in $\that_{j}(\cdot)$ that is feasible would be if, given the $a$-sized shift from $\ww$ to $\wwtil$, \textit{all} data points moved the maximum amount (namely $\lambda a$) in the same direction. Since $\that_{j}(\wwtil)$ is defined by balancing the distance between points to its left and right, the most it could conceivably shift is thus equal to $\lambda a$. That is, smoothness of the loss function immediately implies a Lipschitz property of the estimator,
\begin{align*}
|\that_{j}(\ww)-\that_{j}(\wwtil)| \leq \lambda \|\ww - \wwtil\|.
\end{align*}
Considering the vector of estimates $\widehat{\mv{\theta}}(\ww) \defeq (\that_{1}(\ww),\ldots,\that_{d}(\ww))$, we then have
\begin{align}\label{eqn:estimator_smoothness}
\|\widehat{\mv{\theta}}(\ww)-\widehat{\mv{\theta}}(\wwtil)\| \leq \sqrt{d}\lambda\|\ww - \wwtil\|.
\end{align}
This will be useful for proving uniform bounds on the estimation error shortly.

First, let's use these one-dimensional results for statements about the vector estimator of interest. In $d$ dimensions, using $\widehat{\mv{\theta}}(\ww)$ just defined for any pre-fixed $\ww$, then for any $\varepsilon > 0$ we have
\begin{align*}
\prr\left\{ \|\widehat{\mv{\theta}}(\ww)-\gvec(\ww)\| > \varepsilon \right\} & = \prr\left\{ \|\widehat{\mv{\theta}}(\ww)-\gvec(\ww)\|^{2} > \varepsilon^{2} \right\}\\
& \leq \sum_{j=1}^{d} \prr\left\{ |\that_{j}(\ww) - \gvec_{j}(\ww)| > \frac{\varepsilon}{\sqrt{d}} \right\}.
\end{align*}
Using the notation of $\varepsilon_{j}$ and $\varepsilon^{\ast}$ from (\ref{eqn:grad_estimate_1}), filling in $\varepsilon = \sqrt{d}\varepsilon^{\ast}$, we thus have
\begin{align*}
\prr\left\{ \|\widehat{\mv{\theta}}(\ww)-\gvec(\ww)\| > \sqrt{d}\varepsilon^{\ast} \right\} & \leq \sum_{j=1}^{d}\prr\left\{ |\that_{j}(\ww) - g_{j}(\ww)| > \varepsilon^{\ast} \right\}\\
& \leq \sum_{j=1}^{d}\prr\left\{ |\that_{j}(\ww) - g_{j}(\ww)| > \varepsilon_{j} \right\}\\
& \leq d\delta.
\end{align*}
The second inequality is because $\varepsilon_{j} \leq \varepsilon^{\ast}$ for all $j \in [d]$. It follows that the event
\begin{align*}
\EE(\ww) \defeq \left\{ \|\widehat{\mv{\theta}}(\ww)-\gvec(\ww)\| > 2 \sqrt{\frac{dV\log(2d\delta^{-1})}{n}} c_{0} \right\}
\end{align*}
has probability $\prr \EE(\ww) \leq \delta$. In practice, however, $\wwhat_{(t)}$ for all $t > 0$ will be random, and depend on the sample. We seek uniform bounds using a covering number argument. By \ref{asmp:A1}, $\WW$ is closed and bounded, and thus compact, and it requires no more than $N_{\epsilon} \leq (3\Delta/2\epsilon)^{d}$ balls of $\epsilon$ radius to cover $\WW$, where $\Delta$ is the diameter of $\WW$.\footnote{This is a basic property of covering numbers for compact subsets of Euclidean space \citep{kolmogorov1993SelectWorks3}.} Write the centers of these $\epsilon$ balls by $\{\wwtil_{1},\ldots,\wwtil_{N_{\epsilon}}\}$. Given $\ww \in \WW$, denote by $\wwtil = \wwtil(\ww)$ the center closest to $\ww$, which satisfies $\|\ww - \wwtil\| \leq \epsilon$. Estimation error is controllable using the following new error terms:
\begin{align}\label{eqn:ineq_3errors}
\|\widehat{\mv{\theta}}(\ww) - \gvec(\ww)\| \leq \|\widehat{\mv{\theta}}(\ww)-\widehat{\mv{\theta}}(\wwtil)\| + \|\gvec(\ww) - \gvec(\wwtil)\| + \|\widehat{\mv{\theta}}(\wwtil) - \gvec(\wwtil)\|.
\end{align}
The goal is to be able to take the supremum over $\ww \in \WW$. We bound one term at a time. The first term can be bounded, for any $\ww \in \WW$, by (\ref{eqn:estimator_smoothness}) just proven. The second term can be bounded by
\begin{align}\label{eqn:bound_error2}
\|\gvec(\ww) - \gvec(\wwtil)\| \leq \lambda\|\ww - \wwtil\|
\end{align}
which follows immediately from \ref{asmp:A2}. Finally, for the third term, fixing any $\ww \in \WW$, $\wwtil=\wwtil(\ww) \in \{\wwtil_{1},\ldots,\wwtil_{N_{\epsilon}}\}$ is also fixed, and can be bounded on the $\delta$ event $\EE(\wwtil)$ just defined. The important fact is that
\begin{align*}
\sup_{\ww \in \WW} \left\|\widehat{\mv{\theta}}(\wwtil(\ww)) - \gvec(\wwtil(\ww)) \right\| = \max_{k \in [N_{\epsilon}]} \left\|\widehat{\mv{\theta}}(\wwtil_{k}) - \gvec(\wwtil_{k}) \right\|.
\end{align*}
We construct a ``good event'' naturally as the event in which the bad event $\EE(\cdot)$ holds for no center on our $\epsilon$-net, namely 
\begin{align*}
\EE_{+} = \left(\bigcap_{k \in [N_{\epsilon}]} \EE(\wwtil_{k}) \right)^{c}.
\end{align*}
Taking a union bound, we can say that with probability no less than $1-\delta$, for all $\ww \in \WW$, we have
\begin{align}\label{eqn:bound_error3}
\|\widehat{\mv{\theta}}(\wwtil(\ww)) - \gvec(\wwtil(\ww))\| \leq 2 \sqrt{\frac{dV\log(2d N_{\epsilon} \delta^{-1})}{n}} c_{0}.
\end{align}
Taking the three new bounds together, we have with probability no less than $1-\delta$ that
\begin{align*}
\sup_{\ww \in \WW} \|\widehat{\mv{\theta}}(\ww) - \gvec(\ww)\| \leq \lambda\epsilon(\sqrt{d}+1) + 2 \sqrt{\frac{dV\log(2d N_{\epsilon} \delta^{-1})}{n}} c_{0}.
\end{align*}
Setting $\epsilon = 1/\sqrt{n}$ we have
\begin{align*}
\sup_{\ww \in \WW} \|\widehat{\mv{\theta}}(\ww) - \gvec(\ww)\| \leq \frac{\lambda(\sqrt{d}+1)}{\sqrt{n}} + 2c_{0} \sqrt{\frac{dV(\log(2d\delta^{-1}) + d\log(3\Delta\sqrt{n}/2))}{n}}.
\end{align*}
Since every step of Algorithm \ref{algo:rgd} (with orthogonal projection if required) has $\wwhat_{(t)} \in \WW$, the desired result follows from this uniform confidence interval.
\end{proof}

\begin{proof}[Proof of Lemma \ref{lem:dueling_strong}]
Given $\wwhat_{(t)}$, running the approximate update (\ref{eqn:GD_update_approx}), we have
\begin{align*}
\|\wwhat_{(t+1)}-\wwstar\| & = \|\wwhat_{(t)}-\alpha\gghat(\wwhat_{(t)})-\wwstar\|\\
& \hspace{-2cm} \leq \|\wwhat_{(t)}-\alpha\gvec(\wwhat_{(t)})-\wwstar\| + \alpha\|\gghat(\wwhat_{(t)})-\gvec(\wwhat_{(t)})\|.
\end{align*}
The first term looks at the distance from the target given an optimal update, using $\gvec$. Using the $\kappa$-strong convexity of $R$, via \citet[Thm.~2.1.15]{nesterov2004ConvOpt} it follows that
\begin{align*}
\|\wwhat_{(t)}-\alpha\gvec(\wwhat_{(t)})-\wwstar\|^{2} \leq \left(1-\frac{2\alpha\kappa\lambda}{\kappa+\lambda}\right)\|\wwhat_{(t)}-\wwstar\|^{2}.
\end{align*}
Writing $\beta \defeq 2\kappa\lambda/(\kappa+\lambda)$, the coefficient becomes $(1-\alpha\beta)$.

To control the second term simply requires unfolding the recursion. By hypothesis, we can leverage (\ref{eqn:conf_uniform}) to bound the statistical estimation error by $\varepsilon$ for every step, all on the same $1-\delta$ ``good event.'' For notational ease, write $a \defeq \sqrt{1-\alpha\beta}$. On the good event, we have
\begin{align*}
\|\wwhat_{(t+1)}-\wwstar\| & \leq a^{t+1}\|\wwhat_{(0)}-\wwstar\| + \alpha\varepsilon\left(1+a+a^{2}+\cdots+a^{t}\right)\\
& = a^{t+1}\|\wwhat_{(0)}-\wwstar\| + \alpha\varepsilon\frac{(1-a^{t+1})}{1-a}.
\end{align*}
To clean up the second summand,
\begin{align*}
\alpha\varepsilon\frac{(1-a^{t+1})}{1-a} & \leq \frac{\alpha\varepsilon(1+a)}{(1-a)(1+a)}\\
& = \frac{\alpha\varepsilon(1+\sqrt{1-\alpha\beta})}{\alpha\beta}\\
& \leq \frac{2\varepsilon}{\beta}.
\end{align*}
Taking this to the original inequality yields the desired result.
\end{proof}

\begin{proof}[Proof of Theorem \ref{thm:main_Rbound_strong}]
Using strong convexity and (\ref{eqn:facts_CO_1}), we have that
\begin{align*}
R(\wwhat_{(T)}) - R^{\ast} & \leq \frac{\lambda}{2}\|\wwhat_{(T)} - \wwstar\|^{2}\\
& \leq \lambda(1-\alpha\beta)^{T}D_{0}^{2} + \frac{4\lambda\varepsilon^{2}}{\beta^{2}}.
\end{align*}
The latter inequality holds by direct application of Lemma \ref{lem:dueling_strong}, followed by the elementary fact $(a+b)^{2} \leq 2(a^{2}+b^{2})$. The particular value of $\varepsilon$ under which Lemma \ref{lem:dueling_strong} is valid (i.e., under which (\ref{eqn:conf_uniform}) holds) is given by Lemma \ref{lem:grad_estimate}. Filling in $\varepsilon$ with this concrete setting yields the desired result.
\end{proof}

\begin{proof}[Proof of Lemma \ref{lem:grad_estimate_varknown}]
As in the result statement, we write
\begin{align*}
\Sigma_{(t)} \defeq \exx_{\mu}\left(l^{\prime}(\wwhat_{(t)};\zz) - \gvec(\wwhat_{(t)})\right)\left(l^{\prime}(\wwhat_{(t)};\zz) - \gvec(\wwhat_{(t)})\right)^{T}, \quad \ww \in \WW.
\end{align*}
Running this modified version of Algorithm \ref{algo:rgd}, we are minimizing the bound in Lemma \ref{lem:sharp_Mest} as a function of scale $s_{j}$, $j \in [d]$, which immediately implies that the estimates $\mv{\that}_{(t)}=(\that_{1},\ldots,\that_{d})$ at each step $t$ satisfy
\begin{align}\label{eqn:grad_estimate_varknown_1}
|\that_{j} - g_{j}(\wwhat)| > 4 \left(\frac{C\vaa_{\mu}l^{\prime}_{j}(\wwhat_{(t)};\zz)\log(2\delta^{-1})}{n}\right)^{1/2}
\end{align}
with probability no greater than $\delta$. For clean notation, let us also denote
\begin{align*}
A \defeq 4 \left(\frac{C\log(2\delta^{-1})}{n} \right)^{1/2}, \quad \varepsilon^{\ast} \defeq A\sqrt{\trace(\Sigma_{(t)})}.
\end{align*}
For the vector estimates then, we have
\begin{align*}
\prr\left\{ \|\mv{\that}_{(t)}-\gvec(\wwhat_{(t)})\| > \varepsilon^{\ast} \right\} &\\
& \hspace{-3.5cm} = \prr\left\{ \sum_{j=1}^{d}\frac{(\that_{j} - g_{j}(\wwhat_{(t)}))^{2}}{A^{2}} > \trace(\Sigma_{(t)}) \right\}\\
& \hspace{-3.5cm} = \prr\left\{ \sum_{j=1}^{d}\left(\frac{(\that_{j} - g_{j}(\wwhat_{(t)}))^{2}}{A^{2}}-\vaa_{\mu}l_{j}^{\prime}(\wwhat_{(t)};\zz)\right) > 0 \right\}\\
& \hspace{-3.5cm} \leq \prr \bigcup_{j=1}^{d} \left\{ \frac{(\that_{j} - g_{j}(\wwhat_{(t)}))^{2}}{A^{2}} > \vaa_{\mu}l_{j}^{\prime}(\wwhat_{(t)};\zz) \right\}\\
& \hspace{-3.5cm} \leq d\delta.
\end{align*}
The first inequality uses a union bound, and the second inequality follows from (\ref{eqn:grad_estimate_varknown_1}). Plugging in $A$ and taking confidence $\delta/d$ implies the desired result.
\end{proof}

\begin{proof}[Proof of Theorem \ref{thm:variance_control}]
From Lemma \ref{lem:grad_estimate_varknown}, the estimation error has exponential tails, as follows. Writing
\begin{align*}
A_{1} \defeq 2d, \quad A_{2} \defeq 4\left(\frac{C\trace(\Sigma_{(t)})}{n}\right)^{1/2},
\end{align*}
for each iteration $t$ we have
\begin{align*}
\prr\{\|\mv{\that}_{(t)}-\gvec(\wwhat_{(t)})\| > \varepsilon\} \leq A_{1} \exp\left(-\left(\frac{\varepsilon}{A_{2}}\right)^{2}\right).
\end{align*}
Controlling moments using exponential tails can be done using a fairly standard argument. For random variable $X \in \LL_{p}$ for $p \geq 1$, we have the classic inequality
\begin{align*}
\exx|X|^{p} = \int_{0}^{\infty} \prr\{|X|^{p}>t\}\,dt
\end{align*}
as a starting point. Setting $X = \|\mv{\that}_{(t)}-\gvec(\wwhat_{(t)})\| \geq 0$, and using substitution of variables twice, we have
\begin{align*}
\exx|X|^{p} & = \int_{0}^{\infty} \prr\{X>t^{1/p}\} \, dt\\
& = \int_{0}^{\infty} \prr\{X > t\}pt^{p-1} \, dt\\
& \leq A_{1}p \int_{0}^{\infty} \exp\left(-\left(t/A_{2}\right)^{2}\right) t^{p-1} \, dt\\
& = \frac{A_{1}A_{2}^{p}p}{2} \int_{0}^{\infty}\exp(-t)t^{p/2-1} \, dt.
\end{align*}
The last integral on the right-hand side, written $\Gamma(p/2)$, is the usual Gamma function of Euler evaluated at $p/2$. Setting $p=2$, we have $\Gamma(1)=0!=1$, and plugging in the values of $A_{1}$ and $A_{2}$ yields the desired result.
\end{proof}

\subsection{Computational methods}\label{sec:appendix_computation}

Here we discuss precisely how to compute the implicitly-defined M-estimates of (\ref{eqn:location_rough}) and (\ref{eqn:scale_rough}). Assuming $s>0$ and real-valued observations $x_{1},\ldots,x_{n}$, we first look at the program
\begin{align*}
\min_{\theta} \frac{1}{n} \sum_{i=1}^{n}\rho_{s}\left(x_{i}-\theta\right)
\end{align*}
assuming $\rho$ is as specified in Definition~\ref{defn:rho}, with $\psi = \rho^{\prime}$. Write $\that$ for this unique minimum, and note that it satisfies
\begin{align*}
\frac{s}{n} \sum_{i=1}^{n}\psi_{s}\left(x_{i}-\that\right) = 0.
\end{align*}
Indeed, by monotonicity of $\psi$, this $\that$ can be found via $\rho$ minimization or root-finding. The latter yields standard fixed-point iterative updates, such as
\begin{align*}
\that_{(k+1)} = \that_{(k)} + \frac{s}{n}\sum_{i=1}^{n}\psi_{s}\left(x_{i}-\that_{(k)}\right).
\end{align*}
Note the right-hand side has a fixed point at the desired value. In our routines, we use the Gudermannian function
\begin{align*}
\rho(u) \defeq \int_{0}^{u}\psi(x)\,dx, \quad \psi(u) \defeq 2\atan(\exp(u))-\pi/2
\end{align*}
which can be readily confirmed to satisfy all requirements of Definition~\ref{defn:rho}.

For the dispersion estimate to be used in re-scaling, we introduce function $\chi$, which is even, non-decreasing on $\RR_{+}$, and satisfies
\begin{align*}
0 < \left|\lim\limits_{u \to \pm \infty} \chi(u)\right| < \infty, \quad \chi(0) < 0.
\end{align*}
In practice, we take dispersion estimate $\widehat{\sigma}>0$ as any value satisfying
\begin{align*}
\frac{1}{n} \sum_{i=1}^{n} \chi\left(\frac{x_{i}-\gamma}{\widehat{\sigma}}\right) = 0
\end{align*}
where $\gamma = n^{-1}\sum_{i=1}^{n}x_{i}$, computed by the iterative procedure
\begin{align*}
\widehat{\sigma}_{(k+1)} = \widehat{\sigma}_{(k)}\left(1-\frac{1}{\chi(0)n}\sum_{i=1}^{n}\chi\left(\frac{x_{i}-\gamma}{\widehat{\sigma}_{(k)}}\right)\right)^{1/2}
\end{align*}
which has the desired fixed point, as in the location case. Our routines use the quadratic Geman-type $\chi$, defined
\begin{align*}
\chi(u) \defeq \frac{u^{2}}{1+u^{2}}-c
\end{align*}
with parameter $c > 0$, noting $\chi(0)=-c$. Writing the first term as $\chi_{0}$ so $\chi(u)=\chi_{0}(u)-c$, we set $c = \exx\chi_{0}(x)$ under $x \sim N(0,1)$. Computed via numerical integration, this is $c \approx 0.34$.

\section{Additional test results}\label{sec:more_test_results}

In this section, we provide some additional experimental results obtained via the tests of section \ref{sec:tests}. In particular, we consider the regression application at the end of section \ref{sec:tests_noisyopt}, where due to space limitations, we only showed results for four distinct families of noise distributions. Here, we consider all of the following distribution families: Arcsine (\texttt{asin}), Beta Prime (\texttt{bpri}), Chi-squared (\texttt{chisq}), Exponential (\texttt{exp}), Exponential-Logarithmic (\texttt{explog}), Fisher's F (\texttt{f}), Fr\'{e}chet (\texttt{frec}), Gamma (\texttt{gamma}), Gompertz (\texttt{gomp}), Gumbel (\texttt{gum}), Hyperbolic Secant (\texttt{hsec}), Laplace (\texttt{lap}), Log-Logistic (\texttt{llog}), Log-Normal (\texttt{lnorm}), Logistic (\texttt{lgst}), Maxwell (\texttt{maxw}), Pareto (\texttt{pareto}), Rayleigh (\texttt{rayl}), Semi-circle (\texttt{scir}), Student's t (\texttt{t}), Triangle (asymmetric \texttt{tri\_a}, symmetric \texttt{tri\_s}), U-Power (\texttt{upwr}), Wald (\texttt{wald}), Weibull (\texttt{weibull}).

The content of this section is as follows:
\begin{itemize}
\item Figures \ref{fig:overN_all_distros_1}--\ref{fig:overN_all_distros_2}: performance as a function of sample size $n$. 

\item Figures \ref{fig:overLvl_all_distros_1}--\ref{fig:overLvl_all_distros_2}: performance over noise levels, with fixed $n$ and $d$.

\item Figures \ref{fig:overDim_all_distros_1}--\ref{fig:overDim_all_distros_2}: performance as a function of $d$, with fixed $n/d$ ratio and noise level.
\end{itemize}

\clearpage

\begin{figure}[t]
\centering
\includegraphics[width=0.25\textwidth]{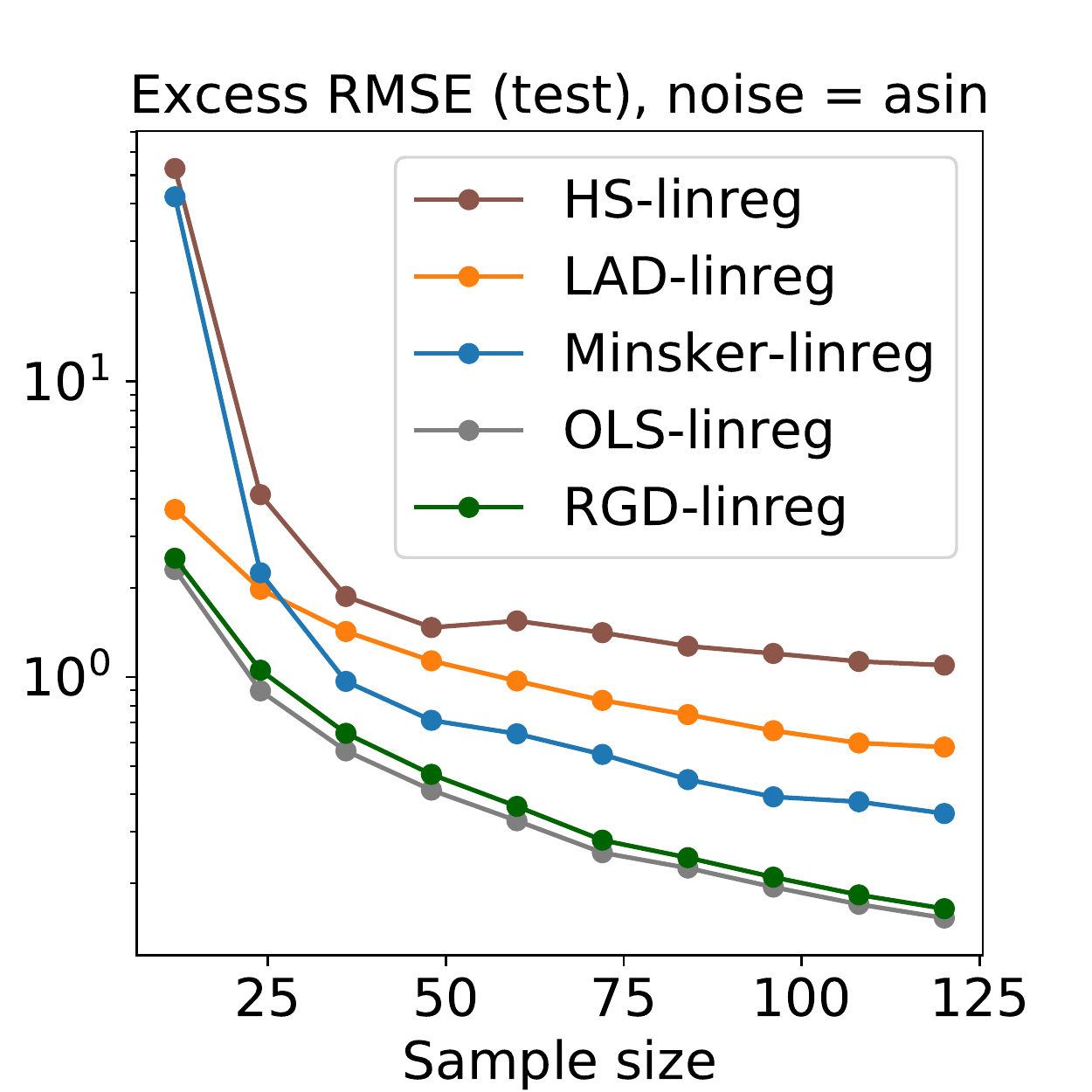}\,\includegraphics[width=0.25\textwidth]{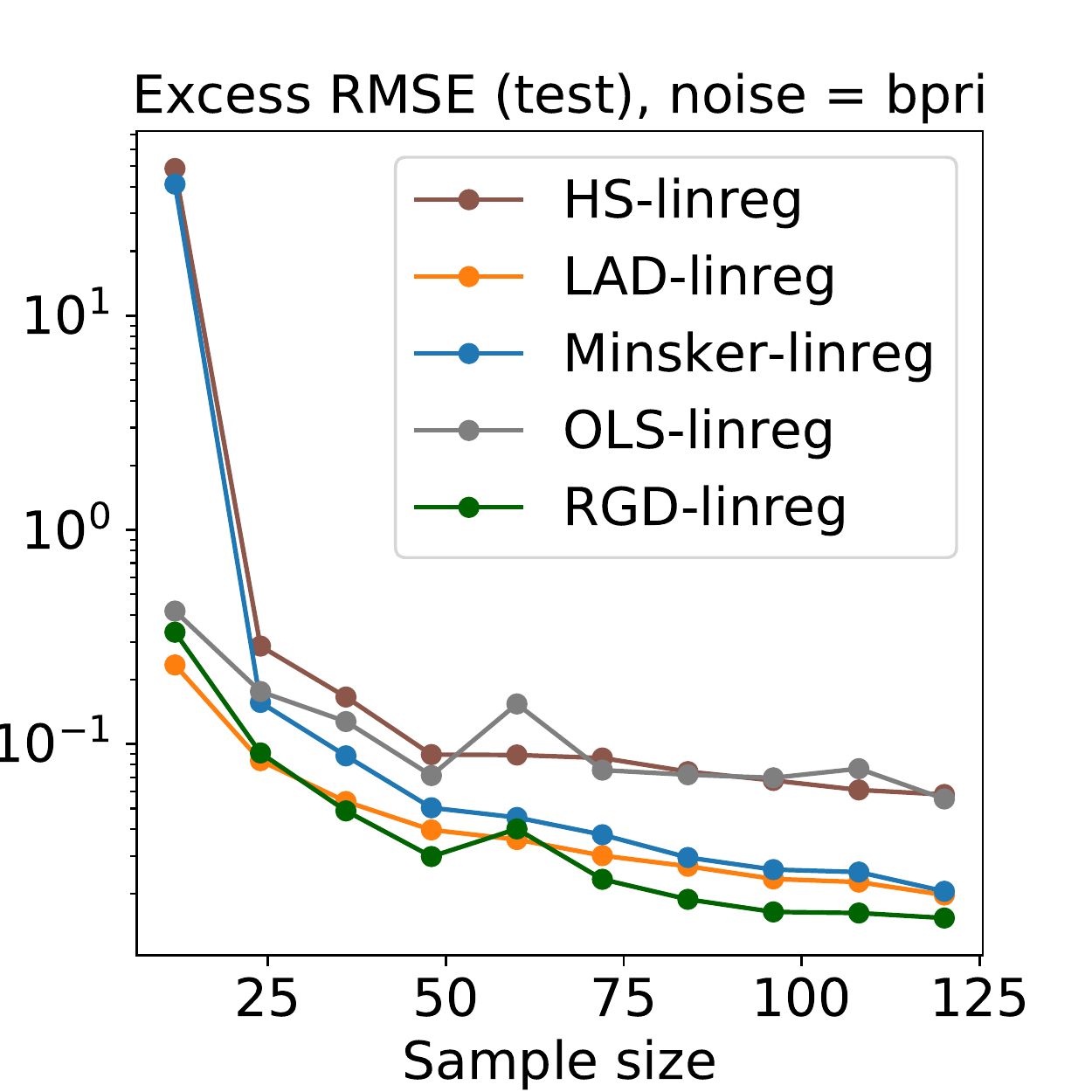}\,\includegraphics[width=0.25\textwidth]{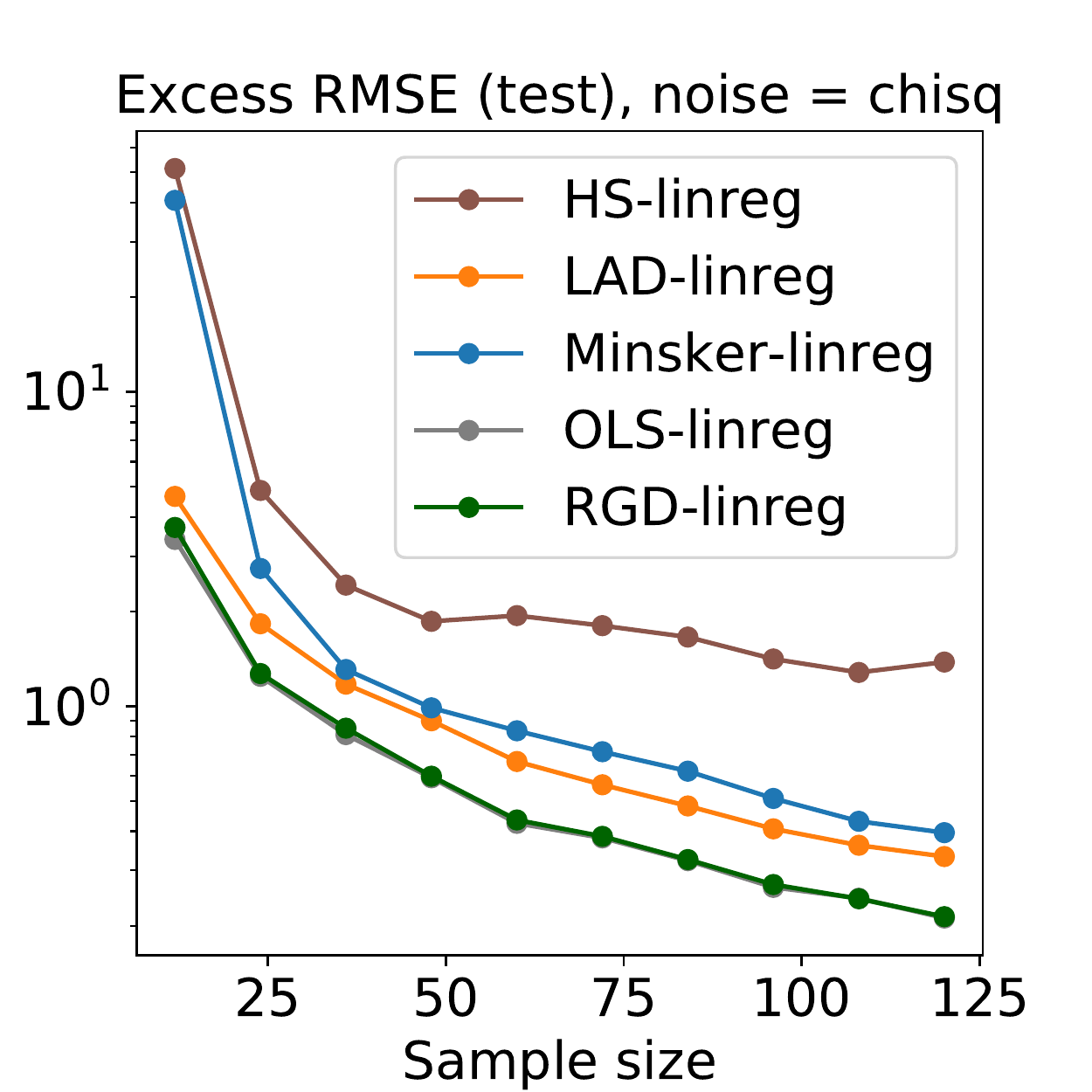}\,\includegraphics[width=0.25\textwidth]{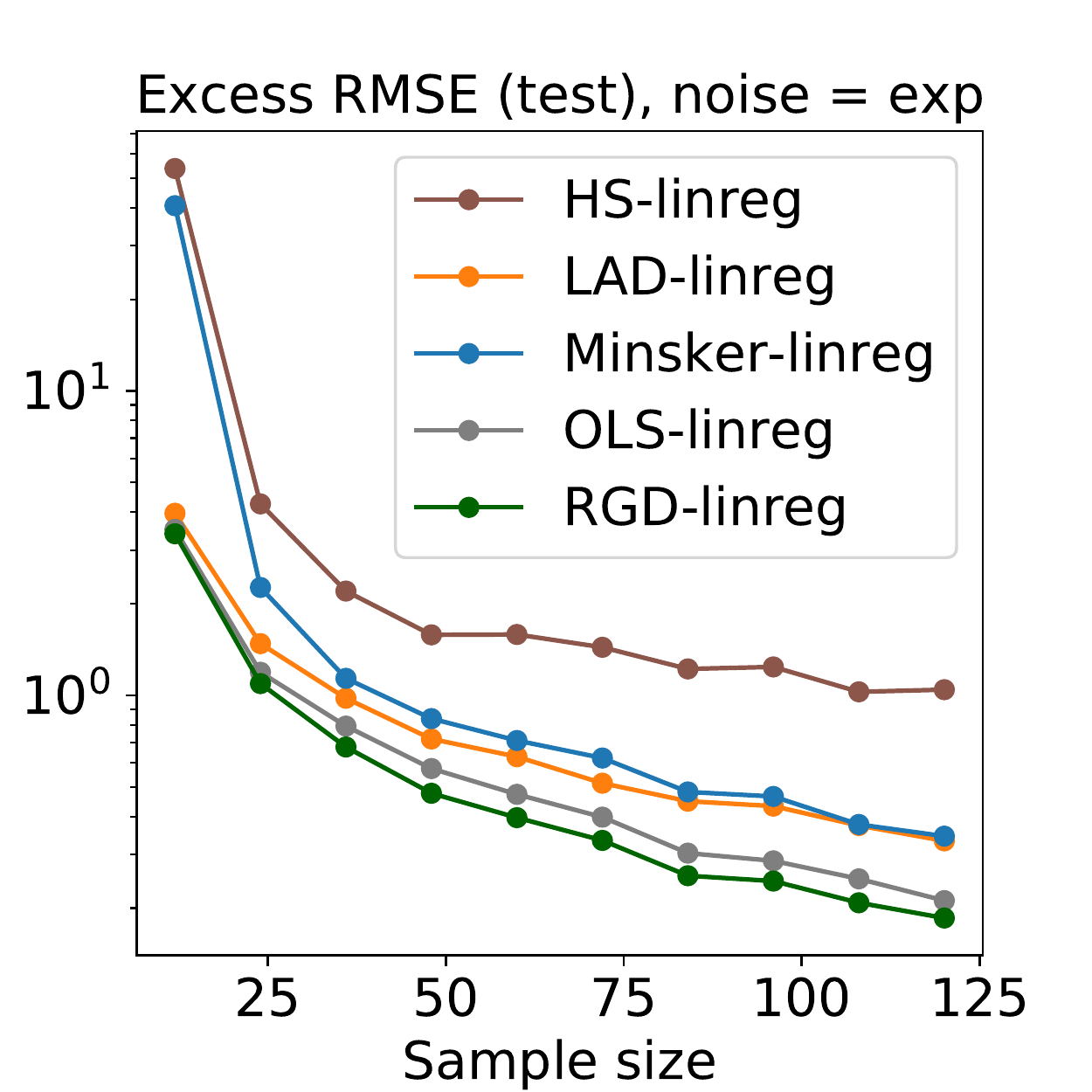}\\
\includegraphics[width=0.25\textwidth]{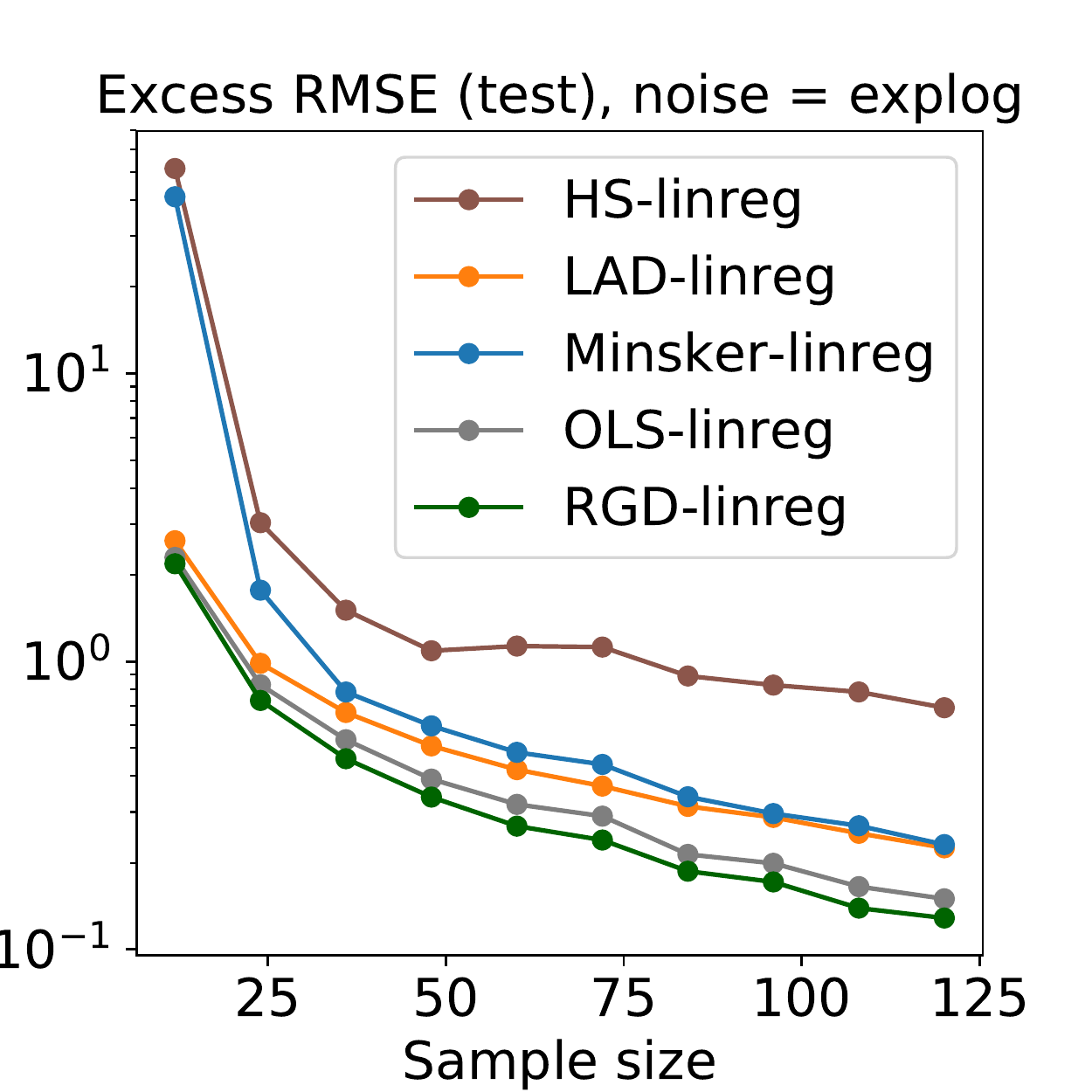}\,\includegraphics[width=0.25\textwidth]{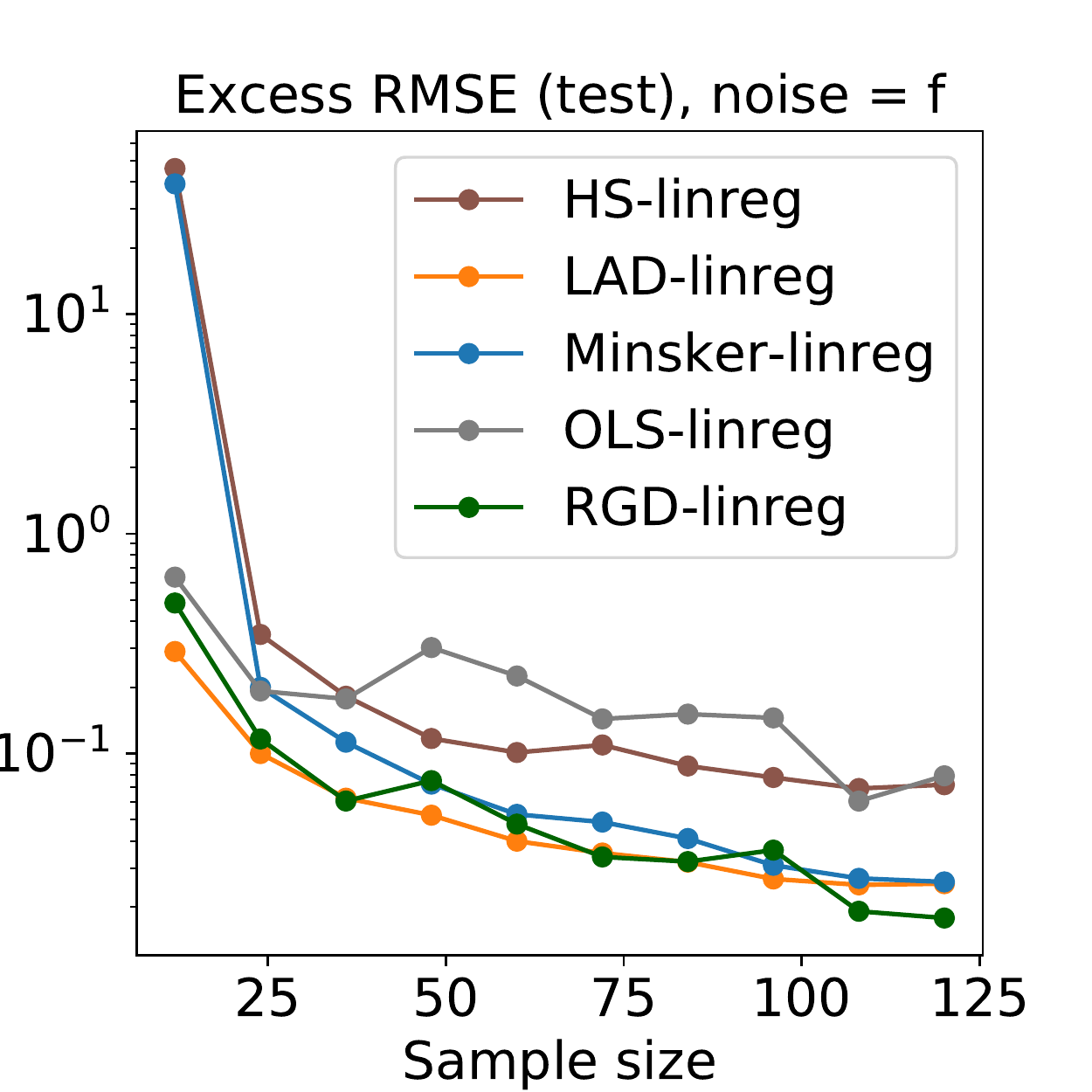}\,\includegraphics[width=0.25\textwidth]{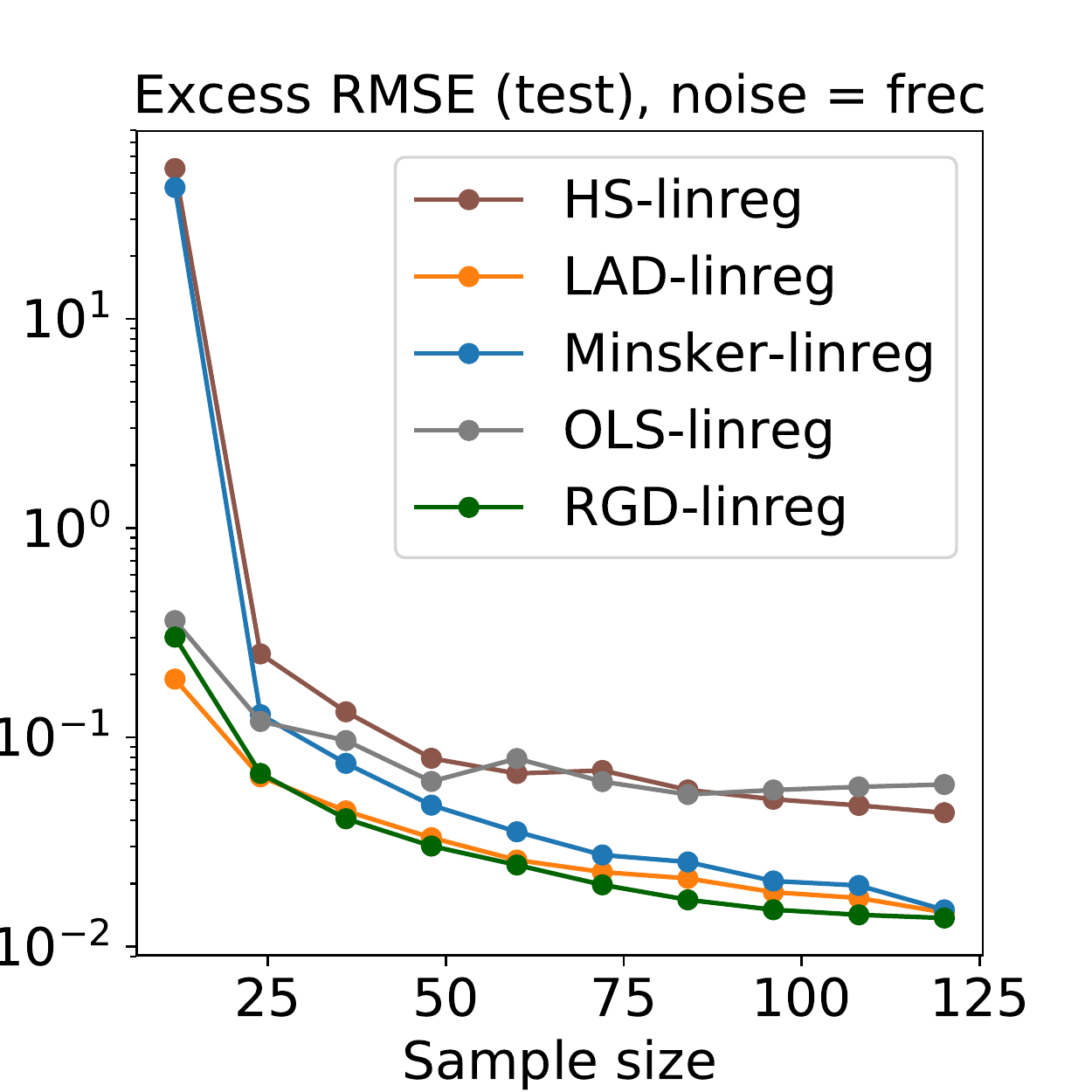}\,\includegraphics[width=0.25\textwidth]{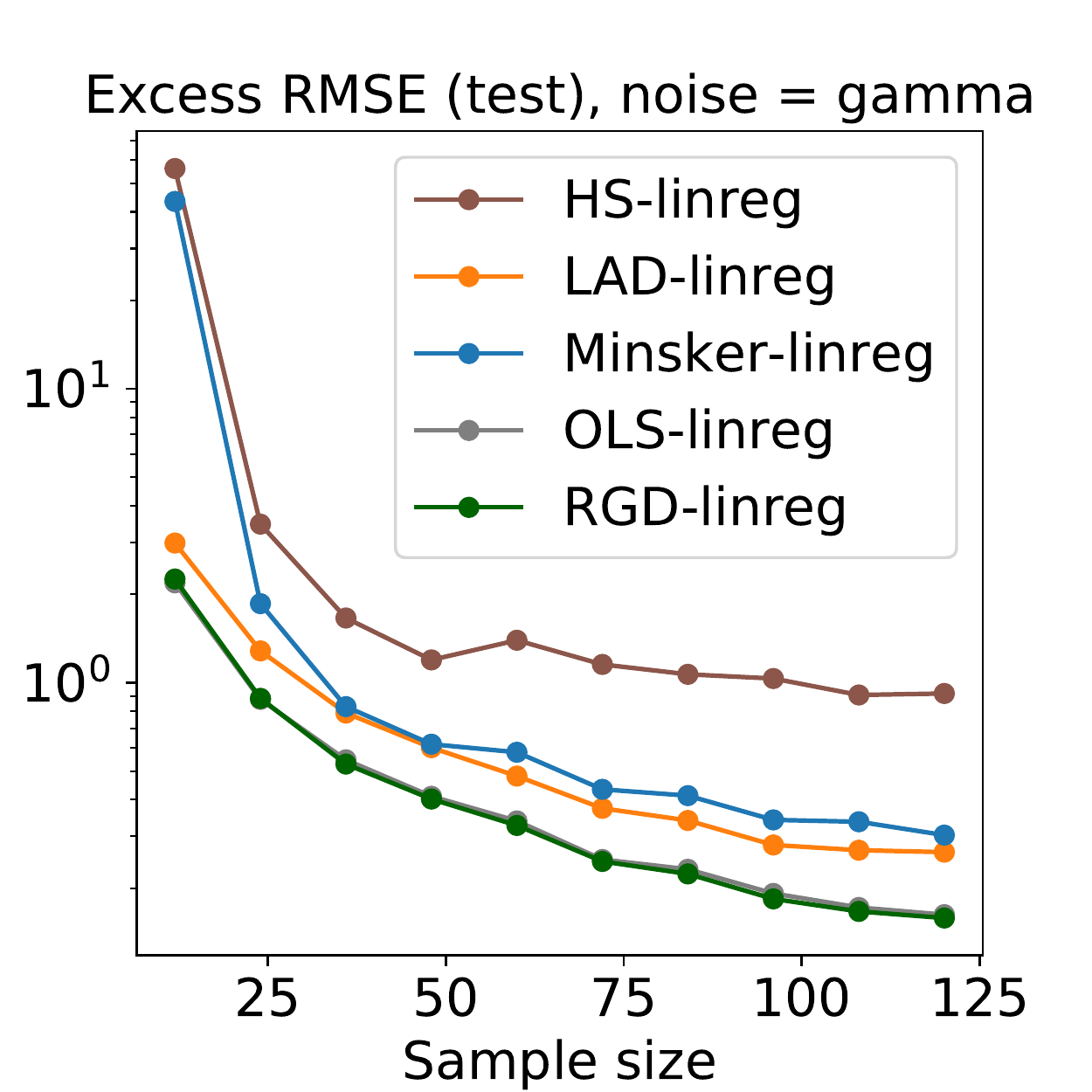}\\
\includegraphics[width=0.25\textwidth]{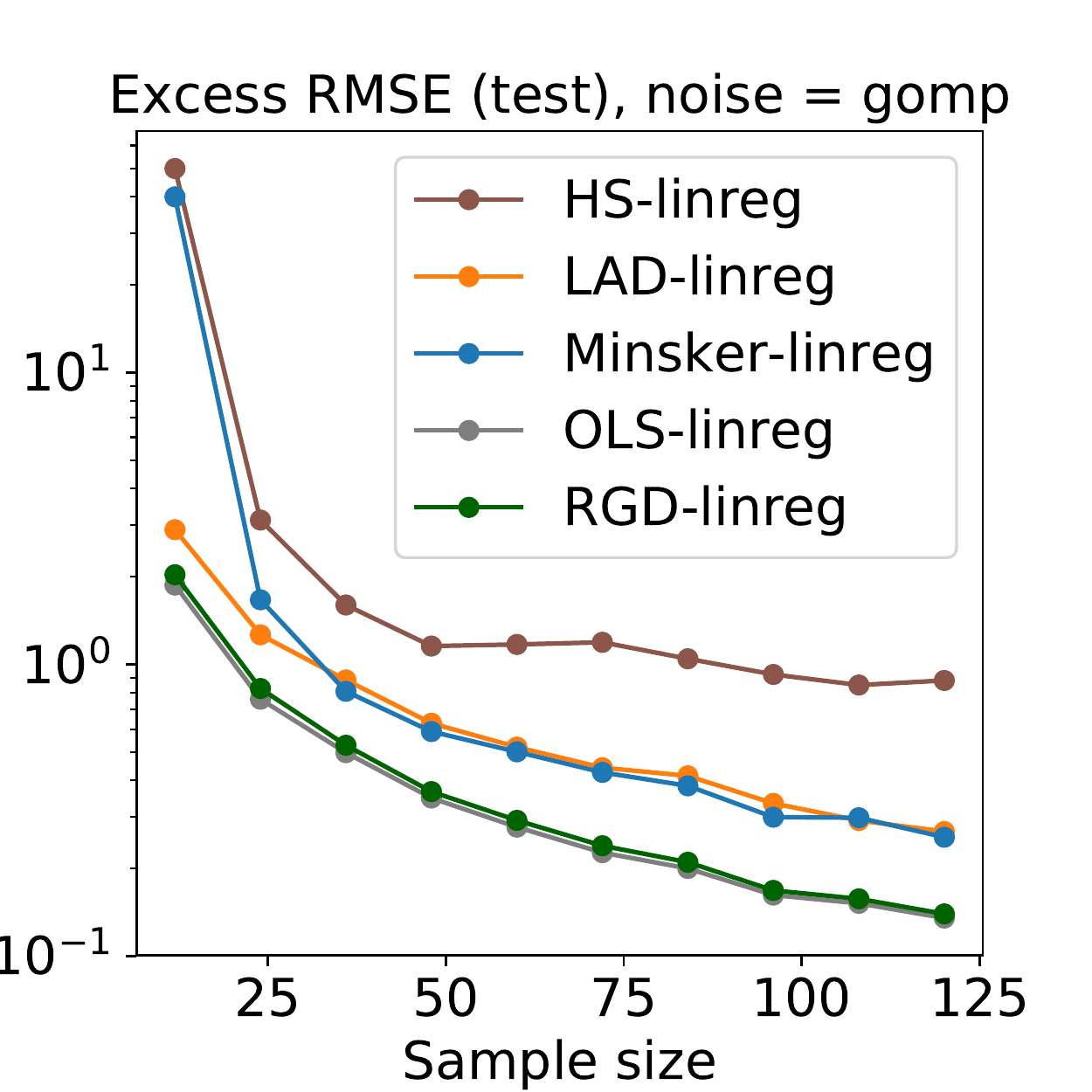}\,\includegraphics[width=0.25\textwidth]{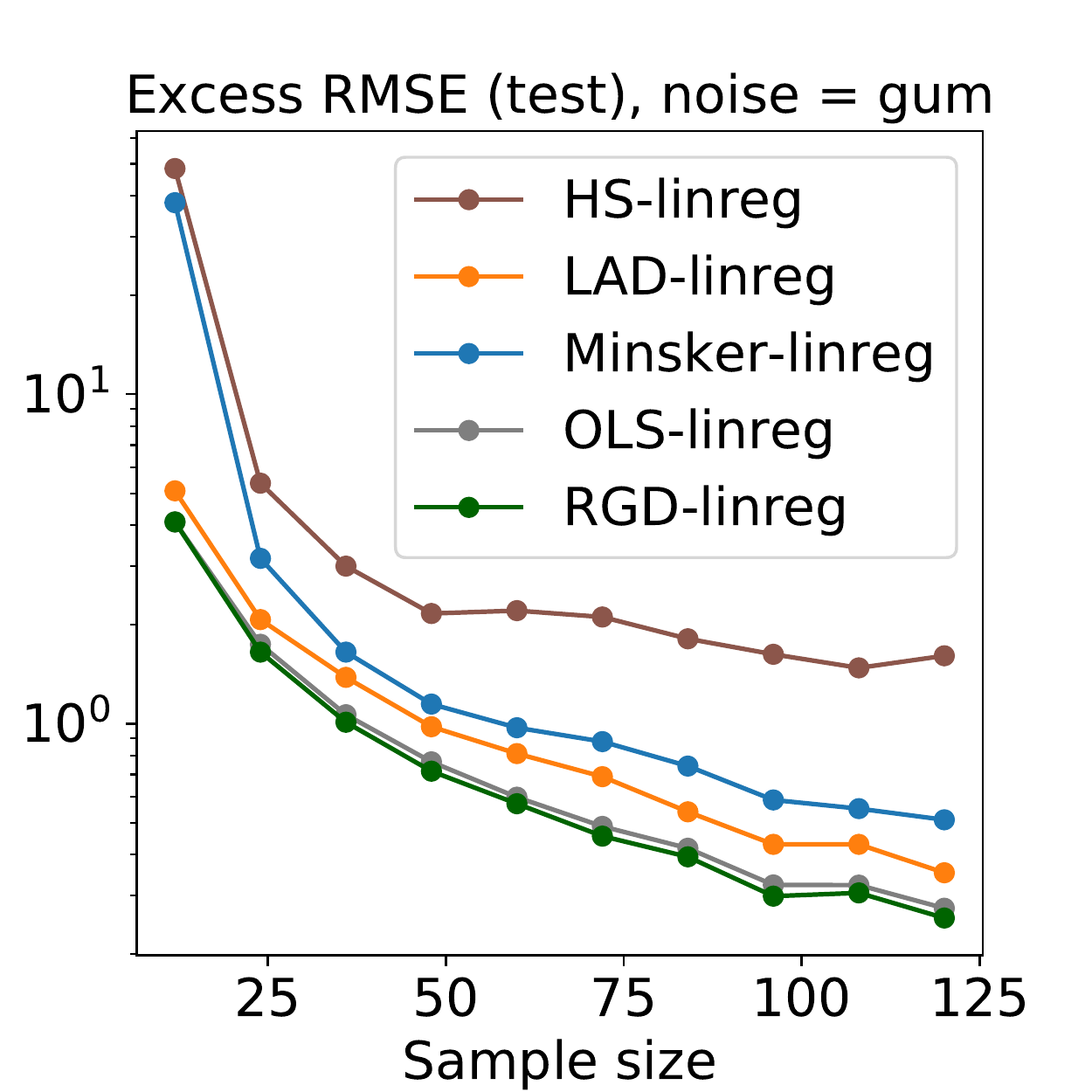}\,\includegraphics[width=0.25\textwidth]{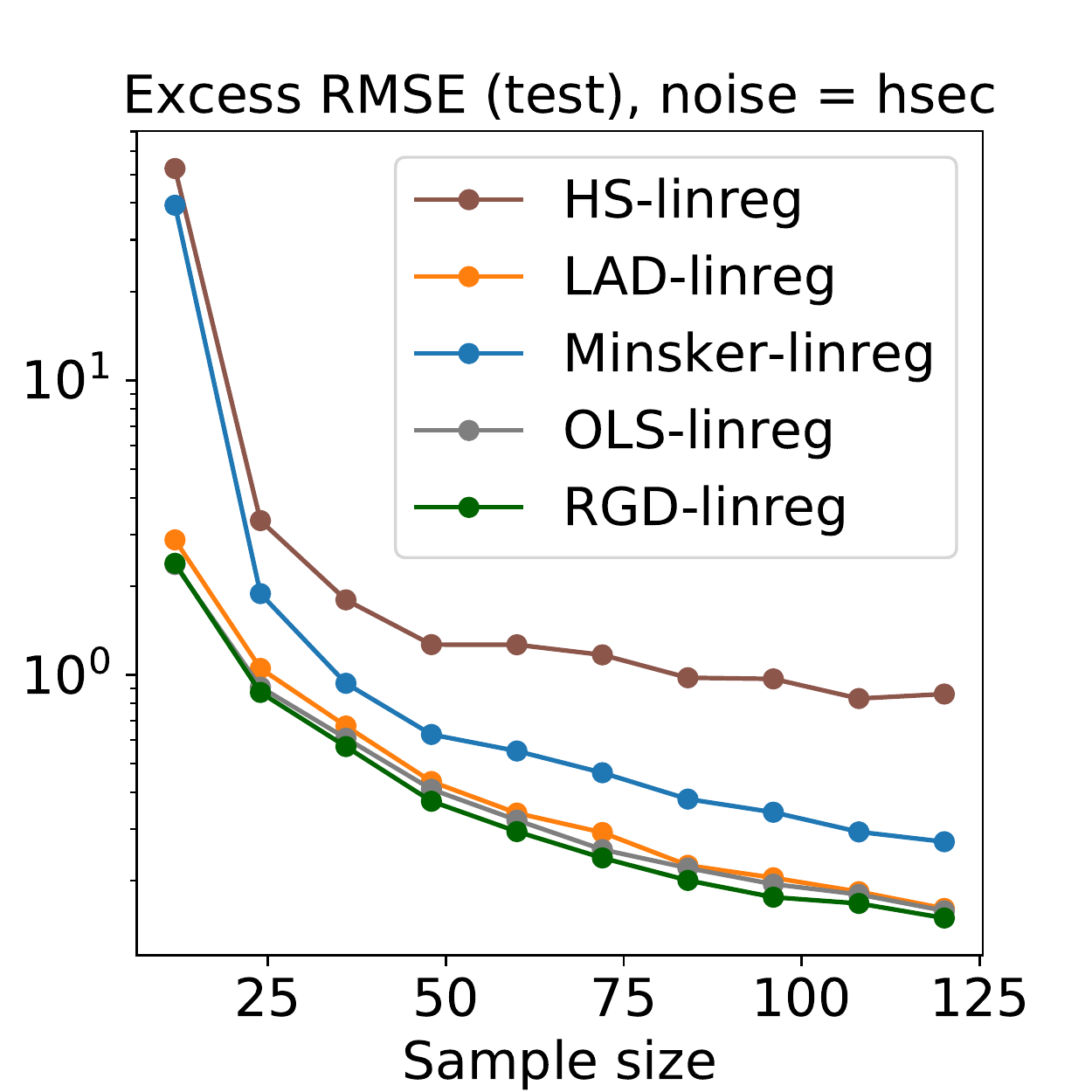}\,\includegraphics[width=0.25\textwidth]{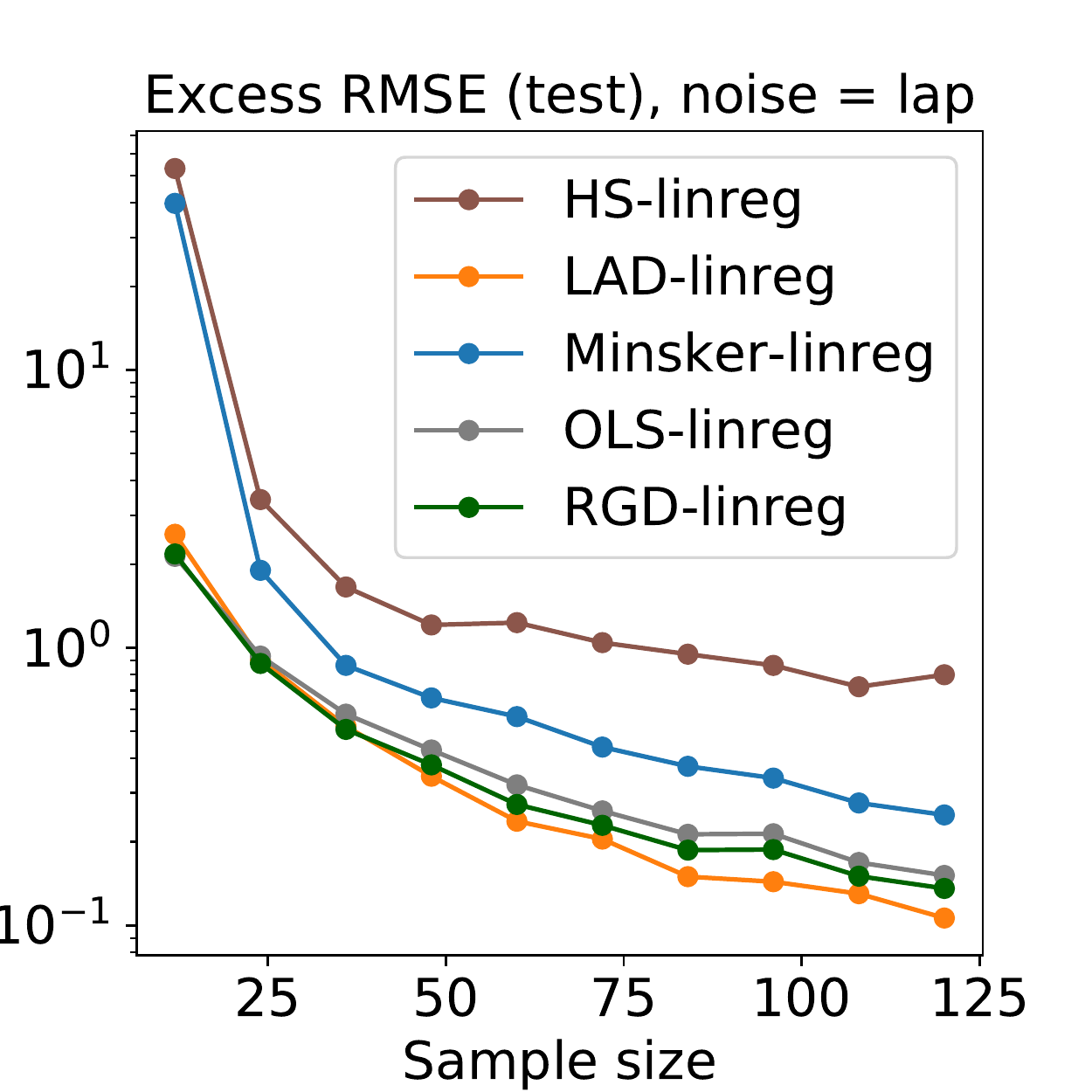}\\
\includegraphics[width=0.25\textwidth]{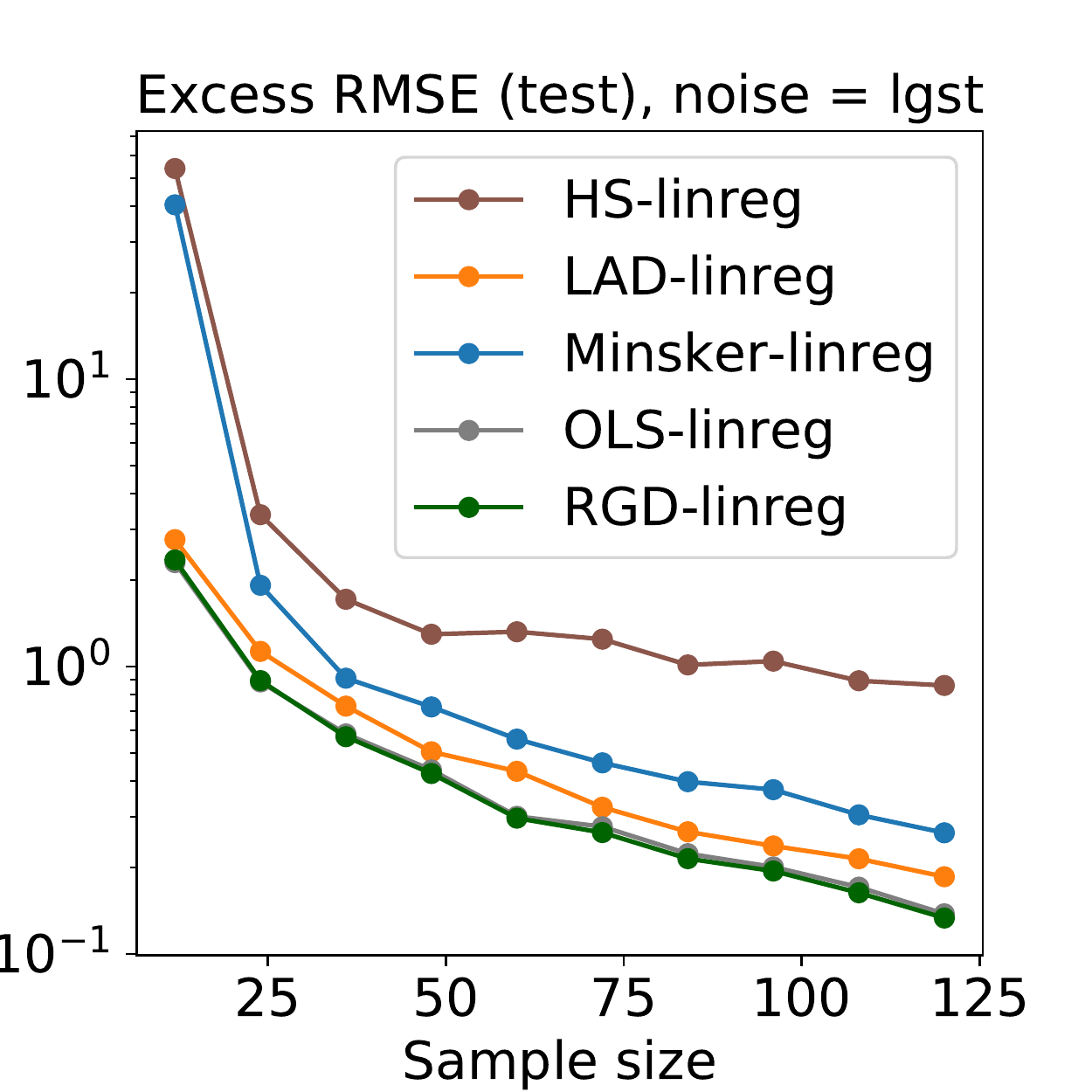}\,\includegraphics[width=0.25\textwidth]{linreg_overN_risk_llog}\,\includegraphics[width=0.25\textwidth]{linreg_overN_risk_lnorm}\,\includegraphics[width=0.25\textwidth]{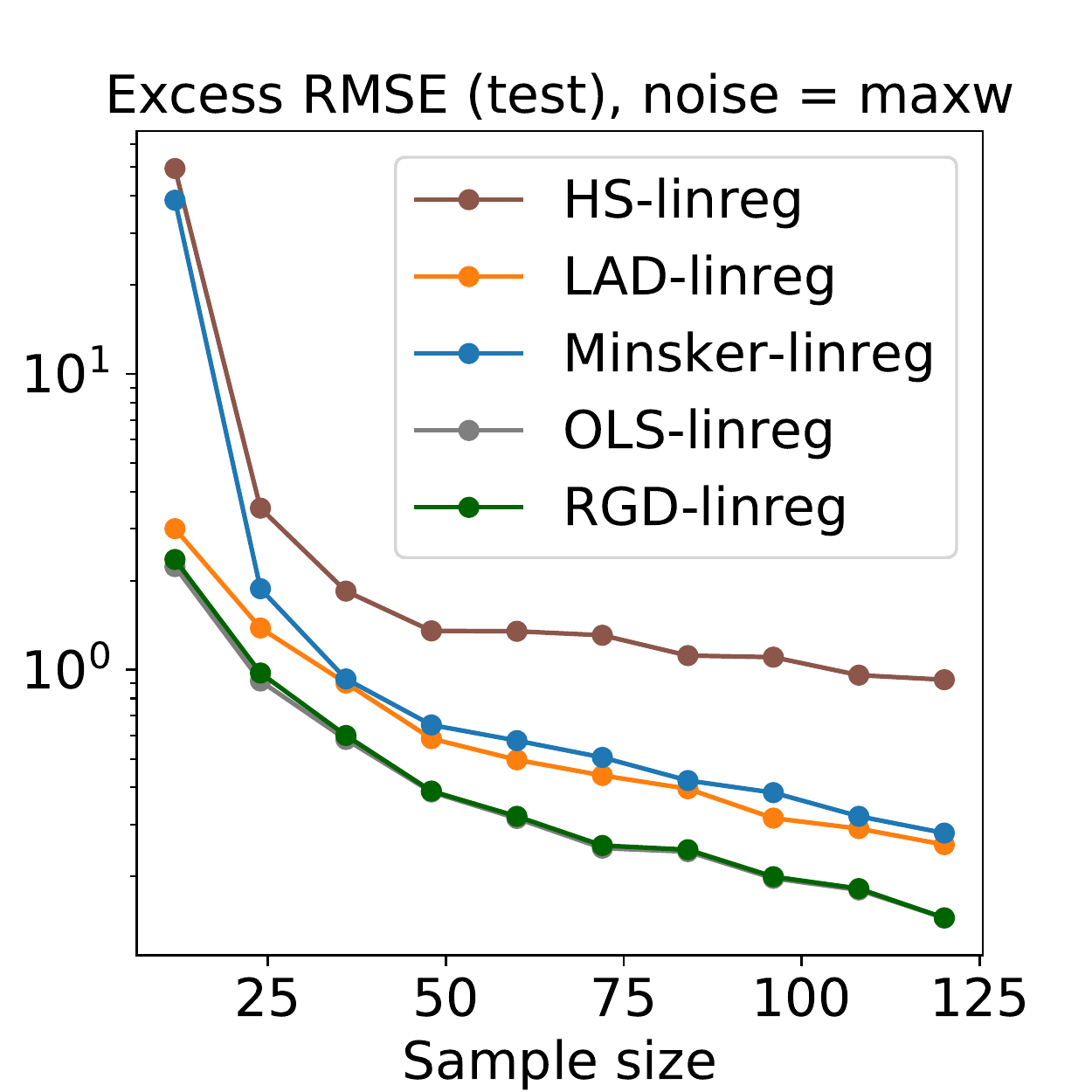}\\
\includegraphics[width=0.25\textwidth]{linreg_overN_risk_norm}\,\includegraphics[width=0.25\textwidth]{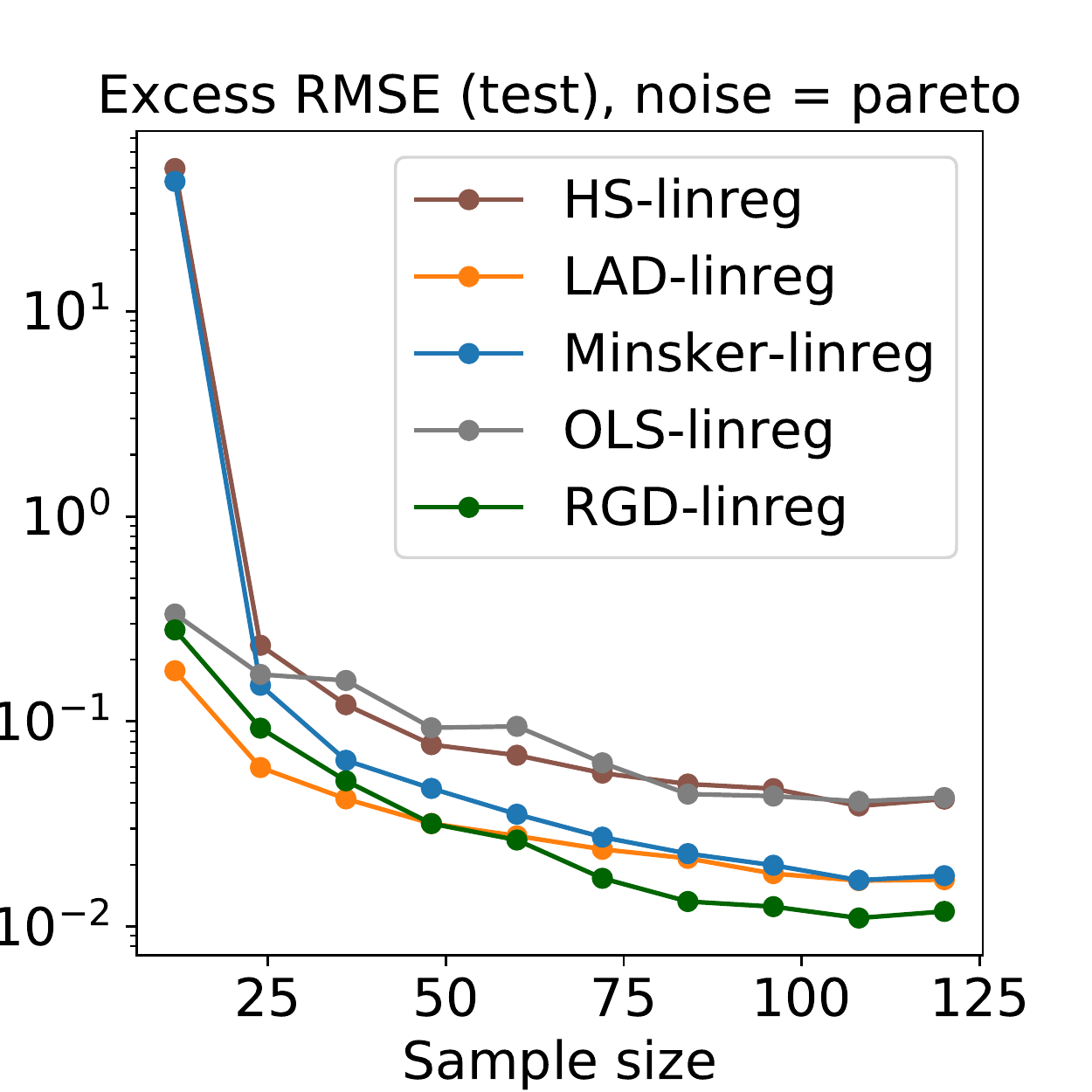}\,\includegraphics[width=0.25\textwidth]{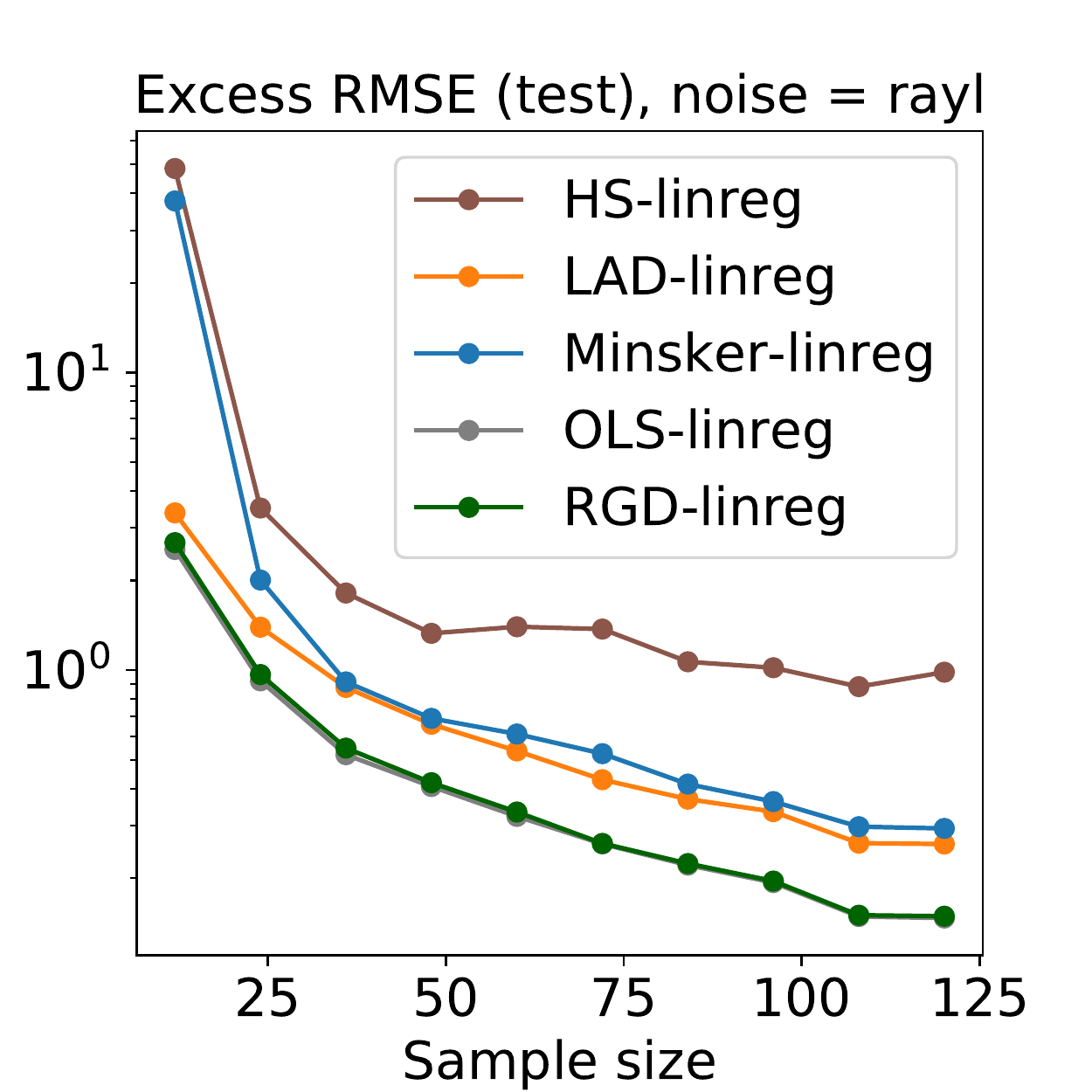}\,\includegraphics[width=0.25\textwidth]{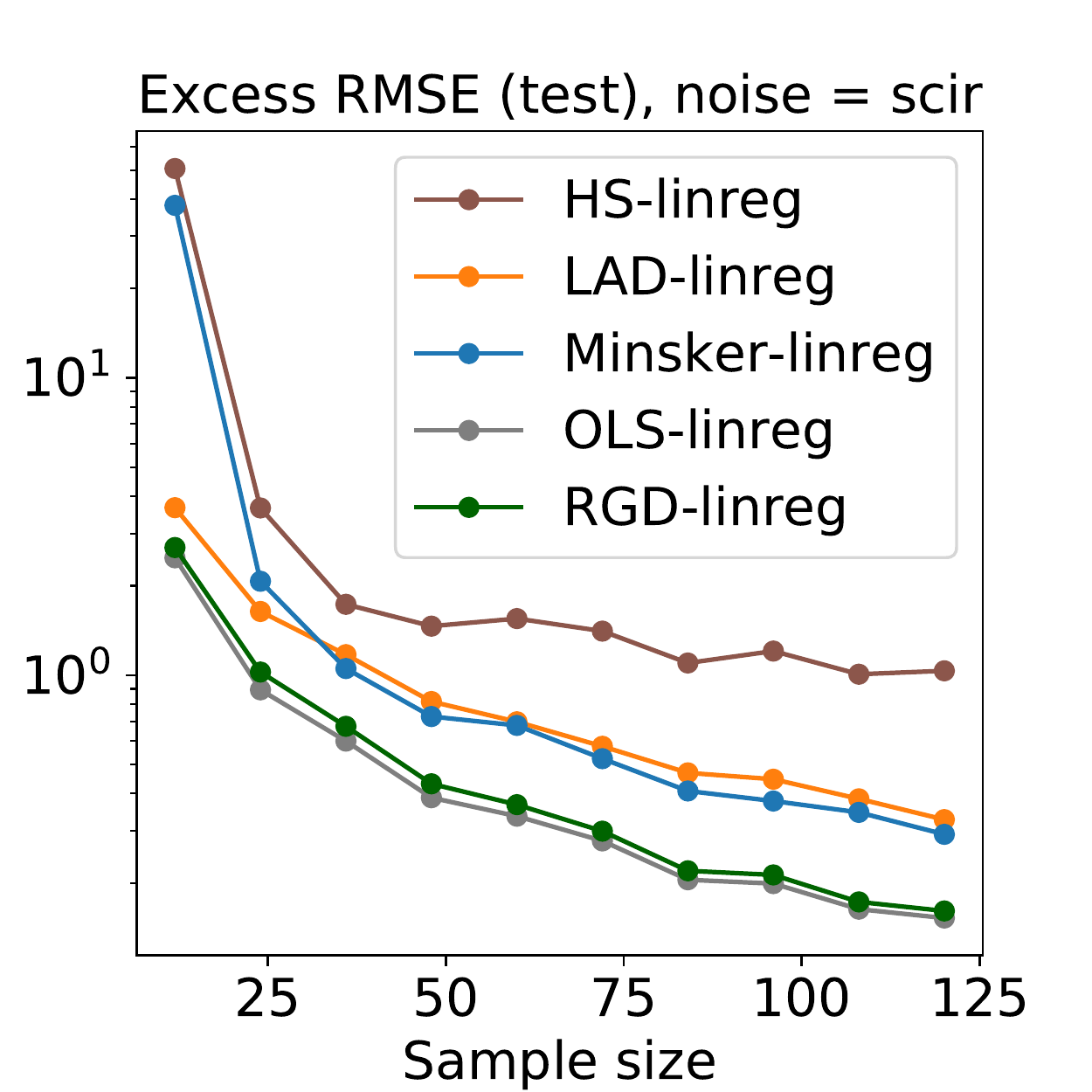}
\caption{Prediction error over sample size $12 \leq n \leq 122$, fixed $d=5$, noise level = $8$. Each plot corresponds to a distinct noise distribution.}
\label{fig:overN_all_distros_1}
\end{figure}

\clearpage

\begin{figure}[t]
\centering
\includegraphics[width=0.25\textwidth]{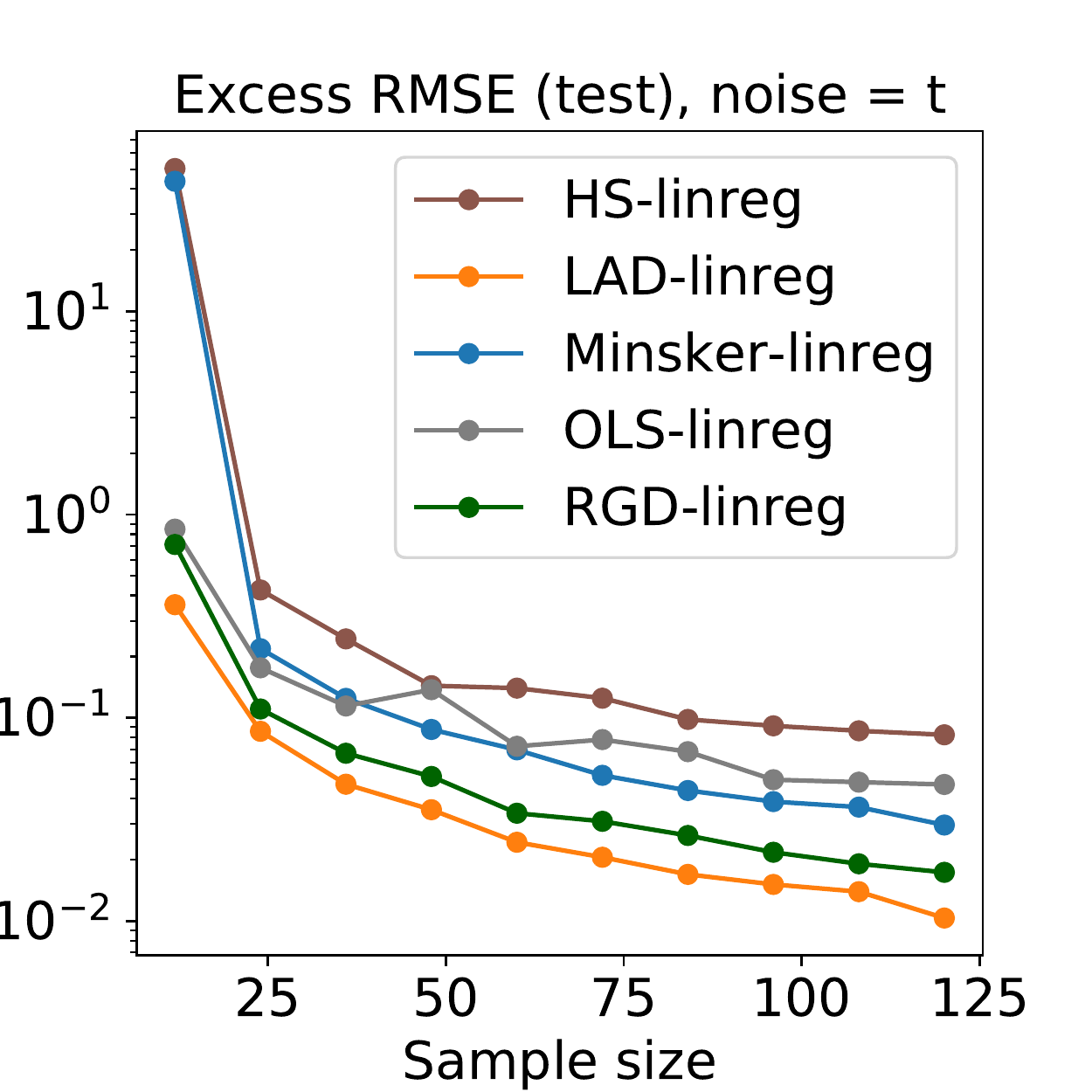}\,\includegraphics[width=0.25\textwidth]{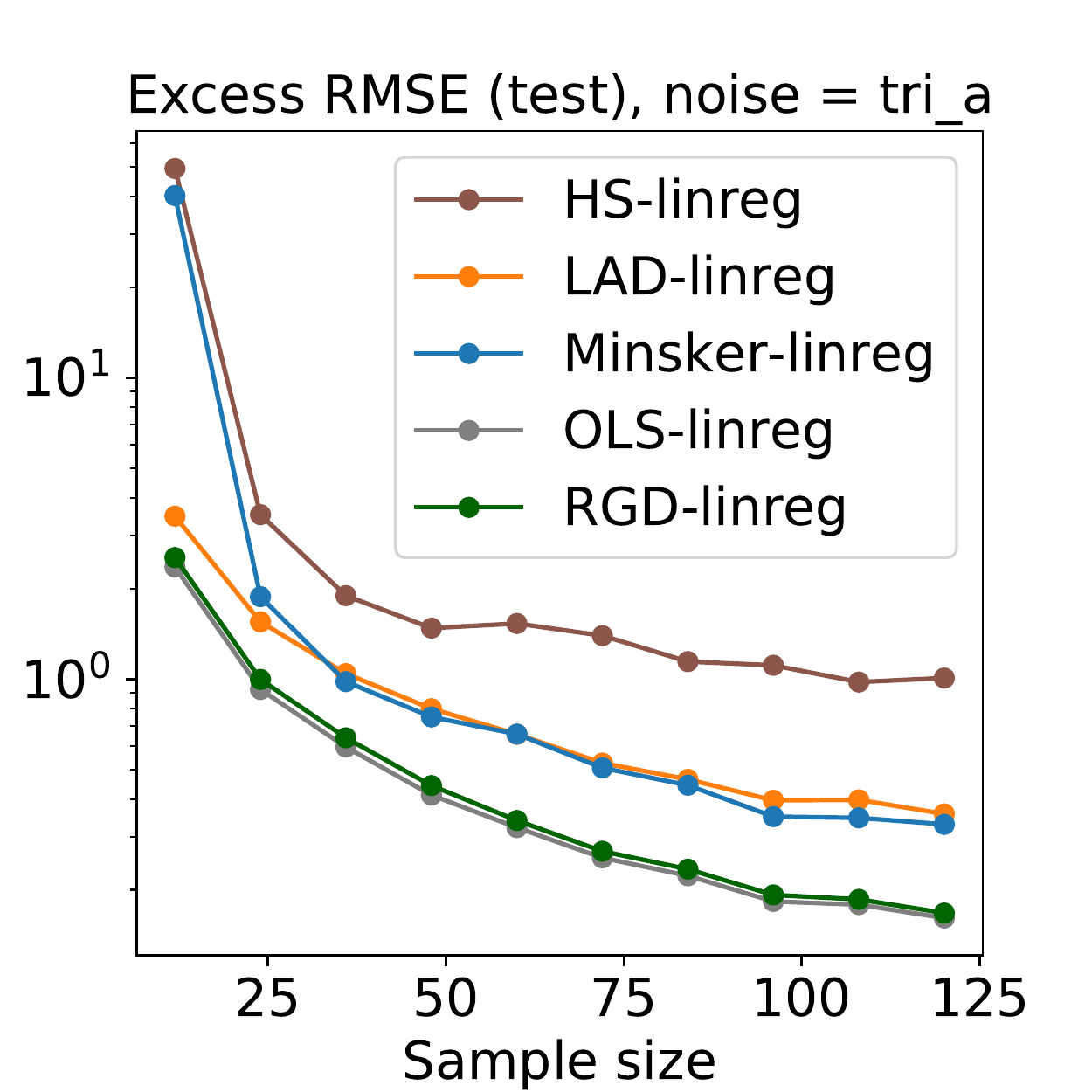}\,\includegraphics[width=0.25\textwidth]{linreg_overN_risk_tri_s}\,\includegraphics[width=0.25\textwidth]{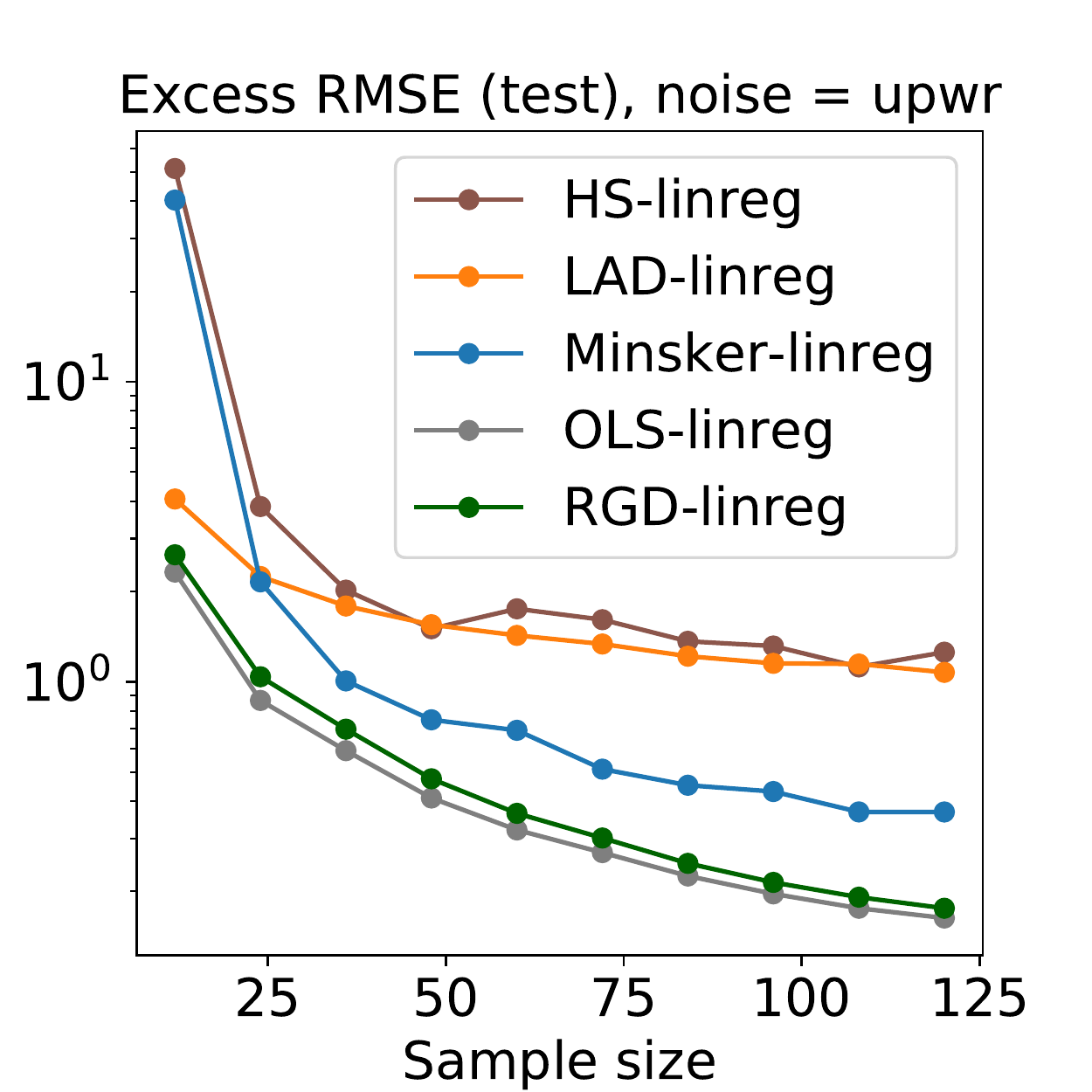}\\
\includegraphics[width=0.25\textwidth]{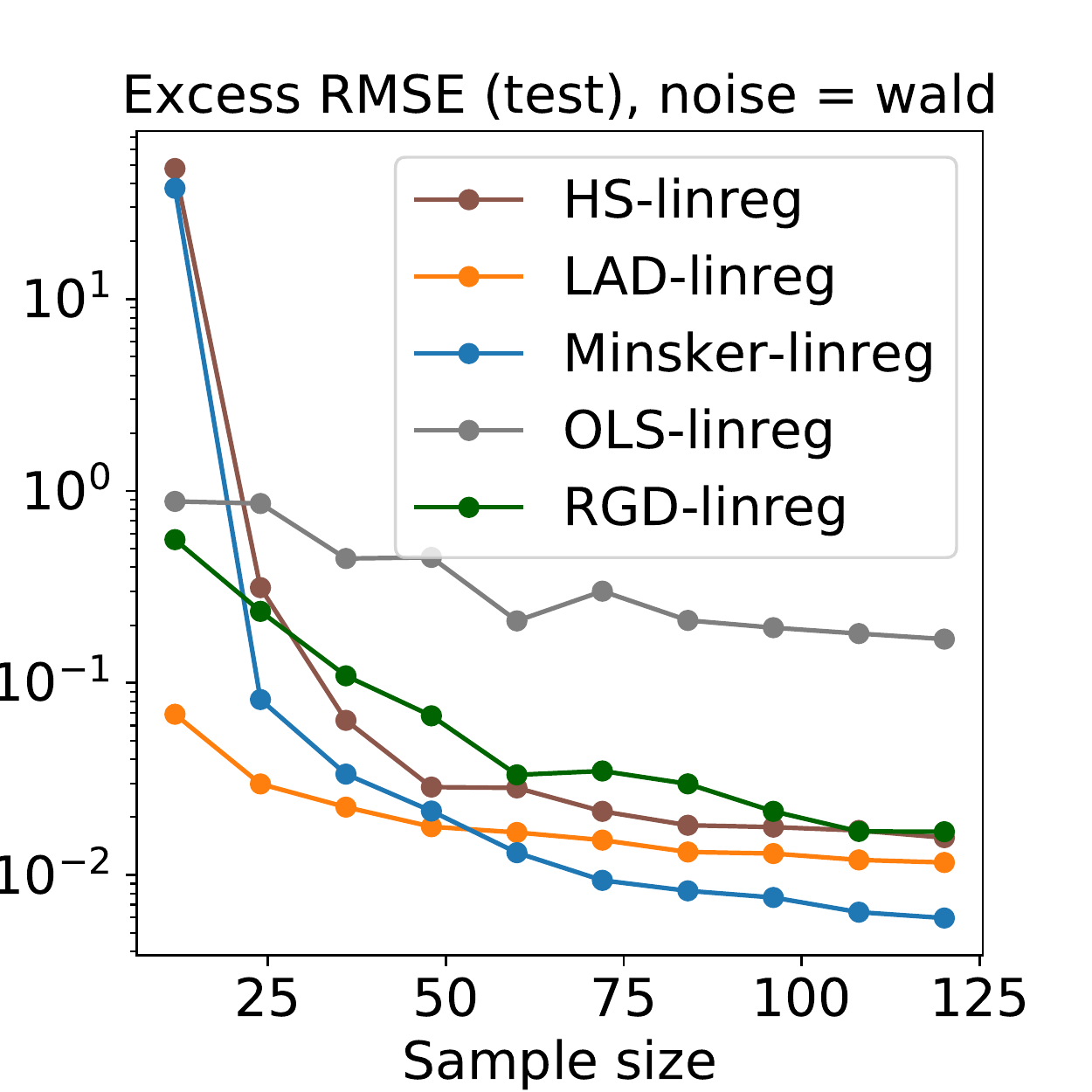}\,\includegraphics[width=0.25\textwidth]{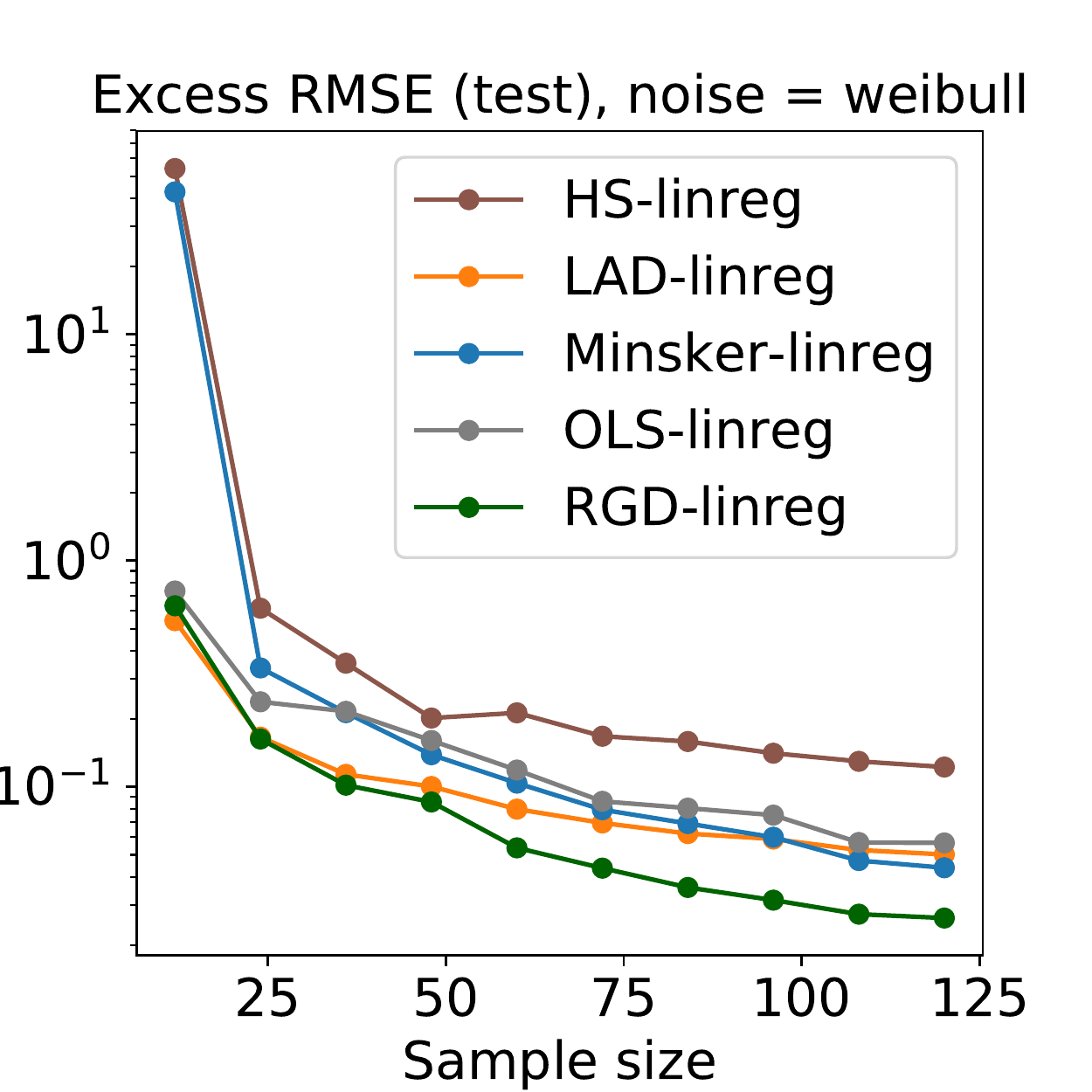}
\caption{Prediction error over sample size $12 \leq n \leq 122$, fixed $d=5$, noise level = $8$. Each plot corresponds to a distinct noise distribution.}
\label{fig:overN_all_distros_2}
\end{figure}

\clearpage

\begin{figure}[t]
\centering
\includegraphics[width=0.25\textwidth]{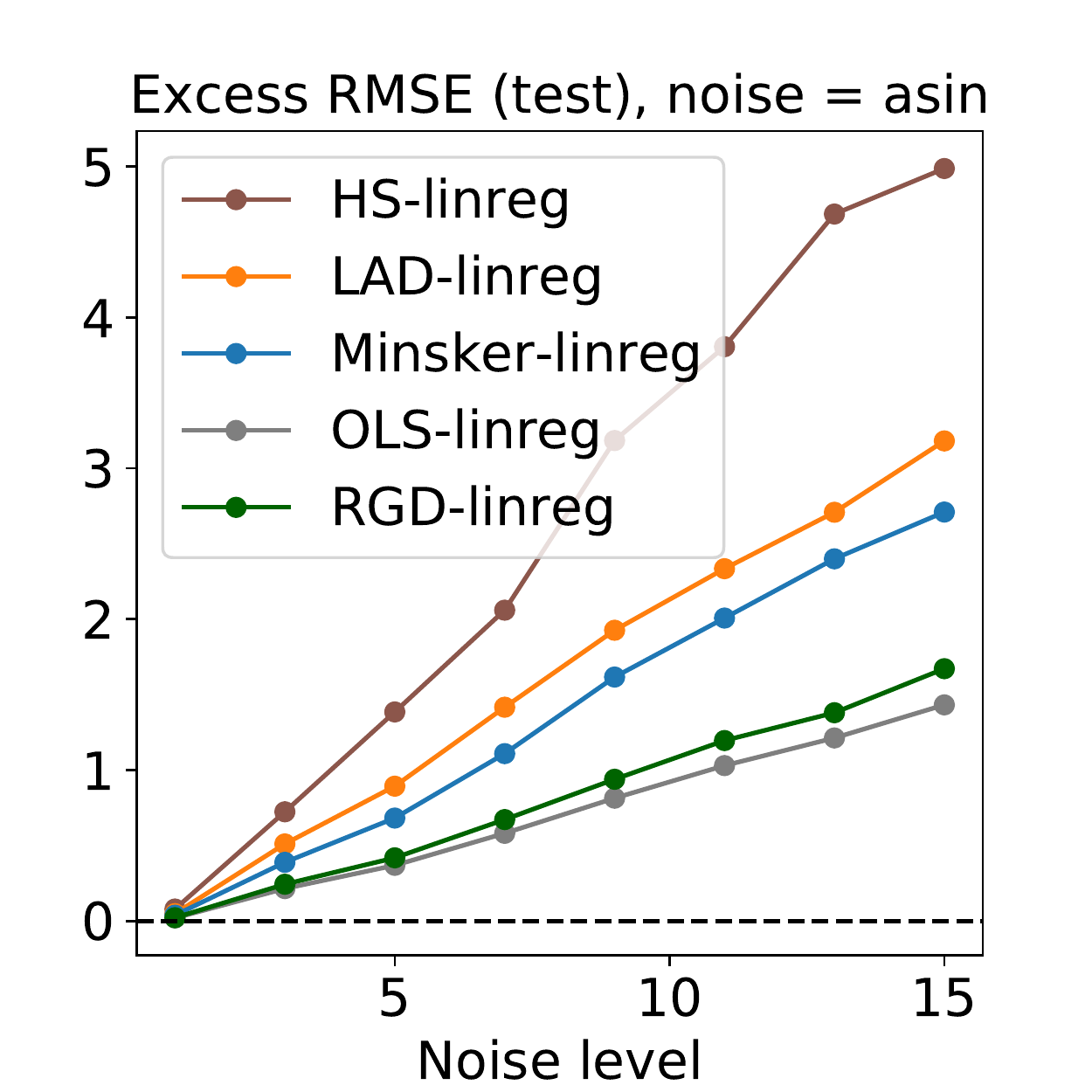}\,\includegraphics[width=0.25\textwidth]{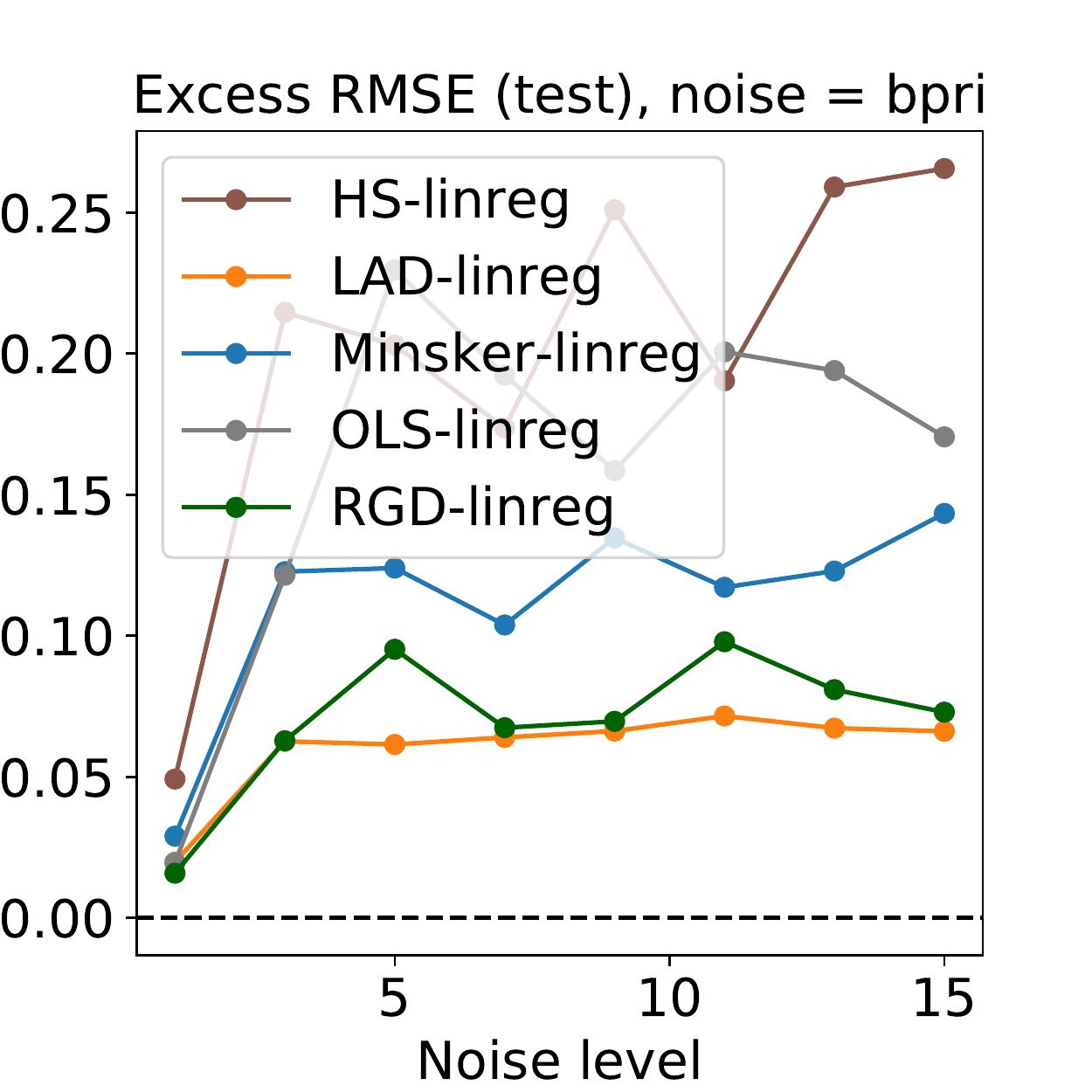}\,\includegraphics[width=0.25\textwidth]{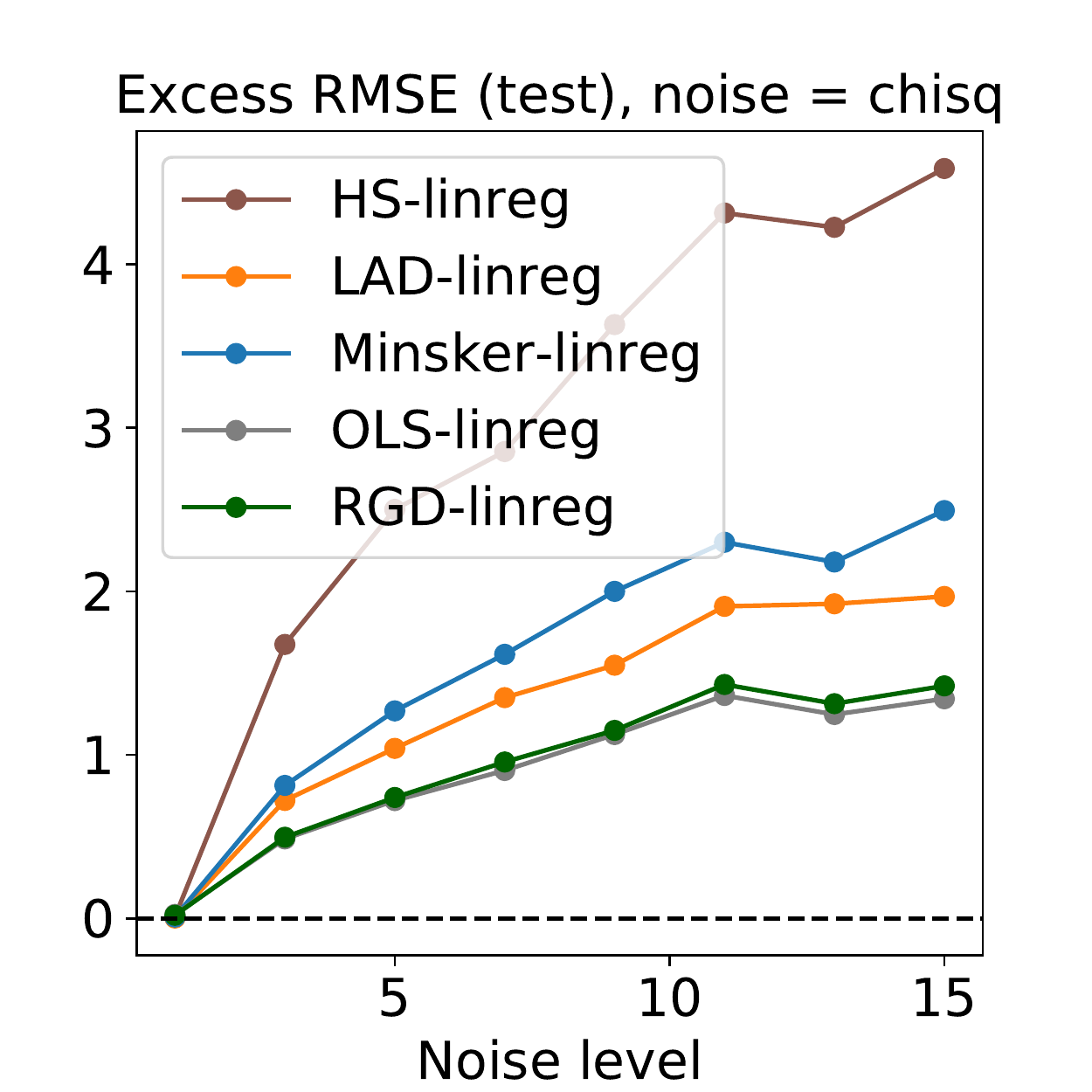}\,\includegraphics[width=0.25\textwidth]{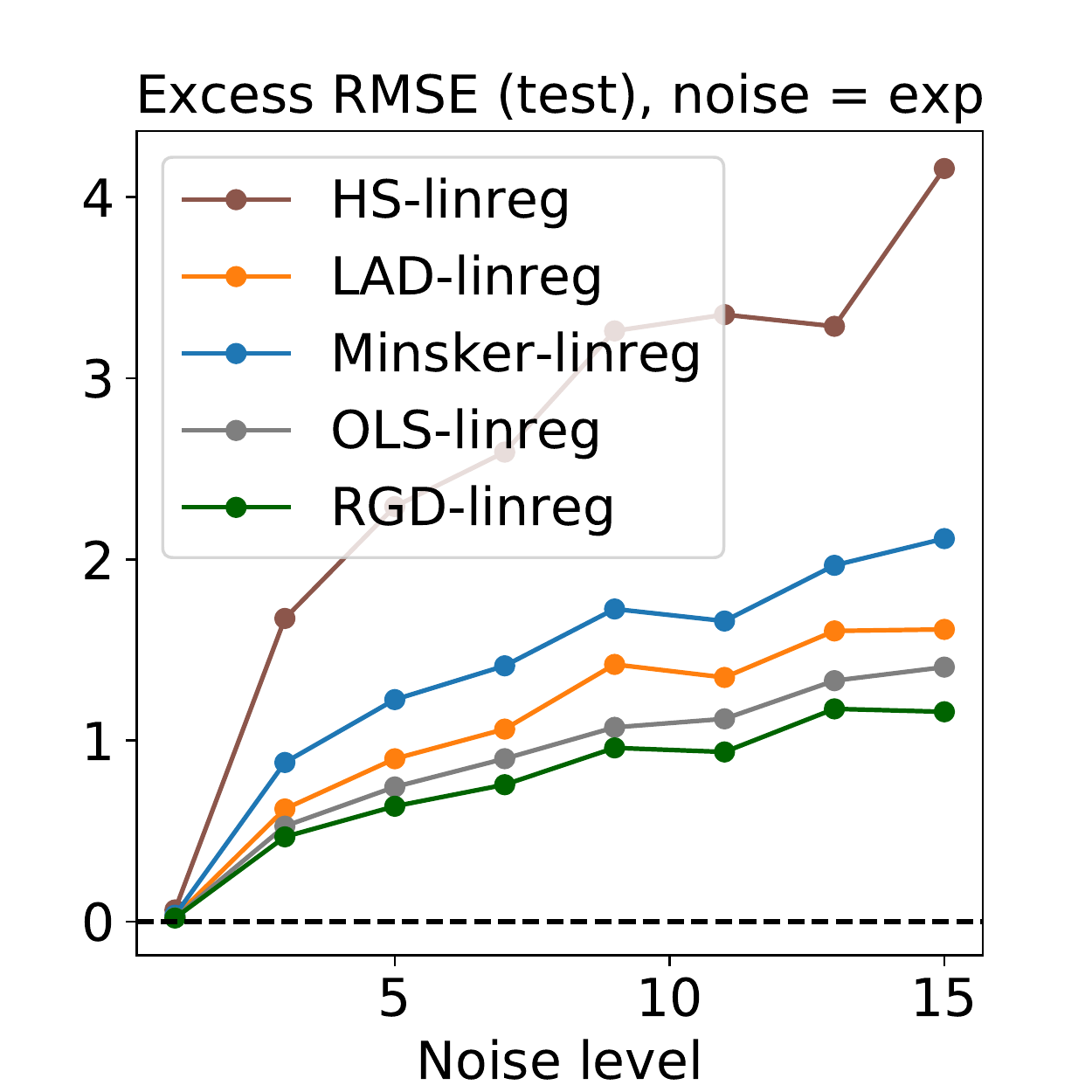}\\
\includegraphics[width=0.25\textwidth]{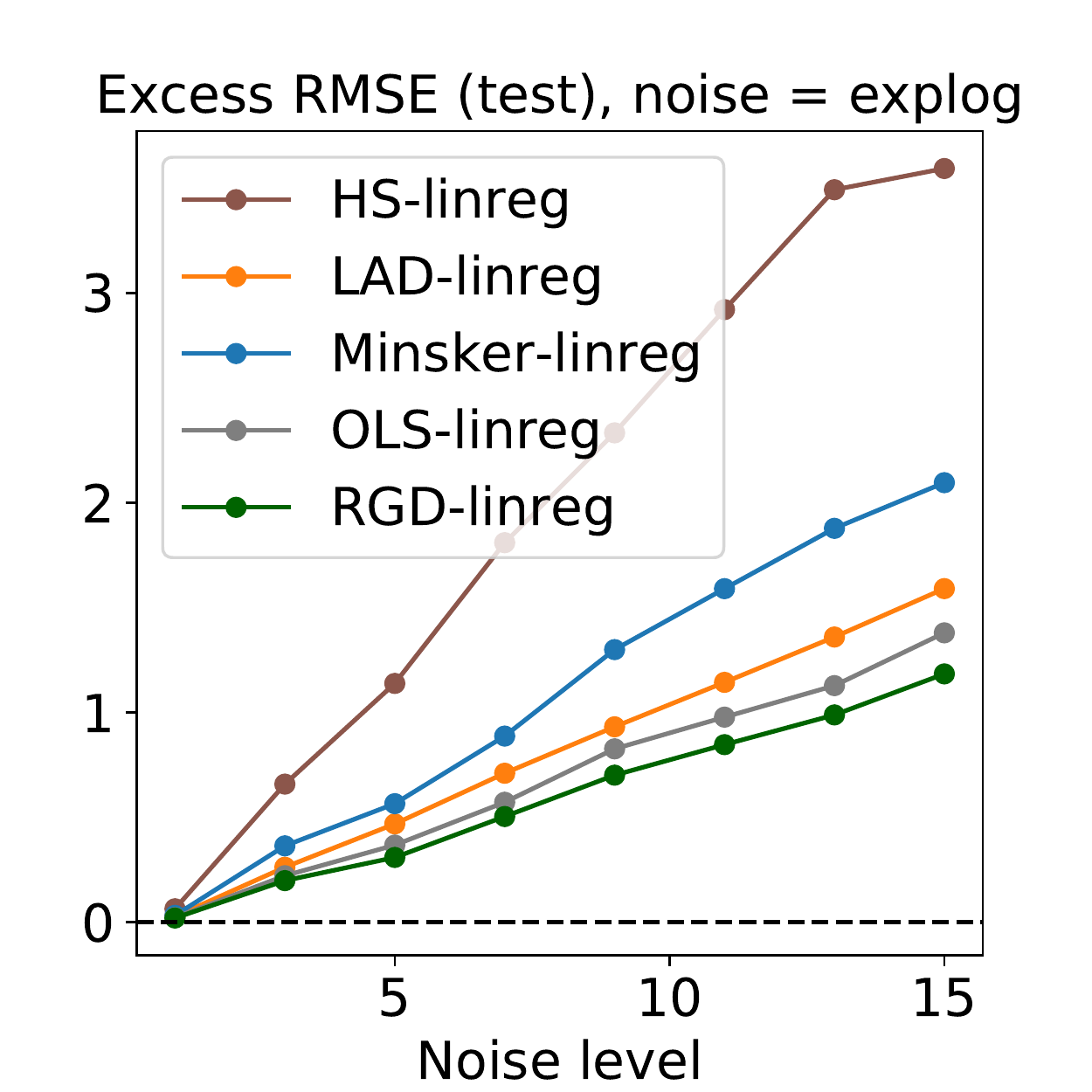}\,\includegraphics[width=0.25\textwidth]{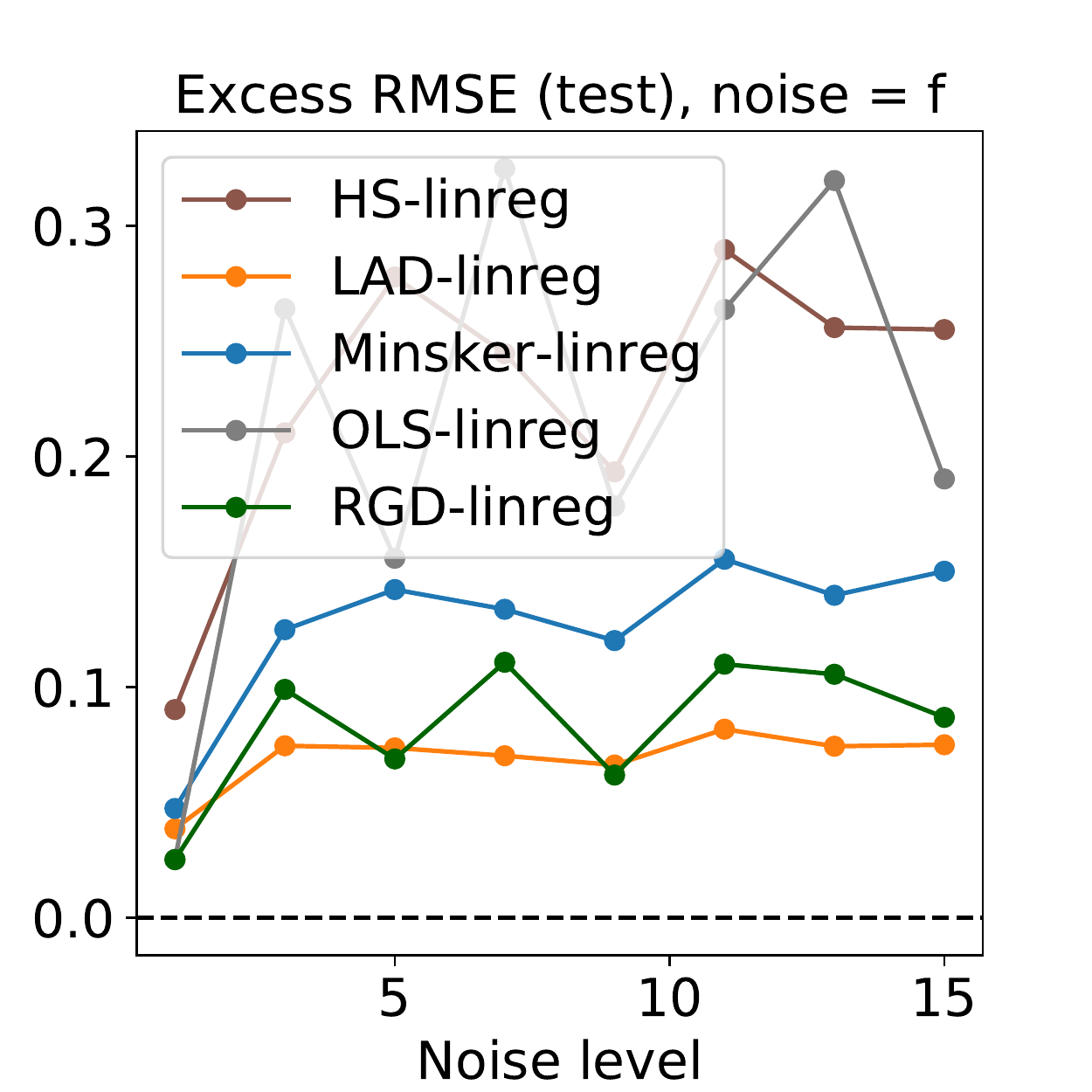}\,\includegraphics[width=0.25\textwidth]{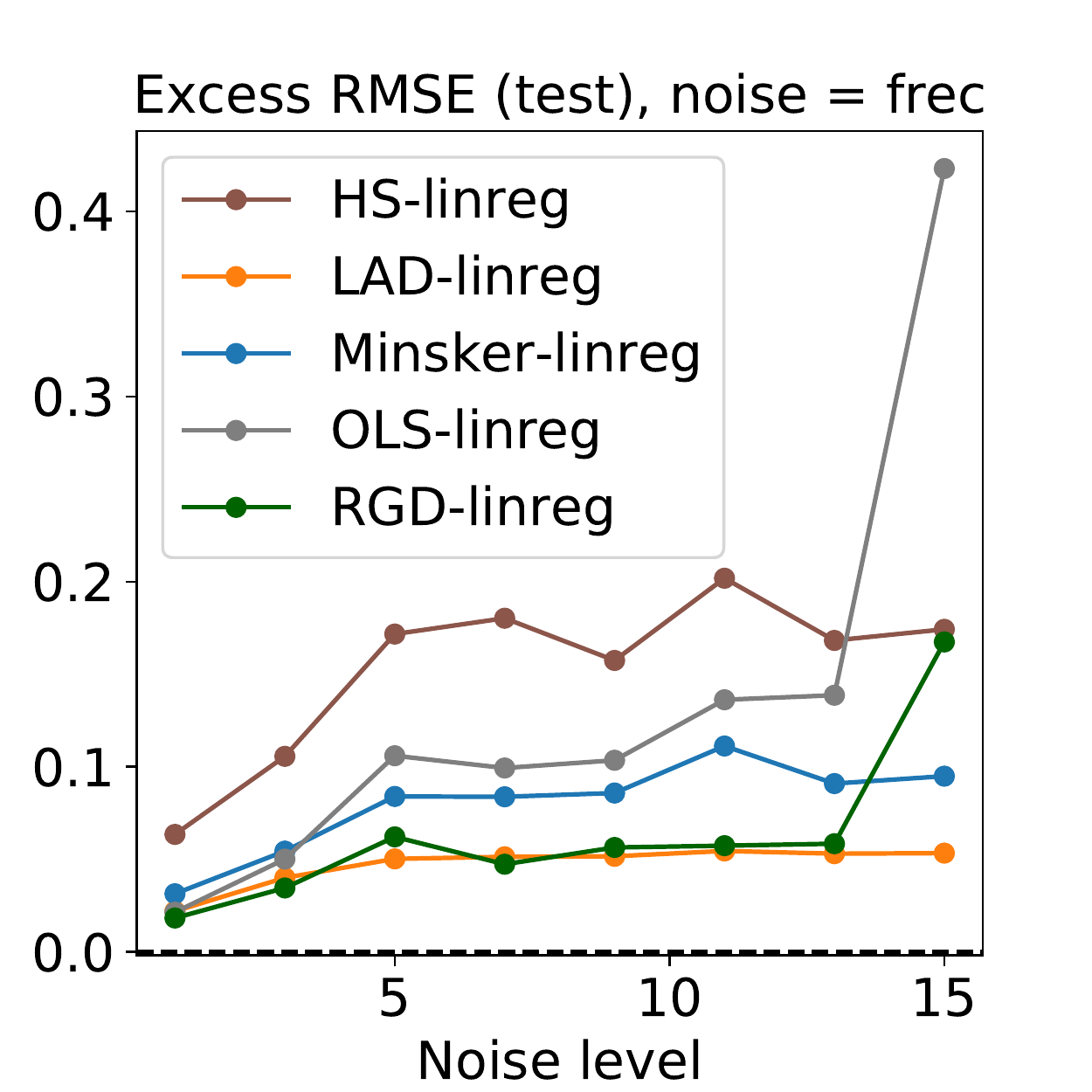}\,\includegraphics[width=0.25\textwidth]{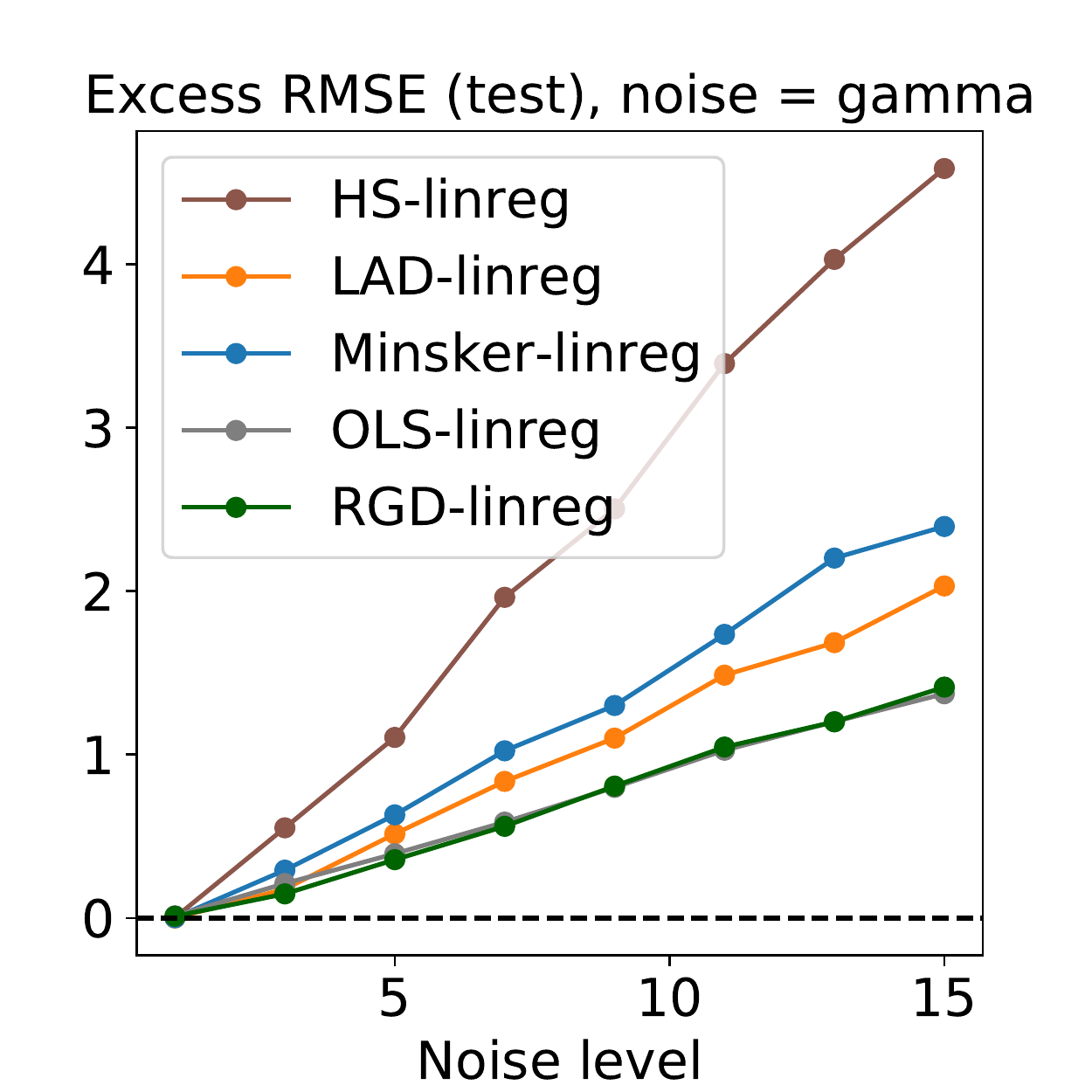}\\
\includegraphics[width=0.25\textwidth]{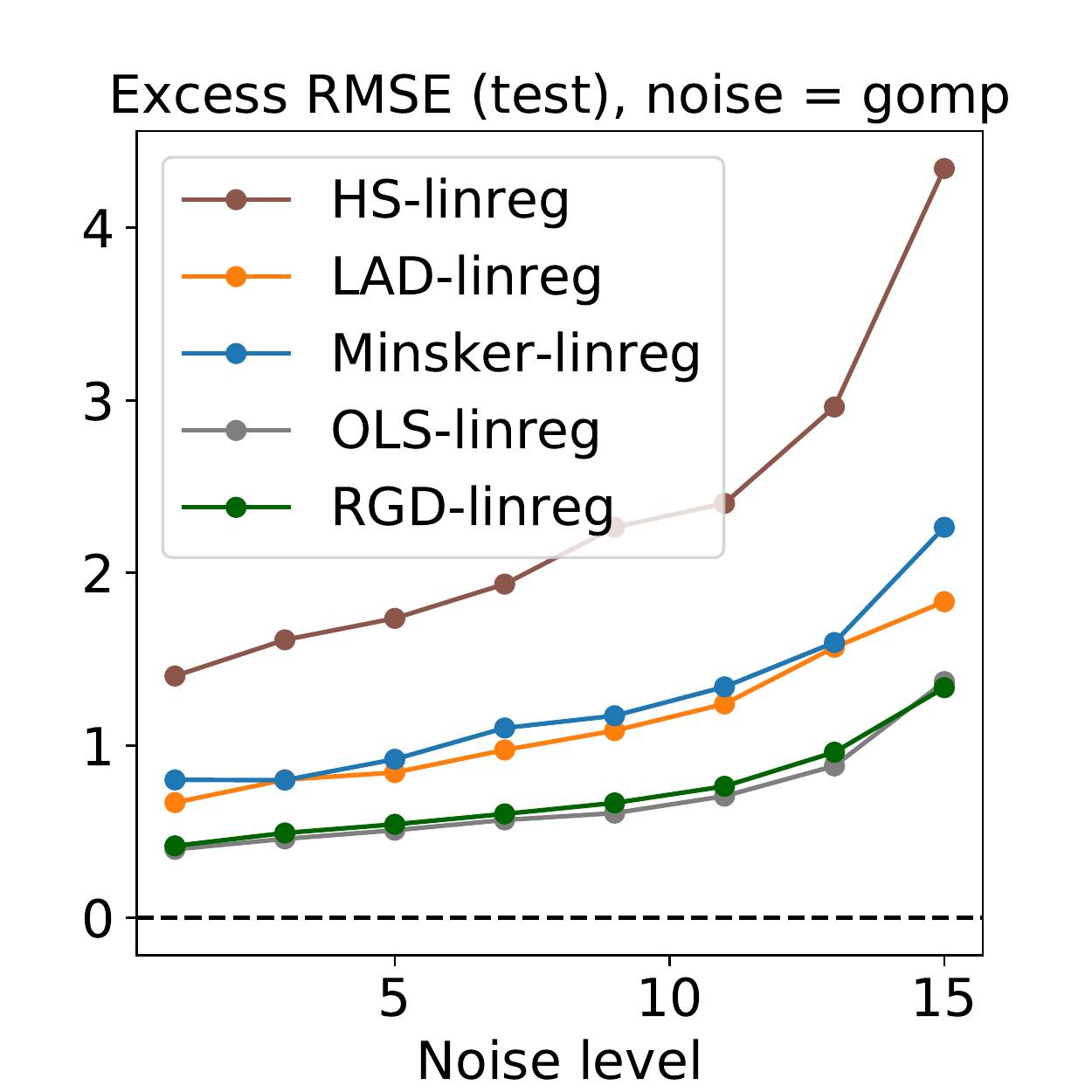}\,\includegraphics[width=0.25\textwidth]{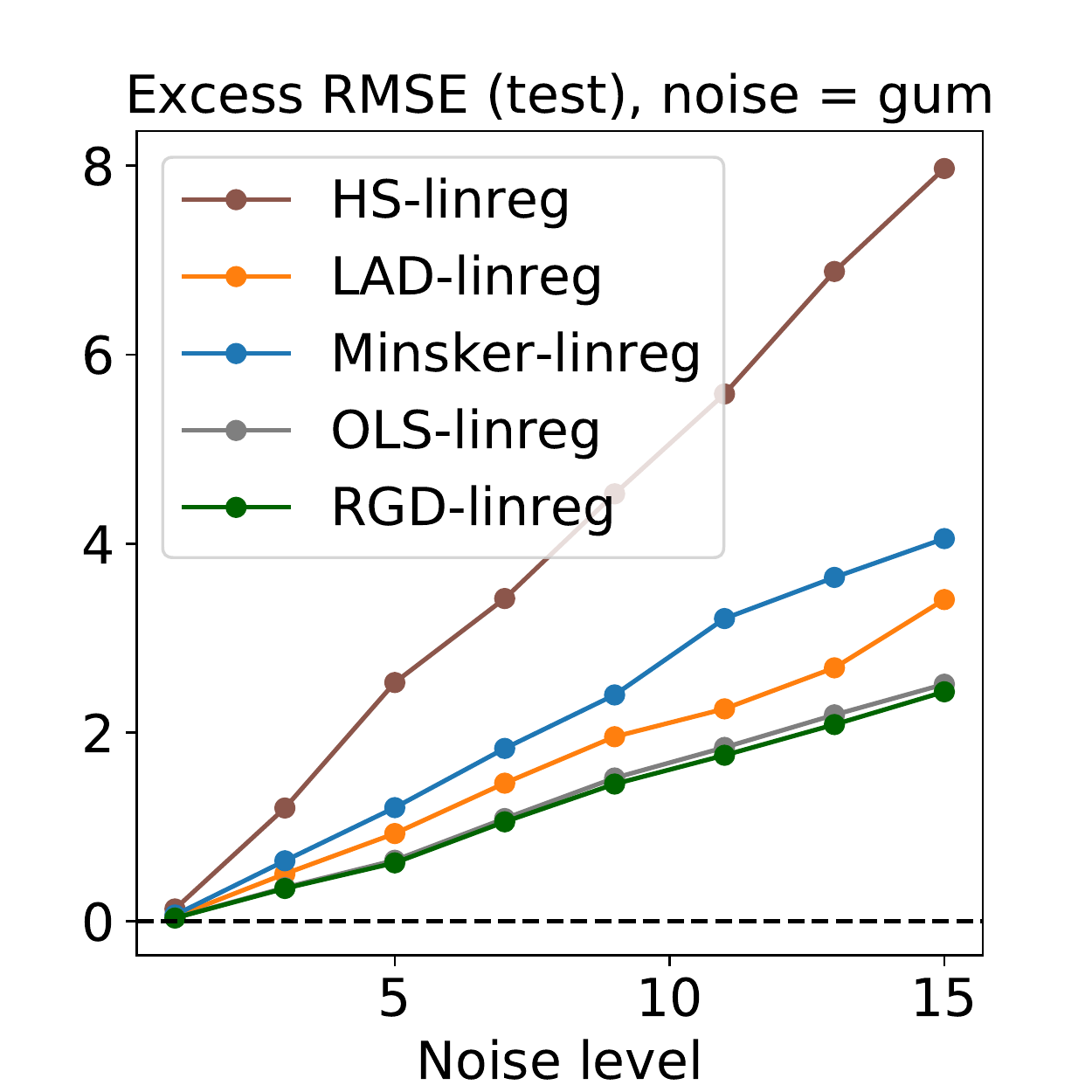}\,\includegraphics[width=0.25\textwidth]{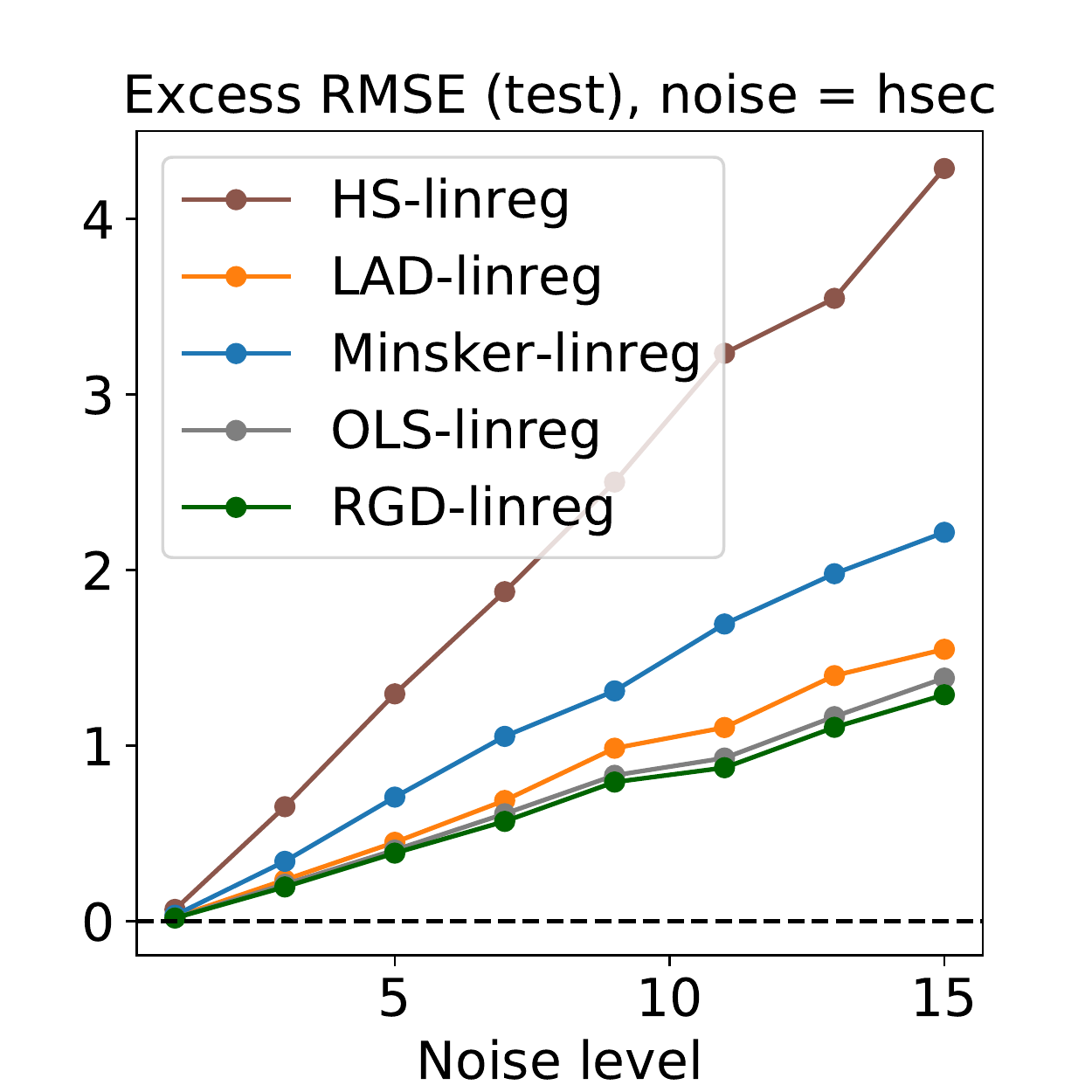}\,\includegraphics[width=0.25\textwidth]{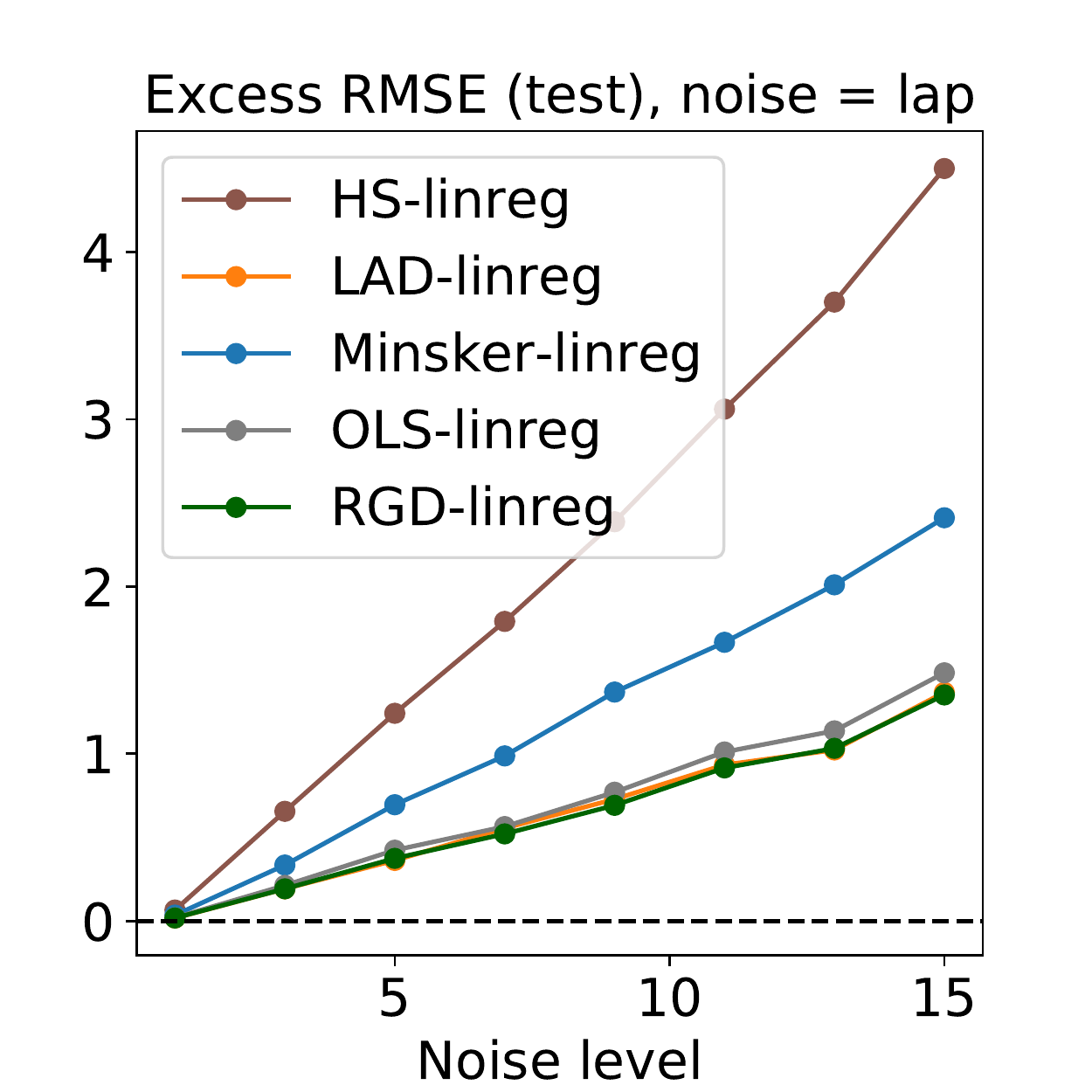}\\
\includegraphics[width=0.25\textwidth]{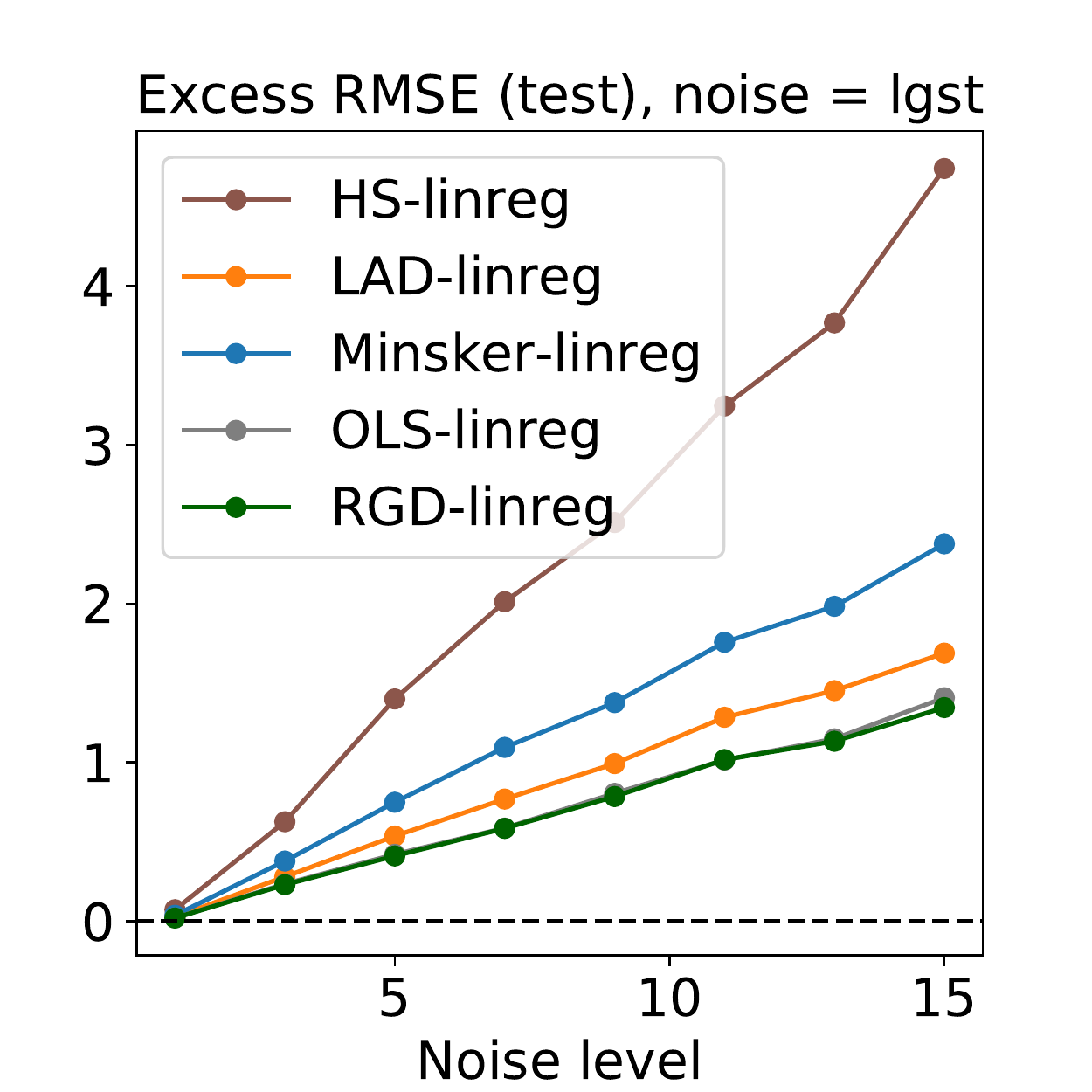}\,\includegraphics[width=0.25\textwidth]{linreg_overLvl_risk_llog}\,\includegraphics[width=0.25\textwidth]{linreg_overLvl_risk_lnorm}\,\includegraphics[width=0.25\textwidth]{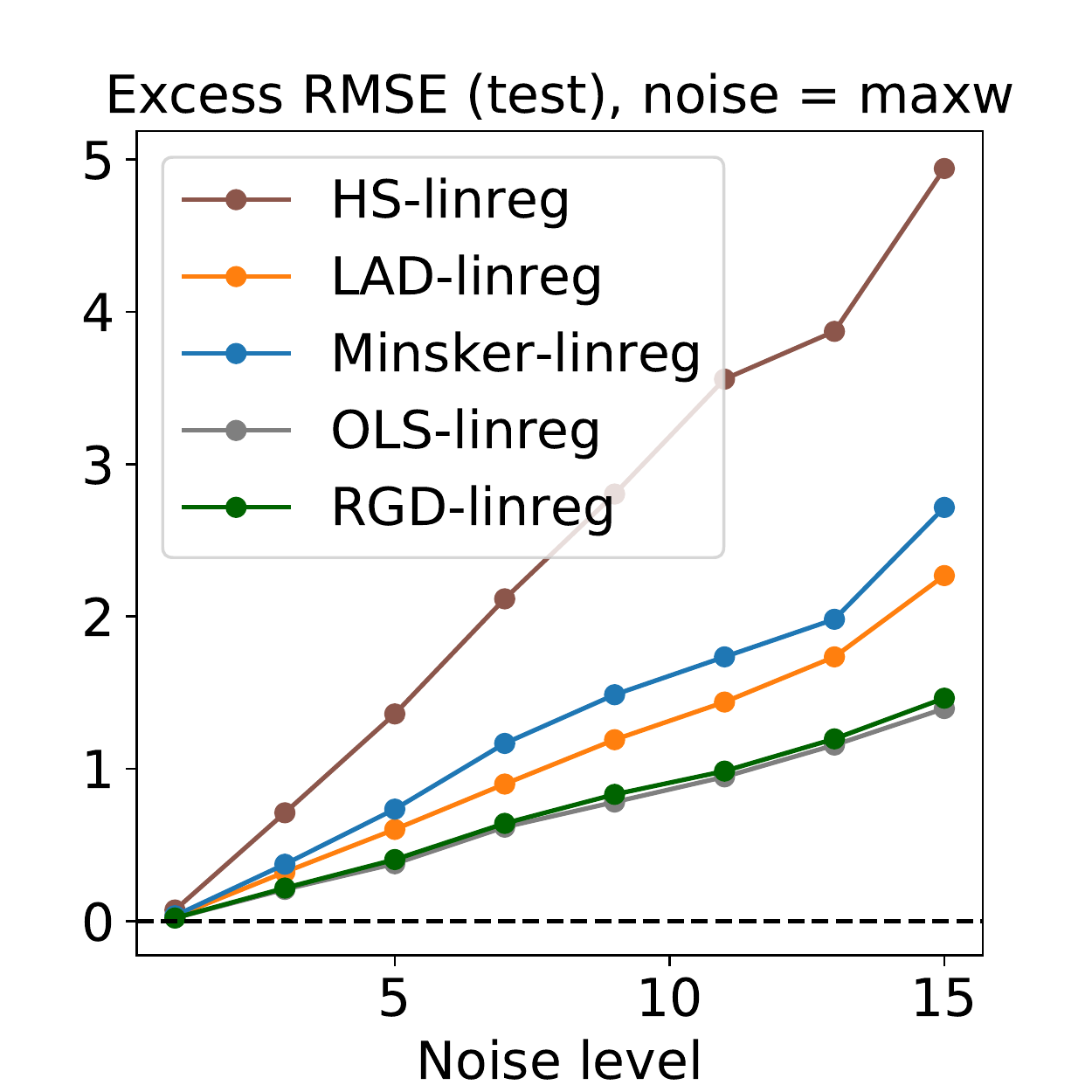}\\
\includegraphics[width=0.25\textwidth]{linreg_overLvl_risk_norm}\,\includegraphics[width=0.25\textwidth]{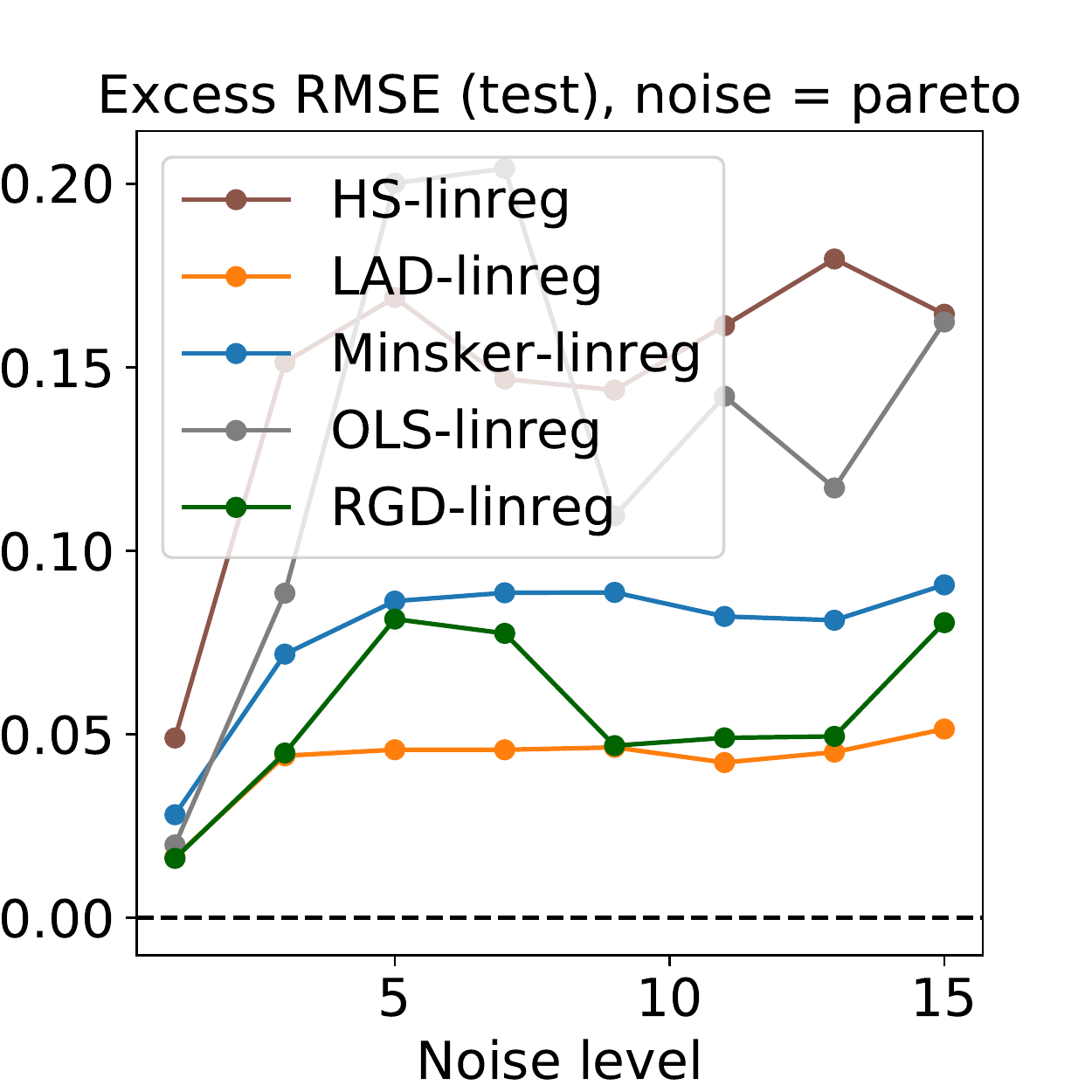}\,\includegraphics[width=0.25\textwidth]{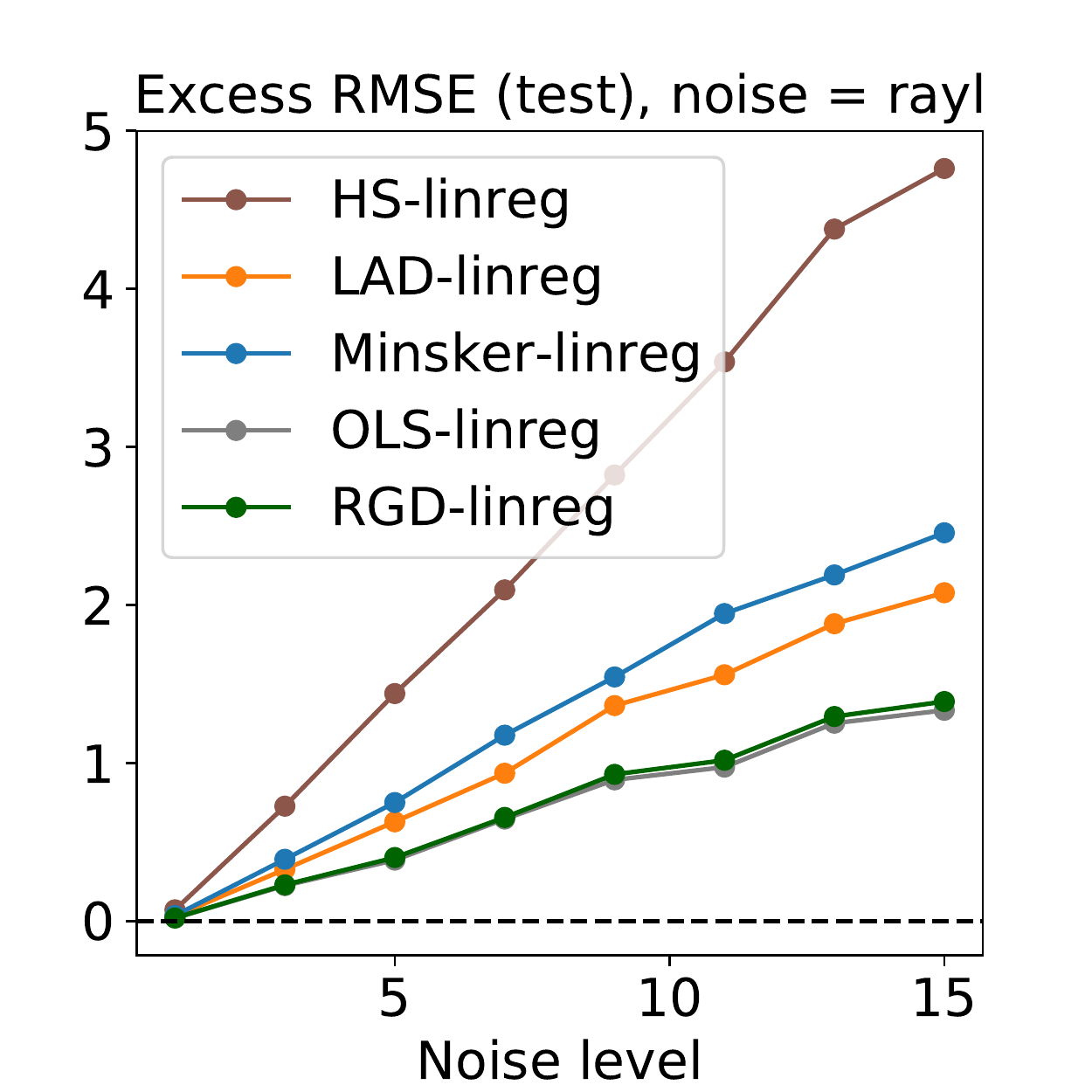}\,\includegraphics[width=0.25\textwidth]{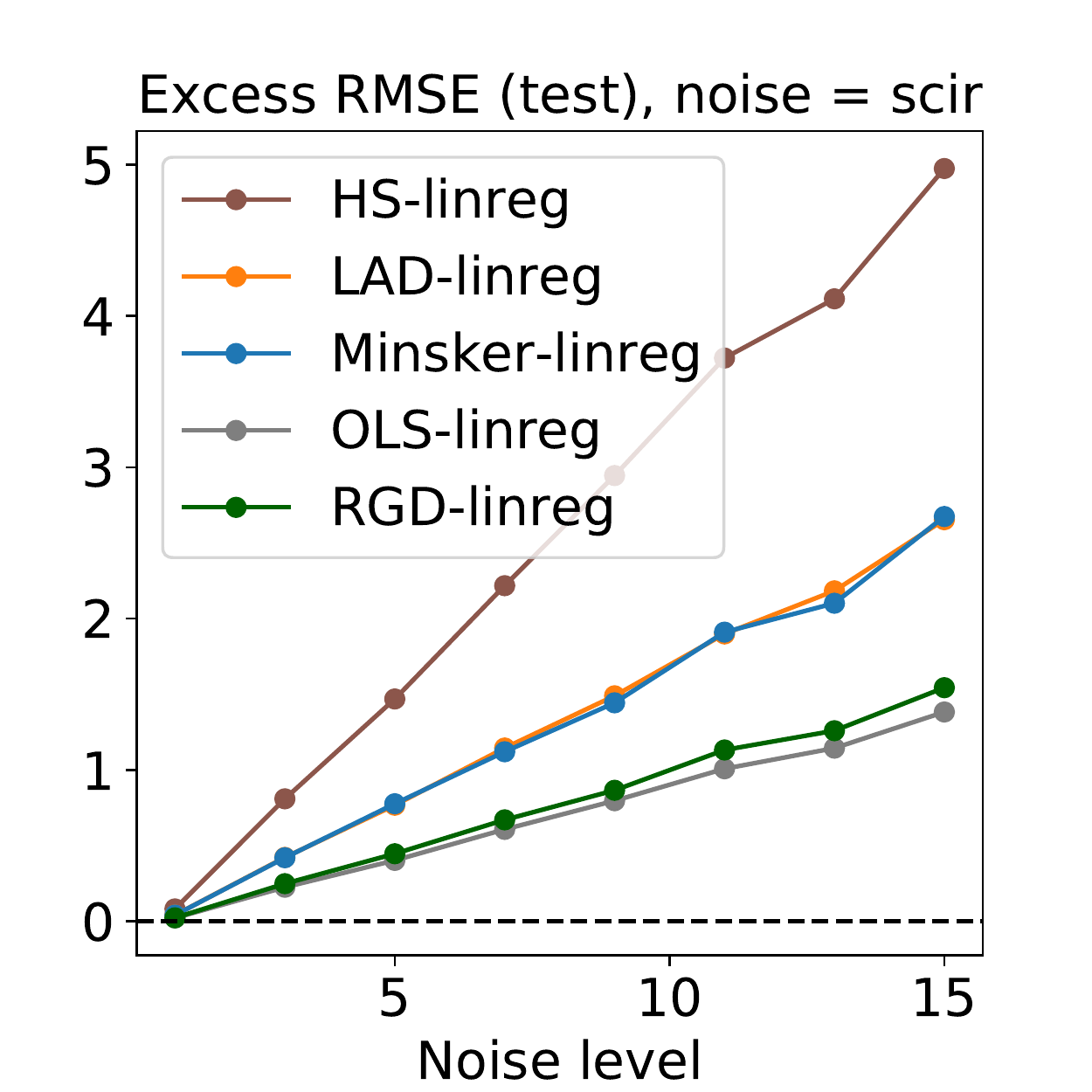}
\caption{Prediction error over noise levels, for $n=30, d=5$. Each plot corresponds to a distinct noise distribution.}
\label{fig:overLvl_all_distros_1}
\end{figure}

\clearpage

\begin{figure}[t]
\centering
\includegraphics[width=0.25\textwidth]{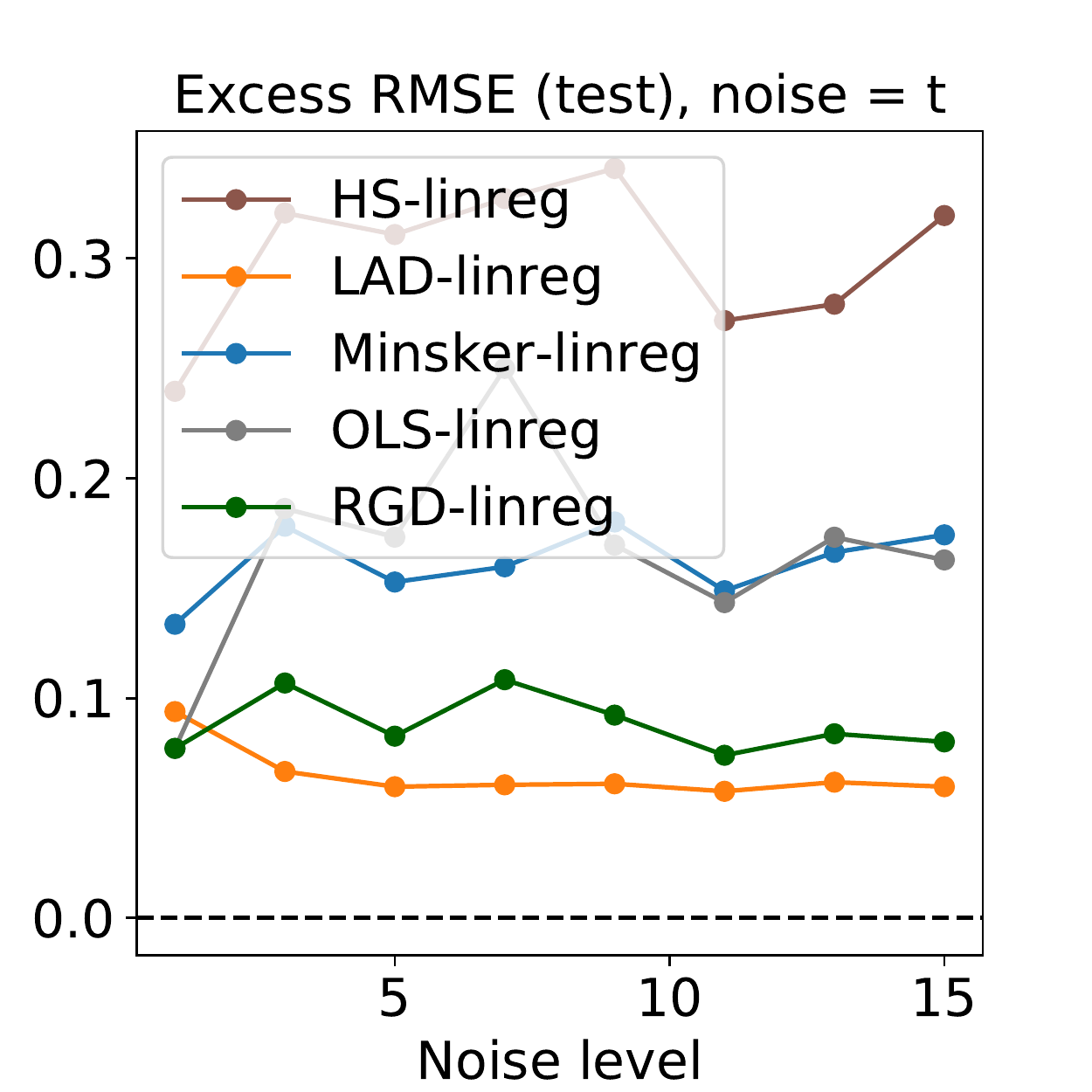}\,\includegraphics[width=0.25\textwidth]{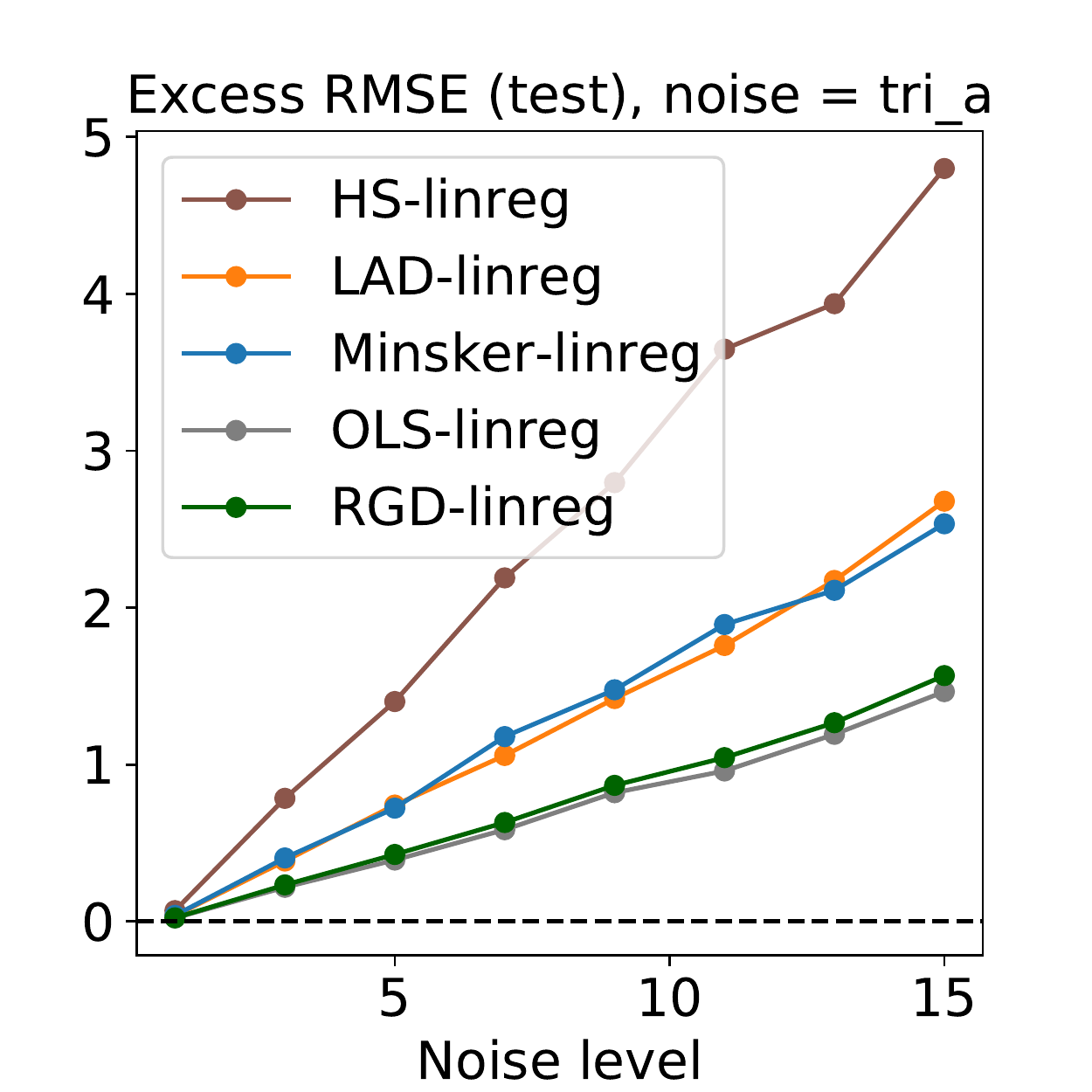}\,\includegraphics[width=0.25\textwidth]{linreg_overLvl_risk_tri_s}\,\includegraphics[width=0.25\textwidth]{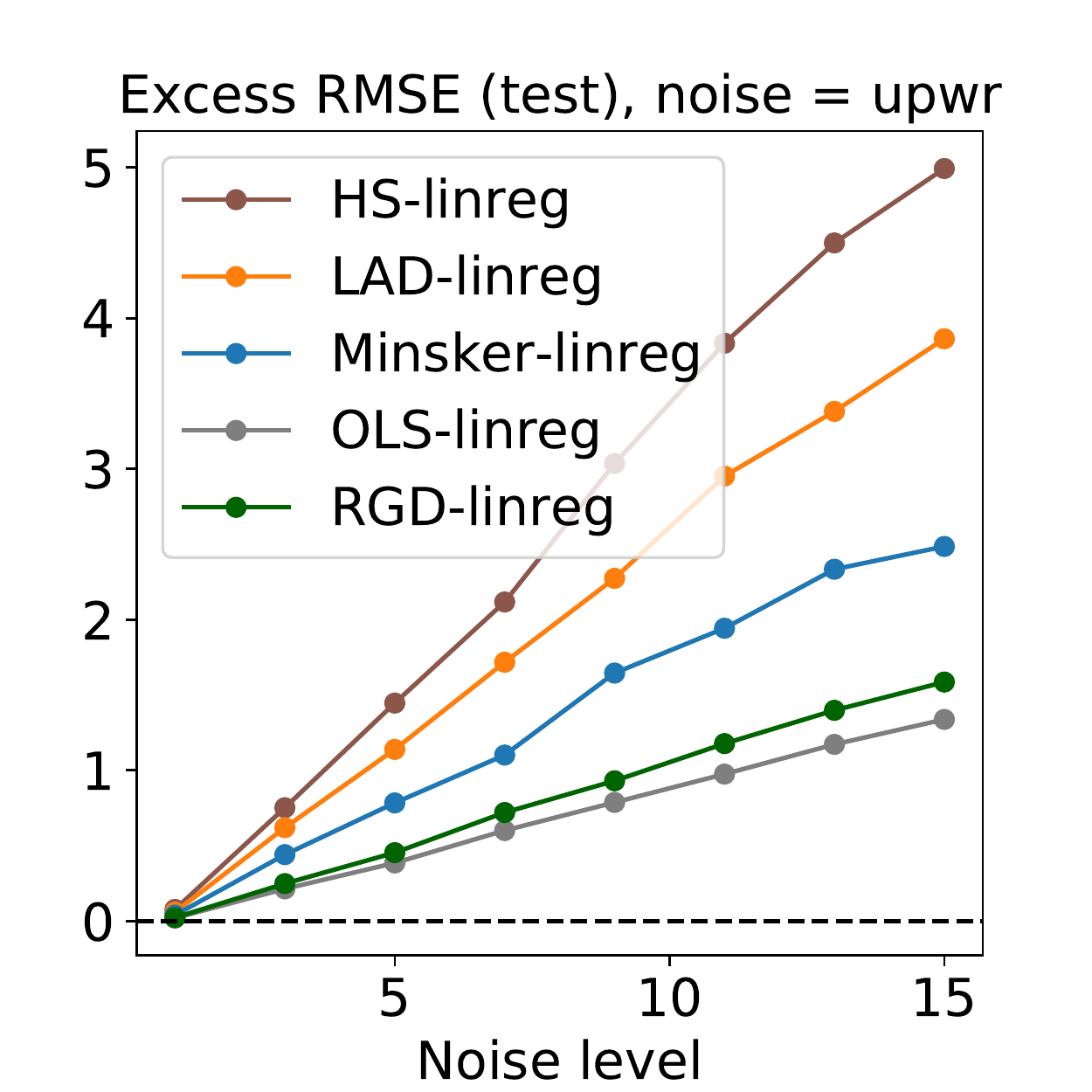}\\
\includegraphics[width=0.25\textwidth]{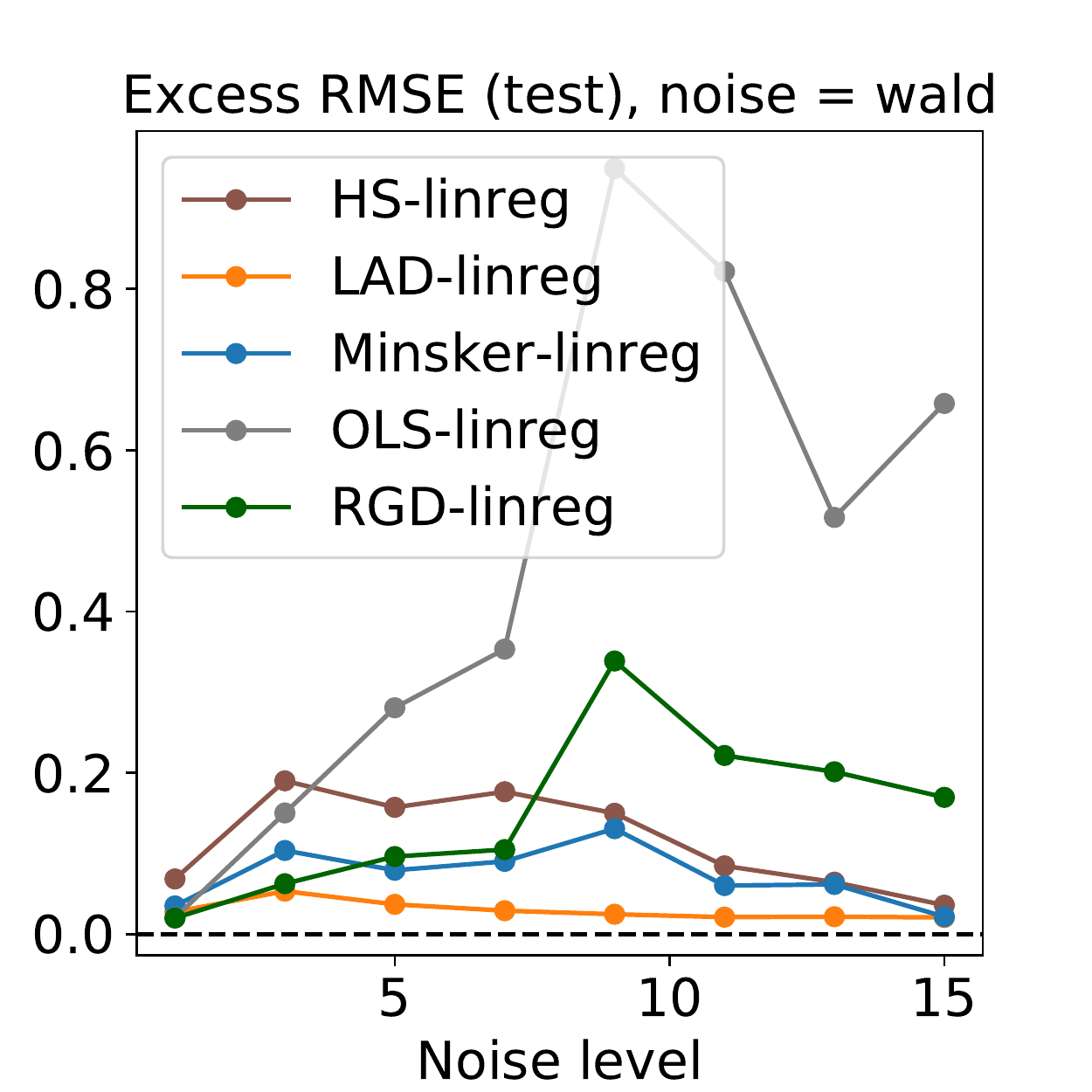}\,\includegraphics[width=0.25\textwidth]{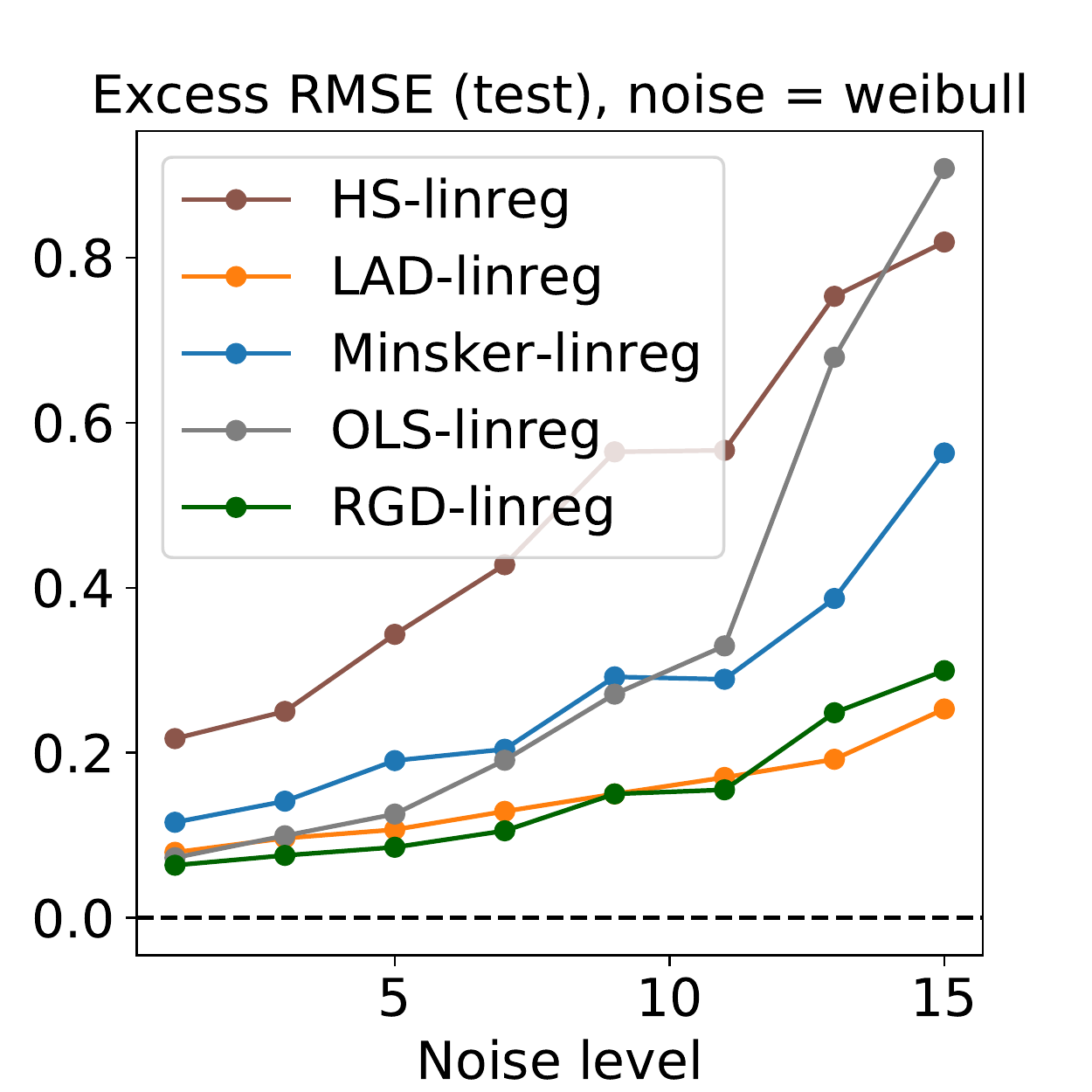}
\caption{Prediction error over noise levels, for $n=30, d=5$. Each plot corresponds to a distinct noise distribution.}
\label{fig:overLvl_all_distros_2}
\end{figure}

\clearpage

\begin{figure}[t]
\centering
\includegraphics[width=0.25\textwidth]{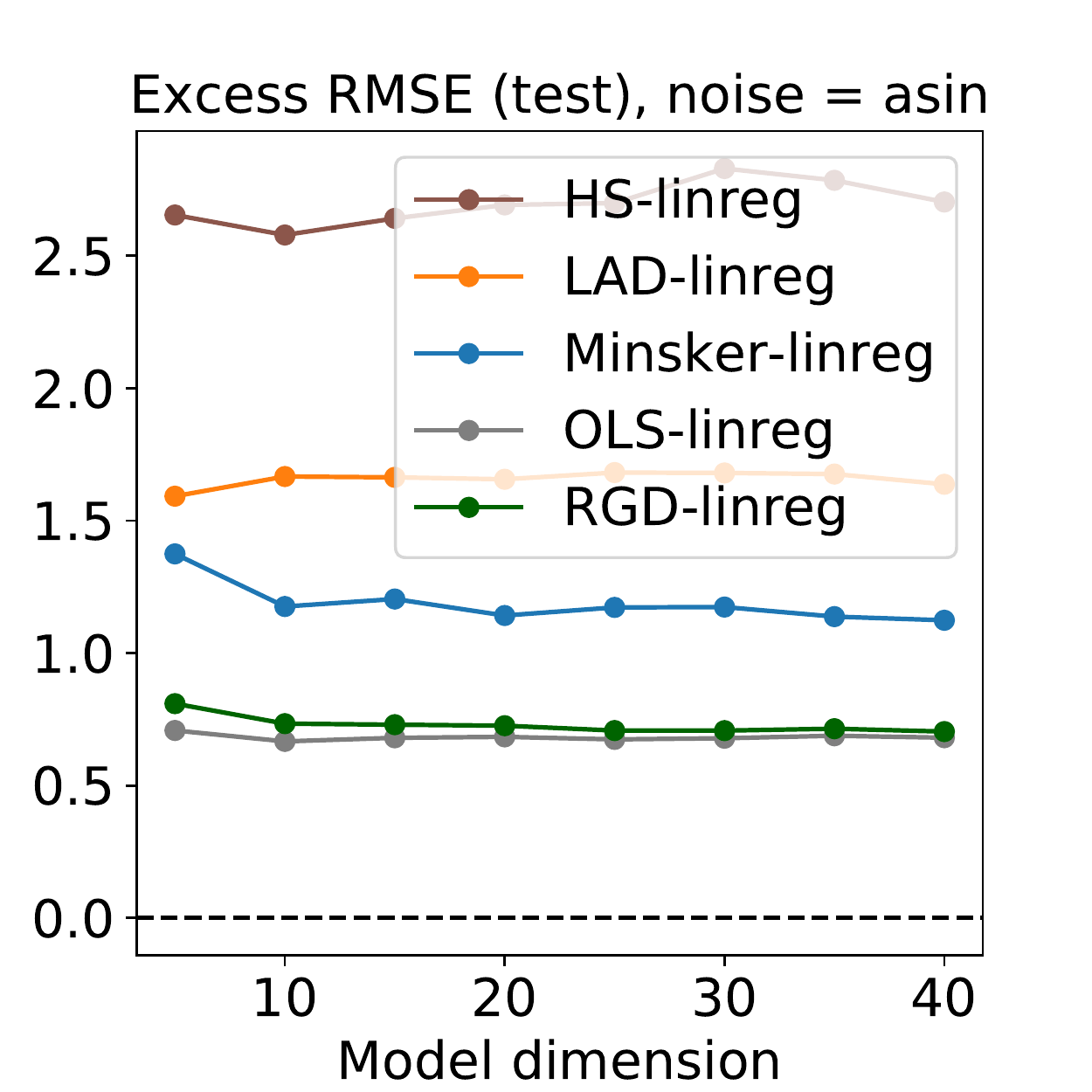}\,\includegraphics[width=0.25\textwidth]{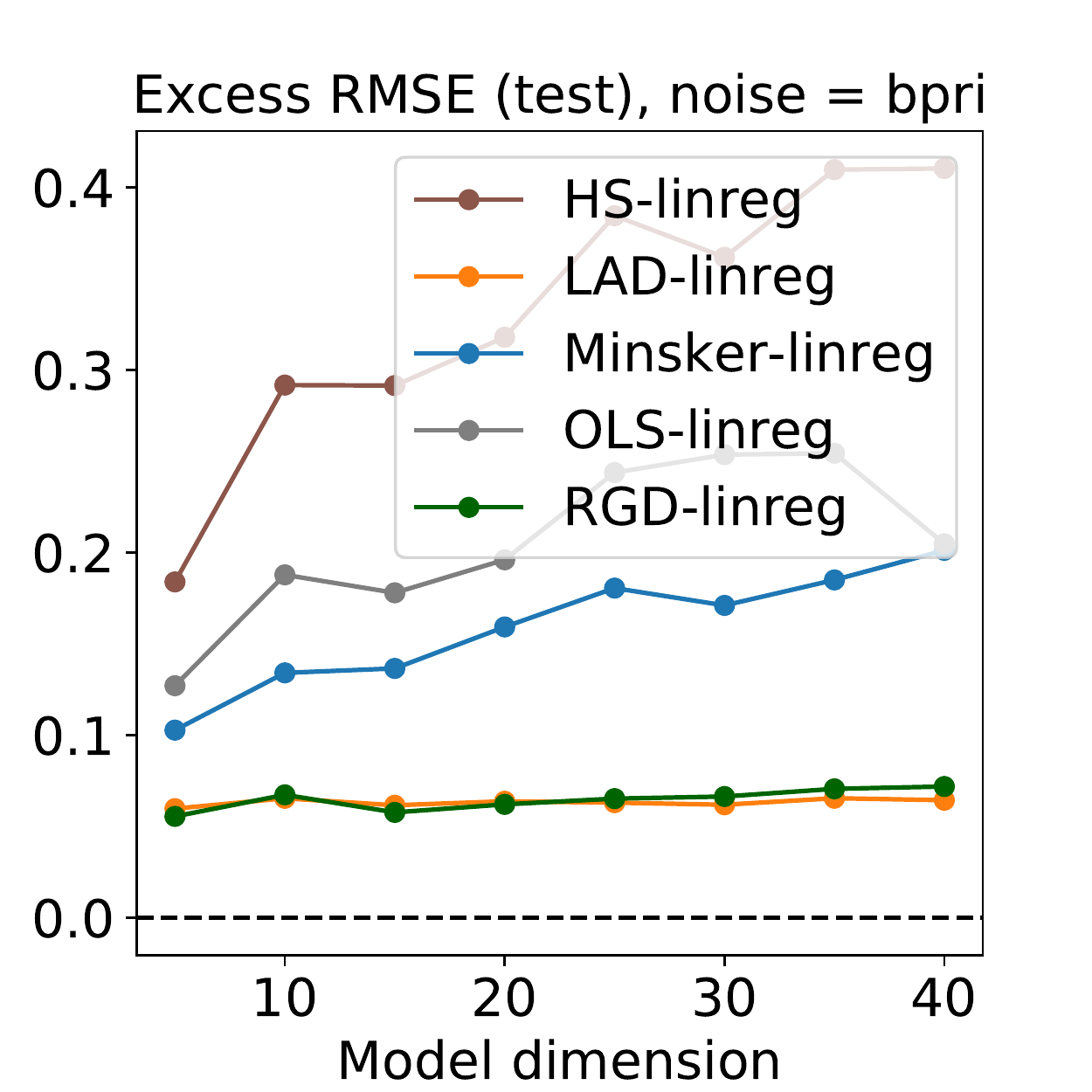}\,\includegraphics[width=0.25\textwidth]{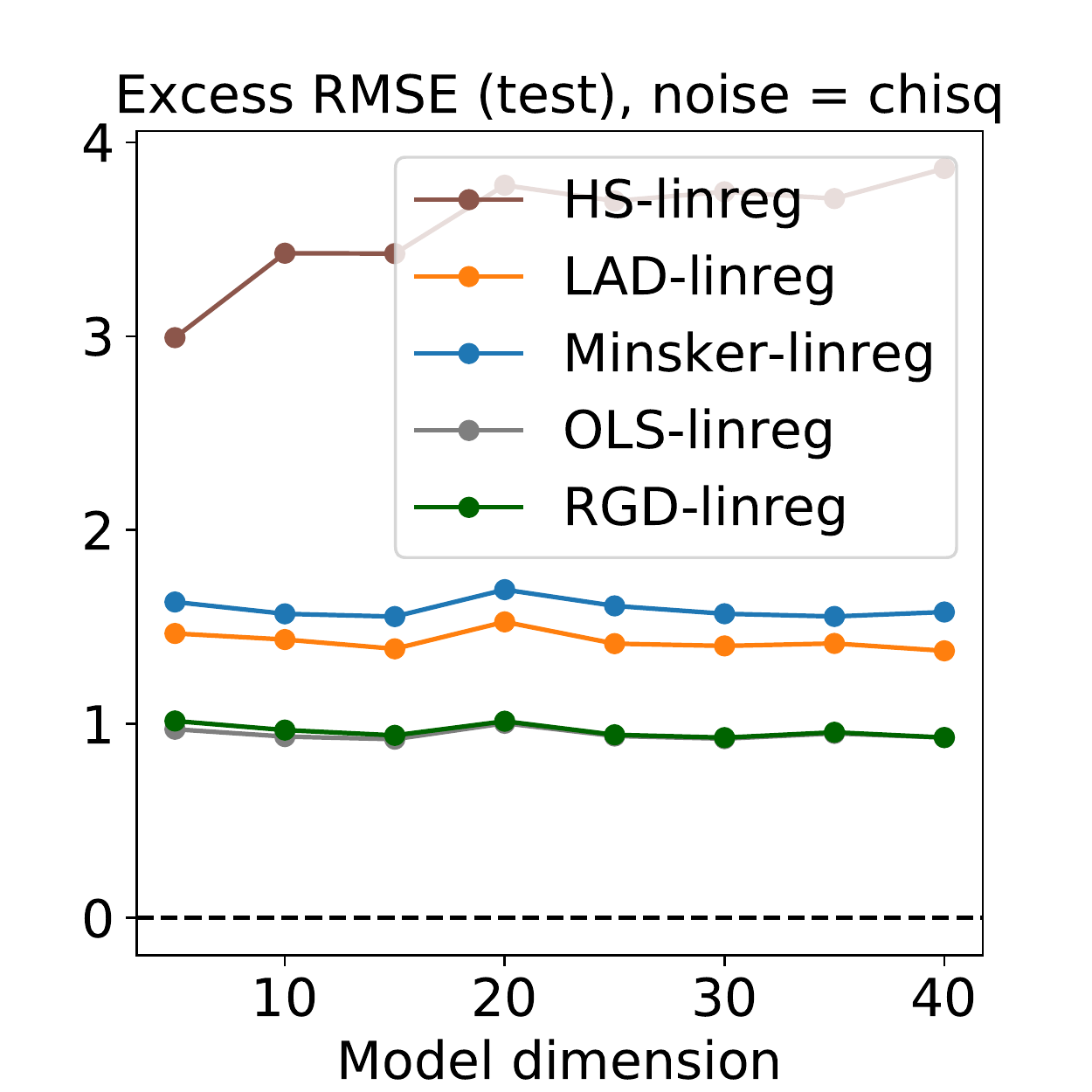}\,\includegraphics[width=0.25\textwidth]{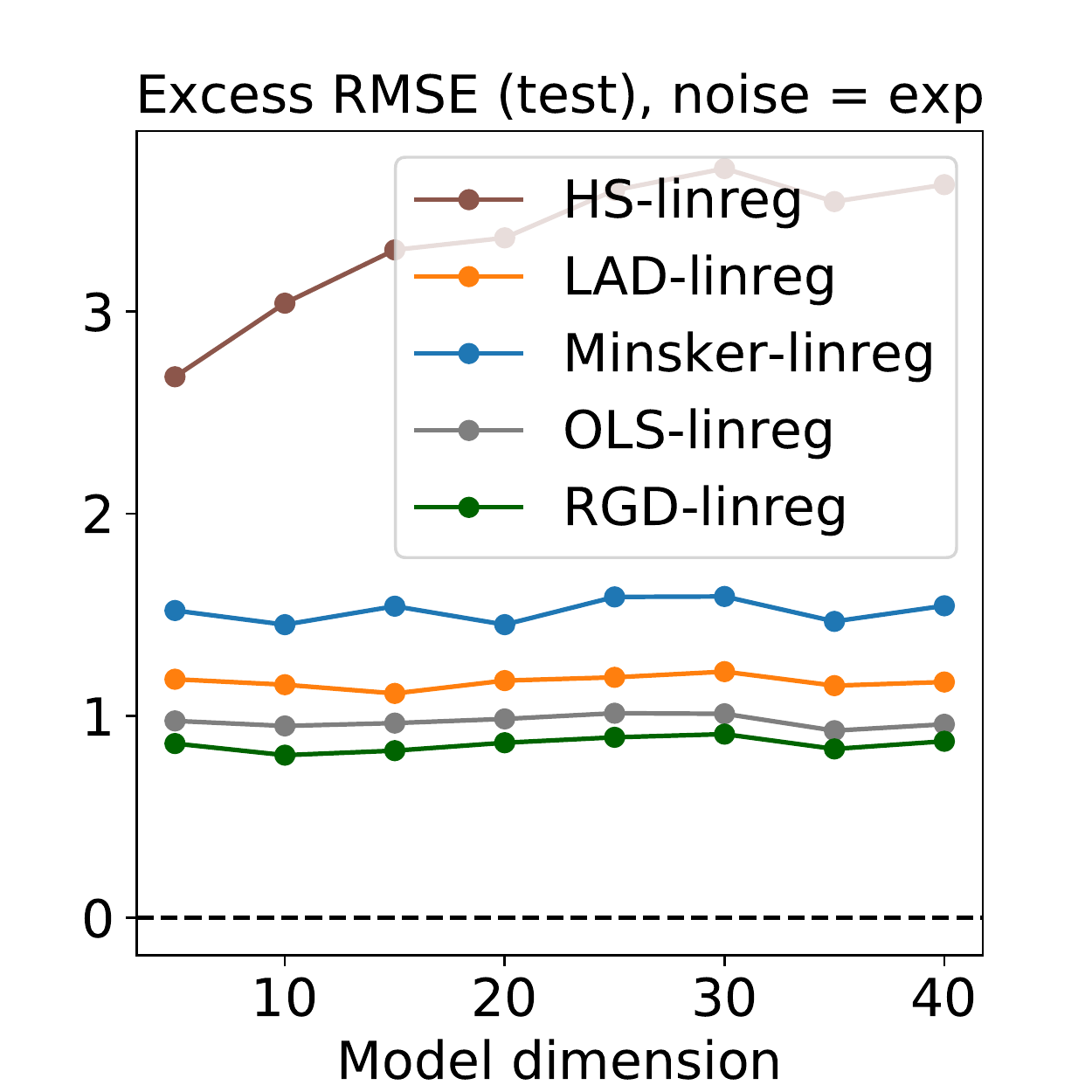}\\
\includegraphics[width=0.25\textwidth]{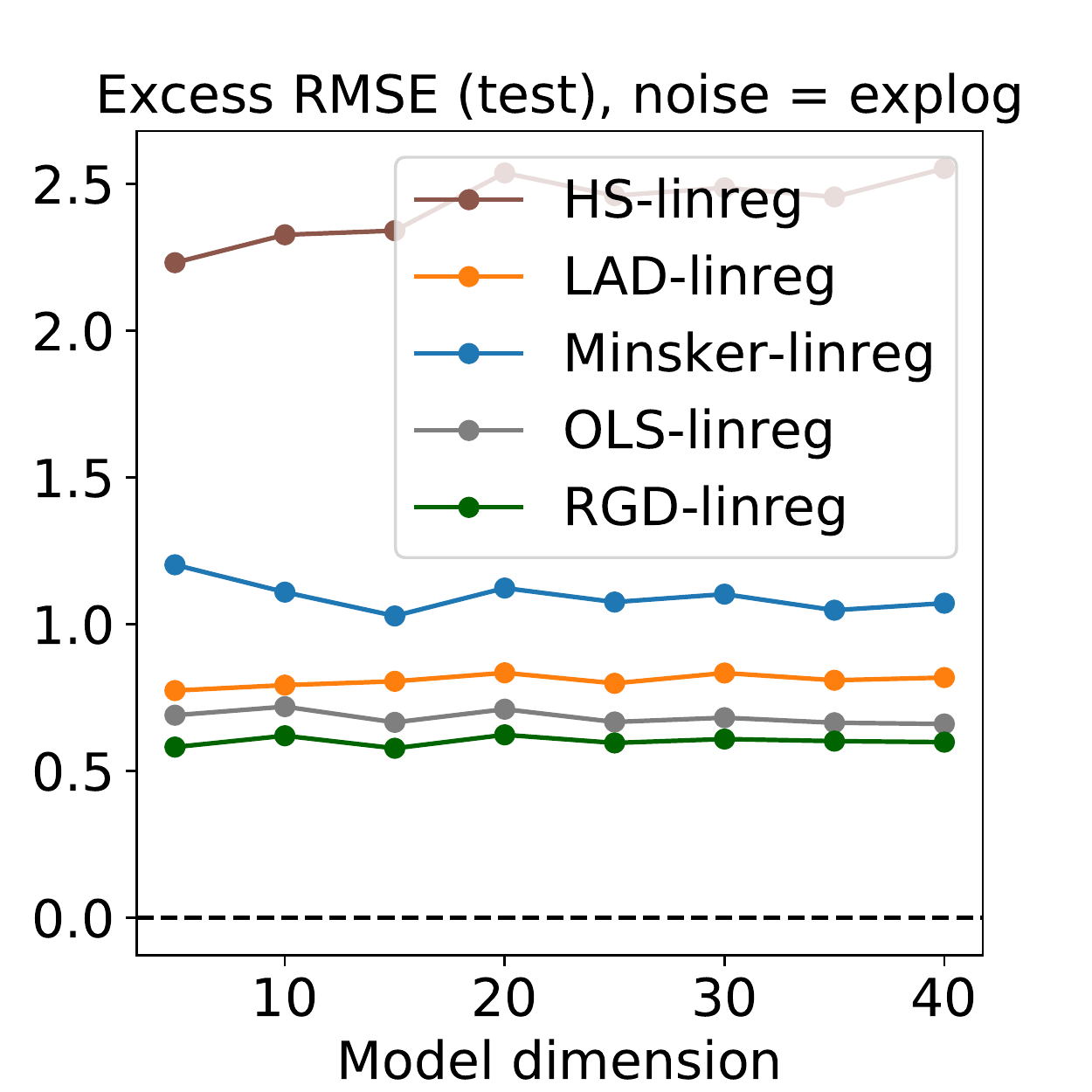}\,\includegraphics[width=0.25\textwidth]{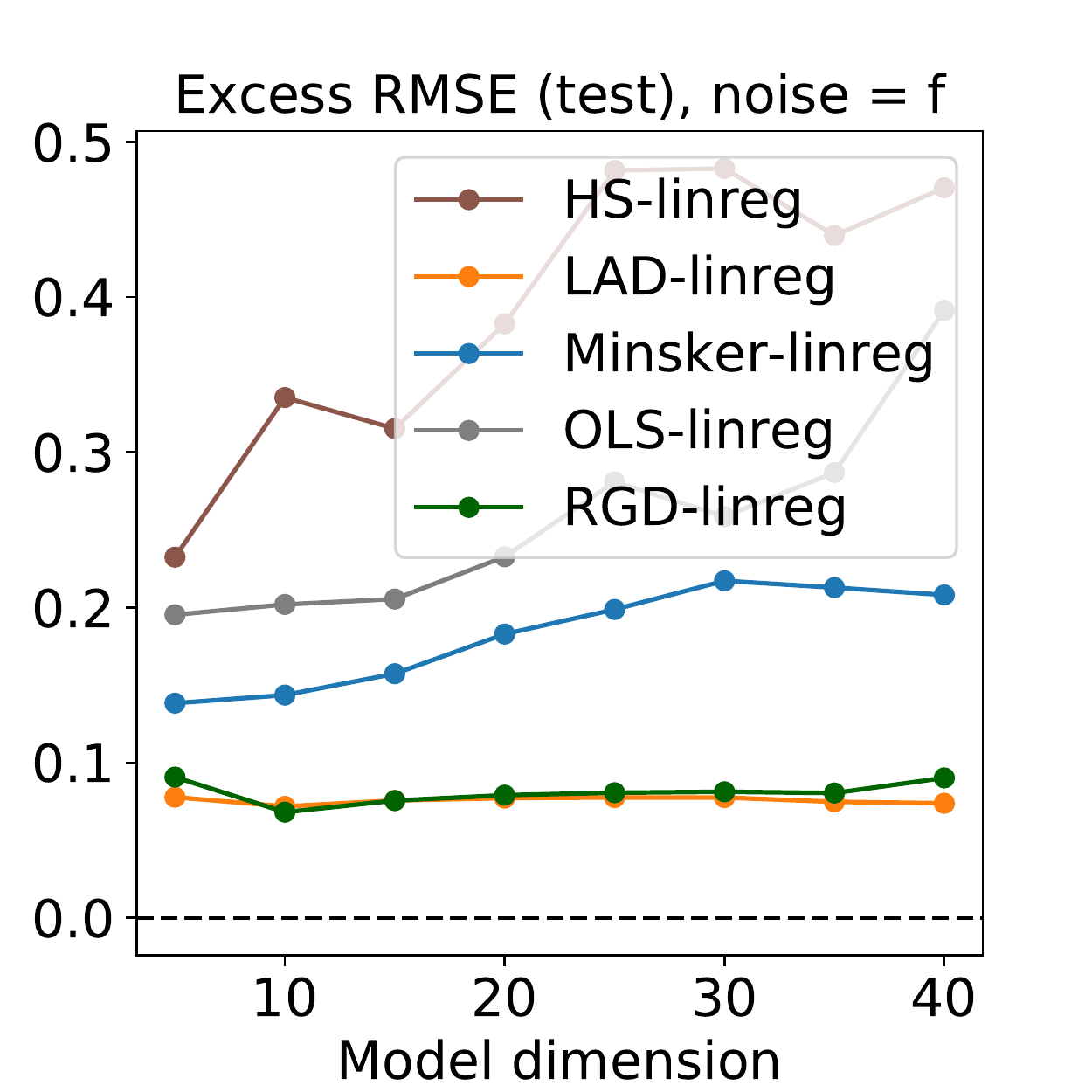}\,\includegraphics[width=0.25\textwidth]{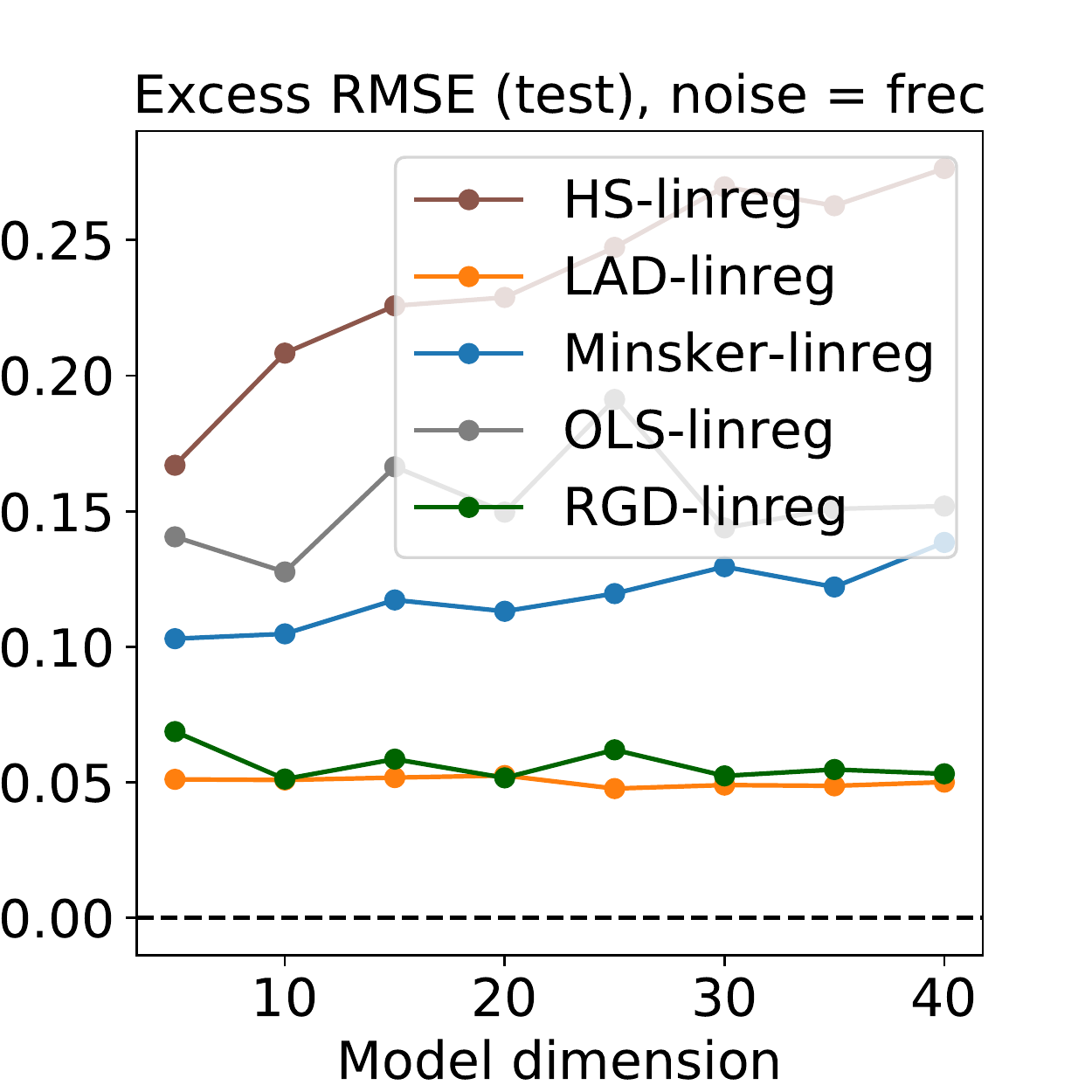}\,\includegraphics[width=0.25\textwidth]{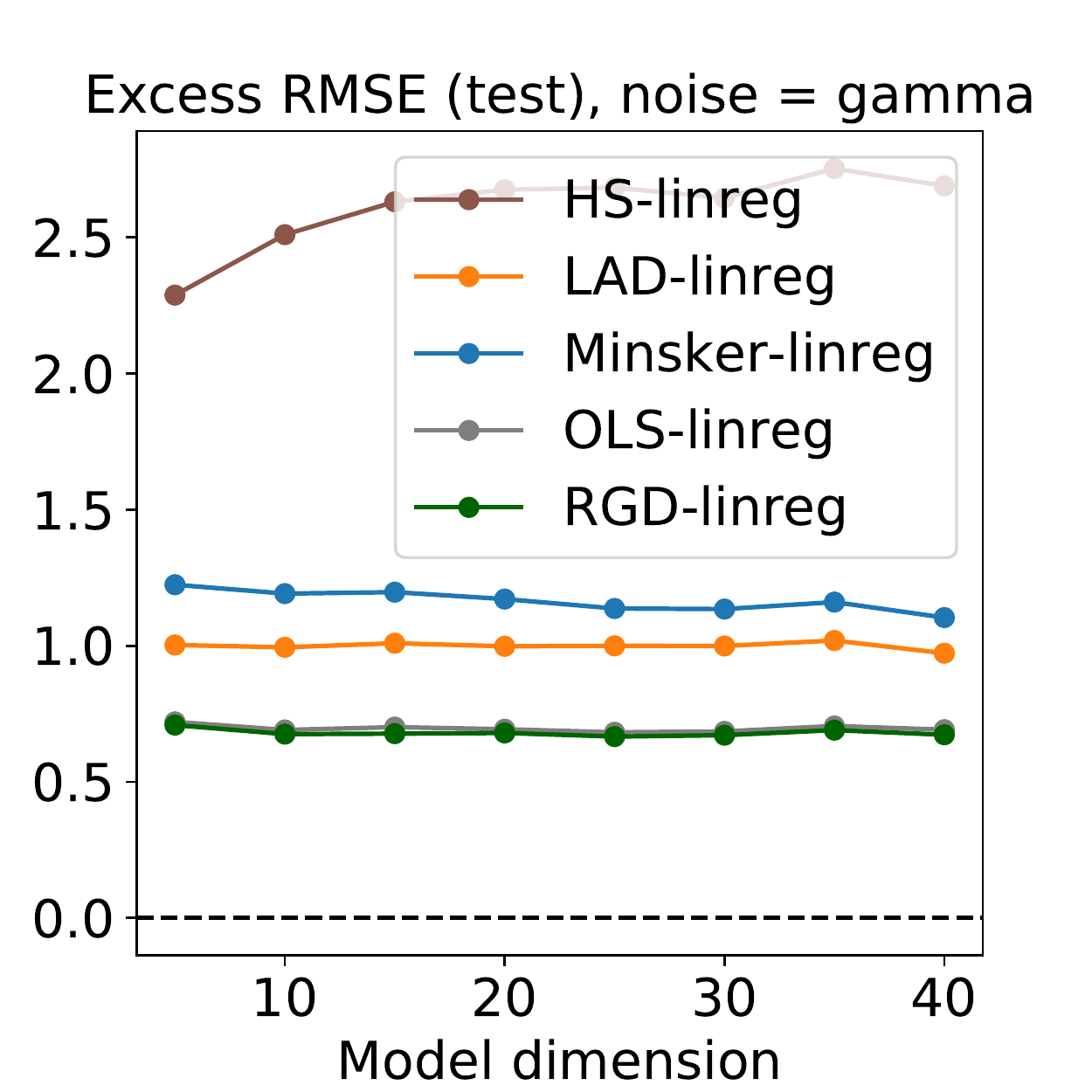}\\
\includegraphics[width=0.25\textwidth]{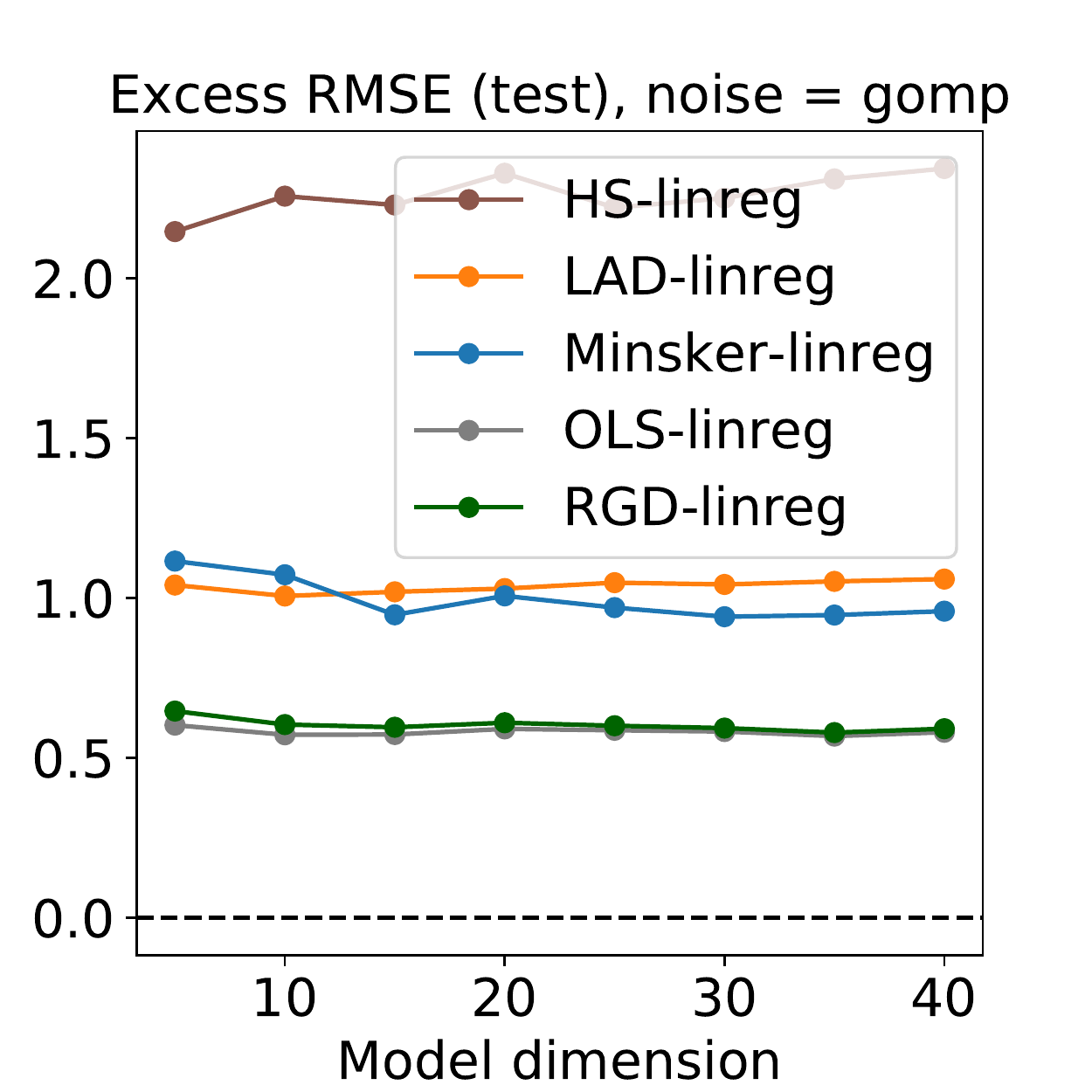}\,\includegraphics[width=0.25\textwidth]{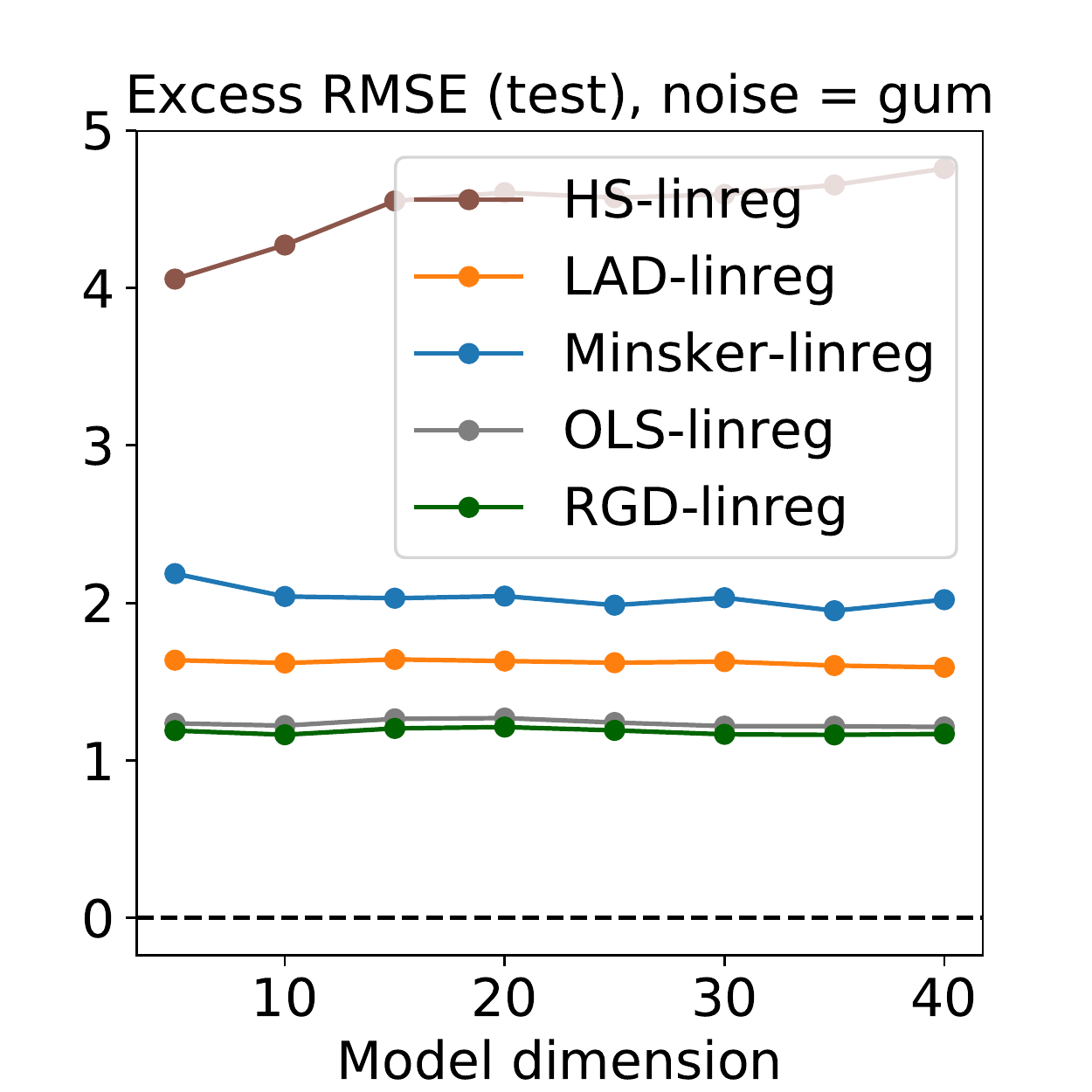}\,\includegraphics[width=0.25\textwidth]{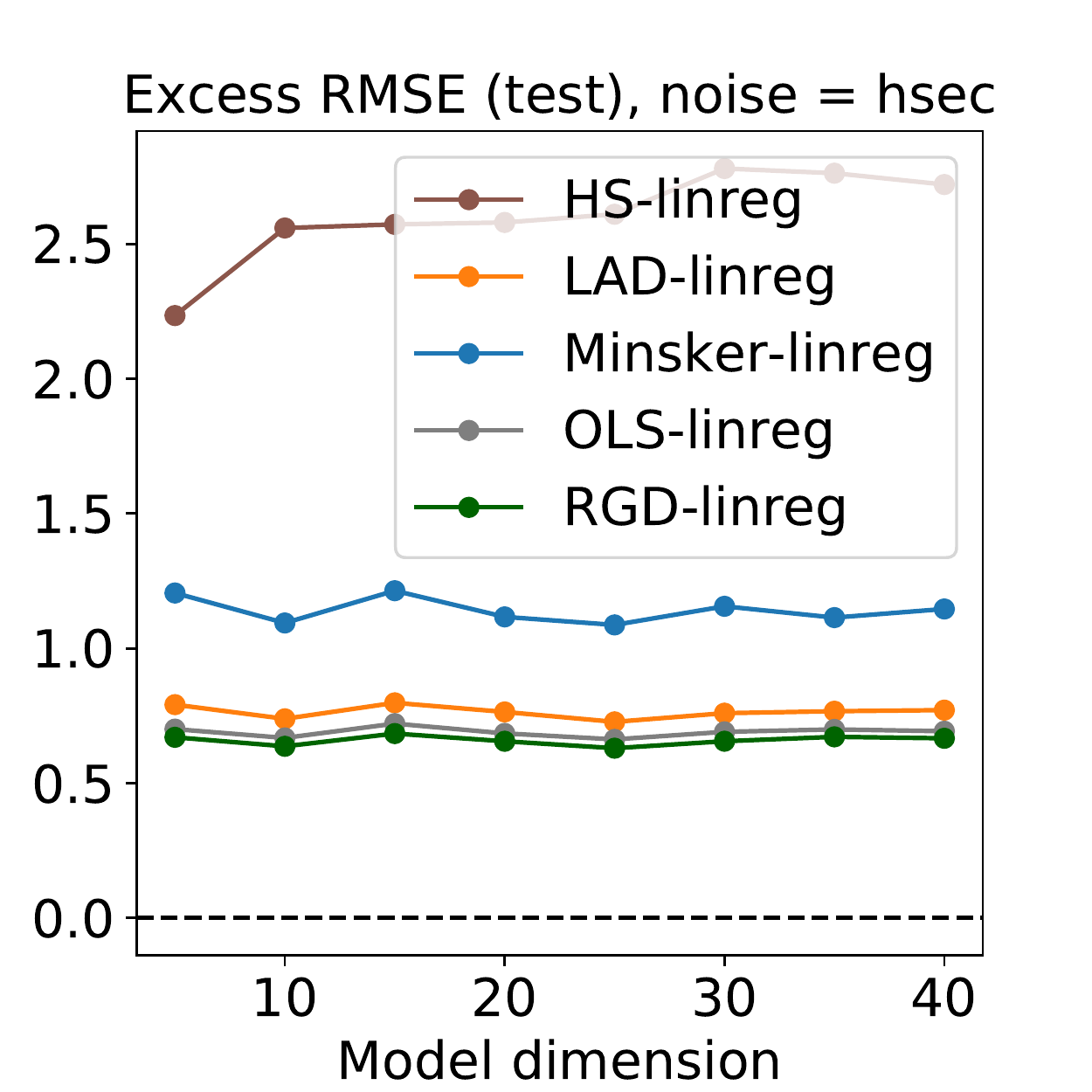}\,\includegraphics[width=0.25\textwidth]{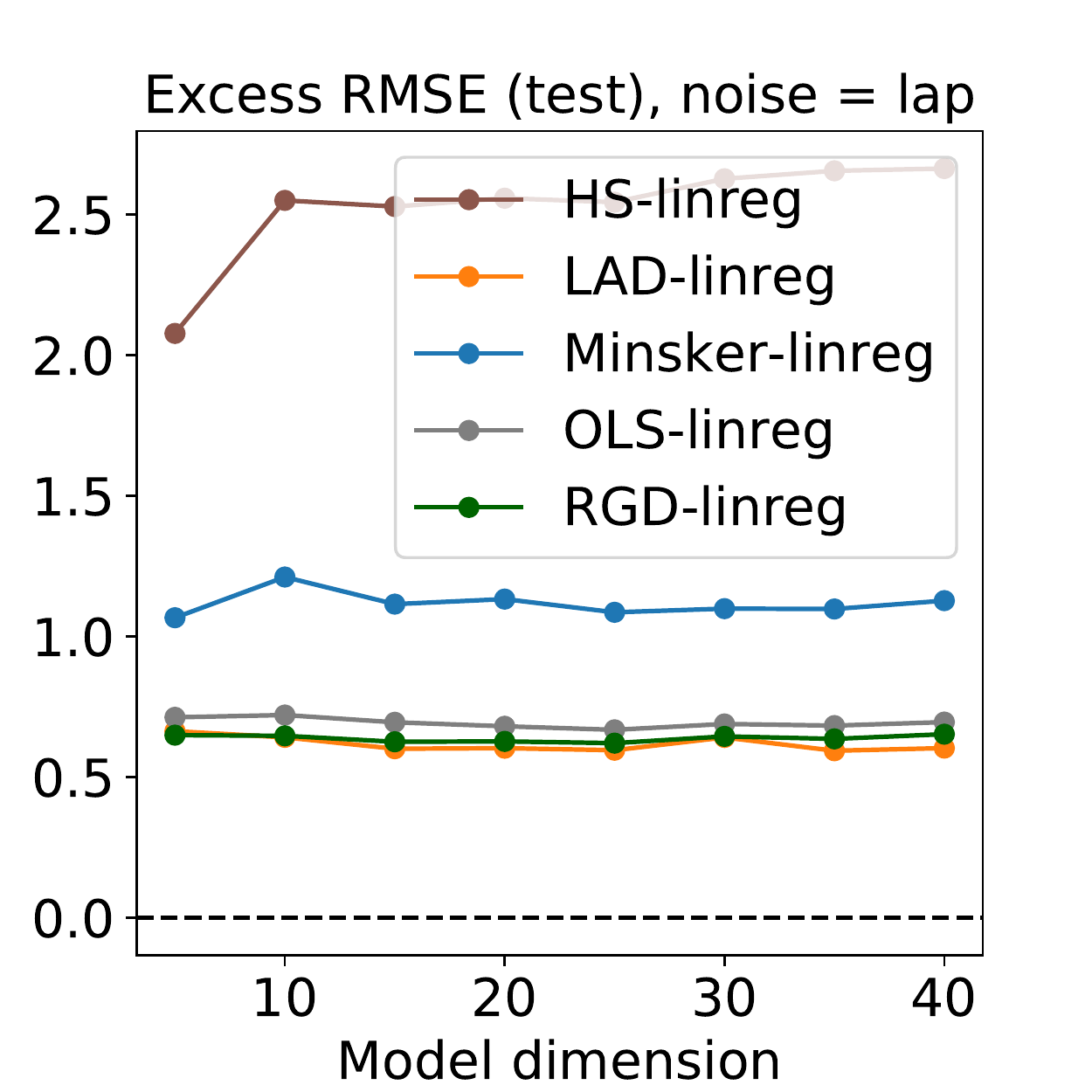}\\
\includegraphics[width=0.25\textwidth]{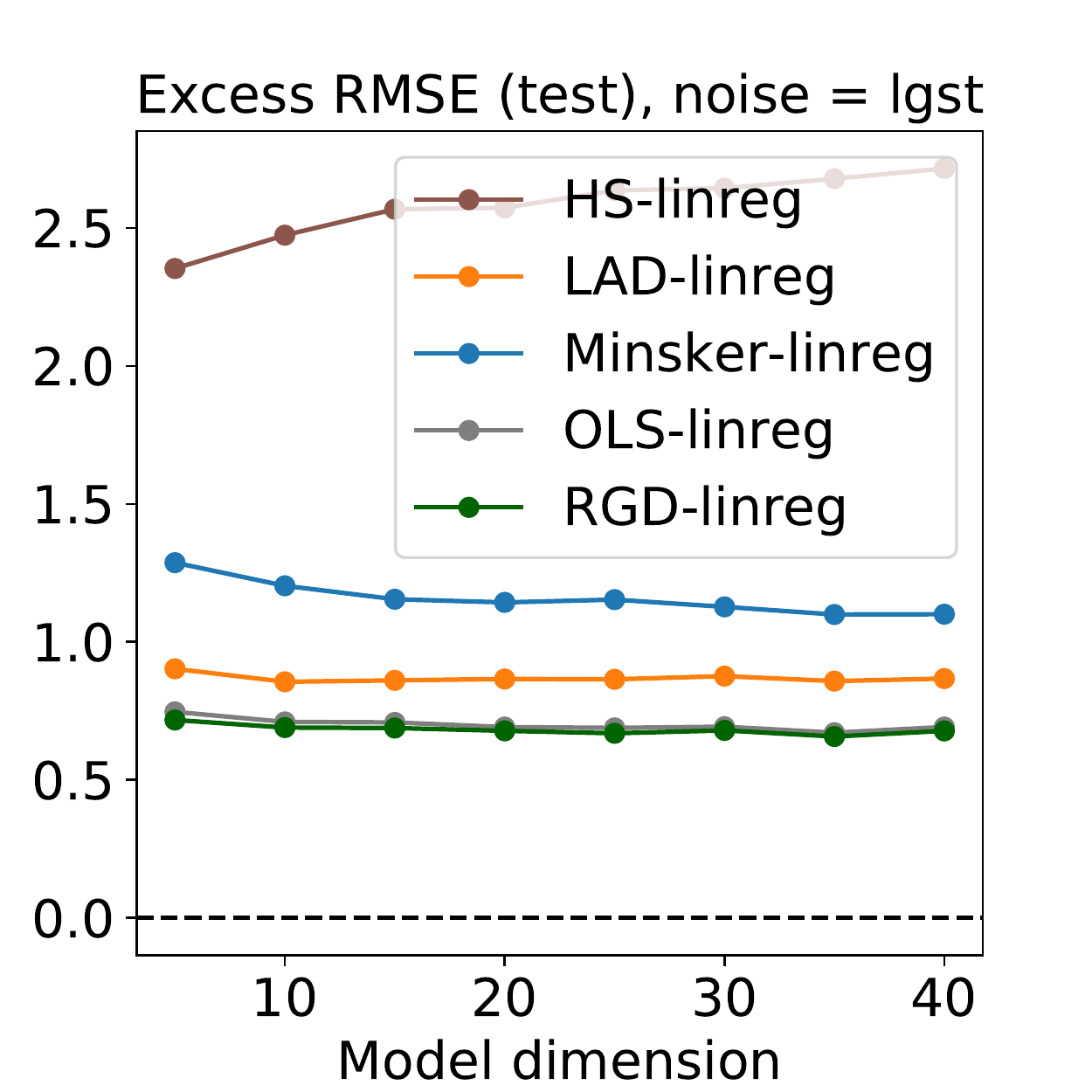}\,\includegraphics[width=0.25\textwidth]{linreg_overDim_risk_llog}\,\includegraphics[width=0.25\textwidth]{linreg_overDim_risk_lnorm}\,\includegraphics[width=0.25\textwidth]{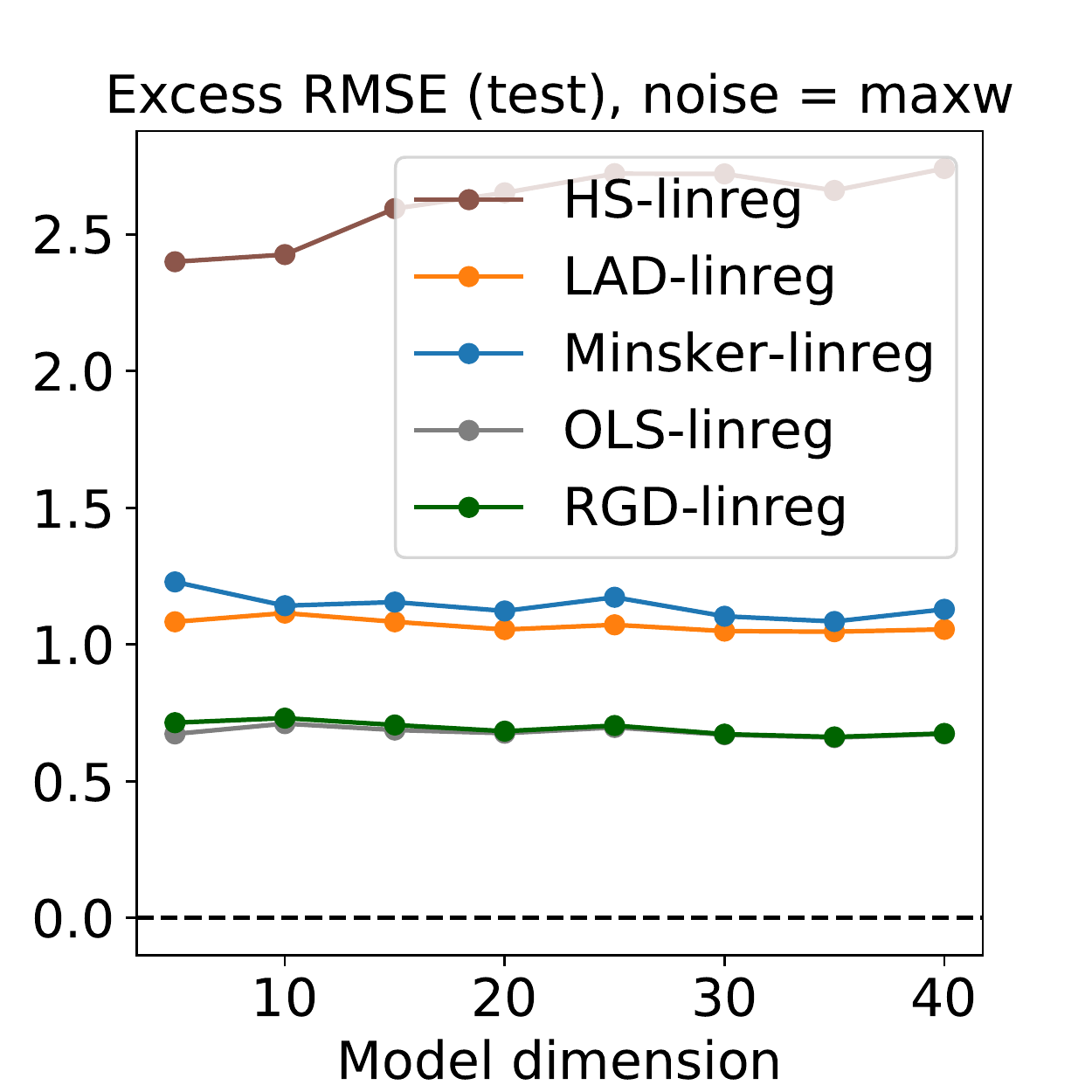}\\
\includegraphics[width=0.25\textwidth]{linreg_overDim_risk_norm}\,\includegraphics[width=0.25\textwidth]{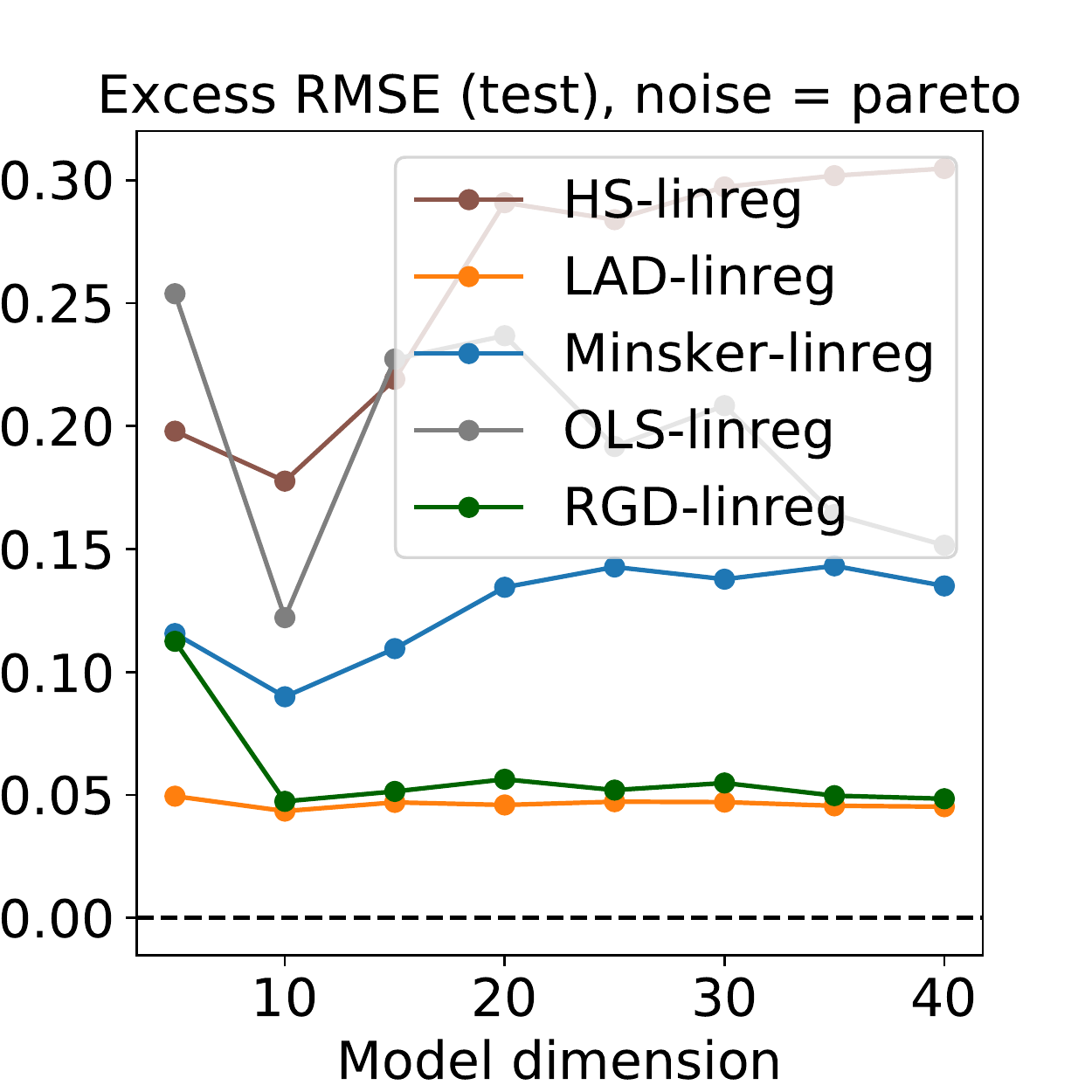}\,\includegraphics[width=0.25\textwidth]{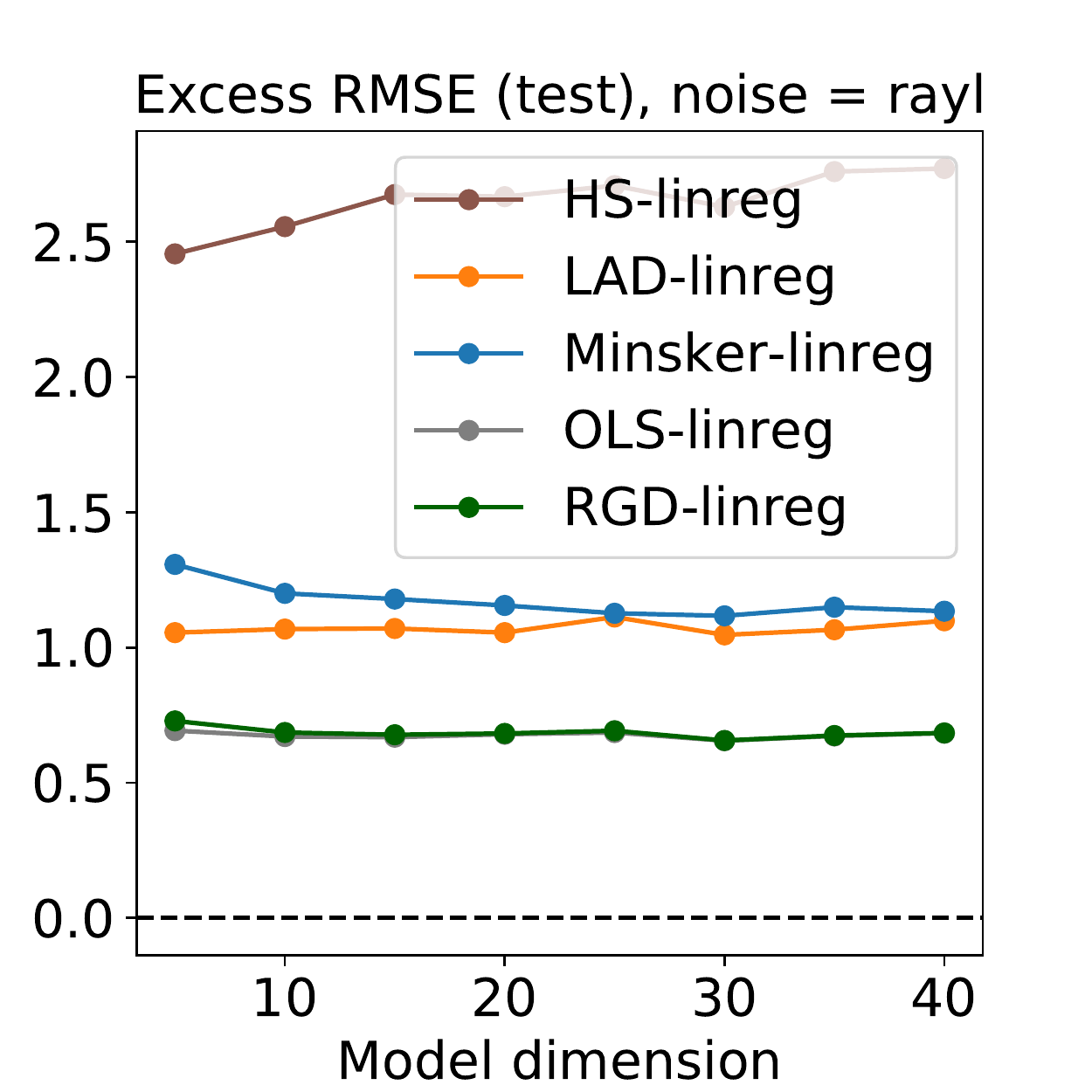}\,\includegraphics[width=0.25\textwidth]{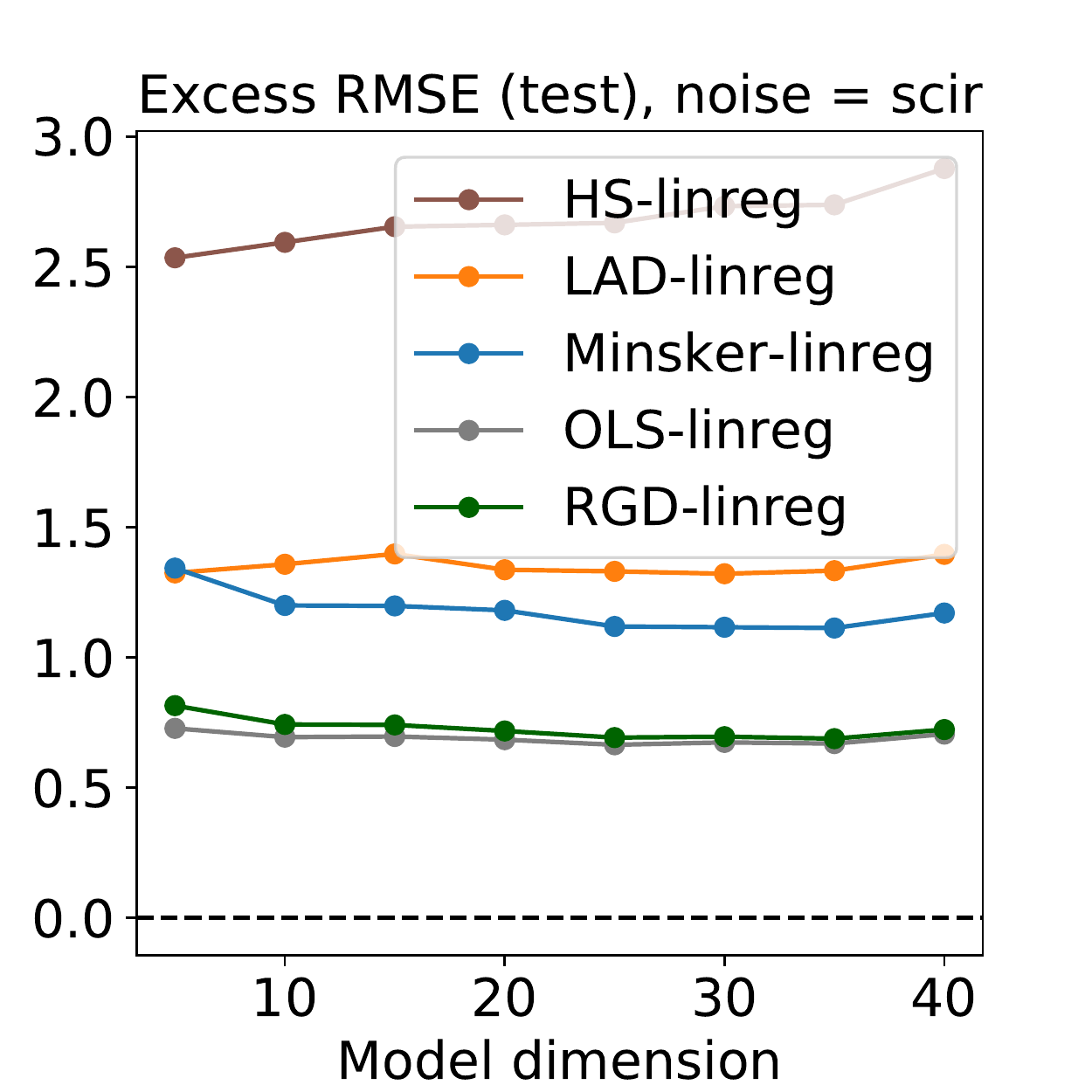}
\caption{Prediction error over dimensions $5 \leq d \leq 40$, with ratio $n/d = 6$ fixed, and noise level = $8$. Each plot corresponds to a distinct noise distribution.}
\label{fig:overDim_all_distros_1}
\end{figure}

\clearpage

\begin{figure}[t]
\centering
\includegraphics[width=0.25\textwidth]{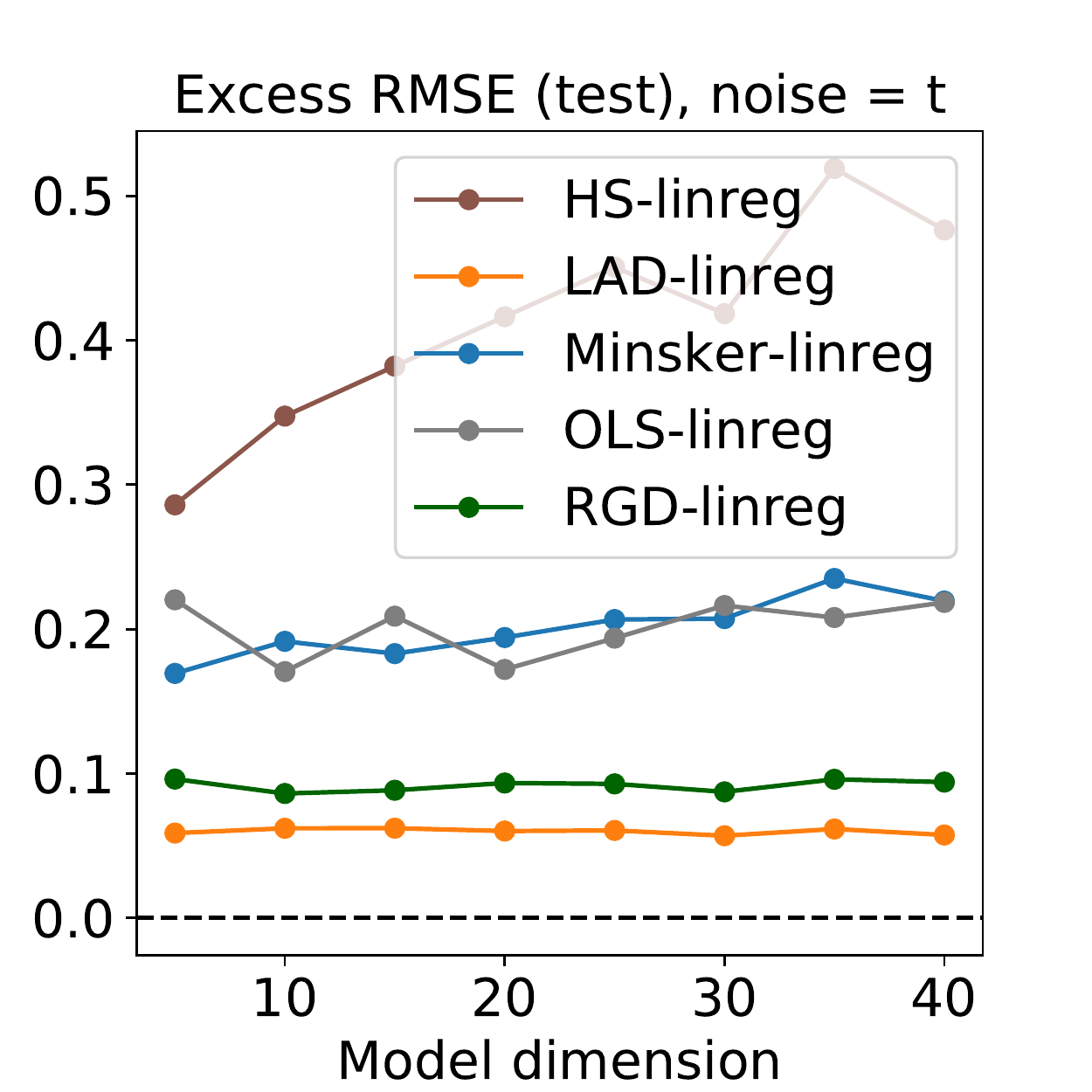}\,\includegraphics[width=0.25\textwidth]{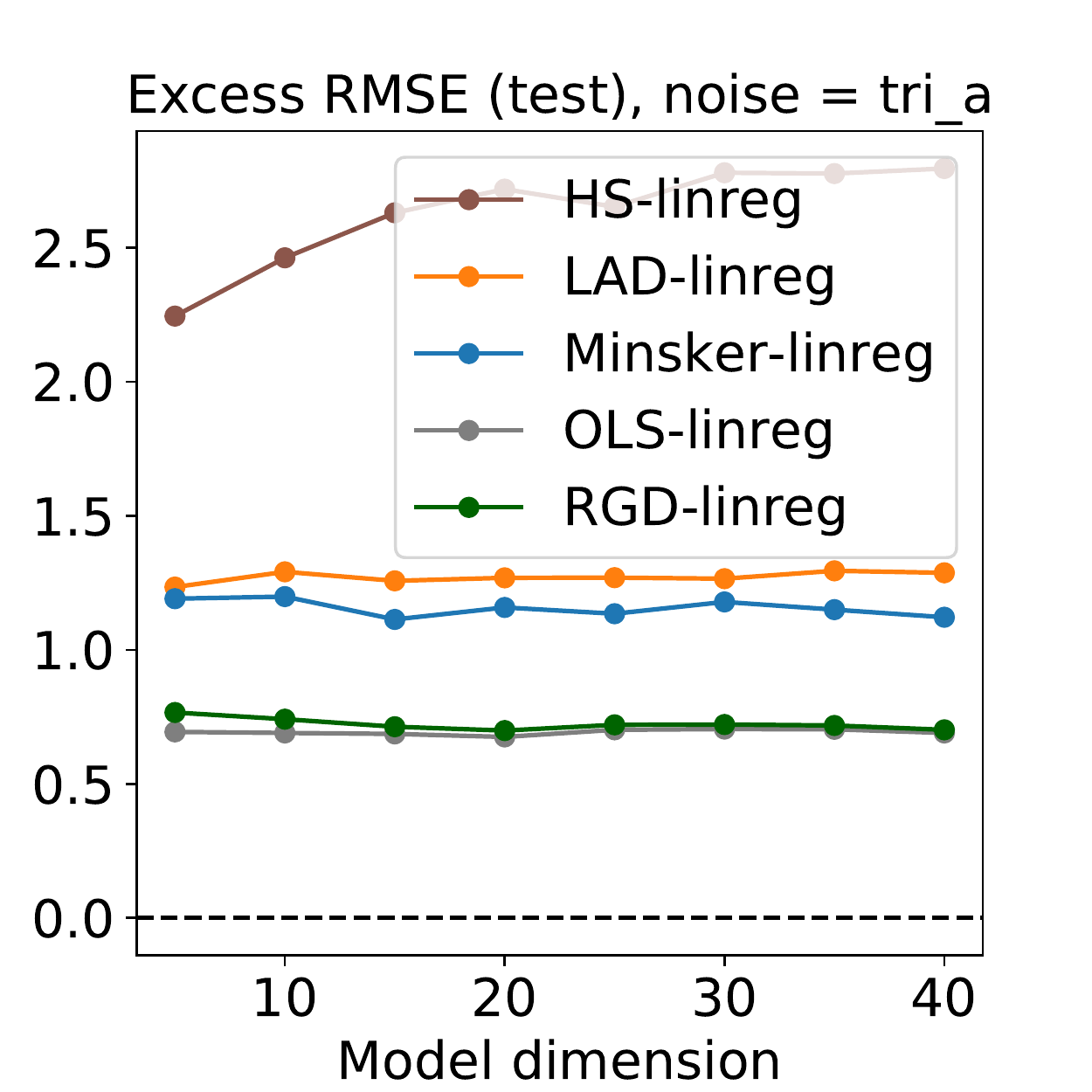}\,\includegraphics[width=0.25\textwidth]{linreg_overDim_risk_tri_s}\,\includegraphics[width=0.25\textwidth]{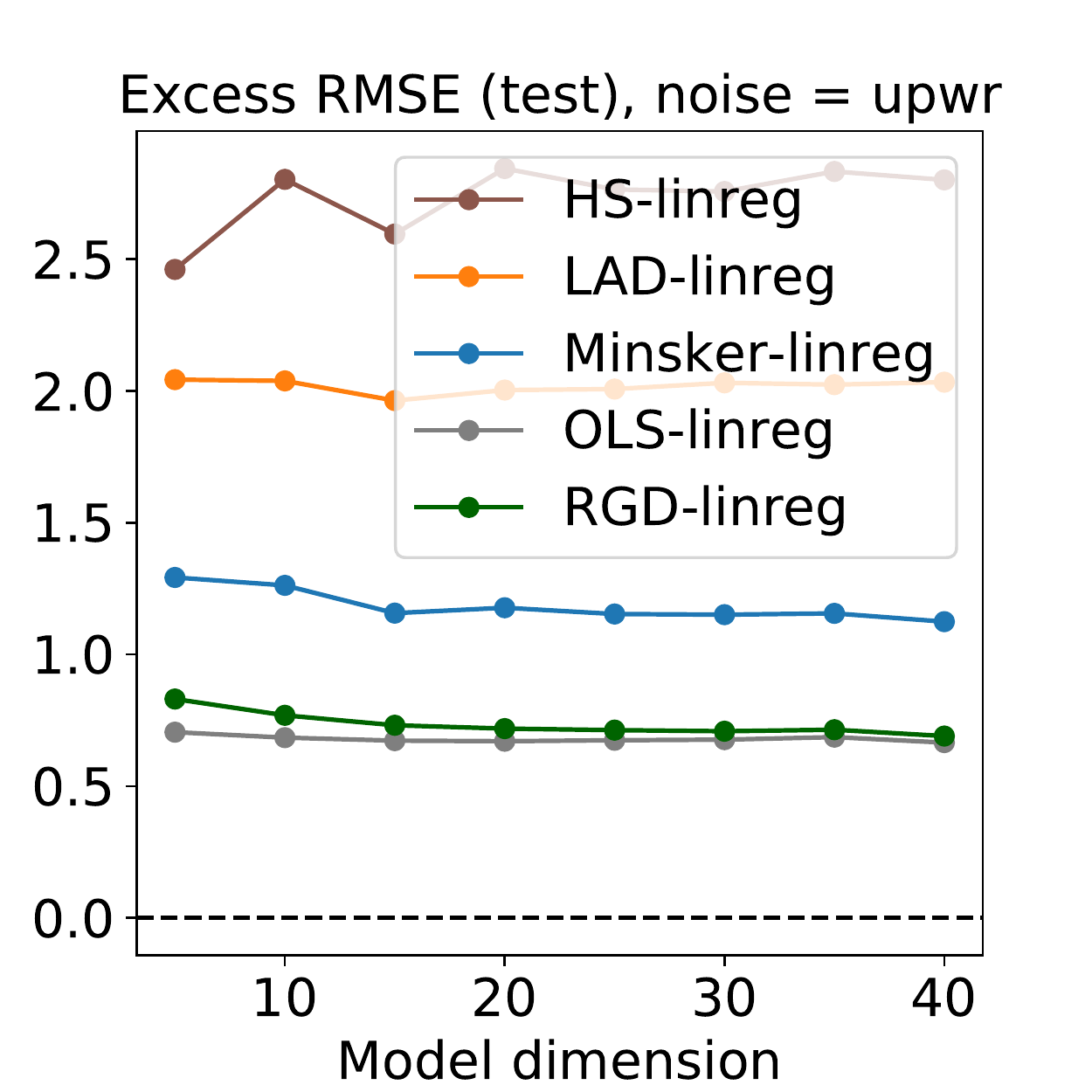}\\
\includegraphics[width=0.25\textwidth]{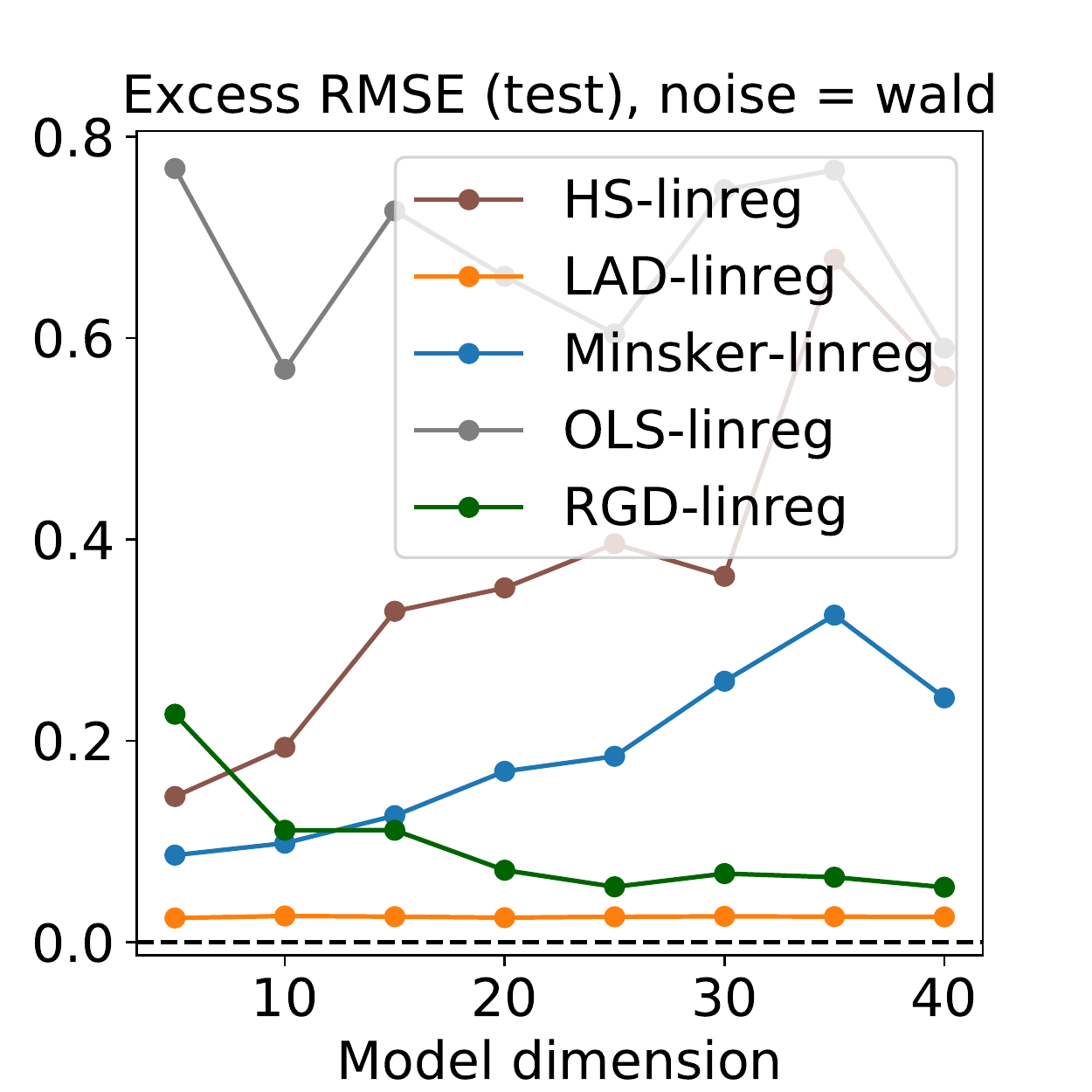}\,\includegraphics[width=0.25\textwidth]{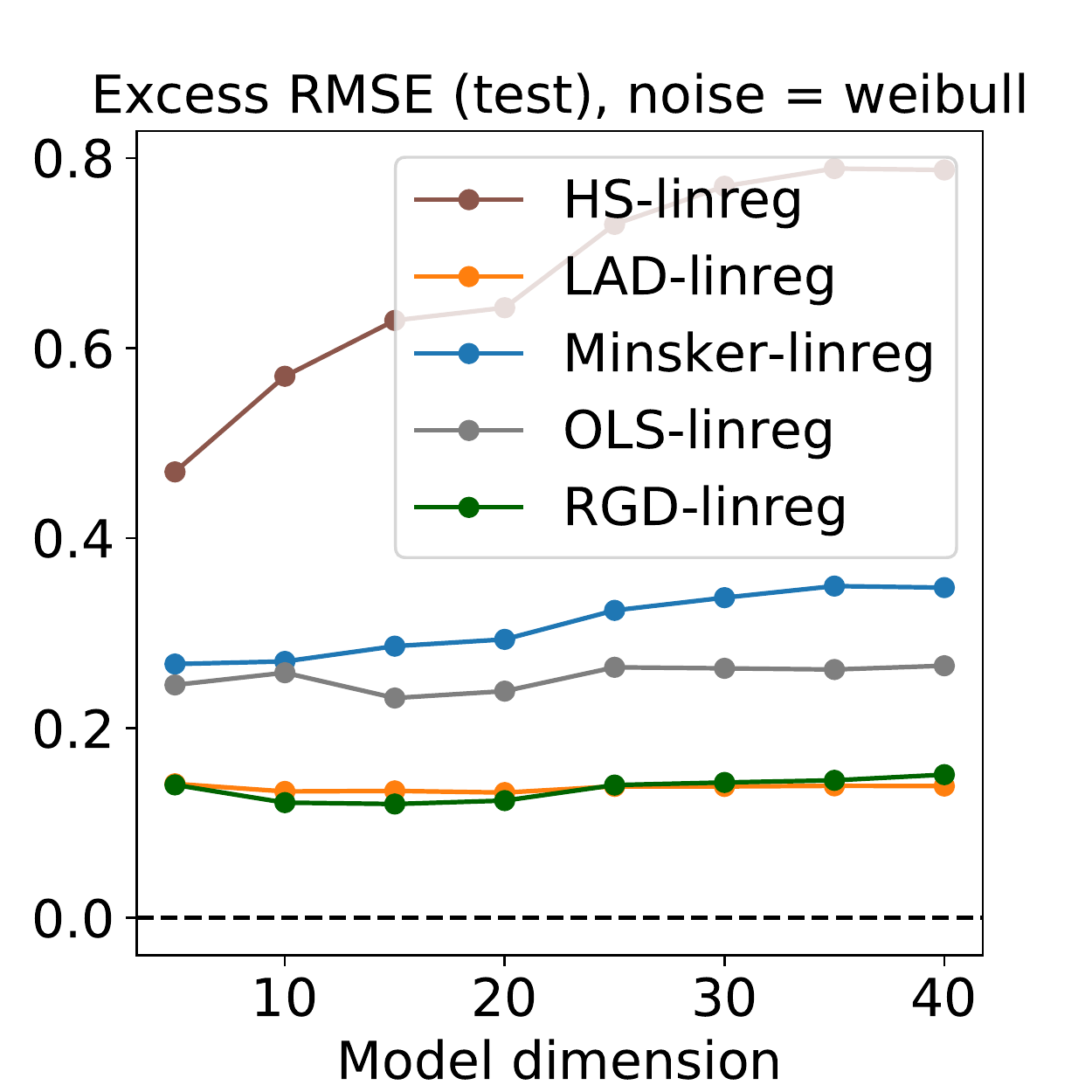}
\caption{Prediction error over dimensions $5 \leq d \leq 40$, with ratio $n/d = 6$ fixed, and noise level = $8$. Each plot corresponds to a distinct noise distribution.}
\label{fig:overDim_all_distros_2}
\end{figure}

\clearpage

\end{document}